%% file: ms.tex
\documentclass[11pt]{article}

\usepackage{amsmath} %Advanced math typing
\usepackage{amsthm} %For theorem styles
\usepackage{amsfonts} %For \mathbb
\usepackage{amssymb} %For more math symbols
\usepackage{stmaryrd} %For even more math symbols (like \llbracket)
\usepackage{mathtools} %For extensible math symbols
\usepackage{thm-restate} %For restatable theorems

\usepackage[style=alphabetic, maxnames=100, maxalphanames=100]{biblatex} %For references
\usepackage{xcolor} %For colored text
\usepackage{bm} %For bold symbols with \bm
\usepackage{bbm} %For alternative blackboardsymbols with \mathbbm

\usepackage{array} %For special commands in tables
\usepackage[inline]{enumitem} %For enumerate styles

\usepackage{microtype} %Corrects kerning
\usepackage{lmodern} %Loads Latin modern fonts (helps microtype package)

\usepackage{pdflscape} %Allows alternating to landscape mode
\usepackage{afterpage} %For \afterpage macro

\usepackage{tikz} %For tikzfigures
\usetikzlibrary{cd} %For commutative diagrams
\usetikzlibrary{calc} %For coordinate calculations

\usepackage[in]{fullpage} %Gets rid of stupidly large margins leaving 1 inch on every side

\usepackage[hypcap=true]{subcaption} % For captions in subfloats

\usepackage[%
  pdftex,% For direct pdf compilation compatibility
  plainpages=false,% Usage of not only arabic numbers
  pdfpagelabels,% Enable page labels
  colorlinks,% Colored links
  citecolor=black,% Citation color
  linkcolor=black,% Link color
  urlcolor=black,% URL color
  filecolor=black,% File color
  bookmarksopen=false% For starting document with bookmarks tree closed
]{hyperref} %For pdf links
\usepackage[all]{hypcap} %Fixes figures and tables anchors for hyperref

\makeatletter
\newcommand*{\newletterthm@internal}{}%Dummy definition to use \renewcommand later
\newcommand*{\newletterthm}[1]{%
  \def\newletterthm@name{#1}%
  \renewcommand*{\newletterthm@internal}[1][]{%
    \def\param@{##1}
    \ifx\param@\empty
    \expandafter\expandafter\expandafter\newtheorem%
    \expandafter\expandafter\expandafter{%
      \expandafter\newletterthm@name%
      \expandafter}%
    \expandafter{%
      \newletterthm@text}%
    \expandafter\renewcommand%
    \expandafter*%
    \expandafter{%
      \csname the#1\endcsname}{\Alph{#1}}%
    \else%
    \expandafter\expandafter\expandafter\newtheorem%
    \expandafter\expandafter\expandafter{%
      \expandafter\newletterthm@name%
      \expandafter}%
    \expandafter{%
      \newletterthm@text}[##1]%
    \expandafter\renewcommand%
    \expandafter*%
    \expandafter{%
      \csname the#1\endcsname}{\csname the##1\endcsname.\Alph{#1}}%
    \fi%
  }%
  \newletterthm@newthm%
}

\newcommand*{\newletterthm@newthm}[2][]{%
  \def\param@{#1}
  \ifx\param@\empty
  \def\newletterthm@text{#2}%
  \newletterthm@internal%
  \else%
  \expandafter\newtheorem\expandafter{\newletterthm@name}[#1]{#2}%
  \fi%
}
\makeatother
\newtheoremstyle{thmstyle}%style name
  {\medskipamount}%space before
  {\smallskipamount}%space after
  {\slshape}%font used
  {0pt}%indentation
  {\bfseries}%modifier theorem head
  {.}%punctuation between theorem head and body
  { }%space after punctuation
  {\thmname{#1}\thmnumber{ #2}{\normalfont\thmnote{ (#3)}}}%theorem specifier
  
\newtheoremstyle{plainstyle}%style name
  {\medskipamount}%space before
  {\smallskipamount}%space after
  {\rmfamily}%font used
  {0pt}%indentation
  {\bfseries}%modifier theorem head
  {.}%punctuation between theorem head and body
  { }%space after punctuation
  {\thmname{#1}\thmnumber{ #2}{\normalfont\thmnote{ (#3)}}}%theorem specifier

\theoremstyle{thmstyle}
\newtheorem{theorem}{Theorem}[section]
\newtheorem{lemma}[theorem]{Lemma}
\newtheorem{corollary}[theorem]{Corollary}
\newtheorem{proposition}[theorem]{Proposition}

\newtheorem{claim}[theorem]{Claim}

\newletterthm{theoremLet}{Theorem}

\theoremstyle{plainstyle}
\newtheorem{definition}[theorem]{Definition}
\newtheorem{convention}[theorem]{Convention}
\newtheorem{remark}[theorem]{Remark}

\newtheorem{example}[theorem]{Example}

\newenvironment{proofof}[1]{\begin{proof}[Proof of #1.]}{\end{proof}}
\newenvironment{proofint}[1][Proof intuition.]{\begin{proof}[#1]}{\end{proof}}

\makeatletter
\def\refdescformat#1{%
  \phantomsection%
  \let\oldlabel\label%
  \let\label\@gobble%
  \edef\@currentlabel{#1}% Label format (in \ref).
  \let\label\oldlabel%
  #1:% Item format (in description).
}
\makeatother

\newlist{refdesc}{description}{1}
\setlist[refdesc]{format={\refdescformat}}
\newlist{enumdef}{enumerate}{1}
\setlist[enumdef]{before={\leavevmode}, label={\arabic*.}, ref={\thetheorem.\arabic*}}

\setlist[enumerate]{label={\roman*.}, ref={(\roman*)}} % Provides default enumitem arguments for enumerate

\makeatletter
\newcommand\notoc@internal[2][]{}%Dummy definition to use \renewcommand later
\newcommand{\notoc}[1]{%Executes the next command (such as \section, \subsection, etc.) without adding it to toc.
  \renewcommand{\notoc@internal}[2][]{%
    \let\old@addtocontents\addtocontents%
    \let\addtocontents\@gobbletwo%
    \def\param@{##1}
    \ifx\param@\empty
    #1{##2}%
    \else
    #1[##1]{##2}%
    \fi
    \let\addtocontents\old@addtocontents%
  }%
  \notoc@internal%
}
\makeatother
\numberwithin{equation}{section} % Changes equation numbering to be within section.

\let\epsilon\varepsilon

\newcommand{\rn}{\bm}
\newcommand{\df}{\stackrel{\text{def}}{=}}
\newcommand{\place}{\mathord{-}}
\newcommand{\comp}{\mathbin{\circ}}
\newcommand{\rest}{\mathord{\vert}}
\newcommand{\floor}[1]{\ensuremath{\left\lfloor#1\right\rfloor}}
\newcommand{\ceil}[1]{\ensuremath{\left\lceil#1\right\rceil}}
\newcommand{\given}[1][]{\mathrel{#1\vert}}

\newcommand{\down}{\mathord{\downarrow}}

\DeclareMathOperator{\im}{im}
\DeclareMathOperator{\id}{id}
\DeclareMathOperator{\rk}{rk}
\DeclareMathOperator{\ag}{ag}
\DeclareMathOperator{\PAC}{PAC}
\DeclareMathOperator{\agPAC}{agPAC}
\DeclareMathOperator{\PACr}{PACr}
\DeclareMathOperator{\agPACr}{agPACr}
\DeclareMathOperator{\UC}{UC}
\DeclareMathOperator{\VC}{VC}
\DeclareMathOperator{\Nat}{Nat}
\DeclareMathOperator{\VCN}{VCN}
\DeclareMathOperator{\DS}{DS}
\DeclareMathOperator{\gVC}{gVC}
\DeclareMathOperator{\dom}{dom}
\newcommand{\kpart}[1][k]{#1\operatorname{-part}}
\newcommand{\kflat}[1][k]{#1\operatorname{-flat}}
\newcommand{\knonlocal}[1][k]{#1\operatorname{-non-local}}
\DeclareMathOperator{\argmin}{argmin}

\DeclareMathOperator{\Bi}{Bi}
\DeclareMathOperator{\ev}{ev}
\DeclareMathOperator{\Var}{Var}
\DeclareMathOperator{\dist}{dist}
\DeclareMathOperator{\bip}{bipartite}

\newcommand{\EE}{\mathbb{E}}
\newcommand{\NN}{\mathbb{N}}
\newcommand{\PP}{\mathbb{P}}
\newcommand{\QQ}{\mathbb{Q}}
\newcommand{\RR}{\mathbb{R}}
\newcommand{\ZZ}{\mathbb{Z}}

\newcommand{\One}{\mathbbm{1}}

\newcommand{\cA}{\mathcal{A}}
\newcommand{\cB}{\mathcal{B}}
\newcommand{\cC}{\mathcal{C}}
\newcommand{\cE}{\mathcal{E}}
\newcommand{\cF}{\mathcal{F}}
\newcommand{\cG}{\mathcal{G}}
\newcommand{\cH}{\mathcal{H}}
\newcommand{\cJ}{\mathcal{J}}
\newcommand{\cL}{\mathcal{L}}
\newcommand{\cM}{\mathcal{M}}
\newcommand{\cN}{\mathcal{N}}
\newcommand{\cP}{\mathcal{P}}
\newcommand{\cT}{\mathcal{T}}
\newcommand{\cU}{\mathcal{U}}

\def\Caratheodory{Carath\'{e}odory}
\def\Erdos{Erd\H{o}s}
\def\Renyi{R\'{e}nyi}
\def\Szemeredi{Szemer\'{e}di}

\title{High-arity PAC learning via exchangeability}
\author{%
  Leonardo N.~Coregliano\thanks{Part of this work carried out while the first author was at
    the Institute for Advanced Study, supported by the National Science Foundation, and by the IAS
    School of Mathematics.} \and%
  Maryanthe Malliaris\thanks{Research partially supported by NSF-BSF 2051825.}%
}
\date{\today}

\begin{document}
\maketitle

\begin{abstract}
  We develop a theory of high-arity PAC learning, which is statistical learning in the presence of
  ``structured correlation''. In this theory, hypotheses are either graphs, hypergraphs or, more
  generally, structures in finite relational languages, and i.i.d.\ sampling is replaced by sampling
  an induced substructure, producing an exchangeable distribution. Our main theorems establish a
  high-arity (agnostic) version of the fundamental theorem of statistical learning.
\end{abstract}

Readers who prefer may start with Section~\ref{sec:exp}, which is a self-contained exposition of the
manuscript.

\notoc\section{Introduction}

Pick a function $F\colon X\to Y$ and a measure $\mu$ over $X$, sample $m$ points
$\rn{x}_1,\ldots,\rn{x}_m$ independently from $\mu$, show me the values
$(\rn{x}_i,F(\rn{x}_i))_{i=1}^m$ and (despite not knowing $\mu$) I will guess your function
$F$. Without being a magician or incredibly lucky, this task is completely impossible. However, the
theory of PAC (probably approximately correct) learning~\cite{Val84} provides a framework in which
the feasibility of this learning task is mathematically interesting (see also the book~\cite{SB14}
for a more thorough introduction to the topic). First, in the original setup of PAC learning, the
adversary is not allowed to play any function, but rather a function within some known family
$\cH\subseteq Y^X$, in which case we refer to \emph{learnability of the family $\cH$}. Second, we
only require our guess to be approximately correct in the sense of having small $L_1(\mu)$-distance
to $F$, and further only require such small distance to happen with high probability over the
revealed sample $\rn{x}_1,\ldots,\rn{x}_m$ (assuming $m$ is large enough). Let us stress that in PAC
learning, it does not suffice to know statistics about the hidden function (which would be deducible
from a simple law of large numbers), e.g., it does not suffice to know that the function $F\colon
X\to\{0,1\}$ assigns value $1$ to about half of the space $X$, but rather we would like to know
which points of $X$ are assigned value $1$.

From a mathematical point of view, the Fundamental Theorem of Statistical Learning, summarized as
Theorem~\ref{thm:FTSL} below, says that the learnability (or not) of a family $\cH$ is a strong
indicator of simplicity or complexity of $\cH$ in the sense that it has independent
characterizations in terms of an appropriate notion of uniform convergence and in terms of a purely
combinatorial dimension. Namely, PAC learnability of $\cH$ is equivalent to finite
\emph{Vapnik--Chervonenkis ($\VC$) dimension}~\cite{VC71,BEH89,VC15} of $\cH$ when $\lvert
Y\rvert=2$, equivalent to finite \emph{Natarajan dimension}~\cite{Nat89} of $\cH$ when $Y$ is finite
and equivalent to finite \emph{Daniely--Shalev-Shwartz ($\DS$) dimension}~\cite{DS14,BCDMY22} in the
general case (see Definition~\ref{def:Natdim} for the first two dimensions; the $\DS$-dimension is
out of the scope of this paper). The appropriate uniform convergence property~\cite{VC15,VC71,Nat89}
(see Definition~\ref{def:UC} below) can be seen as a uniform labeled law of large numbers.

About a decade after Valiant's paper, Haussler~\cite{Hau92} gave a very interesting extension of the
PAC framework (and proved a corresponding theorem) that increases its scope. In Haussler's
extension, called \emph{agnostic PAC learning}, there is a family of functions $\cH$ chosen in
advance, but the adversary is not required to play one of them; rather the adversary picks a
distribution $\nu$ on $X\times Y$, fixed but unknown to us.  In the learning task, the adversary
provides i.i.d.\ samples $(\rn{x}_i,\rn{y}_i)$ directly from $\nu$ (which need not even describe a
partial function) and we must produce some $G\in\cH$ such that
$\EE_{(\rn{x},\rn{y})\sim\nu}[G(\rn{x})\neq\rn{y}]\leq\EE_{(\rn{x},\rn{y})\sim\nu}[H(\rn{x})\neq\rn{y}]
+ \epsilon$ for every $H\in\cH$. Informally, we are choosing an element of $\cH$ which is as close
to $\nu$ as possible. Haussler then extended the equivalence in the Fundamental
Theorem\footnote{This was shown when $Y=\{0,1\}$, but easily extends to the case when $Y$ is
  finite.} to the case of agnostic learnability.

We now informally state the Fundamental Theorem of Statistical Learning.
\begin{theoremLet}[\cite{VC71,BEH89,Nat89,VC15}]\label{thm:FTSL}
  The following are equivalent for a family $\cH\subseteq Y^X$ of functions $X\to Y$ with $Y$
  finite:
  \begin{enumerate}
  \item The Natarajan dimension of $\cH$ is finite (when $\lvert Y\rvert=2$, the more familiar
    statement is that the $\VC$-dimension of $\cH$ is finite).
  \item $\cH$ has the uniform convergence property.
  \item $\cH$ is agnostically PAC learnable.
  \item $\cH$ is PAC learnable.
  \end{enumerate}
\end{theoremLet}
Note that when $\lvert Y\rvert=2$, the more familiar hypothesis is that $\cH$ is a family of subsets
of $X$ (allowing $\lvert Y\rvert > 2$ can be thought of as the natural extension to finite colorings
of $X$).

Looking at Theorem~\ref{thm:FTSL}, there are two directions that the classical theory does not
directly address. The first is how to handle PAC learnability when the sample is not i.i.d. The
second is how to PAC learn more complicated structures such as graphs or hypergraphs (as opposed to
just sets). There has been some work in the first direction over the
years~\cite{HL94,AV95,Gam03,ZZX09,SW10,ZXC12,ZLX12,ZXX14,BGS18,SS23}, mostly by considering
different kinds of Markov processes under reasonable ergodicity assumptions. The second direction
has seen much less activity. In this work, we take up both directions by developing a high-arity
(fully agnostic) PAC learning theory, which allows for learning of more complicated structures and
allows for ``structured'' correlation. Notably, in our setup, correlation is seen as a friend, not a
foe.

To both motivate high-arity PAC and to illustrate why its characterization requires new methods,
consider the problem of learning a graph over a set $X$. By encoding such a graph by its adjacency
matrix $F\colon X\times X\to\{0,1\}$, we could naively apply classic PAC theory, but one quickly
sees that this yields a rather unnatural setup: the adversary picks a measure $\mu$ over $X\times X$
and reveals several pairs of vertices $(\rn{x}_i,\rn{x'}_i)$ drawn i.i.d.\ from $\mu$ along with
whether they are adjacent in $F$ or not: $F(\rn{x}_i, \rn{x'}_i)$. In this naive approach, very few
families of graphs will be learnable, e.g., if the marginals of the measure $\mu$ are atomless, we
cannot even detect whether $F$ has a vertex of degree $2$ or not.

Instead, in our high-arity PAC learning framework, the adversary picks a measure $\mu$ over $X$,
samples $m$ points $\rn{x}_1,\ldots,\rn{x}_m$ independently and reveals the subgraph induced by $F$
on these $m$ points: $\rn{y}\df(F(\rn{x}_i,\rn{x}_j))_{i,j=1}^m$. Note that even though the entries
of $\rn{x}$ are i.i.d., the entries of $\rn{y}$ are potentially correlated when they correspond to
pairs sharing endpoints. This ``structured'' correlation changes the picture in an interesting
way. On one hand, it opens the door for learning a much wider array of classes than in the naive
framework (see Examples~\ref{ex:matching} and~\ref{ex:boundeddegree}) and on the other hand, it
means that methods and techniques of classic PAC which rely on independence do not lift
straightforwardly to high-arity. This is a good place to point out that in analogy to classic PAC, it
does not suffice to know statistics about the hidden graph (which would be deducible from
applications of property testing or \Szemeredi's Regularity Lemma), e.g., it does not suffice to
know that the graph is a clique on about half of the vertices of $X$, but rather we would like to
know exactly which points of $X$ belong to the clique.

Correlation is not the only hurdle of high-arity: already in the case of graphs, it is a challenge
to provide an appropriate definition of agnostic learning. If we aim to produce a high-arity sample
that has structured correlation just as in the non-agnostic case, we cannot simply say that the
adversary should pick a distribution over, say, $X\times X\times\{0,1\}$. Instead, we observe that
in the non-agnostic case, a sample $((\rn{x}_i)_{i=1}^m,(F(\rn{x}_i,\rn{x}_j))_{i,j=1}^m)$ from
$(F,\mu)$ can be seen as a finite marginal of an infinite sample
$((\rn{x}_i)_{i\in\NN_+},(F(\rn{x}_i,\rn{x}_j))_{i,j\in\NN_+})$, which although not i.i.d., has two
key properties: it is \emph{local} in the sense that marginals on disjoint subsets of $\NN_+$ are
independent and it is \emph{exchangeable} in the sense that its distribution is invariant under the
natural action of $S_{\NN_+}$. In this work, we propose that the correct framework for agnostic
high-arity PAC learning is for the adversary to pick a distribution over
$X^{\NN_+}\times\{0,1\}^{\NN_+\times\NN_+}$ that is local and exchangeable and provide us a finite
marginal of it; our main theorems give strong evidence for the correctness of this
definition. Although such a definition seems quite abstract, a normal form for such distributions is
provided by the Aldous--Hoover Theorem from exchangeability theory~\cite{Hoo79,Ald85,Ald81} (see
also~\cite[Theorem~7.22 and Lemma~7.35]{Kal05}).

Before informally stating our main theorems, we make a few comments. First, for simplicity, so far
we only discussed the case of graphs, but our theory is developed for $k$-hypergraphs and even in
the rather more general setting of functions $X^k\to Y$ with $Y$ finite\footnote{This is what a
  logician would call a structure in a finite relational language.}, so we give the informal
statement in this more general language. Second, the reader will notice that in the statement of the
theorem we use the name $k$-PAC learning instead of high-arity learning to emphasize the arity $k$
of the problem. Third, the theorems include a high-arity combinatorial dimension (that we call
$\VCN_k$) whose formal definition is deferred to the main text (for an informal description, see
the expository Section~\ref{subsec:VCNk}). Finally, we have not yet covered an interesting
mathematical phenomenon that does not have an analogue in classic PAC: there is a natural partite
version of high-arity PAC and its interplay with the non-partite is crucial for our theory; we refer
the reader to the expository Section~\ref{subsec:partite}.

Our first main result is:
\begin{theorem}[Informal version of Theorem~\ref{thm:kPAC}]\label{thm:kPACinformal}
  Let $k\geq 1$ and $Y$ be a finite set. The following are equivalent for a family $\cH\subseteq
  Y^{X^k}$ of functions $X^k\to Y$ and its partite version $\cH^{\kpart}$:
  \begin{enumerate}
  \item $\VCN_k(\cH) < \infty$.
  \item $\VCN_k(\cH^{\kpart}) < \infty$.
  \item $\cH^{\kpart}$ has the uniform convergence property.
  \item $\cH$ is agnostically $k$-PAC learnable.
  \item $\cH^{\kpart}$ is partite agnostically $k$-PAC learnable.
  \item $\cH^{\kpart}$ is partite $k$-PAC learnable.
  \end{enumerate}

  Furthermore, any of the items above implies the following:
  \begin{enumerate}[resume]
  \item $\cH$ is $k$-PAC learnable.
  \end{enumerate}
\end{theorem}

Our second main result is:
\begin{theorem}[Informal version of Theorem~\ref{thm:kPACkpart}]\label{thm:kPACkpartinformal}
  Let $k\geq 1$ and $Y$ be a finite set. The following are equivalent for a family
  $\cH\subseteq Y^{X_1\times\cdots\times X_k}$ of functions $X_1\times\cdots\times X_k\to Y$:
  \begin{enumerate}
  \item $\VCN_k(\cH) < \infty$.
  \item $\cH$ has the uniform convergence property.
  \item $\cH$ is partite agnostically $k$-PAC learnable.
  \item $\cH$ is partite $k$-PAC learnable.
  \end{enumerate}
\end{theorem}

Finally, we point out that our framework naturally extends beyond functions $X^k\to Y$ to also make
sense of high-arity PAC learning of Aldous--Hoover representations of local exchangeable
distributions and some of our results will extend naturally as well; see the expository
Section~\ref{subsec:highervar}, which discusses the related topic of higher-order variables.

\notoc\subsection{A brief description of the landscape}
\label{subsec:related}

The $\VC$-dimension, which characterizes classical PAC learning in the case of sets, has found some
mathematical meaning in a range of areas. In developing the combinatorial dimension $\VCN_k$, which
characterizes high-arity learning in the present manuscript, we were guided simply by the
constraints of our problem, but it is quite interesting a posteriori that the literature is already
quite rich in high-arity versions of the different phenomena that in the unary are captured by the
$\VC$-dimension and related concepts.

\notoc\subsubsection{Exchangeability}
\label{subsubsec:exch}

We already mentioned that the backbone of our theory relies on the movement to high-arity of
exchangeability theory. The unary version of exchangeability theory is de~Finetti's Theorem,
characterizing local exchangeable distributions of sequences as precisely those that are
i.i.d.\footnote{The more familiar statement is that exchangeable distributions of sequences are
  mixtures of i.i.d.\ sequences.} The high-arity version of this is the Aldous--Hoover
Theorem~\cite{Hoo79,Ald85,Ald81} characterizing exchangeable and separately exchangeable
distributions. As mentioned before, these are fundamental in the definition of agnostic non-partite
and partite high-arity PAC, respectively and the interplay between the non-partite and partite will
be central in our theory. In turn, this interplay unearths an interesting problem in probability
theory of how to appropriately encode a separately exchangeable distribution by an exchangeable one,
for details, see Remark~\ref{rmk:sepexch->exch}.

\notoc\subsubsection{$k$-dependence}

To our knowledge, the first work towards understanding high-arity versions of phenomena related to
the $\VC$-dimension is by Shelah in the form of ``$k$-dependence'' about a decade ago in an
exploratory section~\cite[\S~(H)]{She14}. Shelah was working in the context of model theory,
studying the class of dependent theories (an analogue of $\VC$-dimension). Shelah's $k$-dependence
gives a higher version of dependence, for formulas, and he established some consequences of
$2$-dependence for groups in~\cite{She17}. It is also worth repeating that the theory of what is now
known as $\VC$-dimension was helped by independent, central concerns in Shelah's model theory in the
seventies, see~\cite{Las92}: Shelah originally developed dependent theories and proved both an
infinitary and a finitary counting argument for them, see~\cite[Chapter~II]{She23}; the finitary
part is now generally referred to in combinatorics and computer science as the Sauer--Shelah--Perles
Lemma for finite $\VC$-dimension. An example of why a complete dependent theory is ``structurally
simpler'' than an arbitrary theory is that all indiscernible sequences in one free variable have an
average type, which can be seen as a model theoretic notion of convergence.

\notoc\subsubsection{Hypergraph regularity lemmas}

The classic $\VC$-dimension also plays an important role in combinatorics. For example, to a graph
$G$ on a vertex set $X$, we can associate a (unary) hypothesis class $\cH_G$ on $X$ as the
collection of all subsets of $X$ that are the neighborhood of some vertex of $G$:
\begin{align*}
  \cH_G & \df \{N_G(x) \mid x\in X\}.
\end{align*}
The $\VC$-dimension of $G$ is then defined as $\VC(G)\df\VC(\cH_G)$ (note the drop in arity: graphs,
which are binary, lead to families of neighborhoods, which are unary). A similar movement from a
single model to a family of sets is performed when defining dependence in model theory.

Under these transformations, classes of graphs (or more generally models) with a bound on the
corresponding $\VC$ dimension are structurally simpler than general classes. For example, for
bounded $\VC$-dimension classes of graphs, a stronger version of \Szemeredi's Regularity Lemma
holds~\cite{AFN07,LS10,FPS19} in which all regular pairs have density close to $0$ or $1$ and where
the number of parts is only polynomial in $\epsilon^{-1}$.

Regarding extensions of $\VC$ to the setting of hypergraphs, we already mentioned~\cite{She14}. More
recently, a series of combinatorics papers informed by model theory extend these notions to capture
the subtlety that is present in regularity lemmas for $n$-ary hypergraphs, $n\geq 3$.  These relate
to our present $\VCN_k$ as follows: the dimension notion obtained by applying our $\VCN_k$-dimension
to neighborhoods is known as $\VC_1$-dimension in~\cite{CT20}, weak $\VC$-dimension in~\cite{TW22}
and slicewise $\VC$-dimension in~\cite{Ter24} and is responsible for yielding stronger versions of
the Hypergraph Regularity Lemma. However, as previously mentioned, these are not sufficient for the
PAC learning task.

\notoc\subsubsection{Graph-based discriminators}

A similar combinatorial dimension has also appeared in~\cite{LM19a,LM19b} in the study of
(hyper)graph-based discriminators, which considers the problem of using samples to distinguish
between distributions. Their dimension is called ``graph $\VC$-dimension'' ($\gVC$), which can be
seen as our $\VCN_k$ specialized to the case of symmetric hypergraphs, and their work includes a
uniform convergence statement appropriate to their dimension. However, we cannot quote their result
out of the box, nor does it cover much of our paper as uniform convergence is only one of the more
straightforward building blocks of high-arity PAC theory.

\notoc\subsubsection{Property testing}
\label{subsubsec:proptest}

A framework related to $2$-PAC learning of families of graphs is that of graph and graphon property
testing (see~\cite{AS06} for a survey on the graph case and~\cite[Chapter~15.3]{Lov12} for the
graphon case), which can be described very informally as probably approximately learning the graph
from a large enough sample, but only up to (approximate) isomorphism. Under this interpretation,
property testing theory says that every hypothesis class is learnable up to approximate isomorphism.

On the other hand, high-arity PAC learning is concerned with learnability \emph{not} up to
isomorphism, e.g., it is not enough to know that the graph approximately looks like a clique on half
of its vertices, we want to know which vertices are in the clique. Naturally, this harder task is
not possible for every hypothesis class.

\notoc\subsection{Our contributions}

To better appreciate what is required to obtain all implications of Theorems~\ref{thm:kPACinformal}
and~\ref{thm:kPACkpartinformal}, we summarize below what we believe to be our main contributions in
this work:
\begin{enumerate}[wide, label={\arabic*.}]
\item Framing the question of high-arity learning, and the long-deferred question of learning under
  structured correlation, in a way that allowed for an interesting theory to be developed (and
  developing what we believe to be such a theory).

  We are encouraged by the fact that the dimension we find specializes to ones studied in several
  different areas of mathematics, albeit currently much less than $\VC$-dimension. It is reasonable
  to hope that high-arity learning would correspond also to a combinatorial class with wide
  applicability, as was the case with $\VC$ classes from the classical theorem.

\item Suggesting the right definition of high-arity agnostic (and proving theorems that support the
  definition). This question was probably not posed earlier because it requires the right framework
  for high-arity learning to make sense.

\item With the important exception of ``agnostic'', the items already present in
  Theorem~\ref{thm:FTSL} are not the difficult ones to generalize: the difficulty is rather that to
  obtain the equivalence in our main theorems, one needs to build a theory having no analogue in the
  classical case (as one would indeed hope if the word ``generalization'' is to carry its
  weight). This will be a main topic of Section~\ref{sec:exp}.

  Rather, the combinatorial core of the challenge in our main theorems has to do with the
  interaction of partite and non-partite, for which we can find no analogue in the mathematical or
  computer science literature, beyond perhaps the connection to exchangeability and separate
  exchangeability mentioned in Section~\ref{subsubsec:exch}, or the very rough model theoretic idea
  of the difference between the random graph and the independence property.
\end{enumerate}

\subsection*{Acknowledgments}

The authors would like to thank Alexander Razborov, Avi Wigderson and Shay Moran for insightful
conversations about an initial version of the manuscript. We would also like to thank Al Baraa Abd
Aldaim for finding an imprecision in the proof of Lemma~\ref{lem:flexibility} of an earlier version
of the paper that lead to the current more general definition of flexibility (see
Footnote~\ref{ftn:flexibility}).

\bigskip
%\clearpage

Below is the structure of the paper; a more detailed overview can be found in
Section~\ref{subsec:org}. Figure~\ref{fig:roadmap} contains a pictorial road map for the different
implications proved in this article.

\tableofcontents

\input{roadmap}

\section{An exposition for the case of graphs}
\label{sec:exp}

The main objective of this expository section is to be a self-contained high-level view of the
central definitions, theorems and methods of the paper in the comparatively simple case of
graphs. As such, some overlap and repetition of topics of the introduction and future sections is
inevitable. (The reader who prefers to start with the main proofs is free to skip this section and
to return as needed.) Let us first review the motivation.

\medskip

Probably approximately correct (PAC) learning deals with the question of when we may have a good
chance of identifying one subset of $X$, or something close to it, from among a family $\cH$ of
subsets given in advance, based on information received from a small i.i.d.\ sample. This is
obviously extremely useful.

However, there are many classes of structures that might intuitively seem tractable, but are not
learnable in the PAC sense. Part of the reason, as we shall see, is that higher arity information
isn't really visible to i.i.d.\ sampling.

In order to statistically learn higher-arity structures (such as graphs, hypergraphs, or more
generally models in a finite relational language) within a family of similar structures, it is very
natural to ask for a weakening of i.i.d.\ sampling that reveals finite induced substructures. This
would extend learnability, and all that it entails (covering, approximation, etc.) to a different
class of phenomena; but it requires rebuilding the superstructure to allow for all aspects of the
Fundamental Theorem of Statistical Learning to be recovered in this setting.

We shall develop a framework in which this much more informative sampling is available but a uniform
convergence law can still be proved, and establish a combinatorial characterization of what we call
``high-arity'' agnostic learnability in both the partite and non-partite settings. Even though the
new fundamental theorem has analogies to the classical one, it turns out that handling high-arity
statistical learning is a problem at a different scale, and it reveals phenomena of independent
interest, such as the role of higher-order variables, the full generality of the agnostic adversary
and the interplay between partite and non-partite objects.

As the title of the paper suggests, we will also draw on a certain understanding from
exchangeability, and from model theory, whose relation in our previous papers lays some foundations
for the present work.

\smallskip

\emph{Note:} Since everything is done in full detail in later sections, the discussion here is less
formal, and there is of necessity some lingering over notation. We will have to say what we mean by
hypothesis class, by sampling, by loss function, by learning, crucially by \emph{agnostic} learning,
and what is the right notion of combinatorial dimension. We use more generality than is needed just
for graphs in order to smooth the transition to the main text for the reader who begins here.

\notoc\subsection{Description of a problem}

Let us start with a very simple example showing why there is something to do: classical PAC learning
can fail on very simple tasks the moment the complexity of the structures involved goes from arity
one (subsets) to arity two (say, graphs).

Suppose we have a family of graphs $\cG\df\{G_i\mid i\in I\}$ all on the same vertex set $X_1$, so
each $G_i = (X_1, E_i)$. To make the graph problem fit classical PAC, we may consider each $G_i$ as
the ``unary'' hypothesis $U_i$ by setting $X\df X_1\times X_1$ and $U_i\df E_i$, the set of ordered
pairs that are edges.\footnote{Since each $G_i$ is a graph, $(a,b)\in E_i$ if and only if $(b,a)\in E_i$.}

Let $\cU\df\{U_i\mid i\in I\}$. Attempting to learn $\cU$ according to classical PAC, we receive
elements sampled i.i.d.\ from $X$ with the adversary's label, that is, we receive the data of
various pairs $(a,b)$ along with whether or not there is an edge from $a$ to $b$. (These different
pairs likely have disjoint endpoints, so indeed, from the sample we may not even see if the edge is
symmetric.) Per classical PAC, $\cU$ is learnable if and only if $\cU$ has finite
$\VC$-dimension. However, for a set of size $n$ to be shattered in $\cU$, it suffices that there be
distinct $a_1,b_1,a_2,b_2,\ldots,a_n,b_n$ from $X_1$ so that for each $A\subseteq\{1,\ldots,n\}$,
for some $i=i(A)$, for each $\ell\in\{1,\ldots,n\}$, $(a_\ell, b_\ell)\in E_i$ if and only if
$\ell\in A$. Thus arbitrary ``unary'' shattering can arise in some $\cU$ built from a very simple
family of graphs, call it $\cG_*$ for future reference, where the degree of each vertex is at most
one.\footnote{$\cG_*$ may even have degree one \emph{across the family}, meaning that for each
  vertex $v$, as the neighborhood $N(v)\subseteq\NN$ varies as the graphs in $\cG_*$ vary, it is
  only ever $\{w(v)\}$ or $\emptyset$.} Whatever our notion of complexity for families of graphs, we
might not feel that degree one should meet it.

Now, if we were not restricted to i.i.d.\ sampling, a very natural idea (the idea of the present
paper) would be to approach the task of learning for $\cG$ by sampling vertices $x_1,\ldots,x_n$
from $X_1$ and receiving the information about the adversary's \emph{induced subgraph} on these
points. The notation we introduce next will allow us to do this in a quite robust way.\footnote{It
 will follow from our first main theorem (specialized to graphs) that under this more powerful
 sampling and the corresponding new notion of high-arity PAC learning given below, the $\cG_*$
 mentioned above is learnable in the new sense.}

\smallskip

This example is just one indication that there is something to do. As logicians know from
diagonalization arguments, noticing that one case was not covered may not adequately represent the
size of the gap.

\notoc\subsection{The learning we would like to do}

Now we set up how we \emph{would} like to learn, and the sorts of classes whose learnability we
would like to characterize.

The theory of \emph{high-arity statistical learning} set out in this paper extends the framework of
PAC learning, both its definitions and its main theorems, from learning subsets of a set to learning
much richer structures over a set. Informally, we shall sample a set of points and receive all
information about the structure induced on those points. This is a much more informative notion of
sampling (which requires rebuilding the probabilistic framework supporting the learning). A priori
this is quite beyond classical PAC, as it is far from independent; and a posteriori our main
theorems confirm it.

Here are some classes of ``more general hypotheses'' we might wish to learn:

\begin{itemize}
\item We are given a set $V$, and a family $\cG$ of graphs on the vertex set $V$, meaning that each
  $G_i = (V_i, E_i)$, where $V_i = V$.
\item Instead of a family of graphs, where the edge relation is symmetric, suppose we are given a
  set $X$, and a family $\cT$ of tournaments on the set $X$.
\item In greater generality, suppose we are given a set $X$, not necessarily finite, and a family
  $\cM$ of models each with domain $X$, in an arbitrary finite relational language.
\end{itemize}

In applications, it may be even more natural to consider partite structures, in which there are one
or more sets, which are formally distinct, and predicates may hold between elements of different
sets, for instance:
\begin{itemize}
\item We are given sets $V_1,V_2,V_3$, and a family of tripartite structures (i.e., all on the
  same vertex set with this same tripartition), where each element of the family has some instances
  of a ternary relation $R\subseteq V_1\times V_2\times V_3$.
\end{itemize}
Some venerable learning problems, such as trying to predict which books or movies a person will
like, are more naturally modeled in a partite way. The paper will also handle such cases, for
arbitrary finite partitions and arbitrary finite relational languages of the appropriate arity. In
fact the partite setting will be central to the paper, and as we will see, there is an interplay
between the non-partite and partite cases in the proofs.

The classes of structures, partite and non-partite, in finite relational languages we have described
so far are the \emph{rank at most $1$ classes} for which our main theorems hold.

An interesting, rather subtle point is that our framework allows for the definition of classes that
go beyond graphs, structures, and partite structures in allowing for higher-order variables (note:
as distinct from higher-arity predicates). Our main theorems are for rank at most one classes, those
without higher-order variables. However, certain key results do go through in greater generality,
which suggests directions for future work. Also notable is that our definition of agnostic learning
always allows for the \emph{adversary} to play higher-order variables even in the rank at most one
setting of our main theorems. \emph{See Section~\ref{subsec:highervar} below}.

Let us look at the main definitions, focusing on the case of graphs.

\notoc\subsection{The definition of hypothesis}
\label{subsec:hypothesis}

We now introduce a more probabilistic language for describing hypotheses, and for maintaining finer
control of sampling. The scope encompassed by this notation may indeed seem excessive for graphs,
but will be justified by the characterizations in the main text.

Let $\NN_+\df\NN\setminus\{0\}$. We first associate a (standard) Borel space to each nonzero natural
number, and a measure to each space. A priori, the spaces need not have anything to do with each
other.\footnote{We will equip each space with a measure, and will deal with random variables indexed
  by nonempty finite subsets $A$ of some background set $V$, with, say, $\rn{x}_A$ drawn from
  $X_{\lvert A\rvert}$ via $\mu_{\lvert A\rvert}$. Similar notation has a history in limit theory
  and exchangeability, but is generalized and updated here to fit our context, as will be explained
  in the main text.}

\begin{definition}
  Call ${\Omega} = (\Omega_i)_{i\in\NN_+}$ a \emph{Borel template} if each $\Omega_i = (X_i,\cB_i)$
  is a (standard) Borel space; here $X_i$ is a set and $\cB_i$ is a $\sigma$-algebra on $X_i$. Call
  ${\mu} = (\mu_i)_{i\in\NN_+}$ a \emph{probability template} on ${\Omega}$ when for each $i\in\NN_+$,
  $\mu_i$ is a probability measure on $\Omega_i$.
\end{definition}

For the case of graphs, we can simply take $\Omega_1$ to be the interesting space (we will think of
$X_1$ as being the vertex set for our family of graphs), and we can take $X_i$ for $i\geq 2$ to be
trivial: say, the singleton set $\{e\}$.

\begin{convention}
  For the purposes of this section, let $X_1$ be whatever set is desired as the vertex set for the
  family of graphs, and let $X_i\df\{e\}$ for all $i\geq 2$.
\end{convention}

Next we introduce notation that will be used for sampling from these spaces. Let $V$ be any finite
or countable set, not necessarily related to $\NN_+$. We will heavily use the
notation:\footnote{Conveniently, any element of $r(V)$ has a size in $\NN_+$.}
\begin{align*}
  r(V) & \df \{A\subseteq V\mid A\text{ is finite and nonempty}\}.
\end{align*}
We will also use the shorthand ``$r(k)$'' in the case $V=[k]\df\{1,\ldots,k\}$.

In full generality, given a finite or countable $V$ and a Borel template ${\Omega}$, define
\begin{align*}
  \cE_V(\Omega) & \df \prod_{A\in r(V)} X_{\lvert A\rvert}
\end{align*}
equipped with the corresponding product $\sigma$-algebra. A probability template $\mu$ on $\Omega$
naturally gives a product measure $\mu^V$ on $\cE_V({\Omega})$ defined by $\bigotimes_{A\in
  r(V)}\mu_{\lvert A\rvert}$. We will use $\cE_k(\Omega)$ and $\mu^k$ as shorthands in the case
$V=[k]\df\{1,\ldots,k\}$.

In the case of graphs, $k = 2$, so
\begin{align*}
  \cE_2(\Omega)
  & \df
  \cE_{\{1,2\}}(\Omega)
  \df
  X_{\lvert\{1\}\rvert}\times X_{\lvert\{2\}\rvert}\times X_{\lvert\{1,2\}\rvert}
  =
  X_1\times X_1\times X_2
\end{align*}
equipped with the corresponding product $\sigma$-algebra. (It will be important that in the product
the copies of each $X_n$ are indexed by subsets of size $n$.) If we have a probability template $\mu$ on
$\Omega$, we have the natural product measure on $\cE_2(\Omega)$, namely
\begin{align*}
  \mu^2 & \df \mu_1\otimes\mu_1\otimes\mu_2
\end{align*}
or when we want to remember the indexing,
$\mu_{\lvert\{1\}\rvert}\otimes\mu_{\lvert\{2\}\rvert}\otimes\mu_{\lvert\{1,2\}\rvert}$. In the case
of graphs, having stipulated that $X_2$ is trivial, we can think about $\cE_2(\Omega) = X_1\times
X_1\times\{e\}$ as ``really'' being $X_1\times X_1$ with the measure given by $\mu_1
\otimes\mu_1$. But we will continue to write the vestigial coordinate to make a bridge to the full
text.\footnote{Readers curious about this coordinate may look ahead to
  Section~\ref{subsec:highervar}.}

\begin{convention}
  In this section, we let $\Lambda$ denote the Borel space $(Y,\cB')$, where $Y\df\{0,1\}$ is
  equipped with the discrete $\sigma$-algebra.
\end{convention}

We will think of $\Lambda$ as denoting true or false, yes or no.

\begin{definition}
  A \emph{hypothesis} in the case of graphs is a measurable function
  $H\colon\cE_2(\Omega)\to\Lambda$.
\end{definition}

A hypothesis takes in an element $(u,v,e)$ of $X_1\times X_1\times\{e\}$ and outputs $1$ or $0$:
this amounts to saying it takes in an ordered pair of elements of $X_1$ and tells us whether or not
there should be an edge between them. Notice that, analogously to the interpretation function in
model theory, $H$ tells us all the ``directed'' edges and to figure out if it is describing a
symmetric binary relation we will need presently to reorganize this information (say, to compare
what $H$ outputs on $(u,v,e)$ and on $(v,u,e)$). As an aside, this definition of hypothesis entails
that even if we are trying to learn a family of graphs, in the improper case\footnote{i.e., when the
  algorithm output is not required to be an element of the hypothesis class.} there is a priori no
reason why the hypothesis we output must be a graph (with symmetric edges).

\begin{definition}
  Let $\cF_2(\Omega,\Lambda)$ denote the measurable functions from $\cE_2(\Omega)$ to $\Lambda$, so
  we can write briefly that a hypothesis class is $\cH\subseteq\cF_2(\Omega,\Lambda)$.
\end{definition}

Some notation will be very helpful in what follows.\footnote{What will we need to describe the
  structure on a $V$-indexed sample of vertices?} Let $(V)_2$ be the set of all injective functions
$\alpha\colon[2]\to V$ (recall $[k]\df\{1,\ldots,k\}$). Recall that $\cE_m(\Omega)$ is a product
having $m$ factors of $X_1$, each indexed by some $\{i\}$; $\binom{m}{2}$ factors of $X_2$, each
indexed by some $\{i,j\}$; and so on up to one factor of $X_m$ indexed by $[m]$ (and in the case of
graphs we have set the $X_i$'s for $i\geq 2$ all equal to $\{e\}$). Observe that any injection
$\alpha\colon[2]\to[m]$ gives a natural map $\alpha^*\colon\cE_m(\Omega)\to\cE_2(\Omega)$ defined as
follows: if $\alpha(1) = i$, $\alpha(2)=j\neq i$, our $\alpha^*$ takes in an element of
$\cE_m(\Omega)$ and returns the triple in $\cE_2(\Omega)$ consisting of the coordinate that came
from $X_{\lvert\{i\}\rvert}$, the coordinate that came from $X_{\lvert\{j\}\rvert}$, and the
coordinate that came from $X_{\lvert\{i,j\}\rvert}$. For us, the third coordinate is invariably $e$.

Given a hypothesis $H\colon\cE_2(\Omega)\to\Lambda$ describing a graph on $X_1$, and a sequence
$(a_1,\ldots,a_m)$ of elements of $X_1$, to describe the induced subgraph on these $m$ vertices it
suffices to know what $H$ does on every $(a_i,a_j,e)$. The following slightly more general notation
will be our way to produce the induced subgraph. Define
\begin{equation*}
  \begin{array}{rrcl}
    H^*_m\colon & \cE_m(\Omega) & \longrightarrow & \Lambda^{(V)_2}\\
    & x & \longmapsto & (H(\alpha^*(x)))_{\alpha\in([m])_2}.
  \end{array}
\end{equation*}
In other words: for each injection $\alpha\colon[2]\to[m]$, and each element $(a_1,\ldots,a_m)$ of
the product, evaluate the hypothesis $H$ on $(a_{\alpha(1)},a_{\alpha(2)},e)$ to obtain an element
of $\Lambda$ (``edge or no edge''). The output of $H^*_m$ on $(a_1,\ldots,a_m)$ is the sequence of
these decisions about edges. For reference we summarize:
\begin{quotation}
  \noindent $H^*_m(a_1,\ldots,a_m)$ returns the (possibly directed) graph on $(a_1,\ldots,a_m)$
induced by $H$.
\end{quotation}
In model theoretic language, $H^*_m$ collects from $H$ all data about the labeled quantifier-free
diagram on $\{a_1,\ldots,a_m\}$, i.e., supposing $a_1,\ldots,a_m$ are named by constants. Soon,
when we remember the measures on $X_1$ and can obtain $(a_1,\ldots,a_m)$ via sampling then this
notation will give us a way of referring to the graph induced on the sample by a given hypothesis.

Also, conveniently, when $m=2$, our $H^*_2$ takes in some $(a,b,e)$ and tells us whether $H$ puts an
edge from $a$ to $b$ \emph{and} whether from $b$ to $a$. We may save ourselves some notation here
(or whenever $k=m$) by describing the range of $H^*_2$ as $\Lambda^{S_2}$, where $S_2$ is the
symmetric group on two elements, rather than using the language of injections. Notice there are also
natural actions of the symmetric group with respect to which $H^*_2$ is $S_2$-equivariant; this
explains how $H^*_2(a,b,e)$ and $H^*_2(b,a,e)$ present the same information: if, say
$H^*_2(a,b,e)=(0,1)$, meaning that $H$ says non-edge from $a$ to $b$ and edge from $b$ to $a$, then
$H^*_2(b,a,e)=(1,0)$, i.e., $H$ says edge from $b$ to $a$ and non-edge from $a$ to $b$.

\notoc\subsection{The basic definitions of loss}

There are at least three crucial definitions in learning. First, what it means to sample. Second,
how we quantify loss. And third, what the adversary is allowed to do. (This will be the subject of
Sections~\ref{subsec:agnostic} and~\ref{subsec:agnosticdef} on agnostic learning, below.)

The job of a basic, as opposed to agnostic, $k$-ary loss function\footnote{The agnostic case will be
  sufficiently distinguished in our present setting to justify, in our view, calling the
  non-agnostic case ``basic'' when we want to single it out.} $\ell$ is to quantify how different
are two hypotheses (two graphs) on a given finite sequence of vertices. For graphs, since $k=2$, the
unit of information is how different are two hypotheses $H$, $F$ on a single
$(a,b,e)\in\cE_2(\Omega)$: and for this it suffices to compare the information returned by
$H^*_2(a,b,e)$ and $F^*_2(a,b,e)$. When $k=2$, then, we require simply that:
\begin{definition}
  A basic \emph{loss function} is a measurable function
  \begin{align*}
    \ell\colon\cE_2(\Omega)\times\Lambda^{S_2}\times\Lambda^{S_2}\to\RR_{\geq 0}.
  \end{align*}
\end{definition}
Writing the loss this way, rather than, say, as a function that takes in $H$, $F$, and $(a,b,e)$
and makes the analogous computations, emphasizes that the \emph{only} information it uses from $H$
and $F$ are these sequences of values. For graphs:

\begin{definition}
  Given further a probability template $\mu$, thus a way of sampling $(\rn{a},\rn{b},\rn{e})$ via
  $\mu_1\otimes\mu_1\otimes\mu_2$, the \emph{total loss} $L_{\mu,F,\ell}(H)$ is the expected value
  of $\ell((\rn{a},\rn{b},\rn{e}), H^*_2((\rn{a},\rn{b},\rn{e})), F^*_2(\rn{a},\rn{b},\rn{e}))$.
\end{definition}

Informally, we are \emph{sampling} a pair of vertices $a,b$ and asking both $H$ and $F$ whether they
put an edge each way, then calculating the penalty for the difference.

A choice of loss function will be an input to various lemmas and theorems. Some desirable properties
of such functions are discussed in the text:
\begin{enumerate}
\item Observe that a basic loss function can receive the same information in multiple ways: per the
  last paragraph of Section~\ref{subsec:hypothesis}, asking the loss function to compare
  $H^*_2(a,b,e)$, $F^*_2(a,b,e)$, or asking it to compare $H^*_2(b,a,e)$, $F^*_2(b,a,e)$ is
  essentially the same question. A priori the loss function need not give the same answer. Those
  that do are called \emph{symmetric}.\footnote{In particular, symmetric is \emph{not} about whether
    $\ell(x,r,s)=\ell(x,s,r)$, but rather whether $\ell$ is $S_2$-invariant with respect to the diagonal
    action.}
\item The loss is \emph{separated} if it evaluates to $0$ when the second and third inputs (in
  $\Lambda^{S_2}$) are equal, and otherwise is bounded away from $0$, i.e., we incur zero penalty if
  we guessed correctly and incur penalty bounded away from $0$ when we guessed incorrectly.
\item The loss is \emph{bounded} if it has a uniform finite upper bound.
\end{enumerate}

A (proper) learning algorithm for $\cH$ is a measurable function
\begin{align*}
  \cA\colon\bigcup_{m\in\NN} (\cE_m(\Omega)\times\Lambda^{([m])_k})\to\cH
\end{align*}
meaning informally in our case that it receives the data of an induced subgraph on some finite
sequence of vertices and outputs some hypothesis in the class. Fix a binary loss function
$\ell\colon\cE_2(\Omega)\times\Lambda^{S_2}\times\Lambda^{S_2}\to\RR_{\geq 0}$. The adversary's
input will be to choose some probability template $\mu$ along with some $F\in\cF_2(\Omega,\Lambda)$. As
in classical PAC, $F$ need not be a hypothesis belonging to $\cH$, but should be arbitrarily well
approximated by $\cH$ (in the sense of total loss):
\begin{align*}
  \inf\{L_{\mu,F,\ell}(H)\mid H\in\cH\} & = 0.
\end{align*}
Such an $F$ is called \emph{realizable} in $\cH$.

\notoc\subsection{The definition of basic (non-agnostic) high-arity learning}
\label{subsec:basicPAC}

We now summarize basic (non-agnostic) $k$-PAC learning for our running example of graphs. We may
call this ``high-arity learning,'' or ``$k$-PAC learning'' when we want to emphasize the value of
$k$, here~2. (We will also often continue to say ``classical'' or ``classic'' PAC for the existing
definition of PAC.)

\begin{quotation}
  \noindent\textbf{Informal version of basic (non-agnostic) graph learning:}

  \noindent
  Given $\Omega$ and a family $\cH$ of graphs over $X_1$, we say $\cH$ is $k$-PAC learnable (for
  $k=2$, and loss $\ell$) if for some learning algorithm $\cA$, for every $\epsilon,\delta>0$ there
  exists $m^{\PAC}_{\cH,\ell,\cA}(\epsilon,\delta)$ so that the following holds:\footnote{For
    technical reasons, we allow this number to be a real; then the minimum sample size is the next
    largest integer.} The adversary fixes any probability template $\mu$ and hypothesis $F$ (realizable
  in the sense of $\mu$), both unknown to us. Our $\cA$ receives a sample of $m\geq
  m^{\PAC}_{\cH,\ell,\cA}$ vertices plus the graph induced on them by $F$. (Formally, it receives an
  element of $\cE_m(\Omega)$ sampled via $\mu^m$, labeled by $F^*_m$.) Our $\cA$ then outputs some
  $H\in\cH$. Then with probability at least $1-\delta$, the total loss of $H$ against $F$ is at most
  $\epsilon$.
\end{quotation}

Note that when $k=1$ we ``morally'' recover classical PAC learning. We might also give an:
\begin{quotation}
  \noindent\textbf{Extremely informal version of basic graph learning:}

  \noindent
  We are given a family $\cH$ of graphs over the same vertex set $X$. The adversary chooses a
  measure over $X$, and some $G$ that is, or looks very much like, a graph in $\cH$; these are both
  unknown to us. We receive a finite sample of vertices from $X$, along with the subgraph induced
  by $G$ on these vertices. Our algorithm $\cA$ outputs some $H\in\cH$. The outcome is then
  judged by sampling a new pair of points from $X$, still according to the adversary's measure, and
  computing the loss (i.e., comparing whether $H$ and $G$ think there should be an edge).
\end{quotation}

\notoc\subsection{Motivation: agnostic learning}
\label{subsec:agnostic}

In the classic PAC setting, the fully general case of learning is Haussler's definition of
\emph{agnostic} learning, which upgrades PAC in several ways:

\begin{itemize}
\item First, an agnostic loss function takes as input not only the two labelings of a point (that of
  the algorithm and the adversary), but also the hypothesis that the algorithm uses to do the
  labeling.\footnote{Maybe it levies a higher penalty if the algorithm $\cA$ labels using a
    high-degree polynomial than to the same labeling arising from a low-degree polynomial. This is
    the topic of regularization.}
\item Second, rather than choosing a labeling of $X$ by $\{0,1\}$, the adversary is allowed to
  choose any distribution over $X\times\{0,1\}$; the algorithm then receives pairs $(a_i,t_i)\in
  X\times\{0,1\}$ sampled i.i.d.\ from the adversary's distribution. This data need not bear any
  relation to the graph of a function, and need not even be realizable.
\end{itemize}
One can still compute the total loss of the hypothesis outputted by $\cA$, as well as the total
loss of any other hypothesis $H\in\cH$, using the given agnostic loss function. This means our
learning task also changes:
\begin{itemize}
\item In (classical) agnostic learning, $\cA$ is successful if its total loss is essentially that
  of the best $H\in\cH$.
\end{itemize}
To emphasize the difference, in agnostic learning the adversary is stronger because it can use
randomness, but our task is slightly simpler because we just have to do as well as could be
expected. In the classical PAC case, the Fundamental Theorem of Statistical Learning characterizes
both non-agnostic PAC learnability and agnostic PAC learnability in the same way.

Agnostic learning will be our main case below. In our present program, in changing the notion of
sampling, the question of what is the right notion of agnostic adversary for high-arity learning was
a test of the framework -- but there is a beautiful answer.

\notoc\subsection{Definition of agnostic learning}
\label{subsec:agnosticdef}

To start with the punchline, the high-arity agnostic case will essentially allow the agnostic
adversary to play distributions that are local and exchangeable. We just motivate this simply
here.

Suppose we are given a finite or infinite graph $G=(V,E)$, considered as a model-theoretic structure
in the full language: the vertices are named by constants $\{c_i\mid i\in V\}$, and we have a binary
edge relation $E\subseteq V\times V$. Suppose we have some measure over $V$, which allows us to
sample points. After taking such a sample, we look at the information received about the constants
$c_i$ that were named in the sample, and about the subgraph induced on them by $E$. Whether we saw
that $E(c_i, c_j)$ doesn't depend on the order in which they appeared, nor on any other information
in the sample.

Thus, in considering a framework where ``sample'' means ``induced substructure,'' and in which we
might want to let the adversary play arbitrary models, naturally leads beyond trying to find some
literal analogue of the ``distribution on $X\times\{0,1\}$'' framework, to local and exchangeable
distributions. A worry, a priori, is that these could be extremely complex. Here we call on theorems
from exchangeability theory, which give a normal form for such distributions. This interfaces with
our setup in a very nice way. It is stated in the main text, and proved in an appendix, that it
suffices to let $\Omega'$ be another Borel template, $\mu'$ a probability template, and then we can
recover full generality by allowing the adversary to play some $F\in
\cF_k(\Omega\otimes\Omega',\Lambda)$, i.e., a $k$-ary hypothesis over the product template.

To summarize, in the central definition of agnostic high-arity learning, the main upgrades are:

\begin{itemize}
\item we allow the hypothesis to appear in the loss function;
\item we allow the adversary to play some $F\in\cF_k(\Omega\otimes\Omega',\Lambda)$;
\item our task is to be as good as possible, measured in the sense of total loss (informally, our
  loss against $F$ when a new pair of points is sampled is at most $\epsilon$-worse than that of the
  best hypothesis in the class).
\end{itemize}

\notoc\subsection{The combinatorial dimension $\VCN_k$}
\label{subsec:VCNk}

We now explain, for the case of graphs, the notion of dimension whose finiteness will characterize
agnostic learnability in our high-arity main theorem. Very informally:

\begin{quotation}
  Given the task of learning a graph from among a family of graphs, it is reasonable to think that
  it would suffice to approximately learn most neighborhoods of single vertices, with the same
  efficiency.
\end{quotation}
In some sense that is what the $\VCN_2$-dimension says, though note that is not exactly how the
learning in our algorithm happens. For graphs, since the edge relation is binary and symmetric,
finite $\VCN_2$-dimension has a simple description. When $\cH$ is a family of graphs on the vertex
set $X_1$, fix each individual vertex and look at the $\VC$-dimension of its set of neighborhoods
varying across the graphs in the family (this is a set of subsets of $X_1$); for each $x$ the
$\VC$-dimension should be finite, and moreover there should be a uniform finite bound as $x$
varies. (If there is no uniform finite bound, the $\VCN_2$-dimension is $\infty$.)

For a family of directed graphs, once one fixes a vertex $x$, every other vertex $y$ can be classified
into one of four classes: no edges with $x$; a single edge from $x$ to $y$; a single edge from $y$
to $x$; or both edges from $x$ to $y$ and from $y$ to $x$. Thus, we can naturally compute the
Natarajan dimension of this family of functions classifying the vertices $y$ into these four classes
and we take $\VCN_2$-dimension as the supremum over $x$ of all these Natarajan dimension (this
explains the ``N'').

For a family of $3$-uniform hypergraphs all on the same vertex set, finite $\VCN_3$-dimension can be
described as follows. Each vertex $u$ induces a (binary) graph on its vertex set, in which we put an
edge between $v$, $w$ if and only if $(u,v,w)$ is a hyperedge. To compute $\VCN_3$, fix each
individual vertex $u$ and look at the graph family described by its neighborhoods. This family
should have finite $\VCN_2$-dimension, and this number should be uniformly finite across the choice
of $u$. Of course, this is the same as picking two vertices $u,v$, computing the $\VC$-dimension of
the resulting family of sets of $w$ such that $(u,v,w)$ is a hyperedge and taking supremum over the
choice of $(u,v)$.

In the general case, the definition of $\VCN_k$-dimension is more intricate as it allows for up to
finitely many not necessarily symmetric relations, larger $k$ and rank greater than $1$ (that is, it
also has to take into account higher-order variables).

\notoc\subsection{The advent of the partite case}
\label{subsec:partite}

A priori, it might be more difficult to learn bipartite graphs than to learn graphs, and the graph
problem does not obviously allow us to simulate the bipartite graph problem: one reason is that a
theory of learning bipartite graphs should allow for the choice of different measures over the two
sides of the vertex partition (each chosen by the adversary and unknown to us), which may be hard to
simulate in a graph equipped with a single measure. Another is that sampling from bipartite graphs
may a priori reveal less information. Compare learning a graph $G$ to learning its natural
bipartization $G^{\bip}$, formed by doubling the vertex set of $G$ to $U$, $W$ and connecting $a\in
U$ with $b\in W$ if and only if $(a,b)$ is an edge in $G$. A priori $G$ and $H\df G^{\bip}$ have the
same data. But even if we assume the measure on both sides of $H$ is the same, sampling may a priori
not reveal the same kind of information.\footnote{If we sample from the vertex sets $U, W$ of $H$
  and obtain $\{a_1,\ldots,a_r\}\in U$, $\{b_1,\ldots,b_\ell\}$ in $W$, then according to our new
  notion of bipartite sampling and learning we will receive information about the neighborhood of
  each $a_i$ in $\{b_1,\ldots,b_\ell\}$. If we sample from the vertex set of $G$ and obtain
  $\{c_1,\ldots,c_m\}$ in $V$, according to our new notion of non-partite graph sampling and
  learning we will receive information about the neighborhood of each $c_i$ in $\{c_1,\ldots,c_m\}$
  along with how all the $c_j$'s and $c_k$'s interact amongst themselves. A priori, this might make
  learning $G$ easier than learning $H$.}

There are two reasons we nonetheless deal with partite structures.

First, because they are very natural for applications (and for model theorists\footnote{because
  inputs to the different variables in formulas need not be drawn from the same place.}); e.g., the
``Netflix problem'' of learning which movies a person likes is naturally modeled in a partite way.
So it is desirable to understand them, and one of our main theorems is a characterization of
high-arity learning for partite structures.

Second, because in our analysis of ordinary (non-partite) structures we will use their ``partite''
versions in order to prove certain arrows in the main theorem, and so partite structures become, in
our approach, a necessary part of the theory.

We now give the main high-arity PAC definitions for bipartite graphs.\footnote{Again, note that
  there will be two ways in which bipartite graphs (or generally, $k$-partite structures) appear in
  our proofs. We may want to learn structures given \emph{as} partite structures, or to learn
  structures that arise as the ``partization'' of some other structure we hope to learn.} We won't
give the same level of detail but will mainly indicate subtleties and differences.

\notoc\subsection{The definition of hypothesis in the partite case}

We make a few natural changes to the graph setup, or what we shall henceforth often refer to as the
\emph{non-partite} case (to distinguish it from the partite one\footnote{This terminology alone says
  something about the centrality of the partite case in the paper.}).

\begin{definition}
  A bipartite Borel template has the form
  \begin{align*}
    \Omega & = (\Omega_A)_{A\in r(2)} = (\Omega_{\{1\}},\Omega_{\{2\}},\Omega_{\{1,2\}}).
  \end{align*}
  A probability bipartite template on $\Omega$ is $\mu = (\mu_A)_{A\in r(2)} = (\mu_{\{1\}},
  \mu_{\{2\}}, \mu_{\{1,2\}})$, where each $\mu_A$ is a probability measure on $\Omega_A$.
\end{definition}

Parallel to the case of graphs, in our bipartite setup we will think of our bipartite graph as
having vertex set $X_{\{1\}}\cup X_{\{2\}}$ (from $\Omega_{\{1\}}$, $\Omega_{\{2\}}$,
respectively), whereas the items indexed by larger subsets won't play a role:

\begin{convention}
  In this section, we assume $\Omega_{\{1,2\}}$ is trivial, meaning its $X_{\{1,2\}}=\{e\}$.
\end{convention}

\begin{convention}
  In this section, $\Lambda$ is still $(Y,\cB')$, where $Y\df\{0,1\}$ equipped with the discrete
  $\sigma$-algebra.
\end{convention}

In the bipartite case, the parallel to sampling $m$ points from the graph via $\cE_m(\Omega)$ is to
sample $m$ points from \emph{each side} of the bipartition. Define\footnote{In the main text the
  expression $\cE_m(\Omega)$ is used for both the usual and the partite case; there should be no
  confusion since $\Omega$ is a formally different object in the two cases, and it is convenient
  there to be able to compactly state theorems that apply to both cases. However, for expositional
  purposes in this section, we shall use $\cE_m$ for the graph case and $\cE_{m,m}$ for the
  bipartite case.} $\cE_{m,m}(\Omega)$ to have $m$ factors of $X_{\{1\}}$, $m$ factors of
$X_{\{2\}}$, and $m\times m$ factors of $X_{\{1,2\}}=\{e\}$. In particular
\begin{align*}
  \cE_{1,1}(\Omega) & = X_{\{1\}}\times X_{\{2\}}\times\{e\}.
\end{align*}
So for bipartite graphs we can define:
\begin{definition}
  Given $\Omega$ a bipartite Borel template, a bipartite hypothesis from $\Omega$ to $\Lambda$ is a
  measurable function $H\colon\cE_{1,1} (\Omega)\to\Lambda$.
\end{definition}
\noindent Note that a bipartite graph hypothesis simply specifies ``edge or non-edge'' on $(x,z,e)
\in X_{\{1\}}\times X_{\{2\}}\times\{e\}$; it never opines on elements of $X_{\{2\}}\times
X_{\{1\}}\times\{e\}$. This reflects that once the vertex set is partitioned we won't assume
relations stay well defined when the provenance of their inputs changes.\footnote{Indeed in the case
  of $k > 2$, or non-graphs, the more complex notation in the main text still only ever asks the
  hypothesis questions about tuples in increasing order across the partition. In terms of
  representing various structures, there is no real loss of generality in assuming that any relation
  symbol in the language specifies the provenance of each of its coordinates in the partition, and
  that those coordinates are given in increasing order.}

\begin{convention}
  Let $\cF_2(\Omega,\Lambda)$ denote the measurable functions from $\cE_{1,1}(\Omega)$ to $\Lambda$,
  so that a bipartite hypothesis class is $\cH\subseteq\cF_2(\Omega,\Lambda)$.\footnote{We will
    keep the same notation for $\cF_2$ for the partite and non-partite case, but the $\Omega$ being
    used will make it clear which case we are in.}
\end{convention}

Finally, parallel to the graph case, we want a function that takes in a hypothesis $H$ and an
element of $\cE_{m,m}(\Omega)$ (note: since for us the remaining coordinates all have value $e$,
this amounts to receiving $(x_1,\ldots,x_m)$ from $X_{\{1\}}\times\cdots\times X_{\{1\}}$ and
$(z_1,\ldots,z_m) $ from $X_{\{2\}}\times\cdots\times X_{\{2\}}$) and returns the collected data of
$H$ on each $(x_i,z_j,e)$ from $X_{\{1\}}\times X_{\{2\}}\times\{e\}$. This function is
\begin{align*}
  H^*_{m,m}\colon\cE_{m,m}(\Omega)\to\Lambda^{m\times m},
\end{align*}
defined in the natural way.

\notoc\subsection{The definition of loss and learning in the partite case}

Parallel to the graph case, in the bipartite case the basic unit of information for the loss
function is how two hypotheses may differ on their labeling of $(a,b,e)\in\cE_{1,1}(\Omega)$. Thus:
\begin{definition}
  A bipartite loss function is a measurable
  $\ell\colon\cE_{1,1}(\Omega)\times\Lambda\times\Lambda\to\RR_{\geq 0}$.
\end{definition}
A fine point: comparing to the loss function for graphs defined earlier, the exponent of $S_k$
disappears. This reflects what was just discussed about inputting the vertices in order of the piece
of the partition: the only question we have for a bipartite $H$ on $(a,b,e)$ is the single value
$H(a,b,e)$. And ``symmetric'' is not a property of partite loss functions since the spaces are
different.

\smallskip
Now the main differences in learning, in the case of bipartite graphs, may be presented at the level
of the ``extremely informal version'' of Section~\ref{subsec:basicPAC}. In particular, notice that
samples have size $2m$ in which $m$ points are chosen from each side.

\begin{quotation}
\noindent\textbf{Extremely informal version of basic \emph{bipartite} learning:}

\noindent
We are given a family $\cH$ of bipartite graphs, each over the same vertex set with the same
bipartition into $X_{\{1\}}$ and $X_{\{2\}}$. The adversary chooses possibly different measures
$\mu_{\{1\}}$, $\mu_{\{2\}}$ over $X_{\{1\}}$, $X_{\{2\}}$, respectively, along with some $G$ that
is, or looks very much like, a bipartite graph in $\cH$; these are both unknown to us. We receive a
finite sample of $2m$ vertices, consisting of $a_1,\ldots,a_m$ drawn from $X_{\{1\}}$ according to
$\mu_{\{1\}}$, and $b_1,\ldots,b_m$ drawn from $X_{\{2\}}$ according to $\mu_{\{2\}}$, along with
all information about the bipartite graph induced between the $a$'s and the $b$'s by $G$. Our
algorithm $\cA$ outputs some $H\in\cH$. The outcome is then judged by sampling a new pair of points,
one from $X_{\{1\}}$ and one from $X_{\{2\}}$ (still according to the respective measures) and
computing the loss (i.e., comparing whether $H$ and $G$ think there should be an edge).
\end{quotation}

\notoc\subsection{Agnostic learning in the partite case}

Analogously to the non-partite case, first, there is a more general notion of loss that also takes
in the hypothesis doing the labeling. Second, there arises the central question in the case of
partite (high-arity) agnostic learning: what can the agnostic adversary do.

In the partite case, the punchline is to allow such an adversary to play local \emph{separately}
exchangeable distributions (which arise naturally in sampling from partite structures where all of
the elements are labeled by constants, and receiving information about the edges between constants
that \emph{cross the partition}). Again, there is an appeal to a normal form for such
distributions. The satisfying conclusion is that the agnostic adversary in the partite case simply
plays some $F\in\cF_k(\Omega\otimes\Omega',\Lambda)$. This is the same notation as in the
non-partite agnostic case, but of course notice that $\Omega$, $\Omega'$ in this case are partite.

Finally, just as in the non-partite case, the outcome of the agnostic learning task is judged by
computing the total loss of our $H$ against the adversary's $F$, compared to the total loss of the
best $H'\in\cH$ against the adversary's $F$.

\notoc\subsection{The definition of dimension in the partite case}

We need formally different, though closely related, definitions of dimension for the non-partite and
the partite cases. For a family of bipartite graphs on a vertex set given with the same bipartition
into $X_{\{1\}}$ and $X_{\{2\}}$, the bipartite $\VCN_2$-dimension also has a simple
description. First, for each vertex $a\in X_{\{1\}}$, compute the $\VC$-dimension of its set of
neighborhoods in $X_{\{2\}}$ (looking across the family of bipartite graphs). Take the supremum of
this number over all $a\in X_{\{1\}}$. Second, for each vertex $b\in X_{\{2\}}$, compute the
$\VC$-dimension of its set of neighborhoods in $X_{\{1\}}$ (looking across the family of bipartite
graphs). Take the supremum of \emph{this} number over all $b\in X_{\{2\}}$. The dimension is the
larger of the two if both are finite, and $\infty$ otherwise (and for high-arity learnability, we
would like the dimension to be finite).\footnote{From this it should be clear that if $\cH$ is a
  family of graphs, and $\cH^{\bip}$ is its bipartization in the sense of
  Section~\ref{subsec:partite}, then the bipartite $\VCN_2$-dimension of $\cH^{\bip}$ should be
  equal to the $\VCN_2$-dimension of $\cH$.}

\medskip
\begin{center}
  \noindent\emph{This completes our discussion of the main definitions. \\ Next we discuss the
    structure of the argument.}
\end{center}

\notoc\subsection{The structure of the argument}
\label{subsec:org}

Here we give a short reader's guide to the text. Again, since we will be dealing both with graphs
and bipartite graphs, we follow the convention in the text of using ``non-partite'' for the first
case and ``partite'' for the second. Here ``learnable'' always means in the new high-arity sense
unless otherwise stated.

\begin{description}[wide]
\item[Sections~\ref{sec:defs} and~\ref{sec:partdefs}] contain the main definitions (hypothesis,
  loss, agnostic, learning, etc.) for the non-partite and partite cases, respectively. See in
  particular agnostic learning (Definitions~\ref{def:agkPAC} and~\ref{def:partagkPAC}).
\item[Section~\ref{sec:mainthms}] contains the statements of the main theorems (repeated below for
  graphs).
\item[Section~\ref{sec:initialred}] lays out the more expected reductions and operations. For
  example, agnostic learnability implies learnability and the fact that in the non-agnostic case,
  $\VCN_k$-dimension and learnability (with randomness or not) are not affected by adding more
  symbols to $\Lambda$ (in the agnostic case, a priori learnability may become harder with more
  symbols, but it does not become easier).
\item[Section~\ref{sec:derand}] is devoted to derandomization. In order to complete the cycle in our
  main theorems, it will be very useful to allow the algorithms to use finitely many bits of
  randomness, e.g., coin flips or other finitely-valued outcomes. (For instance, we will want to
  receive samples from a partite learning problem, and transform them into samples that are
  intelligible as input to a non-partite learning problem, and randomness will help a lot.) But
  since we ultimately want to close the loop, it is important to prove that this doesn't change
  learnability, at the cost of increasing the sample size.

  A posteriori, derandomization means we could have given non-random versions of all of the proofs,
  but this would have made reasoning about distributions unintelligible. (It is also desirable to
  establish that allowing finite randomness doesn't change the main results.) The proofs in
  Section~\ref{sec:derand} apply to derandomize both partite and non-partite, and both agnostic and
  non-agnostic, high-arity learners.

  The first step is Lemma~\ref{lem:empconc} on the concentration of the empirical loss around the
  actual loss. The key technical result is Proposition~\ref{prop:derand}. Informally, suppose we
  have $b$ random bits, and we want a deterministic algorithm that learns the same problem, at the
  cost of increasing the sample size. We take both in a large sample and a large auxiliary
  sample. Using the first sample, we run our algorithm on this sample once for each of the possible
  values of the $b$ bits, and in each case we get the algorithm's output. Now it remains to choose
  the best one. To do so, we evaluate them all on the auxiliary sample. By the concentration of loss
  lemma, this should be sufficiently informative to allow us to choose (and the fact that the
  original algorithm was a PAC learner tells us one should work).
\item[Section~\ref{sec:part}] is a central section of the text, and builds up to what (by its
  proof) is one of the most interesting parts of the paper, namely the result that if we can learn a
  class $\cH$ of graphs then we can learn its bipartization $\cH^{\bip}$.

  Basic results are first established in Lemma~\ref{lem:kpartbasics} and
  Propositions~\ref{prop:kpartrk} and~\ref{prop:kpartVCN}: going from $\cH$ to $\cH^{\bip}$ doesn't
  affect rank or $\VCN_k$-dimension, though of course the dimension must be computed in the
  non-partite sense and the partite sense respectively.

  Learning is a priori a different matter. First consider the easier direction:
  Proposition~\ref{prop:kpart} says that if $\cH^{\bip}$ is learnable, then $\cH$ is learnable,
  whether the learner is agnostic or not, random or not. (As noted in Section~\ref{subsec:partite}
  above, in moving to the bipartization, even though we have not technically lost information, we
  have changed, a priori made weaker, the information that is available by finite sampling. If the
  algorithm can succeed even in the weaker case, it can also learn the original.)

  The main work of the section is to prove the converse, Proposition~\ref{prop:kpart3}. To set this
  up, let us compare the learning problems. In the non-partite case, we receive a sample of $m$
  points with the induced substructure; in the partite case, we receive $k\cdot m$ points, i.e., $m$
  in each of the $k$ pieces (here: $k=2$) given with the partition and only the edges crossing the
  partition. So in the partite case, it may appear that there is a lot of missing
  information. Moreover, in the partite case, the measures on each of the pieces are allowed to be
  different. To obtain the converse, we will find a randomized algorithm that converts an element of
  $\cE_m(\Omega^{\bip})$ into an element of $\cE_m(\Omega)$ (using a finite amount of randomness
  that only depends on $m, k$) and has the property that for any choice of partite measure $\nu$
  there exists a non-partite measure $\rho$, possibly not realizable, that is simulated by first
  sampling from $\nu$ and then applying the algorithm, and that has the property that the losses are
  comparable. Given such an algorithm, we learn the partization by converting our sample to one
  intelligible to the non-partite learner.\footnote{A posteriori, then, learning the partization is
    not harder, except that the sample size must become larger.} (Of course, a trivial way to
  produce something intelligible to the non-partite learner is to ignore the sample and do something
  else entirely, but this is not helpful; we also need, in the appropriate sense, to retain much of
  the original information.)

  One idea in the proof that may be of independent interest is that of a flexible loss function,
  which arises to solve the following problem. In order to distinguish the information we hope to
  transfer to the non-partite learner from the information that the partite structure has hidden
  from us, we would like to introduce a new symbol $\bot$ to $\Lambda$, interpreted to mean ``I
  don't know.'' A priori, in the case of agnostic learning, we can't increase $\Lambda$ without
  possibly changing the learning problem. This is solved in two steps. Define a \emph{neutral
    symbol} to be an element of $\Lambda$ that can in some precise sense be freely used as a symbol
  that does not affect the learning task (Definition~\ref{def:neutsymb}). A loss function need not
  have a neutral symbol; the $0/1$ loss does not. Define a weaker property a loss function may have,
  called \emph{flexible}, which allows us to simulate a neutral symbol, via randomness, if one
  doesn't already exist (Definitions~\ref{def:flexible} and~\ref{def:flexibilitypart} and
  Proposition~\ref{prop:neutsymb}). The $0/1$ loss \emph{is} flexible (and most reasonable loss
  functions should also be flexible).

  We can now communicate the intuition of Proposition~\ref{prop:kpart2} that uses the neutral symbol
  to show that non-partite learnability implies partite learnability via the following
  thought-experiment. One way to get a local exchangeable distribution is to sample from a structure
  $M$, say a graph. Now consider its bipartization. Sampling from the bipartization, and receiving
  exactly $m$ elements from each of the pieces, is itself not local and exchangeable. But suppose we
  consider the result of sampling, call it $N_m$, to be the structure on these $2\cdot m$ elements,
  along with all the edges crossing the partition, and the rest of the edges labeled ``I don't
  know''. \emph{This} is a structure, and sampling from \emph{it} is local and exchangeable. Now a
  priori each such $N_m$ may not ``cohere'' in the sense that we might not expect there would be an
  infinite $N_\infty$, suitably defined, such that the $N_m$'s are restrictions of it. It would
  suffice to show that the distribution obtained from sampling from $N_m$ is the marginal of that
  obtained from sampling from some $N_\infty$. But indeed something close to this is true. (The
  subtlety is that ``agnostic'' is introduced.) In the bipartite case, we really have some
  distribution over $X_1\otimes X'_1$ along with $F\colon\cE_2(\Omega\otimes\Omega')\to\Lambda$, and
  if we take an infinite sample from this, restrict to its first $m$ vertices, and forget everything
  about $X'_1$, then ``morally'' the result is the distribution corresponding to sampling from
  $N_m$. Note that a constant fraction of the tuples will not be labeled ``I don't know'' since we
  receive good information whenever our tuple crosses the partition; and this probability of
  crossing, call it $p$, only depends on $2$ (or in general, on $k$), not on $m$. So with
  probability $1-p$ we are computing the loss against the ``I don't know'' symbol, and with
  probability $p$ we are computing the $k$-partite loss.

  (Alternately, this proof can be understood as converting a certain kind of separately exchangeable
  local distribution to an exchangeable and local one that retains much of the same information.)
\item[Section~\ref{sec:UC}] deals with uniform convergence. Recall that the uniform law of large
  numbers for finite $\VC$-dimension says that the actual and empirical losses converge uniformly,
  using the polynomial/exponential dichotomy in the Sauer--Shelah--Perles Lemma; and from this
  uniform convergence, classical PAC learning follows. This section takes place in the $k$-partite
  setting.

  The first step, Lemma~\ref{lem:VCNk}, is to observe that from the definition of $\VCN_k$-dimension
  there is an analogue of the Sauer--Shelah--Perles Lemma that controls a growth function related to
  fixing $k-1$ vertices and letting the remaining one vary, and that yields a polynomial/exponential
  dichotomy.

  In the section's main technical result (see Section~\ref{subsec:growth}), we would like to
  leverage the polynomial/exponential dichotomy for $\VCN_k$-dimension to obtain uniform
  convergence. In the $k$-partite setting, of course, the domain of our structures comes with a
  partition into $k$ parts (and a priori sampling involves receiving $m$ points from each part).
  The central idea is to appeal to a notion of $p$-partially empirical loss (for $p=0,\ldots,k$)
  that will interpolate between total loss and empirical loss in that $p$ parts are fixed and $k-p$
  are sampled. Note that $0$-partially empirical loss is total loss, and $k$-partially empirical
  loss is empirical loss. (Since the $k$-partite setting for $k=1$ is analogous to the classical
  unary case, in this language the classical uniform law of large numbers for finite $\VC$-dimension
  can be seen as comparing the $0$-partially empirical loss with the $1$-partially empirical loss.)
  In the present $k$-partite setting, the proof proceeds step by step, comparing $p$-partially
  empirical loss with $(p+1)$-partially empirical loss for $p=0,\ldots,k-1$. Linearity of
  expectation along with coupling are applied to show that when comparing the $p$-partially
  empirical loss with the $(p+1)$-partially empirical loss, we can focus on the $(p+1)$th part so
  the growth function is relevant: even though both loss functions are sampling all vertices from
  the $(p+2)$th part onward, we can in some sense couple them together and fix them. Once this is
  addressed, the argument is completely analogous to the law of large numbers in classical PAC (with
  some bookkeeping related to higher-order variables).

  Having established partite uniform convergence, it is then straightforward to derive partite
  agnostic learnability, analogously to the case of classical PAC, in
  Proposition~\ref{prop:VCNdim->UC}. (Very informally, if samples are likely to be uniformly
  accurate, then the task of learning from a sample becomes reasonable.)
\item[Section~\ref{sec:nonlearn}] deals with non-learnability. Similarly to the ``infinite
  $\VC$-dimension implies not PAC learnable'' direction of the classical PAC theorem, known as the
  no-free-lunch Theorem, we would like to show that (say) a family of bipartite graphs with infinite
  $\VCN_2$-dimension, or in general a family of $k$-partite structures with infinite
  $\VCN_k$-dimension (in the bipartite sense), can confound a learning algorithm.

  This proof uses partite in a key way. Lemma~\ref{lem:nonlearn} first reworks a slightly tighter
  quantitative version of the no-free-lunch theorem. Proposition~\ref{prop:partkPAC->VCN}, the main
  result of the section, gives the non-learnability result. In the language of bipartite graphs,
  suppose that there were a vertex on one side of the partition, call it $v$, whose neighborhoods
  had unbounded $\VC$-dimension considered as unary subsets of the other side. Since we are in the
  partite case, we can choose different measures on each side of the vertex partition, and so
  confound the learning algorithm in a way analogous to classical PAC by essentially concentrating
  on $v$ on one side, and its set of neighborhoods on the other. This is a non-trivial advantage of
  the partite setup, and a reason for its role in the paper, as explained in
  Section~\ref{subsec:partite} above.
\item[Section~\ref{sec:equivthms}] assembles the proofs of the main theorems (see
  Sections~\ref{subsec:thmkPAC} and~\ref{subsec:thmkPACkpart} below).
\item[Section~\ref{sec:highorder}] gives a counterexample showing the hypothesis of ``rank at most 1''
  is necessary. More precisely, the ``finite dimension implies learnable'' direction of the main
  theorems, see below, continues to hold even when the hypotheses can contain higher-order variables
  (again, as distinct from high arity predicates). The counterexample shows that when this stronger
  information is available, it may happen that learning is possible even if the dimension is
  infinite. For more on these higher-order variables see Section~\ref{subsec:highervar} below.
\item[Section~\ref{sec:nonlearnnonpart}] gives evidence that the basic (non-agnostic)
  non-partite case fits into the cycle of the first main theorem, by showing that for a very
  structured type of hypothesis class of graphs, the $\VCN_2$-dimension does characterize the basic
  learning as well as the agnostic learning. It also indicates the scope of work needed to sort this
  out in general, so we feel justified in deferring this to future work.
\item[Section~\ref{sec:final}] contains concluding remarks and many open problems.
\item[Appendix~\ref{sec:agexch}] works out the normal forms needed for the definitions of
  agnostic learners. In particular~\ref{subsec:exch} deals with exchangeability,
  and~\ref{subsec:sepexch} with separate exchangeability.
\item[Appendix~\ref{sec:Bayes}] addresses the high-arity version of Bayes predictors.
\end{description}

\begin{center}
\emph{We now turn to the paper's main theorems, in the case of graphs.}
\end{center}

\notoc\subsection{Theorem~\ref{thm:kPAC} for graphs}
\label{subsec:thmkPAC}

Theorem~\ref{thm:kPAC}, in the case of our running example, completely characterizes \emph{agnostic}
$2$-PAC learnability when $\cH$ is a family of graphs. For the reasons discussed, $\cH^{\bip}$, the
bipartite version of $\cH$, will play a leading role.\footnote{To anticipate a question, certainly
  not all bipartite graphs need arise as the bipartization of some graph;
  Theorem~\ref{thm:kPACkpart} will deal with the general bipartite case.} We now very briefly
describe the structure of the theorem and proof.

Some disclaimers: The main text proves quite a bit more beyond these arrows. In this section we
select some arrows that are enough to obtain the equivalence. Furthermore, ``learnable'' in each
case means the appropriate notion of high-arity PAC with the appropriate loss
function.\footnote{Which notion of learning, i.e., partite or not, and which loss function are
  appropriate changes in the different items (though in a natural way); we refer the reader to the
  full statement of the theorem in the main text for details.} Finally, the item numbering is the
same as in Theorem~\ref{thm:kPAC}.

\begin{theorem}[Theorem~\ref{thm:kPAC} for graphs]\label{thm:kPACgraph}
 Given a hypothesis class $\cH$ that is a family of graphs and its bipartization $\cH^{\bip}$, along
 with a symmetric, separated, flexible and bounded loss function $\ell$, the following implications
 hold between the properties defined below:
\begin{gather*}
  \ref{thm:kPACgraph:VCN}
  \iff
  \ref{thm:kPACgraph:VCNpart},
  \\
  \ref{thm:kPACgraph:VCNpart}
  \implies
  \ref{thm:kPACgraph:UC}
  \implies
  \ref{thm:kPACgraph:agPACpart}
  \implies
  \ref{thm:kPACgraph:agPAC}
  \implies
  \ref{thm:kPACgraph:agPACr}
  \implies
  \ref{thm:kPACgraph:agPACrpart}
  \implies
  \ref{thm:kPACgraph:PACrpart}
  \implies
  \ref{thm:kPACgraph:PACpart}
  \implies
  \ref{thm:kPACgraph:VCNpart},
  \\
  \ref{thm:kPACgraph:PAC}
  \iff
  \ref{thm:kPACgraph:PACr},
  \\
  \ref{thm:kPACgraph:PACrpart}
  \implies
  \ref{thm:kPACgraph:PACr}.
\end{gather*}
In particular, properties~\ref{thm:kPACgraph:VCN} through~\ref{thm:kPACgraph:PACrpart} are
equivalent, and properties~\ref{thm:kPACgraph:PAC} and~\ref{thm:kPACgraph:PACr} are equivalent and
implied by any of~\ref{thm:kPACgraph:VCN} through~\ref{thm:kPACgraph:PACrpart}.
\end{theorem}

\begin{proof}[Proof Sketch]
  Remember that here we are given $\cH$, and in this proof $\cH^{\bip}$ always refers to its
  bipartization, not an arbitrary bipartite hypothesis class. The case of arbitrary bipartite
  hypothesis classes will be handled in Theorem~\ref{thm:kPACkpart}.

  First we have the two notions of dimension.
  \begin{refdesc}
  \item[\ref*{thm:kPAC:VCN}\label{thm:kPACgraph:VCN}] $\VCN_2(\cH) < \infty$.
  \item[\ref*{thm:kPAC:VCNkpart}\label{thm:kPACgraph:VCNpart}] $\VCN_2(\cH^{\bip}) < \infty$.
  \end{refdesc}

  \begin{itemize}
  \item\ref{thm:kPACgraph:VCN} and~\ref{thm:kPACgraph:VCNpart} are shown to be equivalent in
    Proposition~\ref{prop:kpartVCN}.
  \end{itemize}

  Next we start on the long loop.
  \begin{refdesc}
  \item[\ref*{thm:kPAC:UC}\label{thm:kPACgraph:UC}] $\cH^{\bip}$ has the bipartite uniform
    convergence property.
  \end{refdesc}

  \begin{itemize}
  \item\ref{thm:kPACgraph:VCNpart}$\implies$\ref{thm:kPACgraph:UC}, that is, finite bipartite
    $\VCN_2$-dimension implies bipartite uniform convergence, is proved in
    Proposition~\ref{prop:VCNdim->UC}. Recall that uniform convergence says that samples are likely
    to be representative. The work of Section~\ref{sec:UC} is to derive bipartite uniform
    convergence from the bipartite dimension bound via its growth function.
  \end{itemize}

  \begin{refdesc}
  \item[\ref*{thm:kPAC:agPACkpart}\label{thm:kPACgraph:agPACpart}] $\cH^{\bip}$ is agnostically
    learnable.
  \end{refdesc}

  \begin{itemize}
  \item\ref{thm:kPACgraph:UC}$\implies$\ref{thm:kPACgraph:agPACpart}, that is, bipartite uniform
    convergence implies $\cH^{\bip}$ agnostically learnable, is proved in
    Proposition~\ref{prop:partUC->partagPAC}. The heavy lifting was done in the previous step:
    bipartite uniform convergence means samples are likely to be representative, from which learning
    can naturally follow.
  \end{itemize}

  \begin{refdesc}
  \item[\ref*{thm:kPAC:agPAC}\label{thm:kPACgraph:agPAC}] $\cH$ is agnostically learnable.
  \end{refdesc}

  \begin{itemize}
  \item\ref{thm:kPACgraph:agPACpart}$\implies$\ref{thm:kPACgraph:agPAC}, that is, $\cH^{\bip}$
    agnostic learnable implies $\cH$ agnostic learnable, is proved in
    Proposition~\ref{prop:kpart}. Recall our intuition that from the learning point of view, the
    bipartization of a graph has less information: if we are able to learn the bipartization,
    learning the original should be even easier.
  \end{itemize}

  \begin{refdesc}
  \item[\ref*{thm:kPAC:agPACr}\label{thm:kPACgraph:agPACr}] $\cH$ is agnostically learnable
    with randomness.
  \end{refdesc}

  \begin{itemize}
  \item\ref{thm:kPACgraph:agPAC}$\implies$\ref{thm:kPACgraph:agPACr}, that is, $\cH$ agnostically
    learnable implies $\cH$ agnostically learnable with randomness, follows from the definitions.
  \end{itemize}

  \begin{refdesc}
  \item[\ref*{thm:kPAC:agPACrkpart}\label{thm:kPACgraph:agPACrpart}] $\cH^{\bip}$ is
    agnostically learnable with randomness.
  \end{refdesc}

  \begin{itemize}
  \item\ref{thm:kPACgraph:agPACr}$\implies$\ref{thm:kPACgraph:agPACrpart}, that is, $\cH$
    agnostically learnable with randomness implies $\cH^{\bip}$ is agnostically learnable with
    randomness, is a main proof of the paper, Propositions~\ref{prop:kpart2} and~\ref{prop:kpart3}.
  \end{itemize}

  \begin{refdesc}
  \item[\ref*{thm:kPAC:PACrkpart}\label{thm:kPACgraph:PACrpart}] $\cH^{\bip}$ is learnable with
    randomness.
  \end{refdesc}

  \begin{itemize}
  \item\ref{thm:kPACgraph:agPACrpart}$\implies$\ref{thm:kPACgraph:PACrpart}, that is, $\cH^{\bip}$
    agnostically learnable with randomness implies $\cH^{\bip}$ learnable with randomness, is
    verified in Proposition~\ref{prop:agPAC->PAC}.
  \end{itemize}

  \begin{refdesc}
  \item[\ref*{thm:kPAC:PACkpart}\label{thm:kPACgraph:PACpart}] $\cH^{\bip}$ is learnable.
  \end{refdesc}

  \begin{itemize}
  \item\ref{thm:kPACgraph:PACrpart}$\implies$\ref{thm:kPACgraph:PACpart}, that is, $\cH^{\bip}$
    learnable with randomness implies $\cH^{\bip}$ learnable, is proved via derandomization
    (Proposition~\ref{prop:derand}): we can replace a learning algorithm using finitely many
    bits of randomness with another learning algorithm using no randomness whose sample size is
    larger.
  \end{itemize}

  \begin{itemize}
  \item\ref{thm:kPACgraph:PACpart} $\implies$\ref{thm:kPACgraph:VCNpart}, that is, if $\cH^{\bip}$
    is learnable then its bipartite $\VCN_2$-dimension must be finite, is proved in
    Proposition~\ref{prop:partkPAC->VCN}.
  \end{itemize}

  All together, the elements~\ref{thm:kPACgraph:VCN} up to~\ref{thm:kPACgraph:PACrpart} give a
  characterization of agnostic learning for $\cH$: quotably we may say $\cH$ is agnostically
  learnable if and only if $\cH^{\bip}$ is agnostically learnable if and only if either,
  equivalently both, have finite $\VCN_2$-dimension.

  Items~\ref{thm:kPACgraph:PAC} and~\ref{thm:kPACgraph:PACr} ask for learnability of $\cH$ without the
  word ``agnostic''. For now, we simply record that they are implied by the main loop. We plan to
  address their relation to the main (non-partite) equivalence in future work.

  \smallskip

  We have a few arrows remaining:

  \begin{refdesc}
  \item[\ref*{thm:kPAC:PAC}\label{thm:kPACgraph:PAC}] $\cH$ is learnable.
  \item[\ref*{thm:kPAC:PACr}\label{thm:kPACgraph:PACr}] $\cH$ is learnable with randomness.
  \end{refdesc}

  \begin{itemize}
  \item\ref{thm:kPACgraph:PACrpart}$\implies$\ref{thm:kPACgraph:PACr}, that is, $\cH^{\bip}$
    learnable with randomness implies $\cH$ learnable with randomness, is analogous to
    \ref{thm:kPACgraph:agPACpart}$\implies$\ref{thm:kPACgraph:agPAC} in the long loop.
  \end{itemize}

  \begin{itemize}
  \item\ref{thm:kPACgraph:PAC}$\implies$\ref{thm:kPACgraph:PACr}, that is, learnable implies
    learnable with randomness, is by definition.
  \end{itemize}

  \begin{itemize}
  \item\ref{thm:kPACgraph:PACr}$\implies$\ref{thm:kPACgraph:PAC}, that is, $\cH$ learnable with
    randomness implies $\cH$ is learnable, is again derandomization (Proposition~\ref{prop:derand}).
  \end{itemize}

  This completes our sketch of the proof.
\end{proof}

\notoc\subsection{Theorem~\ref{thm:kPACkpart} for bipartite graphs}
\label{subsec:thmkPACkpart}

In the case of bipartite graphs (not necessarily arising as the bipartization of a graph, as is
reflected by simply using the notation $\cH$), Theorem~\ref{thm:kPACkpart} gives a full
equivalence, characterizing both agnostic and non-agnostic learning, that is:

\begin{theorem}[Theorem~\ref{thm:kPACkpart} for bipartite graphs]\label{thm:kPACgraphpart}
   Suppose $\cH$ is a hypothesis class of bipartite graphs, and $\ell$ is a separated and bounded
   bipartite loss function. Then the following are equivalent:
   \begin{refdesc}
   \item[\ref*{thm:kPACkpart:VCN}\label{thm:kPACgraphpart:VCN}] $\cH$ has finite $\VCN_2$-dimension.
   \item[\ref*{thm:kPACkpart:UC}] $\cH$ has the bipartite uniform convergence property.
   \item[\ref*{thm:kPACkpart:agPAC}] $\cH$ is agnostically learnable.
   \item[\ref*{thm:kPACkpart:agPACr}] $\cH$ is agnostically learnable with randomness.
   \item[\ref*{thm:kPACkpart:PAC}\label{thm:kPACgraphpart:PAC}] $\cH$ is learnable.
   \item[\ref*{thm:kPACkpart:PACr}] $\cH$ is learnable with randomness.
   \end{refdesc}
 \end{theorem}

 The main steps in the proof can be inferred from the sketch of the proof of
 Theorem~\ref{thm:kPAC} just given. The difference in the bipartite case is being able already
 to include the non-agnostic case in the equivalence, via~\ref{thm:kPACgraphpart:PAC}
 implies~\ref{thm:kPACgraphpart:VCN}, which is Proposition~\ref{prop:partkPAC->VCN}.

 \medskip
 \noindent\textbf{A justification of the assertion that this ``generalizes'' classical PAC:}
 \begin{quotation}
   Since Theorem~\ref{thm:kPACkpart} gives the full equivalence between the different aspects,
   including basic and agnostic learnable, it allows us to recover the classical PAC theorem as a
   corollary of this theorem in the case $k=1$. This is because when $k=1$ the non-partite and
   partite cases coincide and are the same as the classical unary case.
\end{quotation}

\notoc\subsection{On higher-order variables}
\label{subsec:highervar}

We arrive to a forward-looking feature of the setup presented here. So far in this section the
entries indexed by subsets of size greater than one have been present only in the adversary's
choices, not in the hypothesis classes themselves.\footnote{However one direction of the main
  theorems already includes such cases, see below.}

Indeed, the rank at most $1$ hypothesis classes include essentially anything presented in the
language of model theory. We can describe hypothesis classes of graphs, directed graphs,
hypergraphs, and indeed structures in any fixed finite relational language (in their multipartite
versions as well) using only variables of rank 1. Their learnability is studied in the text's main
theorems. We may wish to extend these ideas to study the ``learnability of randomness'' by allowing
higher-order variables.

Let us work out an example of a class with non-trivial higher-order variables. Note that this
example is chosen to illustrate a building block; just by itself, this example won't be very
interesting for learning, since it will be clear from the construction that any pair is independent
from any other, so high-arity PAC essentially reduces to PAC learning, and $\VCN_2$-dimension
matches $\VC$-dimension. (A more involved example in which learning is affected is in
Section~\ref{sec:highorder} of the main text.)

We keep $k = 2$, but we fix a new Borel template $\Omega$ in which\footnote{The choice of $\QQ$ here
  isn't important; it's comfortable to have a countably infinite set, and easier for notational
  purposes to distinguish it from $\NN$ since we will draw graphs over $\NN$.} $X_1=\{e\}$,
$X_2=\QQ$ with $\cP(\QQ)$ as its $\sigma$-algebra, and $X_n=\{e\}$ for $n>2$. In this case
\begin{align*}
  \cE_2(\Omega)
  & =
  X_{\lvert\{1\}\rvert}\times X_{\lvert\{2\}\rvert}\times X_{\lvert\{1,2\}\rvert}
  =
  \{e\}\times\{e\}\times\QQ.
\end{align*}
As before let $Y\df\{0,1\}$. A hypothesis is still $H\colon\cE_2(\Omega)\to\Lambda$. For each $A
\subseteq\QQ$, let $H_A\colon\cE_2(\Omega)\to\Lambda$ be the hypothesis that on $(e,e,q)$ returns
$1$ if $q\in A$, and $0$ otherwise. Let
\begin{align*}
  \cH & \df \{H_A\mid A\subseteq\QQ\}.
\end{align*}
Suppose now we are also given a probability template $\mu$ for $\Omega$, whose important entry will be
$\mu_2$. Let us sample a structure with vertex set $[m]$ using some $H_A$. According to the
instructions of the measure $\mu$ we draw an element of $\cE_m(\Omega)$, recalling that the
Cartesian product $\cE_m(\Omega)$ has $m$ factors of $X_1=\{e\}$, followed by $\binom{m}{2}$ factors
of $X_2$, followed by a factor of $\{e\}$ for each $A\in r(m)$ with $\lvert A\rvert > 2$. So an
element of $\cE_m(\Omega)$ is of the form
\begin{align*}
  (\mathop{\underbrace{e,e,\ldots,e}}
  \limits_{\mathclap{\shortstack{indexed by\\$\{1\},\ldots,\{m\}$,\\respectively}}},
  \;
  p_{\{1,2\}},\ldots,p_{\{i,j\}},\ldots,p_{\{m-1, m\}},
  \;
  \mathop{\underbrace{e,e,\ldots,e}}
  \limits_{\mathclap{\shortstack{indexed by\\sets $A\in r(m)$\\with $\lvert A\rvert\geq 3$}}}
  ).
\end{align*}
From this we'll describe a structure on vertices $v_1,\ldots,v_m$: for each pair $v_i,v_j$, $i\neq
j$, we ask $H_A$ about $(e,e,p_{\{i,j\}})$, and it puts an edge\footnote{More precisely, since the
  $p$ coordinate is indexed by a set $\{i,j\}=\{j,i\}$, we have a directed edge from $v_i$ to $v_j$
  if and only if $H_A(e,e,p_{\{i, j\}})=1$, and a directed edge from $v_j$ to $v_i$ if and only if
  $H_A(e,e,p_{\{i, j\}})=1$. Since by our choice of $H_A$ both answers depend only on whether
  $p=p_{\{i, j\}}=p_{\{j, i\}}\in A$, we end up with a symmetric edge from $v_i$ to $v_j$ if and
  only if this $p\in A$.}  between $v_i$ and $v_j$ if and only if $p_{\{i,j\}}\in A$.

As $m\to\infty$, or just in the case $V=\NN$, what we see is the \Erdos--\Renyi\ random graph with
edge probability $\mu_2(A)$, recalling that $\mu_2$ is a measure on $X_2=\QQ$.

This $\cH$ is an example of a hypothesis class of rank 2.

Again, one direction of the main theorem holds for hypothesis classes of any rank: namely,
$\VCN_k(\cH) < \infty$ implies $\cH$ is $k$-PAC learnable.

The direction ``$k$-PAC learnable implies finite $\VCN_k$-dimension'' is where we use rank at most 1
in an essential way, as we now explain. For even when the $\VCN_k$-dimension is infinite, the more
powerful information contained in higher rank variables might still make learning possible. As
noted, in Section~\ref{sec:highorder} we give an example showing this can occur. To motivate that
section, even though the example just worked out is a bit too simple to show the complexities of
high-arity learning, consider that with both vertex and pair variables in play, the learning picture
may not be as clear-cut. Finding a characterization of high-arity PAC learning that handles
hypothesis classes of arbitrarily large finite rank (again, we emphasize, as distinct from large
finite arity) may turn out to require a corresponding higher understanding of dimension.

\medskip

This concludes this expository section.

\section{Fundamentals in the non-partite setting}
\label{sec:defs}

In this section, we make most of the definitions related to $k$-PAC learning in the non-partite
setting. We note that when $k=1$, these definitions retrieve the usual notions of classic PAC
learning (except for the presence of dummy variables, see also Section~\ref{subsec:highvar}) and we
attempt to use terminology similar to that of~\cite{SB14}.

We start with some general notation. We denote the set of non-negative integers by $\NN$ and the set
of positive integers by $\NN_+\df\NN\setminus\{0\}$. For $k\in\NN$, we let
$[k]\df\{1,\ldots,k\}$. The set of non-negative reals is denoted by $\RR_{\geq 0}$. The space of
injective functions from a set $U$ to a set $V$ is denoted $(V)_U$ and we use the shorthand
$(V)_k\df (V)_{[k]}$. The symmetric group on a set $V$ is denoted by $S_V$ and again we use the
shorthand $S_k\df S_{[k]}$ (note that this is also the same as $([k])_k$). For $n,k\in\NN$, we also
let $(n)_k\df n(n-1)\cdots(n-k+1)$ denote the falling factorial (so $\lvert ([n])_k\rvert =
(n)_k$). For a set $V$ and $k\in\NN$, we let $\binom{V}{k}\df\{A\subset V \mid \lvert A\rvert=k\}$
be the set of all subsets of $V$ of size $k$. When we consider a Cartesian product of the form
$\prod_{v\in V} X_v$, we view it as the space of functions $x\colon V\to\bigcup_{v\in V} X_v$ such
that $x_v\df x(v)\in X_v$ for every $v\in V$, so even though there is a natural bijection between
$\prod_{v\in\{1,2\}} X_v$ and $\prod_{v\in\{2,3\}} X_{v-1}$, these are inherently different sets
because of the domains of the functions. Furthermore, our random variables will always be typeset in
bold font.

\begin{definition}[Borel templates]\label{def:Boreltemplate}
  By a Borel space, we mean a standard Borel space, i.e., a measurable space that is
  Borel-isomorphic to a Polish space when equipped with the $\sigma$-algebra of Borel sets. The
  space of probability measures on a Borel space $\Lambda$ is denoted $\Pr(\Lambda)$.
  \begin{enumdef}
  \item A \emph{Borel template} is a sequence $\Omega=(\Omega_i)_{i\in\NN_+}$, where
    $\Omega_i=(X_i,\cB_i)$ is a non-empty Borel space.
  \item A \emph{probability template} on a Borel template $\Omega$ is a sequence
    $\mu=(\mu_i)_{i\in\NN_+}$, where $\mu_i$ is a probability measure on $\Omega_i$. With a small
    abuse of notation, the space of probability templates on a Borel template $\Omega$ is denoted
    $\Pr(\Omega)$.
  \item\label{def:Boreltemplate:product} For two Borel templates $\Omega=(\Omega_i)_{i\in\NN_+}$ and
    $\Omega'=(\Omega'_i)_{i\in\NN_+}$, we let $\Omega\otimes\Omega'$ be the Borel template given by
    taking products of the Borel spaces: $(\Omega\otimes\Omega')_i\df\Omega_i\otimes\Omega'_i$
    ($i\in\NN_+$). Similarly, if $\mu\in\Pr(\Omega)$ and $\mu'\in\Pr(\Omega')$, we let
    $\mu\otimes\mu'$ be given by $(\mu\otimes\mu')_i\df\mu_i\otimes\mu'_i$ ($i\in\NN_+$).
  \item For a (finite or) countable set $V$ and a Borel template $\Omega$, we define\footnote{This
    notation is borrowed from~\cite{CR20}, but these spaces have a long history in limit
    theory~\cite{Aus08,ES12,AC14} and exchangeability theory~\cite{Hoo79,Ald81,Ald85,Hoo79,Kal05},
    except that since these theories only need to keep track of pushforward distributions generated
    from $\cE_V(\Omega)$ instead of combinatorics of $\cE_V(\Omega)$, they only consider the case
    when all $\Omega_i$ are the same standard Borel space, and typically the space is $[0,1]$,
    equipped with the Borel $\sigma$-algebra.}
    \begin{align*}
      \cE_V(\Omega) & \df \prod_{A\in r(V)} X_{\lvert A\rvert},
    \end{align*}
    equipping it with the product $\sigma$-algebra, where
    \begin{align*}
      r(V) & \df \{A\subseteq V \mid A\text{ finite non-empty}\}.
    \end{align*}

    If $\mu\in\Pr(\Omega)$ is a probability template on $\Omega$, we let
    $\mu^V\df\bigotimes_{A\in r(V)}\mu_{\lvert A\rvert}$ be the product measure.
    
    We will use the shorthands $r(m)\df r([m])$, $\cE_m(\Omega)\df\cE_{[m]}(\Omega)$ and
    $\mu^m\df\mu^{[m]}$ when $m\in\NN$.
  \item\label{def:Boreltemplate:action} For an injective function $\alpha\colon U\to V$ between
    countable sets, we contra-variantly define the map $\alpha^*\colon\cE_V(\Omega)\to\cE_U(\Omega)$
    by
    \begin{align*}
      \alpha^*(x)_A & \df x_{\alpha(A)}\qquad (x\in\cE_V(\Omega), A\in r(U)).
    \end{align*}
  \item\label{def:Boreltemplate:identification} For Borel templates $\Omega^1$ and $\Omega^2$, we
    will also use the natural identification of $\cE_V(\Omega^1\otimes\Omega^2)$ with
    $\cE_V(\Omega^1)\times\cE_V(\Omega^2)$ in which $x\in\cE_V(\Omega^1\otimes\Omega^2)$ corresponds
    to $(z^1,z^2)\in\cE_V(\Omega^1)\times\cE_V(\Omega^2)$ given by $z^i_A\df\pi_{\lvert
      A\rvert,i}(x_A)$, where $\pi_{a,i}\colon\Omega^1_a\otimes\Omega^2_a\to\Omega^i_a$ is the
    projection map. It is straightforward to check that under this identification, if $\alpha\colon
    U\to V$ is an injection between countable sets, then
    $\alpha^*\colon\cE_V(\Omega^1\otimes\Omega^2)\to\cE_U(\Omega^1\otimes\Omega^2)$ is simply the
    product of the $\alpha^*$ corresponding to $\Omega^1$ and $\Omega^2$ and for
    $\mu^i\in\Pr(\Omega^i)$, we have $(\mu^1\otimes\mu^2)^V = (\mu^1)^V\otimes(\mu^2)^V$.
  \end{enumdef}
\end{definition}

\begin{definition}[Hypotheses]\label{def:hypotheses}
  Let $\Omega$ be a Borel template, $\Lambda=(Y,\cB')$ be a non-empty Borel space and $k\in\NN_+$.
  \begin{enumdef}
  \item The set of \emph{$k$-ary hypotheses} from $\Omega$ to $\Lambda$, denoted
    $\cF_k(\Omega,\Lambda)$, is the set of (Borel) measurable functions from $\cE_k(\Omega)$ to
    $\Lambda$.
  \item A \emph{$k$-ary hypothesis class} is a subset $\cH$ of $\cF_k(\Omega,\Lambda)$, equipped
    with a $\sigma$-algebra such that:
    \begin{enumerate}
    \item the evaluation map $\ev\colon\cH\times\cE_k(\Omega)\to\Lambda$ given by $\ev(H,x)\df H(x)$
      is measurable;
    \item for every $H\in\cH$, the set $\{H\}$ is measurable.
    \end{enumerate}
  \item Given $F\in\cF_k(\Omega,\Lambda)$ and a countable set $V$, we define the function
    $F^*_V\colon\cE_V(\Omega)\to\Lambda^{(V)_k}$ by
    \begin{align*}
      F^*_V(x)_\alpha & \df F(\alpha^*(x)) \qquad (x\in\cE_V(\Omega), \alpha\in (V)_k).
    \end{align*}
    Again, we use the shorthand $F^*_m\df F^*_{[m]}$. In particular, $F^*_k$ is a function of the form
    $\cE_k(\Omega)\to\Lambda^{S_k}$.

    Informally, if we think of $F$ as an infinite structure on $\Omega_1$ and ignore higher-order
    variables for a moment, then $F^*_V(x)$ collects all information present on the quantifier-free
    diagram induced by $F$ on $\{x_v \mid v\in V\}$ supposing each of these vertices are named by constants.
  \item For an injective function $\alpha\colon U\to V$ between countable sets, we contra-variantly define
    the map $\alpha^*\colon\Lambda^{(V)_k}\to\Lambda^{(U)_k}$ by
    \begin{align*}
      \alpha^*(y)_\beta & \df y_{\alpha\comp\beta}
      \qquad (y\in\Lambda^{(V)_k}, \beta\in(U)_k).
    \end{align*}
    (This intentionally uses the same notation as the map defined in
    Definition~\ref{def:Boreltemplate:action} to resemble group action notation, see also
    Lemma~\ref{lem:F*Vequiv} below.)
  \end{enumdef}
\end{definition}

The next lemma is a standard observation in exchangeability theory but is made explicit here in case
it is unfamiliar.

\begin{lemma}[Equivariance of $F^*_V$]\label{lem:F*Vequiv}
  For $F\in\cF_k(\Omega,\Lambda)$, the definition $F^*_V$ is equivariant in the sense that for every
  $\beta\colon U\to V$ injective, the diagram
  \begin{equation*}
    \begin{tikzcd}
      \cE_V(\Omega)
      \arrow[r, "F^*_V"]
      \arrow[d, "\beta^*"']
      &
      \Lambda^{(V)_k}
      \arrow[d, "\beta^*"]
      \\
      \cE_U(\Omega)
      \arrow[r, "F^*_U"]
      &
      \Lambda^{(U)_k}
    \end{tikzcd}
  \end{equation*}
  is commutative.
\end{lemma}

\begin{proof}
  Let $x\in\cE_V(\Omega)$ and note that for $\alpha\in (U)_k$, we have
  \begin{align*}
    \beta^*(F^*_V(x))_\alpha
    & =
    F^*_V(x)_{\beta\comp\alpha}
    =
    F((\beta\comp\alpha)^*(x))
    =
    F(\alpha^*(\beta^*(x)))
    =
    F^*_V(\beta^*(x))_\alpha,
  \end{align*}
  so the diagram commutes.
\end{proof}

\begin{remark}\label{rmk:exchangeable}
  As mentioned in the introduction, $k$-ary hypotheses are the combinatorial structures that can
  generate local exchangeable distributions once we equip $\Omega$ with a probability template
  $\mu\in\Pr(\Omega)$. In fact, the Aldous--Hoover Theorem~\cite{Hoo79,Ald85,Ald81} (see
  also~\cite[Theorem~7.22 and Lemma~7.35]{Kal05}) says that any local exchangeable distribution on
  $\Lambda^{(\NN_+)_k}$ is of the form $F^*_{\NN_+}(\rn{x})$ for some $F\in\cF_k(\Omega,\Lambda)$,
  where $\rn{x}\sim\mu^{\NN_+}$ for some $\mu\in\Pr(\Omega)$ and some $\Omega$.

  We also note that an even stronger local exchangeability holds here: if
  $\rn{x}\sim\mu^{\NN_+}$, the distribution of $(\rn{x},F^*_{\NN_+}(\rn{x}))$ is local and
  exchangeable. Locality is obvious: if $U,V\subseteq\NN_+$ are disjoint, then $((\rn{x}_A)_{A\in
    r(U)}, (F^*_{\NN_+}(\rn{x})_\alpha)_{\alpha\in(U)_k})$ is independent from $((\rn{x}_A)_{A\in
    r(V)},(F^*_{\NN_+}(\rn{x})_\alpha)_{\alpha\in(V)_k})$. Exchangeability follows from
  Lemma~\ref{lem:F*Vequiv}: for $\sigma\in S_{\NN_+}$, we have
  \begin{align*}
    \sigma^*(\rn{x},F^*_{\NN_+}(\rn{x}))
    & \df
    (\sigma^*(\rn{x}),\sigma^*(F^*_{\NN_+}(\rn{x})))
    =
    (\sigma^*(\rn{x}), F^*_{\NN_+}(\sigma^*(\rn{x})))
    \sim
    (\rn{x}, F^*_{\NN_+}(\rn{x})).
  \end{align*}
\end{remark}

\begin{definition}[Rank]
  Let $\Omega$ be a Borel template, $\Lambda=(Y,\cB')$ be a non-empty Borel space and $k\in\NN_+$.
  \begin{enumdef}
  \item The \emph{rank} of a $k$-ary hypothesis $F\in\cF_k(\Omega,\Lambda)$, denoted $\rk(F)$ is the
    minimum $r\in\NN$ such that $F$ factors as
    \begin{align*}
      F(x) & \df F'((x_A)_{A\in r(k), \lvert A\rvert\leq r}) \qquad (x\in\cE_k(\Omega))
    \end{align*}
    for some function $F'\colon\prod_{A\in r(k),\lvert A\rvert\leq r} X_A\to\Lambda$.
  \item The \emph{rank} of a $k$-ary hypothesis class $\cH\subseteq\cF_k(\Omega,\Lambda)$ is defined
    as
    \begin{align*}
      \rk(\cH) & \df \sup_{F\in\cH} \rk(F).
    \end{align*}
  \end{enumdef}
\end{definition}

Note that rank is not the same as arity (see Remark~\ref{rmk:rk1} below): instead, high rank is
related to dependence on higher-order variables, which in turn can informally be seen as the
difference between a structure and a random structure (see Section~\ref{subsec:highervar} for
further intuition).

\begin{remark}\label{rmk:rk1}
  As mentioned in the introduction, the setup of Definition~\ref{def:hypotheses} in particular should
  cover the cases of $k$-hypergraphs and structures in a finite relational language, so we formalize
  this here.

  For (measurable) $k$-hypergraphs over some Borel space $\Upsilon$, we simply take $\Omega$ as the
  Borel template given by letting $\Omega_1\df\Upsilon$ and letting each $\Omega_i$ with $i\geq 2$
  be the Borel space with only one point, and we take $\Lambda$ to be $\{0,1\}$, equipped with the
  discrete $\sigma$-algebra. Then there is a natural one-to-one correspondence between measurable
  $k$-hypergraphs over $\Upsilon$ and $k$-ary hypotheses $F\in\cF_k(\Omega,\Lambda)$ satisfying:
  \begin{itemize}
  \item irreflexivity: for every $x\in\cE_k(\Omega)$, if $x_{\{i\}}=x_{\{j\}}$ for some $i\neq j$,
    then $F(x) = 0$;
  \item symmetry: $F$ is $S_k$-invariant in the sense that $F(\sigma^*(x)) = F(x)$ for every
    $x\in\cE_k(\Omega)$ and every $\sigma\in S_k$.
  \end{itemize}
  In this correspondence, a $k$-hypergraph $H$ corresponds to $F_H$ given by
  \begin{align*}
    F_H(x) & \df \One[(x_{\{1\}},\ldots,x_{\{k\}})\in E^H] \qquad (x\in\cE_k(\Omega)),
  \end{align*}
  that is, $F_H$ is morally the adjacency tensor of $H$.

  More generally, we can capture (measurable) structures in a finite relational language $\cL$ over
  some Borel space $\Upsilon$ in a similar way: first, we define $\Omega$ as before and since $\cL$
  is finite, we can let $k\in\NN_+$ be the maximum arity of the predicates of $\cL$ and let
  $\Lambda$ be the (finite) set of all structures in $\cL$ over $[k]$, equipped with the discrete
  $\sigma$-algebra. Then there is a natural injection from the set of measurable structures in $\cL$
  over $\Upsilon$ to the set $\cF_k(\Omega,\Lambda)$ of $k$-ary hypotheses that maps $M$ to
  $F_M$ defined by letting $F_M(x)$ be the unique structure over $[k]$ such
  that $[k]\ni i\mapsto x_{\{i\}}\in\Upsilon$ is an embedding of $F_M(x)$ into $M$.

  Note that these constructions produce $k$-ary hypotheses $F$ with two very important properties:
  $\Lambda$ is finite and $\rk(F)\leq 1$ (the latter follows since $\Omega_i$ has a single point for
  every $i\geq 2$).

  Finally, one may wonder if there is no advantage in picking $k$ to be strictly larger than the
  maximum arity of the predicate symbols in $\cL$. This will be explored in a future work.
\end{remark}

In the introduction, all notions of PAC learnability used what is called the $0/1$-loss, that is, if
the adversary picked some $F$ and our learner outputted $H$, then whenever $F(x)\neq H(x)$, we were
penalized by one unit. The general setup of classic PAC learnability allows for more general loss
functions and in the definition below we allow the same generality.

\begin{definition}[$k$-ary loss functions]
  Let $\Omega$ be a Borel template, $\Lambda$ be a non-empty Borel space and $k\in\NN_+$.
  \begin{enumdef}
  \item A \emph{$k$-ary loss function} over $\Lambda$ is a measurable function
    $\ell\colon\cE_k(\Omega)\times\Lambda^{S_k}\times\Lambda^{S_k}\to\RR_{\geq 0}$.
  \item For a $k$-ary loss function $\ell$, we define
    \begin{align*}
      \lVert\ell\rVert_\infty
      & \df
      \sup_{\substack{x\in\cE_k(\Omega)\\y,y'\in\Lambda^{S_k}}} \ell(x,y,y'),
      &
      s(\ell)
      & \df
      \inf_{\substack{x\in\cE_k(\Omega)\\y,y'\in\Lambda^{S_k}\\y\neq y'}} \ell(x,y,y').
    \end{align*}
  \item A $k$-ary loss function $\ell$ is:
    \begin{description}[format={\normalfont\textit}]
    \item[bounded] if $\lVert\ell\rVert_\infty < \infty$.
    \item[separated] if $s(\ell) > 0$ and $\ell(x,y,y) = 0$ for every $x\in\cE_k(\Omega)$ and every
      $y\in\Lambda^{S_k}$.
    \item[symmetric] if it is $S_k$-invariant in the sense that
      \begin{align*}
        \ell(\sigma^*(x),\sigma^*(y), \sigma^*(y')) & = \ell(x,y,y')
      \end{align*}
      for every $y,y'\in\Lambda^{S_k}$ and every $\sigma\in S_k$.
    \end{description}
  \item For a $k$-ary loss function $\ell$, hypotheses $F,H\in\cF_k(\Omega,\Lambda)$ and a
    probability template $\mu\in\Pr(\Omega)$, the \emph{total loss} of $H$ with respect to $\mu$,
    $F$ and $\ell$ is
    \begin{align*}
      L_{\mu,F,\ell}(H) & \df \EE_{\rn{x}\sim\mu^k}[\ell(\rn{x},H^*_k(\rn{x}),F^*_k(\rn{x}))].
    \end{align*}
  \item We say that $F\in\cF_k(\Omega,\Lambda)$ is \emph{realizable} in
    $\cH\subseteq\cF_k(\Omega,\Lambda)$ with respect to a $k$-ary loss function $\ell$ and
    $\mu\in\Pr(\Omega)$ if $\inf_{H\in\cH} L_{\mu,F,\ell}(H) = 0$, i.e., $F$ can be approximated by
    elements of $\cH$ in terms of total loss.
  \item The \emph{$k$-ary $0/1$-loss function} over $\Lambda$ is defined as
    $\ell_{0/1}(x,y,y')\df\One[y\neq y']$.
  \end{enumdef}
\end{definition}

The value of a $k$-ary loss function $\ell(x,y,y')$ should be interpreted as the penalty we incur if
the point was $x$, we guessed a function $H$ with $H^*_k(x)=y$ and the adversary's function $F$ had
$F^*_k(x)=y'$. Note that a priori a $k$-ary loss function can be extremely nonsensical for learning
purposes. For example, we could incur penalties even when we guess correctly ($\ell(x,y,y) > 0$) or
even worse we could incur no penalty when guessing incorrectly ($\ell(x,y,y')=0$ for some $y\neq
y'$). Of course, if $\ell$ is separated, these problems cannot happen.

Another non-sense that can arise is exclusive to high-arity and concerns the orientation of our loss
computation, namely, if we consider an element
$(x,y,y')\in\cE_k(\Omega)\times\Lambda^{S_k}\times\Lambda^{S_k}$ and its permutation
$(\sigma^*(x),\sigma^*(y),\sigma^*(y'))$ by some $\sigma\in S_k$, there is a priori no guarantee
that $\ell$ assigns the same penalty to them. Loss functions that do are called symmetric. Note also
that symmetry of $\ell$ is only with respect to the order of the tuple and \emph{not} about symmetry
between the last two arguments of $\ell$ (i.e., it says nothing about how $\ell(x,y,y')$ compares
with $\ell(x,y',y)$).

\begin{definition}[Learning algorithms and $k$-PAC learnability]
  Let $\Omega$ be a Borel template, $\Lambda$ be a non-empty Borel space and
  $\cH\subseteq\cF_k(\Omega,\Lambda)$ be a $k$-ary hypothesis class.
  \begin{enumdef}
  \item A \emph{($k$-ary) $V$-sample} with respect to $\Omega$ and $\Lambda$ is an element of
    $\cE_V(\Omega)\times\Lambda^{(V)_k}$.
  \item A \emph{($k$-ary) learning algorithm} for $\cH$ is a measurable function
    \begin{align*}
      \cA\colon
      \bigcup_{m\in\NN} (\cE_m(\Omega)\times\Lambda^{([m])_k})
      \to
      \cH.
    \end{align*}

    We want to interpret $\cA$ as receiving a $k$-ary $[m]$-sample and outputting an element of
    $\cH$ that ought to have small total loss with high probability.\footnote{Since we enforce our
      algorithms to output elements of $\cH$ as opposed to an arbitrary element of
      $\cF_k(\Omega,\Lambda)$, the learning notion that we define is called \emph{proper}, even
      though we will omit this qualifier from the terminology. However, we point out that our
      non-learnability result, Proposition~\ref{prop:partkPAC->VCN}, covers even improper learning
      where the algorithm can output an arbitrary element of $\cF_k(\Omega,\Lambda)$. Thus, the
      entire chain of equivalences in our main theorems also involves the improper versions of all
      PAC learning notions.}
  \item We say that $\cH$ is \emph{$k$-PAC learnable} with respect to a $k$-ary loss function
    $\ell\colon\cE_k(\Omega)\times\Lambda^{S_k}\times\Lambda^{S_k}\to\RR_{\geq 0}$ if there exist a
    learning algorithm $\cA$ for $\cH$ and a function
    $m^{\PAC}_{\cH,\ell,\cA}\colon(0,1)^2\to\RR_{\geq 0}$ such that for every
    $\epsilon,\delta\in(0,1)$, every $\mu\in\Pr(\Omega)$ and every $F\in\cF_k(\Omega,\Lambda)$ that
    is realizable in $\cH$ with respect to $\ell$ and $\mu$, we have
    \begin{align}\label{eq:kPAC}
      \PP_{\rn{x}\sim\mu^m}[L_{\mu,F,\ell}(\cA(\rn{x}, F^*_m(\rn{x})))\leq\epsilon] & \geq 1 - \delta
    \end{align}
    for every integer $m\geq m^{\PAC}_{\cH,\ell,\cA}(\epsilon,\delta)$. A learning algorithm $\cA$ satisfying the
    above is called a \emph{$k$-PAC learner} for $\cH$ with respect to $\ell$.
  \item A \emph{randomized ($k$-ary) learning algorithm} for $\cH$ is a measurable function
    \begin{align*}
      \cA\colon
      \bigcup_{m\in\NN} (\cE_m(\Omega)\times\Lambda^{([m])_k}\times [R_\cA(m)])
      \to
      \cH,
    \end{align*}
    where $R_\cA\colon\NN\to\NN_+$ and $[t]$ is equipped with the discrete $\sigma$-algebra. (The extra
    parameter in $\cA$ should be thought of as a source of randomness.)
  \item We define the notions of \emph{$k$-PAC learnability with randomness}, $m^{\PACr}_{\cH,\ell,\cA}$ and
    of a \emph{randomized $k$-PAC learner} analogously to the non-random case, by requiring instead
    $\cA$ to be a randomized learning algorithm and instead of~\eqref{eq:kPAC}, we require\footnote{Even
      though $\cA$ might not be computable, note that our setup imposes a bound in the complexity of the
      randomness of $\cA$, namely, on a fixed $[m]$-sample $(x,y)$, the distribution of the random
      output of $\cA$ must be the pushforward of a uniform distribution on $[R_\cA(m)]$. In particular,
      this means that it must have finite support of size at most $R_\cA(m)$ and the probability of
      each point must be an integer multiple of $1/R_\cA(m)$. While the latter property will not be
      important to us, the former property of a uniform bound on the size of the support of the
      distribution will be crucial for derandomization in Proposition~\ref{prop:derand}.}
    \begin{align*}
      \PP_{\rn{x}\sim\mu^m}[
        \EE_{\rn{b}\sim U(R_\cA(m))}[L_{\mu,F,\ell}(\cA(\rn{x}, F^*_m(\rn{x}),\rn{b}))]
        \leq\epsilon]
      & \geq
      1 - \delta,
    \end{align*}
    for every integer $m\geq m^{\PACr}_{\cH,\ell,\cA}(\epsilon,\delta)$, where $U(t)$ is the uniform
    probability measure on $[t]$. Note that if $R_\cA$ is the constant $1$ function (which we denote
    by $R_\cA\equiv 1$), then we retrieve the non-random version of $k$-PAC learnability.
  \end{enumdef}
\end{definition}

Similarly to~\cite{SB14}, even though we use the term ``algorithm'', we make no assumptions about
its complexity nor even about its computability. However, we point out that in all of our reductions
between different notions of PAC learnability that require producing an algorithm $\cA$ given that
another algorithm $\cA'$ exists, if $\cA'$ is computable, then so will be $\cA$. In fact, with the
notable exception of the reductions of Section~\ref{sec:derand} (see Remark~\ref{rmk:complexity}),
all reductions will have time-complexity polynomial in $m$.

\medskip

Let us now discuss agnostic PAC learnability. In classic PAC theory, agnostic PAC learning differs
from usual PAC learning in two ways:
\begin{description}[wide]
\item[Agnostic loss functions are more general:] instead of the penalty $\ell(x,H(x),F(x))$ when we
  observed the point $x$, made the guess $H$ and the correct function was $F$, agnostic functions in
  classic PAC are of the form $\ell\colon\cH\times X\times Y\to\RR_{\geq 0}$, that is, in the same
  situation above, we incur penalty $\ell(H,x,F(x))$.
\item[The adversary plays distributions on $X\times Y$:] instead of the adversary picking $F\in\cH$
  and $\mu\in\Pr(X)$ and providing i.i.d.\ samples of the form $(\rn{x},F(\rn{x}))$ with
  $\rn{x}\sim\mu$, they pick instead $\nu\in\Pr(X\times Y)$ and provide i.i.d.\ samples directly
  $(\rn{x},\rn{y})\sim\nu$; of course, since now $\nu$ does not come from a function, it might not
  be realizable: it could be that no single function $H\in\cH$ has small total loss, so our agnostic
  PAC learning task is instead to get very close to the best possible $H\in\cH$. More precisely, in
  the classic PAC setting, $\cH\subseteq Y^X$ is agnostically PAC learnable if there exists a
  function $\cA\colon\bigcup_{m\in\NN} (X^m\times Y^m)\to\cH$ such that for every
  $\epsilon,\delta\in(0,1)$, there exists $m^{\agPAC}_{\cH,\ell,\cA}(\epsilon,\delta)$ such that for
  every $\nu\in\Pr(X\times Y)$, we have
  \begin{align*}
    \PP_{(\rn{x},\rn{y})\sim\nu^m}[
      L_{\nu,\ell}(\cA(\rn{x}, F(\rn{x}_i)_{i=1}^m))
      \leq \inf_{H\in\cH} L_{\nu,\ell}(H) + \epsilon
    ] & \geq 1 - \delta,
  \end{align*}
  for every integer $m\geq m^{\agPAC}_{\cH,\ell,\cA}(\epsilon,\delta)$, where
  \begin{align*}
    L_{\nu,\ell}(H) & \df \EE_{(\rn{x},\rn{y})\sim\nu}[\ell(H,\rn{x},\rn{y})] \qquad(H\in\cH).
  \end{align*}
  The ``realizable agnostic'' setting, i.e., when the distribution $\nu\in\Pr(X\times Y)$ is
  required to satisfy $\inf_{H\in\cH} L_{\nu,\ell}(H)=0$, is sometimes called ``PAC learning noisy
  data''.

  Obviously, if $\ell,\mu,F$ come from the non-agnostic setting, then by defining the agnostic loss
  function $\ell^{\ag}(H,x,y)\df\ell(x,H(x),F(y))$ and defining $\nu$ as the distribution of
  $(\rn{x},F(\rn{x}))$ shows that the non-agnostic setting is a particular case of the agnostic
  setting (since realizability is equivalent to $\inf_{H\in\cH} L_{\nu,\ell^{\ag}}(H)=0$).

  When we move to the $k$-ary case, we have to ask ourselves: what is the correct analogue of the
  distribution $\nu$ on $X\times Y$? Two very naive guesses would be a distribution $\nu$ over
  $\cE_k(\Omega)\times\Lambda$ or $\cE_k(\Omega)\times\Lambda^{S_k}$, but then how do we pick an
  $[m]$-sample from $\nu$? Picking i.i.d.\ copies does not make sense here as the whole point of
  high-arity PAC learning is to use correlations to improve learning.

  Instead, a more instructive line of thought is to focus on an agnostic framework that generalizes
  the non-agnostic one: recall that we can see our $[m]$-samples as the restriction of an
  $\NN_+$-sample to $[m]$ and by Remark~\ref{rmk:exchangeable}, if $\rn{x}\sim\mu^{\NN_+}$, then the
  distribution of $(\rn{x},F^*_{\NN_+}(\rn{x}))$ is local and exchangeable. This inspires us to
  define the agnostic $k$-ary case by letting the adversary pick a distribution $\nu$ over
  $\cE_{\NN_+}(\Omega)\times\Lambda^{(\NN_+)_k}$ that is local and exchangeable and let our learning
  algorithm receive the restriction of an $\NN_+$-sample from $\nu$ to $[m]$.

  By construction and Remark~\ref{rmk:exchangeable}, this will capture the non-agnostic $k$-PAC
  learning setting (see Section~\ref{subsec:agnonag}), but how wild can such a distribution $\nu$
  be? For example, we can leverage Remark~\ref{rmk:exchangeable} itself to produce a more exotic
  $\nu$ with the property above: let $\Omega'$ be another Borel template, let $\mu\in\Pr(\Omega)$
  and $\mu'\in\Pr(\Omega')$ and let $F\in\cF_k(\Omega\otimes\Omega',\Lambda)$ be a $k$-ary
  hypothesis over the product Borel template $\Omega\otimes\Omega'$, then by
  Remark~\ref{rmk:exchangeable}, for $(\rn{x},\rn{x'})\sim(\mu\otimes\mu')^{\NN_+}$ the distribution
  of $(\rn{x},\rn{x'},F^*_{\NN_+}(\rn{x},\rn{x'}))$ is local and exchangeable, which means that if
  $\nu$ is the marginal distribution $(\rn{x},F^*_{\NN_+}(\rn{x},\rn{x'}))$, then $\nu$ is also
  local and exchangeable.

  Similarly to our use of the Aldous--Hoover Theorem~\cite{Hoo79,Ald85,Ald81} to setup non-agnostic
  $k$-PAC learnability, exchangeability theory says that every distribution $\nu$ over
  $\cE_{\NN_+}(\Omega)\times\Lambda^{(\NN_+)_k}$ that is local and exchangeable is of the form above
  for some $(\Omega',\mu,\mu',F)$:
  \begin{proposition}[Informal version of Proposition~\ref{prop:agexch}]
    Let $k\in\NN_+$, let $\Omega$ be a Borel template, let $\Lambda$ be a non-empty Borel space and
    let $\rn{x}$ and $\rn{y}$ be random elements of $\cE_{\NN_+}(\Omega)$ and $\Lambda^{(\NN_+)_k}$,
    respectively such that $(\rn{x},\rn{y})$ is local and exchangeable. Suppose further that
    $\rn{x}\sim\mu^{\NN_+}$ for some $\mu\in\Pr(\Omega)$.

    Then there exist a Borel template $\Omega'$, $\mu'\in\Pr(\Omega')$ and
    $F\in\cF_k(\Omega\otimes\Omega',\Lambda)$ such that for $\rn{x'}\sim(\mu')^{\NN_+}$ picked
    independently from $\rn{x}$, we have
    \begin{align*}
      (\rn{x},\rn{y}) & \sim (\rn{x}, F^*_{\NN_+}(\rn{x},\rn{x'})).
    \end{align*}
  \end{proposition}
  This justifies why we use this representation directly in Definitions~\ref{def:agloss}
  and~\ref{def:agkPAC} below.\footnote{The reader tempted to sanity check this fact using the case
    $k=1$ by trying to describe any distribution $\nu$ on $X\times Y$ as
    $(\rn{x},F(\rn{x},\rn{x'}))$ for some $\rn{x}\sim\mu$, some $\rn{x'}\sim\mu'$ and some $F\colon
    X\times X'\to Y$ will find it necessary to make sense of the conditional distribution of
    $(\rn{x},\rn{y})\sim\nu$ given a potentially zero-measure event $\rn{x}=x$; the formalization of
    such a concept requires the Disintegration Theorem, see Theorem~\ref{thm:disintegration}.}
\end{description}

\begin{definition}[$k$-ary agnostic loss functions]\label{def:agloss}
  Let $\Omega$ be a Borel template, let $\Lambda$ be a non-empty Borel space, let $k\in\NN_+$ and
  let $\cH\subseteq\cF_k(\Omega,\Lambda)$ be a $k$-ary hypothesis class.
  \begin{enumdef}
  \item A \emph{$k$-ary agnostic loss function} over $\Lambda$ with respect to $\cH$ is a measurable
    function $\ell\colon\cH\times\cE_k(\Omega)\times\Lambda^{S_k}\to\RR_{\geq 0}$.
  \item For a $k$-ary agnostic loss function $\ell$, we define
    \begin{align*}
      \lVert\ell\rVert_\infty
      & \df
      \sup_{\substack{H\in\cH\\x\in\cE_k(\Omega)\\y\in\Lambda^{S_k}}} \ell(H,x,y).
    \end{align*}
  \item A $k$-ary agnostic loss function $\ell$ is:
    \begin{description}[format={\normalfont\textit}]
    \item[bounded] if $\lVert\ell\rVert_\infty < \infty$.
    \item[symmetric] if it is $S_k$-invariant in the sense that
      \begin{align*}
        \ell(H,\sigma^*(x),\sigma^*(y)) & = \ell(H,x,y)
      \end{align*}
      for every $H\in\cH$, every $x\in\cE_k(\Omega)$ and every $y\in\Lambda^{S_k}$.
    \end{description}
  \item For a $k$-ary agnostic loss function $\ell$, a hypothesis $H\in\cH$, a Borel template
    $\Omega'$, probability templates $\mu\in\Pr(\Omega)$ and $\mu'\in\Pr(\Omega')$ and a hypothesis
    $F\in\cF_k(\Omega\otimes\Omega',\Lambda)$ over the product space, the \emph{total loss} of $H$
    with respect to $\mu$, $\mu'$, $F$ and $\ell$ is
    \begin{align*}
      L_{\mu,\mu',F,\ell}(H)
      & \df
      \EE_{(\rn{x},\rn{x'})\sim(\mu\otimes\mu')^k}[\ell(H, \rn{x}, F^*_k(\rn{x},\rn{x'}))].
    \end{align*}
  \item The \emph{$k$-ary agnostic $0/1$-loss function} over $\Lambda$ with respect to $\cH$ is defined as
    \begin{align*}
      \ell_{0/1}(H,x,y) & \df \One[H^*_k(x)\neq y].
    \end{align*}
  \end{enumdef}
\end{definition}

\begin{definition}[Agnostic $k$-PAC learnability]\label{def:agkPAC}
  Let $\Omega$ be a Borel template, let $\Lambda$ be a non-empty Borel space, let
  $\cH\subseteq\cF_k(\Omega,\Lambda)$ be a $k$-ary hypothesis class and let
  $\ell\colon\cH\times\cE_k(\Omega)\times\Lambda^{S_k}\to\RR_{\geq 0}$ be a $k$-ary agnostic loss
  function.
  \begin{enumdef}
  \item We say that $\cH$ is \emph{agnostically $k$-PAC learnable} with respect to $\ell$ if there
    exist a learning algorithm $\cA$ for $\cH$ and a function
    $m^{\agPAC}_{\cH,\ell,\cA}\colon(0,1)^2\to\RR_{\geq 0}$ such that for every
    $\epsilon,\delta\in(0,1)$, every $\mu\in\Pr(\Omega)$, every Borel template $\Omega'$, every
    $\mu'\in\Pr(\Omega')$ and every $F\in\cF_k(\Omega\otimes\Omega',\Lambda)$, we have
    \begin{align}\label{eq:agkPAC}
      \PP_{(\rn{x},\rn{x'})\sim(\mu\otimes\mu')^m}[
        L_{\mu,\mu',F,\ell}(\cA(\rn{x}, F^*_m(\rn{x},\rn{x'})))
        \leq
        \inf_{H\in\cH} L_{\mu,\mu',F,\ell}(H) + \epsilon
      ]
      & \geq 1 - \delta
    \end{align}
    for every integer $m\geq m^{\agPAC}_{\cH,\ell,\cA}(\epsilon,\delta)$. A learning algorithm $\cA$
    satisfying the above is called an \emph{agnostic $k$-PAC learner} for $\cH$ with respect to
    $\ell$.
  \item We define the notions of \emph{agnostic $k$-PAC learnability with randomness},
    $m^{\agPACr}_{\cH,\ell,\cA}$ and of a \emph{randomized agnostic $k$-PAC learner} in the same manner,
    except that $\cA$ is instead a randomized learning algorithm and instead of~\eqref{eq:agkPAC}, we
    require
    \begin{align*}
      \PP_{(\rn{x},\rn{x'})\sim(\mu\otimes\mu')^m}[
        \EE_{\rn{b}\sim U(R_\cA(m))}[
          L_{\mu,\mu',F,\ell}(\cA(\rn{x}, F^*_m(\rn{x},\rn{x'}),\rn{b}))
        ]
        \leq
        \inf_{H\in\cH} L_{\mu,\mu',F,\ell}(H) + \epsilon
      ]
      & \geq 1 - \delta,
    \end{align*}
    for every integer $m\geq m^{\agPACr}_{\cH,\ell,\cA}(\epsilon,\delta)$.
  \end{enumdef}
\end{definition}

Similarly to classic agnostic PAC learning, the best possible hypotheses, which are not necessarily
an element of $\cH$, are the Bayes predictors; since these are not the main focus of the paper, we
will elaborate on the topic in Appendix~\ref{sec:Bayes}.

Both in our setting and classic PAC the agnostic loss function has an additional feature: namely, we
could have $\ell(H,x,y)\neq\ell(H',x,y)$ even when $H(x)=H'(x)$, that is, we could incur different
penalties based solely on which hypothesis we answered, even though they classify the point $x$ in
the same way. Intuitively, this is some kind of ``non-locality'': $\ell(H,x,y)$ can depend on $H$ in
a more complex way than just via the value $H(x)$. However, in applications of classic PAC learning,
this could be a desired feature: maybe one of $H$ or $H'$ is a more complex hypothesis and we want
to penalize it to prevent overfitting and one way to perform this is to use an agnostic loss
function of the form
\begin{align*}
  \ell(H,x,y) & \df \ell_r(x,H(x),y) + r(H),
\end{align*}
where $\ell_r$ is a non-agnostic loss function and $r\colon\cH\to\RR$ is a regularization term
(see~\cite[Chapter~13]{SB14}). The actual definition of locality below is a happy compromise that
says that except for regularization, $\ell(H,x,y)$ can only depend on $H$ via the value $H(x)$.

\begin{definition}[Locality and regularization]
  A $k$-ary agnostic loss function $\ell\colon\cH\times\cE_k(\Omega)\times\Lambda^{S_k}\to\RR_{\geq
    0}$ is \emph{local} if there exists a function $r\colon\cH\to\RR$ such that for every
  $F,H\in\cH$, every $x\in\cE_k(\Omega)$ and every $y\in\Lambda^{S_k}$, we have
  \begin{align*}
    F^*_k(x) = H^*_k(x) & \implies \ell(F,x,y) - r(F) = \ell(H,x,y) - r(H) \geq 0.
  \end{align*}
  A function $r$ satisfying the above is called a \emph{regularization term} of $\ell$.

  Equivalently, $\ell$ is local if and only if it can be factored as
  \begin{align*}
    \ell(H,x,y) & = \ell_r(x,H^*_k(x),y) + r(H)
  \end{align*}
  for some (non-agnostic) $k$-ary loss function
  $\ell_r\colon\cE_k(\Omega)\times\Lambda^{S_k}\times\Lambda^{S_k}\to\RR_{\geq 0}$ and some
  regularization term $r\colon\cH\to\RR$.
\end{definition}

With the aim of getting to the definition of $\VCN_k$-dimension, we now recall the definition of
Natarajan dimension.

\begin{definition}[Natarajan dimension~\cite{Nat89}]\label{def:Natdim}
  Let $\cF$ be a collection of functions of the form $X\to Y$ and let $A\subseteq X$.
  \begin{enumdef}
  \item We say that $\cF$ \emph{(Natarajan-)shatters} $A$ if there exist functions $f_0,f_1\colon
    A\to Y$ such that
    \begin{enumerate}
    \item for every $a\in A$, $f_0(a)\neq f_1(a)$;
    \item for every $U\subseteq A$, there exists $F_U\in\cF$ such that
      \begin{align*}
        F_U(a) & = f_{\One[a\in U]}(a)
      \end{align*}
      for every $a\in A$.
    \end{enumerate}
  \item The \emph{Natarajan dimension} of $\cF$ is defined as
    \begin{align*}
      \Nat(\cF) & \df \sup\{\lvert A\rvert \mid A\subseteq X\land\cF\text{ Natarajan-shatters } A\}.
    \end{align*}
  \end{enumdef}
\end{definition}

When $\lvert Y\rvert=2$, the definitions of shattering and dimension above are equivalent to those
of the \emph{Vapnik--Chervonenkis ($\VC$) dimension} ($\VC(\cF)$)~\cite{VC71}.

\begin{definition}[$\VCN_k$-dimension]
  Let $\Omega$ be a Borel template, $\Lambda$ be a non-empty Borel space, $k\in\NN_+$ and
  $\cH\subseteq\cF_k(\Omega,\Lambda)$ be a $k$-ary hypothesis class.
  \begin{enumdef}
  \item For $H\in\cF_k(\Omega,\Lambda)$ and $x\in\cE_{k-1}(\Omega)$, let
    \begin{align*}
      H^*_k(x,\place)\colon \prod_{A\in r(k)\setminus r(k-1)} X_{\lvert A\rvert}\to \Lambda^{S_k}
    \end{align*}
    be the function obtained from $H^*_k$ by fixing its $\cE_{k-1}(\Omega)$ arguments to be $x$ and
    let
    \begin{align}\label{eq:cHx}
      \cH(x) & \df \{H^*_k(x,\place) \mid H\in\cH\}.
    \end{align}
  \item The \emph{Vapnik--Chervonenkis--Natarajan $k$-dimension} of $\cH$ (\emph{$\VCN_k$-dimension}) is
    defined as
    \begin{align*}
      \VCN_k(\cH) & \df \sup_{x\in\cE_{k-1}(\Omega)} \Nat(\cH(x)).
    \end{align*}
  \end{enumdef}
\end{definition}

Note that $\VCN_k(\cH)$ is not the same as the Natarajan dimension of $\cH$ as a family of functions
$\cE_k(\Omega)\to\Lambda$. By definition, $\VCN_k(\cH)$ is always at most the Natarajan dimension of
$\cH$, but the next two examples show that the latter can be infinite while the former is finite.

\begin{example}\label{ex:matching}
  For $\Omega_1\df\NN$, $\Omega_i$ being a singleton for every $i\geq 2$ and $\Lambda=\{0,1\}$, the
  family of spanning subgraphs of a countably infinite matching
  \begin{align*}
    \cH & \df \{F_A \mid A\subseteq\NN\},
  \end{align*}
  where
  \begin{align*}
    F_A(x) & \df \One[\exists i\in A, \{x_{\{1\}},x_{\{2\}}\}=\{2i,2i+1\}]
    \qquad (A\subseteq\NN)
  \end{align*}
  satisfies $\VCN_2(\cH)=1$ and $\VC(\cH)=\infty$.

  The former holds since $\lvert\cH(z)\rvert=2$ for every $z\in\NN$ and the latter holds because the
  infinite set $\{x\in\cE_2(\Omega) \mid \exists i\in\NN, (x_{\{1\}}=2i\land x_{\{2\}}=2i+1)\}$ is
  shattered by $\cH$ as a family of functions $\cE_2(\Omega)\to\Lambda$.
\end{example}

\begin{example}\label{ex:boundeddegree}
  Let $\Omega_1$ be some Borel space, let $\Omega_i$ be a singleton for every $i\geq 2$ and let
  $\Lambda=\{0,1\}$. Let $d\in\NN$ and let $\cH\in\cF_2(\Omega,\Lambda)$ the set of all graphs on
  $\Omega_1$ whose maximum degree is at most $d$.

  It is straightforward to check that if $\Omega_1$ is infinite, then $\VC(\cH)=\infty$ (this is
  because any infinite collection of pairs that is pairwise disjoint is shattered by $\cH$). On the
  other hand, since the maximum degree of any graph in $\cH$ is at most $d$, we conclude that
  $\cH(x)$ never shatters any set of size $d+1$ for any $x\in\Omega_1$. Thus, we have
  $\VCN_2(\cH)\leq d$ (in fact, it is straightforward to check that we in fact have
  $\VCN_2(\cH)=\min\{d,\lvert\Omega_1\rvert-1\}$).
\end{example}

We conclude the section with the definition of flexibility of a loss function. Unfortunately, we
will only be able to provide some intuition on this definition in
Remark~\ref{rmk:flexibilityintuition}. At this point, we can only say that flexibility has two main
components: first, it will be used to produce what we call ``neutral symbols'' (see
Definition~\ref{def:neutsymb}), which intuitively are symbols that do not affect the learning task
(this is the role of~\eqref{eq:EEindepy} and~\eqref{eq:EEindepH}); second, the production of such
symbols can be done using only a finite amount of randomness (this is the role of the
conditions~\eqref{eq:GsimcN} and~\eqref{eq:GsimcNag} on $\cN$).

\begin{definition}[Flexibility]\label{def:flexible}
  Let $\Omega$ be a Borel template, $\Lambda$ be a non-empty Borel space, $k\in\NN_+$ and
  $\cH\subseteq\cF_k(\Omega,\Lambda)$ be a $k$-ary hypothesis class.
  \begin{enumdef}  
  \item A $k$-ary loss function
    $\ell\colon\cE_k(\Omega)\times\Lambda^{S_k}\times\Lambda^{S_k}\to\RR_{\geq 0}$ is
    \emph{flexible} if there exist a Borel template $\Sigma$, a probability template
    $\nu\in\Pr(\Sigma)$, a $k$-ary hypothesis $G\in\cF_k(\Omega\otimes\Sigma,\Lambda)$ and a
    measurable function
    \begin{align*}
      \cN\colon\bigcup_{m\in\NN} (\cE_m(\Omega)\times [R_\cN(m)])
      \to
      \Lambda^{([m])_k},
    \end{align*}
    where $R_\cN\colon\NN\to\NN_+$, such that for every $x\in\cE_k(\Omega)$ and every
    $y,y'\in\Lambda^{S_k}$, we have
    \begin{align}\label{eq:EEindepy}
      \EE_{\rn{z}\sim\nu^k}[\ell(x,y,G^*_k(x,\rn{z}))]
      & =
      \EE_{\rn{z}\sim\nu^k}[\ell(x,y',G^*_k(x,\rn{z}))]
    \end{align}
    and for every $m\in\NN$ and every $x\in\cE_m(\Omega)$, if $\rn{z}\sim\nu^m$ and $\rn{b}$ is
    picked uniformly in $[R_\cN(m)]$, then
    \begin{align}\label{eq:GsimcN}
      G^*_m(x,\rn{z}) & \sim \cN(x,\rn{b}).
    \end{align}
    In this case, we say that $(\Sigma,\nu,G,\cN)$ is a \emph{witness of flexibility} of $\ell$.

    Note that when $\ell$ is flexible with witness $(\Sigma,\nu,G,\cN)$, we can define a function
    $\ell^{\Sigma,\nu,G}\colon\cE_k(\Omega)\to\RR_{\geq 0}$ by
    \begin{align*}
      \ell^{\Sigma,\nu,G}(x)
      & \df
      \EE_{\rn{z}\sim\nu^k}[\ell(x,y,G^*_k(x,\rn{z}))]
      \qquad (x\in\cE_k(\Omega)),
    \end{align*}
    where $y$ is any point in $\Lambda^{S_k}$.
  \item A $k$-ary agnostic loss function
    $\ell\colon\cH\times\cE_k(\Omega)\times\Lambda^{S_k}\to\RR_{\geq 0}$ is \emph{flexible} if there
    exist a Borel template $\Sigma$, a probability template $\nu\in\Pr(\Sigma)$, a $k$-ary hypothesis
    $G\in\cF_k(\Omega\otimes\Sigma,\Lambda)$ and a measurable function
    \begin{align*}
      \cN\colon\bigcup_{m\in\NN} (\cE_m(\Omega)\times [R_\cN(m)])
      \to
      \Lambda^{([m])_k},
    \end{align*}
    where $R_\cN\colon\NN\to\NN_+$, such that for every $H,H'\in\cH$ and every $x\in\cE_k(\Omega)$,
    we have
    \begin{align}\label{eq:EEindepH}
      \EE_{\rn{z}\sim\nu^k}[\ell(H,x,G^*_k(x,\rn{z}))]
      & =
      \EE_{\rn{z}\sim\nu^k}[\ell(H',x,G^*_k(x,\rn{z}))]
    \end{align}
    and for every $m\in\NN$ and every $x\in\cE_m(\Omega)$, if $\rn{z}\sim\nu^m$ and $\rn{b}$ is
    picked uniformly in $[R_\cN(m)]$, then
    \begin{align}\label{eq:GsimcNag}
      G^*_m(x,\rn{z}) & \sim \cN(x,\rn{b}).
    \end{align}
    In this case, we say that $(\Sigma,\nu,G,\cN)$ is a \emph{witness of flexibility} of $\ell$.

    Similarly, when $\ell$ is flexible with witness $(\Sigma,\nu,G,\cN)$, we can define a function
    $\ell^{\Sigma,\nu,G}\colon\cE_k(\Omega)\to\RR_{\geq 0}$ by
    \begin{align}\label{eq:flexibleellSigmanuG}
      \ell^{\Sigma,\nu,G}(x)
      & \df
      \EE_{\rn{z}\sim\nu^k}[\ell(H,x,G^*_k(x,\rn{z}))]
      \qquad (x\in\cE_k(\Omega)),
    \end{align}
    where $H$ is any hypothesis in $\cH$.
  \end{enumdef}
\end{definition}

\begin{remark}\label{rmk:flexibilityfinite}
  An easy way of obtaining the $\cN$ required in the definition of flexibility of $\ell$ is
  when~\eqref{eq:EEindepy} (\eqref{eq:EEindepH}, respectively) holds with some $(\Sigma,\nu,G)$ such
  that all $\Sigma_i$ are finite and all $\nu_i$ are uniform measures. This is
  because~\eqref{eq:GsimcN} (\eqref{eq:GsimcNag}, respectively) follows by taking
  \begin{align*}
    R_\cN(m)
    & \df
    \prod_{A\in r(m)}\lvert\Sigma_{\lvert A\rvert}\rvert
    =
    \prod_{i=1}^k\lvert\Sigma_i\rvert^{\binom{m}{i}},
    &
    \cN(x,b) & \df G(x,\phi_m(b)), 
  \end{align*}
  where $\phi_m\colon[R_\cN(m)]\to\prod_{A\in r(m)}\Sigma_{\lvert A\rvert}$ is any fixed
  bijection\footnote{In an older version of the manuscript, the definition of flexibility instead
    asked for this stronger property that each $\Sigma_i$ is finite and each $\nu_i$ is the uniform
    measure. However, this made proving flexibility of the $0/1$-loss difficult in
    Lemma~\ref{lem:flexibility} below (in fact, the proof in the older version was incorrect). The
    slightly weaker definition of flexibility as we have now makes the $0/1$-loss flexible while
    still being a definition strong enough to create neutral symbols in
    Proposition~\ref{prop:neutsymb} later.\label{ftn:flexibility}}.
\end{remark}

\begin{lemma}[Flexibility of $\ell_{0/1}$]\label{lem:flexibility}
  Both the agnostic and non-agnostic $0/1$-loss functions are flexible when $\Lambda$ is finite.
\end{lemma}

\begin{proof}
  We consider two cases.

  When $k=1$, we let $\Sigma_1\df\Lambda$, equipped with discrete $\sigma$-algebra, let all other
  $\Sigma_i$ have a single point, let each $\nu_i$ be the uniform measure on $\Sigma_i$ and let
  $G\in\cF_k(\Omega\otimes\Sigma,\Lambda)$ be given by
  \begin{align*}
    G(x,z) & \df z_{[1]}
    \qquad (x\in\cE_1(\Omega), z\in\cE_1(\Sigma)).
  \end{align*}
  Then for every $x\in\cE_1(\Omega)$, if $\rn{z}\sim\nu^1$, then $G^*_1(x,\rn{z})$ is uniformly
  distributed in $\Lambda^{S_1}$, so for the non-agnostic $0/1$-loss function, we have
  \begin{align*}
    \EE_{\rn{z}\sim\nu^1}[\ell_{0/1}(x,y,G^*_1(x,\rn{z}))]
    & =
    1 - \frac{1}{\lvert\Lambda\rvert}
    \qquad (x\in\cE_1(\Omega),y\in\Lambda^{S_1})
  \end{align*}
  and for the agnostic $0/1$-loss function with respect to $\cH\subseteq\cF_1(\Omega,\Lambda)$, we
  have
  \begin{align*}
    \EE_{\rn{z}\sim\nu^1}[\ell_{0/1}(H,x,G^*_1(x,\rn{z}))]
    & =
    1 - \frac{1}{\lvert\Lambda\rvert}
    \qquad (H\in\cH, x\in\cE_1(\Omega)).
  \end{align*}
  Since the right-hand sides of the equations above do not depend on $y$ and $H$ (in fact, they do
  not even depend on $x$), we get~\eqref{eq:EEindepy} and~\eqref{eq:EEindepH}. The existence of the
  corresponding $\cN$ then follows by Remark~\ref{rmk:flexibilityfinite}.

  \medskip

  We now consider the case $k\geq 2$. Let $\Sigma_1\df[0,1]$, equipped with the usual Borel
  $\sigma$-algebra, let $\Sigma_k\df\Lambda^{S_k}$ and let all other $\Sigma_i$ have a single
  point. We also let $\nu_1\df\lambda$ be the Lebesgue measure and let all other $\nu_i$ be the
  uniform measure on $\Sigma_i$.

  For $z\in\cE_k(\Sigma)$ such that $z_{\{1\}},\ldots,z_{\{k\}}$ are pairwise distinct,
  let $\sigma(z)\in S_k$ be the unique permutation such that
  \begin{align*}
    z_{\{\sigma(z)(1)\}} < z_{\{\sigma(z)(2)\}} < \cdots < z_{\{\sigma(z)(k)\}}
  \end{align*}
  and define $\sigma(z)\in S_k$ arbitrarily (but measurably) when there are repetitions among
  $z_{\{1\}},\ldots,z_{\{k\}}$.

  We now define $G\in\cF_k(\Omega\otimes\Sigma,\Lambda)$ by
  \begin{align*}
    G(x,z) & \df (z_{[k]})_{\sigma(z)}
    \qquad (x\in\cE_k(\Omega), z\in\cE_k(\Sigma)).
  \end{align*}

  It is straightforward to check that if $z\in\cE_m(\Sigma)$ ($m\in\NN$) is such that
  $z_{\{1\}},\ldots,z_{\{m\}}$ are pairwise distinct, then for every $\alpha,\beta\in ([m])_k$ with
  $\im(\alpha)=\im(\beta)$, we have
  \begin{align*}
    \sigma(\alpha^*(z)) = \sigma(\beta^*(z))
    & \iff
    \alpha = \beta.
  \end{align*}
  Since $\nu_1$ is the Lebesgue measure and $\nu_k$ is the uniform measure on
  $\Sigma_k=\Lambda^{S_k}$, we conclude that for every $x\in\cE_m(\Omega)$, if $\rn{z}\sim\nu^m$,
  then $G^*_m(x,\rn{z})$ is uniformly distributed in $\Lambda^{([m])_k}$. In particular, when $m=k$,
  we get
  \begin{align*}
    \EE_{\rn{z}\sim\nu^k}[\ell_{0/1}(x,y,G^*_k(x,\rn{z}))]
    & =
    \left(1 - \frac{1}{\lvert\Lambda\rvert}\right)^{k!}
    \qquad (x\in\cE_k(\Omega),y\in\Lambda^{S_k})
  \end{align*}
  for the non-agnostic $0/1$-loss function and
  \begin{align*}
    \EE_{\rn{z}\sim\nu^k}[\ell_{0/1}(H,x,G^*_k(x,\rn{z}))]
    & =
    \left(1 - \frac{1}{\lvert\Lambda\rvert}\right)^{k!}
    \qquad (H\in\cH, x\in\cE_k(\Omega))
  \end{align*}
  for the agnostic $0/1$-loss function with respect to $\cH\subseteq\cF_k(\Omega,\Lambda)$.

  Finally, by taking
  \begin{align*}
    R_\cN(m) & \df \lvert\Lambda\rvert^{(m)_k},
    &
    \cN(x,b) & \df \phi_m(b),
  \end{align*}
  where $\phi_m\colon[R_\cN(m)]\to\Lambda^{(m)_k}$ is any fixed bijection, it is clear that for
  every $x\in\cE_m(\Omega)$, if $\rn{b}$ is picked uniformly in $[R_\cN(m)]$, then $\cN(x,\rn{b})$
  is uniformly distributed in $\Lambda^{(m)_k}$, hence has the same distribution as
  $G^*_m(x,\rn{z})$ when $\rn{z}\sim\nu^m$.
\end{proof}

\section{Fundamentals in the partite setting}
\label{sec:partdefs}

This section is the partite analogue of Section~\ref{sec:defs}, containing most of the definitions
related to $k$-PAC learning in the partite setting. We initially follow the same order of
Section~\ref{sec:defs}, but in Section~\ref{subsec:partitespecific}, we arrive at some new and
important definitions.

\begin{definition}[Borel $k$-partite templates]
  Let $k\in\NN_+$.
  \begin{enumdef}
  \item A \emph{Borel $k$-partite template} is a sequence $\Omega=(\Omega_A)_{A\in r(k)}$, where
    $\Omega_A=(X_A,\cB_A)$ is a non-empty (standard) Borel space.
  \item A \emph{probability $k$-partite template} on a Borel $k$-partite template $\Omega$ is a
    sequence $\mu=(\mu_A)_{A\in r(k)}$, where $\mu_A$ is a probability measure on $\Omega_A$. The
    space of probability $k$-partite templates on $\Omega$ is denoted $\Pr(\Omega)$.
  \item Analogously to Definition~\ref{def:Boreltemplate:product}, if $\Omega$ and $\Omega'$ are
    Borel $k$-partite templates and $\mu\in\Pr(\Omega)$ and $\mu'\in\Pr(\Omega')$ are probability
    $k$-partite templates, we define the product Borel $k$-partite template $\Omega\otimes\Omega'$
    by $(\Omega\otimes\Omega')_A\df\Omega_A\otimes\Omega'_A$ and define the product probability
    $k$-partite template $\mu\otimes\mu'\in\Pr(\Omega\otimes\Omega')$ by
    $(\mu\otimes\mu')_A\df\mu_A\otimes\mu'_A$.
  \item For a Borel $k$-partite template $\Omega$, (finite or) countable sets $V_1,\ldots,V_k$ and a
    non-empty Borel space $\Lambda$, we define\footnote{Except in the concluding section,
      Section~\ref{sec:final}, all partite concepts defined for $V_1,\ldots,V_k$ will only be used
      when $V_1=\cdots=V_k$ (and in fact they will either be $\NN_+$ or $[m]$ for some $m\in\NN$).}
    \begin{align*}
      \cE_{V_1,\ldots,V_k}(\Omega) & \df \prod_{f\in r_k(V_1,\ldots,V_k)} X_{\dom(f)}
    \end{align*}
    equipping it with the product $\sigma$-algebra, where
    \begin{align*}
      r_k(V_1,\ldots,V_k)
      & \df
      \left\{f\colon A\to\bigcup_{i=1}^k V_i
      \;\middle\vert\;
      A\in r(k)\land\forall i\in A, f(i)\in V_i
      \right\}.
    \end{align*}

    If $\mu\in\Pr(\Omega)$ is a probability $k$-partite template on $\Omega$, we let
    \begin{align*}
      \mu^{V_1,\ldots,V_k}
      & \df
      \bigotimes_{A\in r_k(V_1,\ldots,V_k)} \mu_A
    \end{align*}
    be the product measure.

    For $m\in\NN$, we will use the shorthands $r_k(m)\df r_k([m],\ldots,[m])$,
    $\cE_m(\Omega)\df\cE_{[m],\ldots,[m]}(\Omega)$ and $\mu^m\df\mu^{[m],\ldots,[m]}$.
  \item For $\alpha\in\prod_{i=1}^k V_i$, we define the map
    $\alpha^*\colon\cE_{V_1,\ldots,V_k}(\Omega)\to\cE_1(\Omega)$ by
    \begin{align*}
      \alpha^*(x)_f & \df x_{\alpha\rest_{\dom(f)}} \qquad (x\in\cE_{V_1,\ldots,V_k}, f\in r_k(1)),
    \end{align*}
    where in the above we interpret $\alpha$ as a function $[k]\to\bigcup_{i=1}^k V_i$.
  \item Similarly to Definition~\ref{def:Boreltemplate:identification}, we naturally identify
    $\cE_{V_1,\ldots,V_k}(\Omega^1\otimes\Omega^2)$ with
    $\cE_{V_1,\ldots,V_k}(\Omega^1)\times\cE_{V_1,\ldots,V_k}(\Omega^2)$.
  \end{enumdef}
\end{definition}

\begin{definition}[$k$-partite hypotheses]
  Let $k\in\NN_+$, let $\Omega$ be a Borel $k$-partite template and let $\Lambda=(Y,\cB')$ be a
  non-empty Borel space.
  \begin{enumdef}
  \item The set of \emph{$k$-partite hypotheses} from $\Omega$ to $\Lambda$, denoted
    $\cF_k(\Omega,\Lambda)$, is the set of (Borel) measurable functions from $\cE_1(\Omega)$ to
    $\Lambda$ (note that here we use $\cE_1(\Omega)$ instead of $\cE_k(\Omega)$ as in the
    non-partite case).
  \item A \emph{$k$-partite hypothesis class} is a subset $\cH$ of $\cF_k(\Omega,\Lambda)$, equipped
    with a $\sigma$-algebra such that:
    \begin{enumerate}
    \item the evaluation map $\ev\colon\cH\times\cE_1(\Omega)\to\Lambda$ given by $\ev(H,x)\df H(x)$
      is measurable;
    \item for every $H\in\cH$, the set $\{H\}$ is measurable.
    \end{enumerate}
  \item For a $k$-partite hypothesis $F\in\cF_k(\Omega,\Lambda)$, we let
    $F^*_{V_1,\ldots,V_k}\colon\cE_{V_1,\ldots,V_k}(\Omega)\to(\Lambda)^{\prod_{i=1}^k V_i}$ be
    given by
    \begin{align*}
      F^*_{V_1,\ldots,V_k}(x)_\alpha & \df F(\alpha^*(x)).
    \end{align*}
    For $m\in\NN$, we use the shorthand $F^*_m\df F^*_{[m],\ldots,[m]}$. Note that $F^*_1 = F$ once we
    naturally identify $\Lambda^{[1]^k}$ with $\Lambda$ (as $[1]^k$ has a single point).
  \end{enumdef}
\end{definition}

The next lemma is the partite analogue of Lemma~\ref{lem:F*Vequiv}.

\begin{lemma}[Equivariance of $F^*_{V_1,\ldots,V_k}$]\label{lem:partiteF*Vequiv}
  For $F\in\cF_k(\Omega,\Lambda)$, the definition of $F^*_{V_1,\ldots,V_k}$ is equivariant in the
  sense that for every $\beta_i\colon U_i\to V_i$ ($i\in[k]$), the diagram
  \begin{equation*}
    \begin{tikzcd}
      \cE_{V_1,\ldots,V_k}(\Omega)
      \arrow[r, "F^*_{V_1,\ldots,V_k}"]
      \arrow[d, "\beta^{\#}"']
      &
      \Lambda^{\prod_{i=1}^k V_i}
      \arrow[d, "\beta^{\#}"]
      \\
      \cE_{U_1,\ldots,U_k}(\Omega)
      \arrow[r, "F^*_{U_1,\ldots,U_k}"]
      &
      \Lambda^{\prod_{i=1}^k U_i}
    \end{tikzcd}
  \end{equation*}
  is commutative, where vertical arrows are given by
  \begin{gather}
    \beta^{\#}(x)_f \df x_{\beta_{\#}(f)}
    \qquad (x\in\cE_{V_1,\ldots,V_k}(\Omega), f\in r_k(U_1,\ldots,U_k)),
    \label{eq:betax}
    \\
    \beta_{\#}(f)\colon\dom(f)\ni i \mapsto\beta_i(f(i))\in\bigcup_{i=1}^k V_i
    \qquad (f\in r_k(U_1,\ldots,U_k)),
    \label{eq:betaf}
    \\
    \beta^{\#}(y)_\alpha \df y_{\beta_1(\alpha_1),\ldots,\beta_k(\alpha_k)}
    \qquad \left(y\in\Lambda^{\prod_{i=1}^k U_i}, \alpha\in\prod_{i=1}^k U_i\right).
    \label{eq:betay}
  \end{gather}
\end{lemma}

\begin{proof}
  Let $x\in\cE_{V_1,\ldots,V_k}$, let $\alpha\in\prod_{i=1}^k U_i$, let $\gamma\in\prod_{i=1}^k V_i$
  be given by $\gamma_i\df\beta_i(\alpha_i)$ ($i\in[k]$) and note that for $f\in r_k(1)$, we have
  \begin{align*}
    \alpha^*(\beta^{\#}(x))_f
    & =
    \beta^{\#}(x)_{\alpha\rest_{\dom(f)}}
    =
    x_{\beta_{\#}(\alpha\rest_{\dom(f)})}
    =
    x_{\gamma\rest_{\dom(f)}}
    =
    \gamma^*(x)_f,
  \end{align*}
  hence $\alpha^*(\beta^{\#}(x))=\gamma^*(x)$, so we conclude that
  \begin{align*}
    \beta^{\#}(F^*_{V_1,\ldots,V_k}(x))_\alpha
    & =
    F^*_{V_1,\ldots,V_k}(x)_\gamma
    =
    F(\gamma^*(x))
    =
    F(\alpha^*(\beta^{\#}(x)))
    =
    F^*_{U_1,\ldots,U_k}(\beta^{\#}(x))_\alpha,
  \end{align*}
  so the diagram commutes.
\end{proof}

\begin{remark}\label{rmk:partiteexchangeable}
  Analogously to Remark~\ref{rmk:exchangeable}, $k$-partite hypotheses are combinatorial structures
  that can generate what is called a \emph{local separately exchangeable distribution} once we equip
  $\Omega$ with a probability $k$-partite template $\mu\in\Pr(\Omega)$. Namely, if
  $\rn{x}\sim\mu^{\NN_+,\ldots,\NN_+}$, then the distribution of $(\rn{x},F^*_{\NN_+,\ldots,\NN_+}(\rn{x}))$ is:
  \begin{description}
  \item[Local:] If $(U_1,\ldots,U_k)$ and $(V_1,\ldots,V_k)$ are $k$-tuples of subsets of $\NN_+$
    such that $U_i\cap V_i = \varnothing$ for every $i\in[k]$, then
    \begin{align*}
      ((\rn{x}_f)_{f\in r_k(U_1,\ldots,U_k)}, &
      (F^*_{\NN_+,\ldots,\NN_+}(\rn{x})_\alpha)_{\alpha\in\prod_{i=1}^k U_i})
      \intertext{is independent from}
      ((\rn{x}_f)_{f\in r_k(V_1,\ldots,V_k)}, &
      (F^*_{\NN_+,\ldots,\NN_+}(\rn{x})_\alpha)_{\alpha\in\prod_{i=1}^k V_i}).
    \end{align*}
  \item[Separately exchangeable:] For every $(\beta_1,\ldots,\beta_k)\in S_{\NN_+}^k$, we have
    \begin{align*}
      \beta^{\#}(\rn{x},F^*_{\NN_+,\ldots,\NN_+}(\rn{x}))
      & \df
      (\beta^{\#}(\rn{x}), \beta^{\#}(F^*_{\NN_+,\ldots,\NN_+}(\rn{x})))
      \\
      & =
      (\beta^{\#}(\rn{x}), F^*_{\NN_+,\ldots,\NN_+}(\beta^{\#}(\rn{x})))
      \sim
      (\rn{x},F^*_{\NN_+,\ldots,\NN_+}(\rn{x})),
    \end{align*}
    where the two $\beta^{\#}$ are given by~\eqref{eq:betax} and~\eqref{eq:betay}, the
    equality follows from Lemma~\ref{lem:partiteF*Vequiv} and the distributional equality follows from
    $\rn{x}\sim\beta^{\#}(\rn{x})$ since $\rn{x}$ has a product distribution.
  \end{description}

  The version of the Aldous--Hoover Theorem on separate exchangeability due to Hoover~\cite{Hoo79}
  (see also~\cite[Corollary~7.23 and Lemma~7.35]{Kal05}) says that every separately exchangeable
  distribution on $\Lambda^{\NN_+^k}$ is of the form $F^*_{\NN_+,\ldots,\NN_+}(\rn{x})$ for some
  $F\in\cF_k(\Omega,\Lambda)$, where $\rn{x}\sim\mu^{\NN_+,\ldots,\NN_+}$ for some
  $\mu\in\Pr(\Omega)$ and some $\Omega$.
\end{remark}

\begin{definition}[Rank]
  Let $k\in\NN_+$, let $\Omega$ be a Borel $k$-partite template and let $\Lambda$ be a non-empty
  Borel space.
  \begin{enumdef}
  \item The \emph{rank} of a $k$-partite hypothesis $F\in\cF_k(\Omega,\Lambda)$, denoted $\rk(F)$ is
    the minimum $r\in\NN$ such that $F$ factors as
    \begin{align*}
      F(x) & \df F'((x_f)_{f\in r_k(1), \lvert\dom(f)\rvert\leq r}) \qquad (x\in\cE_1(\Omega))
    \end{align*}
    for some function $F'\colon\prod_{f\in r_k(1),\lvert\dom(f)\rvert\leq r} X_{\dom(f)}\to\Lambda$.
  \item The \emph{rank} of $\cH$ is defined as
    \begin{align*}
      \rk(\cH) & \df \sup_{F\in\cH} \rk(F).
    \end{align*}
  \end{enumdef}
\end{definition}

\begin{remark}\label{rmk:partiterk1}
  Similarly to Remark~\ref{rmk:rk1}, the partite setting covers the cases of $k$-partite
  $k$-hypergraphs and $k$-partite structures in finite relational languages (i.e., structures in
  which predicates can only hold on tuples that cross a given $k$-partition in the sense that they
  contain at most one vertex of each part).

  For (measurable) $k$-partite $k$-hypergraphs with $k$-partition $(\Upsilon_1,\ldots,\Upsilon_k)$,
  we let $\Omega_{\{i\}}\df\Upsilon_i$ for every $i\in[k]$, let all other $\Omega_A$ have a single
  point and take $\Lambda=\{0,1\}$, so there is a natural one-to-one correspondence in which a
  measurable $k$-partite $k$-hypergraph $H$ corresponds to $F_H\in\cF_k(\Omega,\Lambda)$ given by
  \begin{align*}
    F_H(x) & \df \One[(x_{1^{\{1\}}},\ldots,x_{1^{\{k\}}})\in E^H] \qquad (x\in\cE_1(\Omega)),
  \end{align*}
  where $1^B$ is the unique function $B\to[1]$, that is, $F_H$ is morally the $k$-partite adjacency
  tensor of $H$.

  For (measurable) $k$-partite structures in a finite relational language $\cL$ with a given
  $k$-partition $(\Upsilon_1,\ldots,\Upsilon_k)$, we define $\Omega_{\{i\}}\df\Upsilon_{\{i\}}$, let
  all other $\Omega_A$ have a single point and let $\Lambda$ be the (finite) set of all $k$-partite
  structures in $\cL$ with $k$-partition $(\{1\},\ldots,\{k\})$, equipped with the discrete
  $\sigma$-algebra. Then there is a natural injection from the set of measurable $k$-partite
  structures in $\cL$ with $k$-partition $(\Upsilon_1,\ldots,\Upsilon_k)$ to the set
  $\cF_k(\Omega,\Lambda)$ that maps $M$ to $F_M$ defined by letting $F_M(x)$ be the unique
  $k$-partite structure with $k$-partition $(\{1\},\ldots,\{k\})$ such that $[k]\ni i\mapsto
  x_{1^{\{i\}}}\in\bigcup_{j=1}^k\Upsilon_j$ is an embedding of $F_M(x)$ into $M$.

  Again, these constructions produce $k$-partite hypotheses $F$ with $\Lambda$ finite and
  $\rk(F)\leq 1$.
\end{remark}

\begin{definition}[$k$-partite loss functions]
  Let $k\in\NN_+$, let $\Omega$ be a Borel $k$-partite template and let $\Lambda$ be a non-empty
  Borel space.
  \begin{enumdef}
  \item A \emph{$k$-partite loss function} over $\Lambda$ is a measurable function
    $\ell\colon\cE_1(\Omega)\times\Lambda\times\Lambda\to\RR_{\geq 0}$. (It is important to note
    that, differently from $k$-ary loss functions, $k$-partite loss functions use $\cE_1(\Omega)$
    and $\Lambda$ instead of $\cE_k(\Omega)$ and $\Lambda^{S_k}$, respectively.)
  \item For a $k$-partite loss function $\ell$, we define
    \begin{align*}
      \lVert\ell\rVert_\infty
      & \df
      \sup_{\substack{x\in\cE_1(\Omega)\\y,y'\in\Lambda}} \ell(x,y,y'),
      &
      s(\ell)
      & \df
      \inf_{\substack{x\in\cE_1(\Omega)\\y,y'\in\Lambda\\y\neq y'}} \ell(x,y,y').
    \end{align*}
  \item A $k$-partite loss function is:
    \begin{description}[format={\normalfont\textit}]
    \item[bounded] if $\lVert\ell\rVert_\infty < \infty$.
    \item[separated] if $s(\ell) > 0$ and $\ell(x,y,y) = 0$ for every $x\in\cE_1(\Omega)$ and every
      $y\in\Lambda$. 
    \end{description}
  \item For a $k$-partite loss function $\ell$, hypotheses $F,H\in\cF_k(\Omega,\Lambda)$ and a
    probability $k$-partite template $\mu\in\Pr(\Omega)$, the \emph{total loss} of $H$ with respect
    to $\mu$, $F$ and $\ell$ is
    \begin{align*}
      L_{\mu,F,\ell}(H) & \df \EE_{\rn{x}\sim\mu^1}[\ell(\rn{x},H(\rn{x}),F(\rn{x}))].
    \end{align*}
  \item We say that $F\in\cF_k(\Omega,\Lambda)$ is \emph{realizable} in
  $\cH\subseteq\cF_k(\Omega,\Lambda)$ with respect to a $k$-partite loss function
  $\ell\colon\cE_1(\Omega)\times\Lambda\times\Lambda\to\RR_{\geq 0}$ and $\mu\in\Pr(\Omega)$ if
    $\inf_{H\in\cH} L_{\mu,F,\ell}(H) = 0$, i.e., $F$ can be approximated by elements of $\cH$ in
    terms of total loss.
  \item The \emph{$k$-partite $0/1$-loss function} over $\Lambda$ is defined as
    $\ell_{0/1}(x,y,y')\df\One[y\neq y']$.
  \end{enumdef}
\end{definition}

Differently from the non-partite setting, it does not make much sense to define the notion of a
symmetric loss function in the partite setting as there is no natural action of $S_k$ on
$\cE_1(\Omega)$ (as we may potentially have $\Omega_{\{i\}}\neq\Omega_{\{j\}}$ for $i\neq
j$). However, a natural $S_k$-action will arise when $\Omega$ is the $k$-partite version of a
(non-partite) Borel template (see Definition~\ref{def:kpart:Omega} and
Lemma~\ref{lem:kpartbasics}\ref{lem:kpartbasics:action} below).

\begin{definition}[Learning algorithms and partite $k$-PAC learnability]
  Let $k\in\NN_+$, let $\Omega$ be a Borel $k$-partite template, let $\Lambda$ be a non-empty Borel
  space and let $\cH\subseteq\cF_k(\Omega,\Lambda)$ be a $k$-partite hypothesis class.
  \begin{enumdef}  
  \item A \emph{($k$-partite) $(V_1,\ldots,V_k)$-sample} with respect to $\Omega$ and $\Lambda$ is
    an element of $\cE_{V_1,\ldots,V_k}(\Omega)\times\Lambda^{\prod_{i=1}^k V_i}$.
  \item A \emph{($k$-partite) learning algorithm} for $\cH$ is a measurable function
    \begin{align*}
      \cA\colon
      \bigcup_{m\in\NN} (\cE_m(\Omega)\times\Lambda^{[m]^k})
      \to
      \cH.
    \end{align*}
  \item We say that $\cH$ is \emph{$k$-PAC learnable} with respect to a $k$-partite loss function
    $\ell\colon\cE_1(\Omega)\times\Lambda\times\Lambda\to\RR_{\geq 0}$ if there exist a learning
    algorithm $\cA$ for $\cH$ and a function $m^{\PAC}_{\cH,\ell,\cA}\colon(0,1)^2\to\RR_{\geq 0}$
    such that for every $\epsilon,\delta\in(0,1)$, every $\mu\in\Pr(\Omega)$ and every
    $F\in\cF_k(\Omega,\Lambda)$ that is realizable in $\cH$ with respect to $\ell$ and $\mu$, we
    have
    \begin{align}\label{eq:partkPAC}
      \PP_{\rn{x}\sim\mu^m}[
        L_{\mu,F,\ell}(\cA(\rn{x}, F^*_m(\rn{x})))\leq\epsilon
      ]
      & \geq 1 - \delta
    \end{align}
    for every integer $m\geq m^{\PAC}_{\cH,\ell,\cA}(\epsilon,\delta)$. A learning algorithm $\cA$
    satisfying the above is called a \emph{$k$-PAC learner} for $\cH$ with respect to $\ell$.
  \item A \emph{randomized ($k$-partite) learning algorithm} for $\cH$ is a measurable function
    \begin{align*}
      \cA\colon
      \bigcup_{m\in\NN} (\cE_m(\Omega)\times\Lambda^{[m]^k}\times [R_\cA(m)])
      \to
      \cH,
    \end{align*}
    where $R_\cA\colon\NN\to\NN_+$ and $[t]$ is equipped with the discrete $\sigma$-algebra.
  \item We define the notions of \emph{$k$-PAC learnability with randomness}, $m^{\PACr}_{\cH,\ell,\cA}$ and
    of a \emph{randomized $k$-PAC learner} in the same manner, except that $\cA$ is instead a randomized
    learning algorithm and instead of~\eqref{eq:partkPAC}, we require
    \begin{align*}
      \PP_{\rn{x}\sim\mu^m}[
        \EE_{\rn{b}\sim U(R_\cA(m))}[L_{\mu,F,\ell}(\cA(\rn{x}, F^*_m(\rn{x}), \rn{b}))]
        \leq\epsilon]
      & \geq
      1 - \delta,
    \end{align*}
    for every integer $m\geq m^{\PACr}_{\cH,\ell,\cA}(\epsilon,\delta)$, where $U(t)$ is the uniform
    probability measure on $[t]$.
  \end{enumdef}
\end{definition}

Analogously to the non-partite case, in the agnostic setting, inspired by the analogue of
Remark~\ref{rmk:exchangeable}, Remark~\ref{rmk:partiteexchangeable}, we allow our adversary to play
a distribution $\nu$ over $\cE_{\NN_+,\ldots,\NN_+}(\Omega)\times\Lambda^{\NN_+^k}$ that is local
and separately exchangeable and exchangeability theory says that every such a $\nu$ is of the
expected form:
\begin{proposition}[Informal version of Proposition~\ref{prop:agsepexch}]
  Let $k\in\NN_+$, let $\Omega$ be a Borel $k$-partite template, let $\Lambda$ be a non-empty
  Borel space and let $\rn{x}$ and $\rn{y}$ be random elements of $\cE_{\NN_+}(\Omega)$ and
  $\Lambda^{\NN_+^k}$, respectively such that $(\rn{x},\rn{y})$ is local and separately
  exchangeable. Suppose further that $\rn{x}\sim\mu^{\NN_+,\ldots,\NN_+}$ for some $\mu\in\Pr(\Omega)$.

  Then there exist a Borel $k$-partite template $\Omega'$, $\mu'\in\Pr(\Omega')$ and
  $F\in\cF_k(\Omega\otimes\Omega',\Lambda)$ such that for $\rn{x'}\sim(\mu')^{\NN_+,\ldots,\NN_+}$
  picked independently from $\rn{x}$, we have
  \begin{align*}
    (\rn{x},\rn{y}) & \sim (\rn{x}, F^*_{\NN_+,\ldots,\NN_+}(\rn{x},\rn{x'})).
  \end{align*}
\end{proposition}

\begin{definition}[$k$-partite agnostic loss functions]\label{def:partagloss}
  Let $k\in\NN_+$, let $\Omega$ be a Borel $k$-partite template, let $\Lambda$ be a non-empty Borel
  space and let $\cH\subseteq\cF_k(\Omega,\Lambda)$ be a $k$-partite hypothesis class.
  \begin{enumdef}
  \item A \emph{$k$-partite agnostic loss function} over $\Lambda$ with respect to $\cH$ is a
    measurable function $\ell\colon\cH\times\cE_1(\Omega)\times\Lambda\to\RR_{\geq 0}$. (Again, the
    $k$-partite version uses $\cE_1(\Omega)$ and $\Lambda$ instead of $\cE_k(\Omega)$ and
    $\Lambda^{S_k}$, respectively.)
  \item For a $k$-partite agnostic loss function $\ell$, we define
    \begin{align*}
      \lVert\ell\rVert_\infty
      & \df
      \sup_{\substack{H\in\cH\\x\in\cE_1(\Omega)\\y\in\Lambda}} \ell(H,x,y).
    \end{align*}
  \item A $k$-partite agnostic loss function $\ell$ is \emph{bounded} if $\lVert\ell\rVert_\infty <
    \infty$.
  \item For a $k$-partite agnostic loss function $\ell$, a hypothesis $H\in\cH$, a Borel $k$-partite
    template $\Omega'$, probability $k$-partite templates $\mu\in\Pr(\Omega)$ and
    $\mu'\in\Pr(\Omega')$ and a hypothesis $F\in\cF_k(\Omega\otimes\Omega',\Lambda)$ over the
    product space, the \emph{total loss} of $H$ with respect to $\mu$, $\mu'$, $F$ and $\ell$ is
    \begin{align*}
      L_{\mu,\mu',F,\ell}(H)
      & \df
      \EE_{(\rn{x},\rn{x'})\sim(\mu\otimes\mu')^1}[\ell(H, \rn{x}, F(\rn{x},\rn{x'}))].
    \end{align*}
  \item The \emph{$k$-partite agnostic $0/1$-loss function} over $\Lambda$ with respect to $\cH$ is defined
    as
    \begin{align*}
      \ell_{0/1}(H,x,y) & \df \One[H(x)\neq y].
    \end{align*}
  \end{enumdef}
\end{definition}

\begin{definition}[Partite agnostic $k$-PAC learnability]\label{def:partagkPAC}
  Let $k\in\NN_+$, let $\Omega$ be a Borel $k$-partite template, let $\Lambda$ be a non-empty Borel
  space, let $\cH\subseteq\cF_k(\Omega,\Lambda)$ be a $k$-partite hypothesis class and let
  $\ell\colon\cH\times\cE_1(\Omega)\times\Lambda\to\RR_{\geq 0}$ be a $k$-partite agnostic loss
  function.
  \begin{enumdef}
  \item We say that $\cH$ is \emph{agnostically $k$-PAC learnable} with respect to a $k$-partite
    agnostic loss function $\ell\colon\cH\times\cE_1(\Omega)\times\Lambda\to\RR_{\geq 0}$ if there
    exist a learning algorithm $\cA$ for $\cH$ and a function
    $m^{\agPAC}_{\cH,\ell,\cA}\colon(0,1)^2\to\RR_{\geq 0}$ such that for every
    $\epsilon,\delta\in(0,1)$, every $\mu\in\Pr(\Omega)$, every Borel $k$-partite template
    $\Omega'$, every $\mu'\in\Pr(\Omega')$ and every $F\in\cF_k(\Omega\otimes\Omega',\Lambda)$, we
    have
    \begin{align}\label{eq:partagkPAC}
      \PP_{(\rn{x},\rn{x'})\sim(\mu\otimes\mu')^m}
         [L_{\mu,\mu',F,\ell}(\cA(\rn{x}, F^*_m(\rn{x},\rn{x'})))
           \leq
           \inf_{H\in\cH} L_{\mu,\mu',F,\ell}(H) + \epsilon
         ]
         & \geq 1 - \delta
    \end{align}
    for every integer $m\geq m^{\agPAC}_{\cH,\ell,\cA}(\epsilon,\delta)$. A learning algorithm $\cA$
    satisfying the above is called an \emph{agnostic $k$-PAC learner} for $\cH$ with respect to
    $\ell$.
  \item We define the notions of \emph{agnostic $k$-PAC learnability with randomness},
    $m^{\agPACr}_{\cH,\ell,\cA}$ and of a \emph{randomized agnostic $k$-PAC learner} in the same
    manner, except that $\cA$ is instead a randomized learning algorithm and instead
    of~\eqref{eq:partagkPAC}, we require
    \begin{align*}
      \PP_{(\rn{x},\rn{x'})\sim(\mu\otimes\mu')^m}[
        \EE_{\rn{b}\sim U(R_\cA(m))}[
          L_{\mu,\mu',F,\ell}(\cA(\rn{x}, F^*_m(\rn{x},\rn{x'}), \rn{b}))
        ]
        \leq
        \inf_{H\in\cH} L_{\mu,\mu',F,\ell}(H) + \epsilon
      ]
      & \geq 1 - \delta,
    \end{align*}
    for every integer $m\geq m^{\agPACr}_{\cH,\ell,\cA}(\epsilon,\delta)$
  \end{enumdef}
\end{definition}

\begin{definition}[Locality and regularization]
  A $k$-partite agnostic loss function $\ell\colon\cH\times\cE_1(\Omega)\times\Lambda\to\RR_{\geq
    0}$ is \emph{local} if there exists a function $r\colon\cH\to\RR$ such that for every
  $F,H\in\cH$, every $x\in\cE_1(\Omega)$ and every $y\in\Lambda$, we have
  \begin{align*}
    F(x) = H(x) & \implies \ell(F,x,y) - r(F) = \ell(H,x,y) - r(H) \geq 0.
  \end{align*}
  A function $r$ satisfying the above is called a \emph{regularization term} of $\ell$.

  Equivalently, $\ell$ is local if and only if it can be factored as
  \begin{align*}
    \ell(H,x,y) & = \ell_r(x,H(x),y) + r(H)
  \end{align*}
  for some (non-agnostic) $k$-partite loss function
  $\ell_r\colon\cE_1(\Omega)\times\Lambda\times\Lambda\to\RR_{\geq 0}$ and some regularization term
  $r\colon\cH\to\RR$.
\end{definition}

\begin{definition}[$\VCN_k$-dimension]
  Let $k\in\NN_+$, let $\Omega$ be a Borel $k$-partite template, let $\Lambda$ be a non-empty Borel
  space and let $\cH\subseteq\cF_k(\Omega,\Lambda)$ be a $k$-partite hypothesis class.
  \begin{enumdef}
  \item For $A\in\binom{[k]}{k-1}$, let
    \begin{align}\label{eq:rkA}
      r_{k,A} & \df \{f\in r_k(1) \mid \dom(f)\subseteq A\},
    \end{align}
    and for $x\in\prod_{f\in r_{k,A}} X_{\dom(f)}$ and $H\in\cF_k(\Omega,\Lambda)$, let
    \begin{align*}
      H(x,\place)\colon\prod_{f\in r_k(1)\setminus r_{k,A}} X_{\dom(f)}\to
      \Lambda
    \end{align*}
    be the function obtained from $H$ by fixing its arguments in $\prod_{f\in r_{k,A}} X_{\dom(f)}$ to
    be $x$ and let
    \begin{align}\label{eq:partitecHx}
      \cH(x) & \df \{H(x,\place) \mid H\in\cH\}.
    \end{align}
    \item The \emph{Vapnik--Chervonenkis--Natarajan $k$-dimension} of $\cH$
      (\emph{$\VCN_k$-dimension}) is defined as
      \begin{align*}
        \VCN_k(\cH)
        & \df
        \sup_{\substack{A\in\binom{[k]}{k-1}\\x\in\prod_{f\in r_{k,A}} X_{\dom(f)}}} \Nat(\cH(x)).
      \end{align*}
  \end{enumdef}
\end{definition}

\begin{definition}[Flexibility]\label{def:flexibilitypart}
  Let $k\in\NN_+$, let $\Omega$ be a Borel $k$-partite template, let $\Lambda$ be a non-empty Borel
  space and let $\cH\subseteq\cF_k(\Omega,\Lambda)$ be a $k$-partite hypothesis class.
  \begin{enumdef}
  \item A $k$-partite loss function $\ell\colon\cE_1(\Omega)\times\Lambda\times\Lambda\to\RR_{\geq
    0}$ is \emph{flexible} if there exist a Borel $k$-partite template $\Sigma$, a probability
    $k$-partite template $\nu\in\Pr(\Sigma)$, a $k$-partite hypothesis
    $G\in\cF_k(\Omega\otimes\Sigma,\Lambda)$ and a measurable function
    \begin{align*}
      \cN\colon\bigcup_{m\in\NN} (\cE_m(\Omega)\times [R_\cN(m)])
      \to
      \Lambda^{[m]^k},
    \end{align*}
    where $R_\cN\colon\NN\to\NN_+$, such that for every $x\in\cE_1(\Omega)$ and every
    $y,y'\in\Lambda$, we have
    \begin{align}\label{eq:partiteEEindepy}
      \EE_{\rn{z}\sim\nu^1}[\ell(x,y,G(x,\rn{z}))]
      & =
      \EE_{\rn{z}\sim\nu^1}[\ell(x,y',G(x,\rn{z}))]
    \end{align}
    and for every $m\in\NN$ and every $x\in\cE_m(\Omega)$, if $\rn{z}\sim\nu^m$ and $\rn{b}$ is
    picked uniformly in $[R_\cN(m)]$, then
    \begin{align}\label{eq:partiteGsimcN}
      G^*_m(x,\rn{z}) & \sim \cN(x,\rn{b}).
    \end{align}
    In this case, we say that $(\Sigma,\nu,G,\cN)$ is a \emph{witness of flexibility} of $\ell$.

    When $\ell$ is flexible with witness $(\Sigma,\nu,G,\cN)$, we define the function
    $\ell^{\Sigma,\nu,G}\colon\cE_1(\Omega)\to\RR_{\geq 0}$ by
    \begin{align*}
      \ell^{\Sigma,\nu,G}(x)
      & \df
      \EE_{\rn{z}\sim\nu^1}[\ell(x,y,G(x,\rn{z}))]
      \qquad (x\in\cE_1(\Omega)),
    \end{align*}
    where $y$ is any point in $\Lambda$.
  \item A $k$-partite agnostic loss function
    $\ell\colon\cH\times\cE_1(\Omega)\times\Lambda\to\RR_{\geq 0}$ is \emph{flexible} if there
    exist a Borel $k$-partite template $\Sigma$, a probability $k$-partite template, a $k$-partite
    hypothesis $G\in\cF_k(\Omega\otimes\Sigma,\Lambda)$ and a measurable function
    \begin{align*}
      \cN\colon\bigcup_{m\in\NN} (\cE_m(\Omega)\times [R_\cN(m)])
      \to
      \Lambda^{[m]^k},
    \end{align*}
    where $R_\cN\colon\NN\to\NN_+$, such that for every for every $H,H'\in\cH$ and every
    $x\in\cE_1(\Omega)$, we have
    \begin{align}\label{eq:partiteEEindepH}
      \EE_{\rn{z}\sim\nu^1}[\ell(H,x,G(x,\rn{z}))]
      & =
      \EE_{\rn{z}\sim\nu^1}[\ell(H',x,G(x,\rn{z}))]
    \end{align}
    and for every $m\in\NN$ and every $x\in\cE_m(\Omega)$, if $\rn{z}\sim\nu^m$ and $\rn{b}$ is
    picked uniformly in $[R_\cN(m)]$, then
    \begin{align}\label{eq:partiteGsimcNag}
      G^*_m(x,\rn{z}) & \sim \cN(x,\rn{b}).
    \end{align}
    In this case, we say that $(\Sigma,\nu,G,\cN)$ is a \emph{witness of flexibility} of $\ell$.

    When $\ell$ is flexible with witness $(\Sigma,\nu,G,\cN)$, we define the function
    $\ell^{\Sigma,\nu,G}\colon\cE_1(\Omega)\to\RR_{\geq 0}$ by
    \begin{align}\label{eq:partiteflexibleellSigmanuG}
      \ell^{\Sigma,\nu,G}(x)
      & \df
      \EE_{\rn{z}\sim\nu^1}[\ell(H,x,G(x,\rn{z}))]
      \qquad (x\in\cE_1(\Omega)),
    \end{align}
    where $H$ is any hypothesis in $\cH$.
  \end{enumdef}
\end{definition}

\begin{remark}\label{rmk:partiteflexibilityfinite}
  Similarly to Remark~\ref{rmk:flexibilityfinite}, an easy way of obtaining the $\cN$ required in the
  definition of flexibility of $\ell$ is when~\eqref{eq:partiteEEindepy} (\eqref{eq:partiteEEindepH},
  respectively) holds with some $(\Sigma,\nu,G)$ such that all $\Sigma_A$ are finite and all $\nu_A$
  are uniform measures. Again, this is because~\eqref{eq:partiteGsimcN} (\eqref{eq:partiteGsimcNag},
  respectively) follows by taking
  \begin{align*}
    R_\cN(m)
    & \df
    \prod_{f\in r_k(m)}\lvert\Sigma_{\dom(f)}\rvert
    =
    \prod_{A\in r(k)} \lvert\Sigma_A\rvert^{m^{\lvert A\rvert}},
    &
    \cN(x,b) & \df G(x,\phi_m(b)), 
  \end{align*}
  where $\phi_m\colon[R_\cN(m)]\to\prod_{f\in r_k(m)}\Sigma_{\dom(f)}$ is any fixed bijection.
\end{remark}

The lemma below, which is the partite analogue of Lemma~\ref{lem:flexibility}, turns out to be much
easier to prove.

\begin{lemma}[Flexibility of $\ell_{0/1}$]
  Both the agnostic and the non-agnostic $k$-partite $0/1$-loss functions are flexible when
  $\Lambda$ is finite.
\end{lemma}

\begin{proof}
  Let $\Sigma_{[k]}\df\Lambda$, equipped with discrete $\sigma$-algebra, let all other $\Sigma_A$
  have a single point, let each $\nu_A$ be the uniform measure on $\Sigma_A$ and let
  $G\in\cF_k(\Omega\otimes\Sigma,\Lambda)$ be given by
  \begin{align*}
    G(x,z) & \df z_{1^{[k]}}
    \qquad (x\in\cE_1(\Omega), z\in\cE_1(\Sigma)),
  \end{align*}
  where $1^{[k]}$ is the unique function $[k]\to[1]$. Then for every $x\in\cE_1(\Omega)$, if
  $\rn{z}\sim\nu^1$, then $G(x,\rn{z})$ is uniformly distributed in $\Lambda$, so for the
  non-agnostic $k$-partite $0/1$-loss function, we have
  \begin{align*}
    \EE_{\rn{z}\sim\nu^1}[\ell_{0/1}(x,y,G(x,\rn{z}))]
    & =
    1 - \frac{1}{\lvert\Lambda\rvert}
    \qquad (x\in\cE_1(\Omega), y\in\Lambda)
  \end{align*}
  and for the agnostic $k$-partite $0/1$-loss function with respect to
  $\cH\subseteq\cF_k(\Omega,\Lambda)$, we have
  \begin{align*}
    \EE_{\rn{z}\sim\nu^1}[\ell_{0/1}(H,x,G(x,\rn{z}))]
    & =
    1 - \frac{1}{\lvert\Lambda\rvert}
    \qquad (H\in\cH, x\in\cE_1(\Omega)).
  \end{align*}
  Since the right-hand sides of the equations above do not depend on $y$ and $H$ (in fact, they do
  not even depend on $x$), we get~\eqref{eq:partiteEEindepy} and~\eqref{eq:partiteEEindepH}. The
  existence of the corresponding $\cN$ then follows by Remark~\ref{rmk:partiteflexibilityfinite}.
\end{proof}

\notoc\subsection{Partite-specific definitions}
\label{subsec:partitespecific}

This section contains very useful definitions in the partite setting. Even though some of them have
non-partite counter-parts, we have not stated them in Section~\ref{sec:defs} since they are either
not used or not needed to state our main theorems. In particular, the first definition of empirical
loss below will actually have a counter-part in the non-partite setting, but since it is a bit more
technical than the partite case and not necessary to state our main results, we postpone it from
Section~\ref{sec:defs} to Section~\ref{sec:derand}.

\begin{definition}[Empirical loss]\label{def:emploss}
  Let $k\in\NN_+$, let $\Omega$ be a Borel $k$-partite template, let $\Lambda$ be a non-empty Borel
  space, let $\cH\subseteq\cF_k(\Omega,\Lambda)$ be a $k$-partite hypothesis class, let
  $V_1,\ldots,V_k$ be finite non-empty sets, let
  $(x,y)\in\cE_{V_1,\ldots,V_k}(\Omega)\times\Lambda^{\prod_{i=1}^k V_i}$ be a
  $(V_1,\ldots,V_k)$-sample and let $H\in\cH$.
  \begin{enumdef}
  \item The \emph{empirical loss} (or \emph{empirical risk}) of $H$ with respect to $(x,y)$ and a
    $k$-partite loss function $\ell\colon\cE_1(\Omega)\times\Lambda\times\Lambda\to\RR_{\geq 0}$ is
    \begin{align*}
      L_{x,y,\ell}(H)
      & \df
      \frac{1}{\prod_{i=1}^k\lvert V_i\rvert}
      \sum_{\alpha\in\prod_{i=1}^k V_i}
      \ell(\alpha^*(x), H^*_{V_1,\ldots,V_k}(x)_\alpha, y_\alpha).
    \end{align*}
  \item\label{def:emploss:ag} The \emph{empirical loss} (or \emph{empirical risk}) of $H$ with
    respect to $(x,y)$ and a $k$-partite agnostic loss function
    $\ell\colon\cH\times\cE_1(\Omega)\times\Lambda\to\RR_{\geq 0}$ is
    \begin{align*}
      L_{x,y,\ell}(H)
      & \df
      \frac{1}{\prod_{i=1}^k \lvert V_i\rvert}
      \sum_{\alpha\in\prod_{i=1}^k V_i}
      \ell(H, \alpha^*(x), y_\alpha).
    \end{align*}
  \item We say that a learning algorithm $\cA$ is an \emph{empirical risk minimizer} for an
    (agnostic or not) loss function $\ell$ if
    \begin{align}\label{eq:empriskmin}
      L_{x,y,\ell}(\cA(x,y))
      & =
      \inf_{H\in\cH} L_{x,y,\ell}(H)
    \end{align}
    for every $m\in\NN_+$ and every $([m],\ldots,[m])$-sample $(x,y)$.
  \end{enumdef}
\end{definition}

\begin{remark}\label{rmk:empriskmin}
  Note that empirical risk minimizers might not exist due to the infimum in~\eqref{eq:empriskmin}
  not being attained. Fortunately, it will be clear from the proofs that in terms of learnability,
  it will be enough to consider almost empirical risk minimizers in the sense
  that~\eqref{eq:empriskmin} holds with an extra additive term $f(m)$ on the right-hand side for
  some function $f\colon\NN_+\to\RR_{\geq 0}$ with $\lim_{m\to\infty} f(m) = 0$.

  However, even the existence of such almost empirical risk minimizers might not be guaranteed due
  to measurability issues if the agnostic loss function and hypothesis class are too wild.

  Nevertheless, in most applications, the fact that algorithms for almost empirical risk
  minimization are (efficiently) implemented implicitly implies measurability.
\end{remark}

\begin{definition}[Uniform convergence]\label{def:UC}
  Let $k\in\NN_+$, let $\Omega$ and $\Omega'$ be Borel $k$-partite templates, let
  $\mu\in\Pr(\Omega)$, $\mu'\in\Pr(\Omega')$ be probability $k$-partite templates, let
  $\cH\subseteq\cF_k(\Omega,\Lambda)$ be a $k$-partite hypothesis class, let
  $F\in\cF_k(\Omega\otimes\Omega',\Lambda)$ and let
  $\ell\colon\cH\times\cE_1(\Omega)\times\Lambda\to\RR_{\geq 0}$ be a $k$-partite agnostic loss
  function.
  \begin{enumdef}
  \item A $(V_1,\ldots,V_k)$-sample
    $(x,y)\in\cE_{V_1,\ldots,V_k}(\Omega)\times\Lambda^{\prod_{i=1}^k V_i}$ is
    \emph{$\epsilon$-representative} with respect to $\cH$, $\mu$, $\mu'$, $F$ and $\ell$ if
    \begin{align*}
      \lvert L_{x,y,\ell}(H) - L_{\mu,\mu',F,\ell}(H)\rvert & \leq \epsilon
    \end{align*}
    for every $H\in\cH$.
  \item We say that $\cH$ has the \emph{uniform convergence property} with respect to $\ell$ if
    there exists a function $m^{\UC}_{\cH,\ell}\colon(0,1)^2\to\RR_{\geq 0}$ such that for every
    $\epsilon,\delta\in(0,1)$, every $\mu\in\Pr(\Omega)$, every Borel $k$-partite template
    $\Omega'$, every $\mu'\in\Pr(\Omega')$ and every $F\in\cF_k(\Omega\otimes\Omega',\Lambda)$, we
    have
    \begin{align*}
      \PP_{(\rn{x},\rn{x'})\sim(\mu\otimes\mu')^m}[(\rn{x},F^*_m(\rn{x},\rn{x'}))
        \text{ is $\epsilon$-representative w.r.t.\ $\cH$, $\mu$, $\mu'$, $F$ and $\ell$}]
      & \geq 1 - \delta,
    \end{align*}
    for every integer $m\geq m^{\UC}_{\cH,\ell}(\epsilon,\delta)$.
  \end{enumdef}
\end{definition}

Note that uniform convergence can also be defined for arbitrary functions
$\ell\colon\cH\times\cE_1(\Omega)\times\Lambda\to\RR$ that can take negative values in an
analogous way.

If such an $\ell$ is bounded, that is, if
\begin{align*}
  \lVert\ell\rVert_\infty
  & \df
  \sup_{\substack{H\in\cH\\x\in\cE_1(\Omega)\\y\in\Lambda}} \lvert\ell(H,x,y)\rvert
  <
  \infty,
\end{align*}
then uniform convergence with respect to $\ell$ is equivalent to uniform convergence with respect
to $\ell+\lVert\ell\rVert_\infty$, which is non-negative.

\medskip

We conclude this section with the natural ``partization'' construction that moves objects from the
non-partite to the partite setting.

\begin{definition}[Partization]\label{def:kpart}
  Let $k\in\NN_+$, let $\Omega$ be a Borel template and let $\Lambda$ be a non-empty Borel space.
  \begin{enumdef}
  \item\label{def:kpart:Omega} The \emph{$k$-partite version} of $\Omega$ is the Borel $k$-partite
    template $\Omega^{\kpart}$ given by $\Omega^{\kpart}_A\df\Omega_{\lvert A\rvert}$ ($A\in r(k)$).
  \item For $\mu\in\Pr(\Omega)$, the \emph{$k$-partite version} of $\mu$ is
    $\mu^{\kpart}\in\Pr(\Omega^{\kpart})$ given by $\mu^{\kpart}_A\df\mu_{\lvert A\rvert}$ ($A\in
    r(k)$).
  \item For a hypothesis $F\in\cF_k(\Omega,\Lambda)$, the \emph{$k$-partite version} of $F$ is the
    $k$-partite hypothesis $F^{\kpart}\in\cF_k(\Omega^{\kpart},\Lambda^{S_k})$ given by
    \begin{align*}
      F^{\kpart}(x) & \df F^*_k(\iota_{\kpart}(x)) \qquad (x\in\cE_1(\Omega^{\kpart})),
    \end{align*}
    where $\iota_{\kpart}\colon\cE_1(\Omega^{\kpart})\to\cE_k(\Omega)$ is given by
    \begin{align}\label{eq:iotakpart}
      \iota_{\kpart}(x)_A & \df x_{1^A} \qquad (x\in\cE_1(\Omega^{\kpart}), A\in r(k))
    \end{align}
    and $1^A$ is the unique function $A\to[1]$. It is important to note that the codomain of
    $F^{\kpart}$ is $\Lambda^{S_k}$ as opposed to $\Lambda$.
  \item\label{def:kpart:cH} For $\cH\in\cF_k(\Omega,\Lambda)$, the \emph{$k$-partite version} of
    $\cH$ is $\cH^{\kpart}\df\{H^{\kpart} \mid H\in\cH\}$, equipped with the pushforward
    $\sigma$-algebra of the one of $\cH$.

    Lemma~\ref{lem:kpartbasics}\ref{lem:kpartbasics:phi} below shows that $\iota_{\kpart}$ is a
    bijection, which in turn implies that $\cH\ni F\mapsto F^{\kpart}\in\cH^{\kpart}$ is a bijection
    (so singletons of $\cH^{\kpart}$ are indeed measurable) and that $\cH\mapsto\cH^{\kpart}$ is an
    injection. We denote by $\cH^{\kpart}\ni G\mapsto G^{\kpart,-1}\in\cH$ the inverse of $\cH\ni
    F\mapsto F^{\kpart}\in\cH^{\kpart}$ (note that there is no ambiguity regarding $\cH$ since $\cH\ni
    F\mapsto F^{\kpart}\in\cH^{\kpart}$ is the restriction of $\cF_k(\Omega,\Lambda)\ni F\mapsto
    F^{\kpart}\in\cF_k(\Omega,\Lambda)^{\kpart}$ to $\cH$).
  \item For a $k$-ary loss function over $\Lambda$, the \emph{$k$-partite version} of $\ell$ is
    $\ell^{\kpart}\colon\cE_1(\Omega)\times\Lambda^{S_k}\times\Lambda^{S_k}$ given by
    \begin{align*}
      \ell^{\kpart}(x,y,y') & \df \ell(\iota_{\kpart}(x),y,y')
      \qquad (x\in\cE_1(\Omega^{\kpart}), y,y'\in\Lambda^{S_k}).
    \end{align*}
  \item For a $k$-ary agnostic loss function
    $\ell\colon\cH\times\cE_k(\Omega)\times\Lambda^{S_k}\to\RR_{\geq 0}$, the \emph{$k$-partite
    version} of $\ell$ is
    $\ell^{\kpart}\colon\cH^{\kpart}\times\cE_1(\Omega^{\kpart})\times\Lambda^{S_k}\to\RR_{\geq 0}$
    be given by
    \begin{align*}
      \ell^{\kpart}(H, x, y) & \df \ell(H^{\kpart,-1}, \iota_{\kpart}(x), y)
      \qquad
      (H\in\cH^{\kpart}, x\in\cE_1(\Omega^{\kpart}), y\in\Lambda^{S_k})
    \end{align*}
  \end{enumdef}
\end{definition}

\section{Statements of main theorems}
\label{sec:mainthms}

In this section we state our main theorems.

The first theorem is the $k$-ary version of Theorem~\ref{thm:FTSL} and concerns equivalences that
can be obtained by starting from a non-partite hypothesis class along with a non-agnostic loss
function. We remind the reader that the rank hypothesis $\rk(\cH)\leq 1$ of Theorem~\ref{thm:kPAC}
is always satisfied when the hypothesis class comes from a family of $k$-hypergraphs or from a
family of structures in a finite relational language (see Remark~\ref{rmk:rk1}). Also, it is
straightforward to check that the $0/1$-loss function $\ell_{0/1}(x,y,y')\df\One[y\neq y']$ is
symmetric, separated and bounded, and Lemma~\ref{lem:flexibility} shows that $\ell_{0/1}$ is
flexible when $\Lambda$ is finite; hence $\ell_{0/1}$ satisfies all hypotheses of
Theorem~\ref{thm:kPAC} below.

\begin{restatable}{theorem}{thmkPAC}\label{thm:kPAC}
  Let $\Omega$ be a Borel template, let $\Lambda$ be a finite non-empty Borel space, let $\ell$ be a
  symmetric, separated, bounded and flexible $k$-ary loss function, let
  $\cH\subseteq\cF_k(\Omega,\Lambda)$ be a $k$-ary hypothesis class with $\rk(\cH)\leq 1$. Let also
  $\ell^{\ag}\colon\cH\times\cE_k(\Omega)\times\Lambda^{S_k}\to\RR_{\geq 0}$ be the $k$-ary agnostic
  loss function defined by
  \begin{align*}
    \ell^{\ag}(H,x,y) & \df \ell(x,H^*_k(x),y)
    \qquad (H\in\cH, x\in\cE_k(\Omega), y\in\Lambda^{S_k}).
  \end{align*}

  Then the following are equivalent:
  \begin{enumerate}[label={\arabic*.}, ref={(\arabic*)}]
  \item\label{thm:kPAC:VCN} $\VCN_k(\cH) < \infty$.
  \item\label{thm:kPAC:VCNkpart} $\VCN_k(\cH^{\kpart}) < \infty$.
  \item\label{thm:kPAC:UC} $\cH^{\kpart}$ has the uniform convergence property with respect to
    $\ell^{\kpart}$.
  \item\label{thm:kPAC:agPAC} $\cH$ is agnostically $k$-PAC learnable with respect to $\ell^{\ag}$.
  \item\label{thm:kPAC:agPACr} $\cH$ is agnostically $k$-PAC learnable with randomness with respect
    to $\ell^{\ag}$.
  \item\label{thm:kPAC:agPACkpart} $\cH^{\kpart}$ is agnostically $k$-PAC learnable with respect to
    $(\ell^{\ag})^{\kpart}$.
  \item\label{thm:kPAC:agPACrkpart} $\cH^{\kpart}$ is agnostically $k$-PAC learnable with randomness
    with respect to $(\ell^{\ag})^{\kpart}$.
  \item\label{thm:kPAC:PACkpart} $\cH^{\kpart}$ is $k$-PAC learnable with respect to
    $\ell^{\kpart}$.
  \item\label{thm:kPAC:PACrkpart} $\cH^{\kpart}$ is $k$-PAC learnable with randomness with respect
    to $\ell^{\kpart}$.
  \end{enumerate}

  Furthermore, the following are equivalent and are implied by any of the items above:
  \begin{enumerate}[resume*]
  \item\label{thm:kPAC:PAC} $\cH$ is $k$-PAC learnable with respect to $\ell$. 
  \item\label{thm:kPAC:PACr} $\cH$ is $k$-PAC learnable with randomness with respect to $\ell$.
  \end{enumerate}
\end{restatable}

One might wonder about obtaining an analogue of Theorem~\ref{thm:kPAC} above that starts in the
non-partite agnostic setting from the several implications proved in this document, as not every
agnostic loss function is of the form $\ell^{\ag}$ for some non-agnostic loss function
$\ell$. However, since the implication \ref{thm:kPAC:VCNkpart}$\implies$\ref{thm:kPAC:UC} of
Theorem~\ref{thm:kPAC} above (Proposition~\ref{prop:VCNdim->UC}) requires locality of the underlying
agnostic loss function, at the very least, we would have to restrict to agnostic loss functions that
are of the form $\ell^{\ag} + r$ for some non-agnostic loss function $\ell$ and some regularization
term $r$. On the other hand, since the implication
\ref{thm:kPAC:PACkpart}$\implies$\ref{thm:kPAC:VCNkpart} (Proposition~\ref{prop:partkPAC->VCN})
requires the separation of the underlying loss function, we would then have to require the
regularization term $r$ to be zero, effectively forcing our agnostic loss function to be of the form
$\ell^{\ag}$.

\medskip

The next theorem is the $k$-partite version of Theorem~\ref{thm:FTSL} and covers any partite
hypothesis class, even those that do not arise as partizations. In this case, there is no natural
construction to move to the non-partite setting. Again, the rank hypothesis $\rk(\cH)\leq 1$ of
Theorem~\ref{thm:kPACkpart} is always satisfied when the hypothesis class comes from a family of
$k$-partite $k$-hypergraphs or from a family of $k$-partite structures in a finite relational
language (see Remark~\ref{rmk:partiterk1}). Finally, it is straightforward to check that the
(partite) $0/1$-loss function $\ell_{0/1}(x,y,y')\df\One[y\neq y']$ is separated and bounded, that
is, it satisfies all hypotheses of Theorem~\ref{thm:kPACkpart} below.

\begin{restatable}{theorem}{thmkPACkpart}\label{thm:kPACkpart}
  Let $\Omega$ be a Borel $k$-partite template, let $\Lambda$ be a finite non-empty Borel space, let
  $\ell$ be a separated and bounded $k$-partite loss function, let
  $\cH\subseteq\cF_k(\Omega,\Lambda)$ be a $k$-partite hypothesis class with $\rk(\cH)\leq 1$. Let
  also $\ell^{\ag}\colon\cH\times\cE_1(\Omega)\times\Lambda\to\RR_{\geq 0}$ be the $k$-partite
  agnostic loss function defined by
  \begin{align*}
    \ell^{\ag}(H,x,y) & \df \ell(x, H(x), y)
    \qquad (H\in\cH, x\in\cE_1(\Omega), y\in\Lambda).
  \end{align*}

  Then the following are equivalent:
  \begin{enumerate}[label={\arabic*.}, ref={(\arabic*)}]
  \item\label{thm:kPACkpart:VCN} $\VCN_k(\cH) < \infty$.
  \item\label{thm:kPACkpart:UC} $\cH$ has the uniform convergence property with respect to $\ell$.
  \item\label{thm:kPACkpart:agPAC} $\cH$ is agnostically $k$-PAC learnable with respect to
    $\ell^{\ag}$.
  \item\label{thm:kPACkpart:agPACr} $\cH$ is agnostically $k$-PAC learnable with randomness with
    respect to $\ell^{\ag}$.
  \item\label{thm:kPACkpart:PAC} $\cH$ is $k$-PAC learnable with respect to $\ell$.
  \item\label{thm:kPACkpart:PACr} $\cH$ is $k$-PAC learnable with randomness with respect to $\ell$.
  \end{enumerate}
\end{restatable}

For completeness, we record how the classic PAC characterization, Theorem~\ref{thm:FTSL} follows
from Theorem~\ref{thm:kPACkpart}.

\begin{corollary}
  Theorem~\ref{thm:FTSL} follows from Theorem~\ref{thm:kPACkpart} with $k=1$.
\end{corollary}

\begin{proof}
  When $k=1$, a Borel $1$-partite template $\Omega$ amounts simply to a non-empty Borel space
  $\Omega_{\{1\}}$, so $1$-PAC learning coincides with classic PAC learning and $1$-ary uniform
  convergence coincides with classic uniform convergence. Finally, $\VCN_1$-dimension coincides with
  Natarajan dimension (and is equivalent to $\VC$-dimension when $\lvert\Lambda\rvert=2$).
\end{proof}

The theorems above cover hypothesis classes of rank at most $1$, in particular, classes of graphs,
hypergraphs and structures in finite relational languages. As we have seen, high-arity PAC learning
still makes sense in higher rank and informally captures learning randomness (see
Section~\ref{subsec:highervar}). Already our results give at least one equivalence in high rank:
\begin{restatable}{theorem}{thmagkPAC}\label{thm:agkPAC}
  Let $\Omega$ be a Borel template, let $\Lambda$ be a non-empty Borel space, let
  $\cH\subseteq\cF_k(\Omega,\Lambda)$ be a $k$-ary hypothesis class and let
  $\ell\colon\cH\times\cE_k(\Omega)\times\Lambda^{S_k}\to\RR_{\geq 0}$ be a symmetric, bounded and
  flexible $k$-ary agnostic loss function.

  Then the following are equivalent:
  \begin{enumerate}[label={\arabic*.}, ref={(\arabic*)}]
  \item\label{thm:agkPAC:agPAC} $\cH$ is agnostically $k$-PAC learnable with respect to $\ell$.
  \item\label{thm:agkPAC:agPACr} $\cH$ is agnostically $k$-PAC learnable with randomness with
    respect to $\ell$.
  \item\label{thm:agkPAC:agPACkpart} $\cH^{\kpart}$ is agnostically $k$-PAC learnable with respect
    to $\ell^{\kpart}$.
  \item\label{thm:agkPAC:agPACrkpart} $\cH^{\kpart}$ is agnostically $k$-PAC learnable with
    randomness with respect to $\ell^{\kpart}$.
  \end{enumerate}
\end{restatable}

\section{Initial reductions}
\label{sec:initialred}

As a first test of the definitions, in this section we will see that many basic desired properties
follow easily.

\subsection{Variables of order higher than $k$}
\label{subsec:highvar}

The next proposition says that we can ignore all variables indexed by subsets of size greater than $k$.

\begin{proposition}[Dummy variables]\label{prop:ho}
  Let $\Omega$ be a Borel template, let $k\in\NN_+$, let $\Lambda$ be a non-empty Borel space and
  let $\cH\subseteq\cF_k(\Omega,\Lambda)$ be a $k$-ary hypothesis class.
  \begin{enumerate}
  \item If $\cA$ is a (randomized, respectively) $k$-PAC learner for $\cH$ with respect to a $k$-ary
    loss function $\ell$, then there exists a (randomized, respectively) $k$-PAC learner $\cA'$ for
    $\cH$ with respect to $\ell$ with $m^{\PAC}_{\cH,\ell,\cA'}=m^{\PAC}_{\cH,\ell,\cA}$
    ($m^{\PACr}_{\cH,\ell,\cA'}=m^{\PACr}_{\cH,\ell,\cA}$ and $R_{\cA'}=R_\cA$, respectively) and that does
    not depend on variables indexed by sets of size greater than $k$ in the sense that for every
    $x,x'\in\cE_m(\Omega)$ with $x_C = x'_C$ whenever $\lvert C\rvert\leq k$, we have $\cA'(x,\place)
    = \cA'(x',\place)$.
  \item If $\cA$ is a (randomized, respectively) agnostic $k$-PAC learner for $\cH$ with respect to a
    $k$-ary agnostic loss function $\ell$, then there exists a (randomized, respectively) agnostic
    $k$-PAC learner $\cA'$ for $\cH$ with respect to $\ell$ with
    $m^{\agPAC}_{\cH,\ell,\cA'}=m^{\agPAC}_{\cH,\ell,\cA}$
    ($m^{\agPACr}_{\cH,\ell,\cA'}=m^{\agPACr}_{\cH,\ell,\cA}$ and $R_{\cA'}=R_\cA$, respectively) and that
    does not depend on variables indexed by sets of size greater than $k$ in the sense that for
    every $x,x'\in\cE_m(\Omega)$ with $x_C = x'_C$ whenever $\lvert C\rvert\leq k$, we have
    $\cA'(x,\place) = \cA'(x',\place)$.
  \end{enumerate}
\end{proposition}

\begin{proof}
  All four cases, agnostic or not, with randomness or not, have completely analogous proofs, so we
  only show the case that would naively seem to be the most complex: the agnostic with randomness
  case.

  For each $i\in\NN_+$ with $i > k$, fix a point $z^i\in X_i$ and define
  \begin{align*}
    \cA'(x,y,b)
    & \df
    \cA(\widehat{x},y,b)
    \qquad (x\in\cE_m(\Omega), y\in\Lambda^{([m])_k}, b\in [R_\cA(m)]),
  \end{align*}
  where
  \begin{align*}
    \widehat{x}_C & \df
    \begin{dcases*}
      x_C, & if $\lvert C\rvert\leq k$,\\
      z^{\lvert C\rvert}, & if $\lvert C\rvert > k$.
    \end{dcases*}
  \end{align*}

  Given $\mu\in\Pr(\Omega)$, we let $\widehat{\mu}\in\Pr(\Omega)$ be given by
  \begin{align*}
    \widehat{\mu}_i & \df
    \begin{dcases*}
      \mu_i, & if $i\leq k$,\\
      \delta_{z^i}, & if $i > k$,
    \end{dcases*}
  \end{align*}
  where $\delta_a$ is the Dirac delta concentrated on $a$.

  Given further $H\in\cF_k(\Omega,\Lambda)$, a Borel template $\Omega'$,
  $F\in\cF_k(\Omega\otimes\Omega',\Lambda)$, note that for every $x\in\cE_m(\Omega)$ and every
  $x'\in\cE_m(\Omega')$, we have
  \begin{align*}
    H^*_m(x) & = H^*_m(\widehat{x}), &
    F^*_m(x,x') & = F^*_m(\widehat{x},x'),
  \end{align*}
  simply because these functions do not depend on coordinates indexed by sets of size greater than
  $k$. This immediately implies that if $\mu\in\Pr(\Omega)$ and $\mu'\in\Pr(\Omega')$, then
  \begin{align*}
    L_{\mu,\mu',F,\ell}(H) & = L_{\widehat{\mu},\mu',F,\ell}(H),
  \end{align*}
  which in turn implies
  \begin{align*}
    \inf_{H\in\cH} L_{\mu,\mu',F,\ell}(H) & = \inf_{H\in\cH} L_{\widehat{\mu},\mu',F,\ell}(H).
  \end{align*}
  Let $I$ be the infimum above and note that for $\epsilon,\delta\in(0,1)$ and an integer $m\geq
  m^{\agPACr}_{\cH,\ell,\cA}(\epsilon,\delta)$, we have
  \begin{align*}
    & \!\!\!\!\!\!
    \PP_{(\rn{x},\rn{x'})\sim(\mu\otimes\mu')^m}[\EE_{\rn{b}\sim U(R_{\cA'}(m))}[
        L_{\mu,\mu',F,\ell}(\cA'(\rn{x},F^*_m(\rn{x},\rn{x'}),\rn{b}))
      ] \leq I + \epsilon
    ]
    \\
    & =
    \PP_{(\rn{x},\rn{x'})\sim(\mu\otimes\mu')^m}[\EE_{\rn{b}\sim U(R_\cA(m))}[
        L_{\widehat{\mu},\mu',F,\ell}(\cA(\widehat{\rn{x}},F^*_m(\widehat{\rn{x}},\rn{x'}),\rn{b}))
      ] \leq I + \epsilon
    ]
    \\
    & =
    \PP_{(\rn{w},\rn{x'})\sim(\widehat{\mu}\otimes\mu')^m}[\EE_{\rn{b}\sim U(R_\cA(m))}[
        L_{\widehat{\mu},\mu',F,\ell}(\cA(\rn{w},F^*_m(\rn{w},\rn{x'}),\rn{b}))
      ] \leq I + \epsilon
    ]
    \geq
    1 - \delta,
  \end{align*}
  where the second equality follows since $x\mapsto\widehat{x}$ is measure-preserving with respect
  to $\mu^m$ and $\widehat{\mu}^m$ and the inequality follows since $\cA$ is a randomized agnostic
  $k$-PAC learner.
\end{proof}

\subsection{Increasing codomain}

The next proposition says that by increasing the codomain of all hypotheses without changing them as
functions, we preserve the $\VCN_k$-dimension and we do not affect non-agnostic $k$-PAC learnability
(with randomness or not). For agnostic $k$-PAC learnability, we can only prove that agnostic
learning must be harder or equal after increasing the codomain; the other direction will require the
agnostic loss function to be flexible and will be covered in Proposition~\ref{prop:neutsymb} (the
reason why the other direction is not simple is because even though hypotheses do not use the extra
symbols, an agnostic adversary could use them; in the non-agnostic setting, this is indirectly ruled
out by the realizability assumption).

\begin{proposition}[Increasing codomain]\label{prop:cod}
  Let $k\in\NN_+$, let $\Omega$ be a Borel ($k$-partite, respectively) template, let $\Lambda$ be a
  non-empty Borel space, let $\ell$ be a $k$-ary ($k$-partite, respectively) loss function over
  $\Lambda$ and let $\cH\subseteq\cF_k(\Omega,\Lambda)$ be a $k$-ary ($k$-partite, respectively)
  hypothesis class.

  Let $\Lambda'$ be a Borel space containing $\Lambda$ as a measurable subspace in the sense that
  the inclusion map $\iota\colon\Lambda\to\Lambda'$ is measurable and has measurable image. Let
  $\cH'$ be $\cH$ when we increase the codomain of its elements to $\Lambda'$; formally, let
  \begin{align*}
    \cH' & \df \{\iota\comp H \mid H\in\cH\}.
  \end{align*}
  Let also $\iota_\cH\colon\cH\to\cH'$ be the bijection given by $\iota_\cH(H)\df\iota\comp H$. We
  also view $\Lambda^{([m])_k}$ ($\Lambda^{[m]^k}$, respectively) as a subset of
  $(\Lambda')^{([m])_k}$ ($(\Lambda')^{[m]^k}$, respectively) naturally.
  
  Then the following hold:
  \begin{enumerate}
  \item\label{prop:cod:VCN} $\VCN_k(\cH)=\VCN_k(\cH')$.
  \item\label{prop:cod:learn} Let $\ell$ be a $k$-ary ($k$-partite, respectively) loss function over
    $\Lambda$ and let $\ell'$ be a $k$-ary ($k$-partite, respectively) loss function over $\Lambda'$
    extending $\ell$ such that $\ell'(x,y,y')>0$ whenever $y\in\Lambda^{S_k}$ and
    $y'\in(\Lambda')^{S_k}\setminus\Lambda^{S_k}$ ($y\in\Lambda$ and
    $y'\in\Lambda'\setminus\Lambda$, respectively). Then $\cH$ is $k$-PAC learnable (with
    randomness, respectively) with respect to $\ell$ if and only if $\cH'$ is $k$-PAC learnable
    (with randomness, respectively) with respect to $\ell'$.

    Furthermore, in all directions we have $m^{\PAC}_{\cH',\ell',\cA'}=m^{\PAC}_{\cH,\ell,\cA}$
    ($m^{\PACr}_{\cH',\ell',\cA'}=m^{\PACr}_{\cH,\ell,\cA}$ and $R_{\cA'}=R_\cA$, respectively).
  \item\label{prop:cod:agPAC} Let $\ell$ be an agnostic $k$-ary ($k$-partite, respectively) loss
    function over $\Lambda$ and let $\ell'$ and agnostic $k$-ary ($k$-partite, respectively) loss
    function over $\Lambda'$ extending $\ell$ (in the sense $\ell'(H,x,y)=\ell(\iota_\cH(H),x,y)$
    whenever $y\in\Lambda^{S_k}$). If $\cH'$ is agnostically $k$-PAC learnable (with randomness,
    respectively) with respect to $\ell'$, then $\cH$ is agnostically $k$-PAC learnable with respect
    to $\ell$ with $m^{\agPAC}_{\cH,\ell,\cA}=m^{\agPAC}_{\cH',\ell',\cA'}$
    ($m^{\agPACr}_{\cH,\ell,\cA}=m^{\agPACr}_{\cH',\ell',\cA'}$ and $R_\cA=R_{\cA'}$, respectively).
  \end{enumerate}
\end{proposition}

\begin{proof}
  Item~\ref{prop:cod:VCN} follows straightforwardly from definitions.

  \medskip

  For item~\ref{prop:cod:learn}, it is straightforward to check that if $\cA'$ is a (randomized,
  respectively) $k$-PAC learner for $\cH'$, then $\cA\df\iota_\cH^{-1}\comp \cA'$ is a (randomized,
  respectively) $k$-PAC learner for $\cH$ with $m^{\PAC}_{\cH,\ell,\cA}=m^{\PAC}_{\cH',\ell',\cA'}$
  ($m^{\PACr}_{\cH,\ell,\cA}=m^{\PACr}_{\cH',\ell',\cA'}$ and $R_\cA=R_{\cA'}$, respectively).

  The converse has a small technicality: if $\cA$ is a (randomized, respectively) $k$-PAC learner
  for $\cH$, then we might not be able to run $\cA$ directly on an input for $\cA'$ as it may have
  entries in $\Lambda'\setminus\Lambda$. To fix this, we let $y_0\in\Lambda$ be arbitrary and given
  an input for $\cA'$, we first change all its entries in $\Lambda'\setminus\Lambda$ to $y_0$ and
  run $\cA$ on the result.

  Now, for $F'\in\cF_k(\Omega,\Lambda')$, let $F\in\cF_k(\Omega,\Lambda)$ be obtained by changing
  all values of $F'$ in $\Lambda'\setminus\Lambda$ to $y_0$ and note that if $F'$ is realizable in
  $\cH'$ with respect to $\ell'$ and $\mu\in\Pr(\Omega)$, then for every $m\in\NN$, $(F')^*_m$ can
  only output values containing points in $\Lambda'\setminus\Lambda$ in a set of $\mu^m$-measure
  zero (this is because $\ell(x,y,y')>0$ whenever $y\in\Lambda^{S_k}$ and
  $y'\in(\Lambda')^{S_k}\setminus\Lambda^{S_k}$ ($y\in\Lambda$ and $y'\in\Lambda'\setminus\Lambda$,
  respectively)). In particular, this gives $L_{\mu,F',\ell'}(\iota_\cH(H)) = L_{\mu,F,\ell}(H)$ for
  every $H\in\cH$, which in turn implies that $F$ is realizable in $\cH$ with respect to $\mu$ and
  $\ell$. It is now straightforward to check that the fact that $\cA$ can learn $F$ implies that
  $\cA'$ can learn $F'$.

  \medskip

  We now prove item~\ref{prop:cod:agPAC}. We will prove only the case with randomness, as
  the case without randomness amounts to $R_{\cA'}\equiv 1$.

  Let $\cA'$ be a randomized agnostic $k$-PAC learner for $\cH'$ with respect to $\ell'$ and define
  $\cA$ with $R_\cA\df R_{\cA'}$ by $\cA\df\iota_\cH^{-1}\comp\cA'$, that is, we let
  \begin{align*}
    \cA(x,y,b) & \df \iota_\cH^{-1}(\cA'(x,y,b))
  \end{align*}
  for every $m\in\NN$, every $x\in\cE_m(\Omega)$, every $y\in\Lambda^{([m])_k}$
  ($y\in\Lambda^{[m]^k}$, respectively) and every $b\in[R_\cA(m)]$.

  Let $\Omega'$ be a Borel ($k$-partite, respectively) template, let $\mu'\in\Pr(\Omega')$ and let
  $F\in\cF_k(\Omega\otimes\Omega',\Lambda)$. Define $F'\in\cF_k(\Omega\otimes\Omega',\Lambda')$ by
  $F'\df\iota\comp F$ and note that for every $H\in\cH$, every $m\in\NN$, every $x\in\cE_m(\Omega)$
  and every $x'\in\cE_m(\Omega')$, we have
  \begin{align*}
    (F')^*_m(x,x') = F^*_m(x,x'),
  \end{align*}
  which in particular means that all entries of $(F')^*_m(x)$ are in $\Lambda$. Thus, we get
  \begin{align*}
    L_{\mu,\mu',F',\ell'}(\iota_\cH(H)) & = L_{\mu,\mu',F,\ell}(H)
    \qquad (H\in\cH),
  \end{align*}
  Since $\iota_\cH$ is bijective, the above also implies that
  \begin{align*}
    \inf_{H\in\cH'} L_{\mu,\mu',F',\ell'}(H) & =
    \inf_{H\in\cH} L_{\mu,\mu',F,\ell}(H).
  \end{align*}
  
  It is now straightforward to check that the agnostic $k$-PAC learnability guarantee of $\cA'$ for
  $(\mu,\mu',F',\ell')$ yields the agnostic $k$-PAC learnability guarantee needed for $\cA$ for
  $(\mu,\mu',F,\ell)$.
\end{proof}

\subsection{Agnostic versus non-agnostic}
\label{subsec:agnonag}

In this subsection, we reduce non-agnostic $k$-PAC learning to agnostic $k$-PAC learning (in both
non-partite and partite settings).

\begin{proposition}[Agnostic to non-agnostic]\label{prop:agPAC->PAC}
  Let $k\in\NN_+$, let $\Omega$ be a Borel ($k$-partite, respectively) template, let $\Lambda$ be a
  non-empty Borel space, let $\ell$ be a $k$-ary ($k$-partite, respectively) loss function over
  $\Lambda$ and let $\cH\subseteq\cF_k(\Omega,\Lambda)$ be a $k$-ary ($k$-partite, respectively)
  hypothesis class.

  Define the $k$-ary ($k$-partite, respectively) agnostic loss function $\ell^{\ag}$ over $\Lambda$ with
  respect to $\cH$ by
  \begin{align*}
    \ell^{\ag}(H,x,y) & \df \ell(x,H^*_k(x),y) & (x\in\cE_k(\Omega), y\in\Lambda^{S_k})
    \intertext{in the $k$-ary case and by}
    \ell^{\ag}(H,x,y) & \df \ell(x,H(x),y) & (x\in\cE_1(\Omega), y\in\Lambda)
  \end{align*}
  in the $k$-partite case. Then the following hold:
  \begin{enumerate}
  \item\label{prop:agPAC->PAC:local} $\ell^{\ag}$ is local.
  \item\label{prop:agPAC->PAC:bound} $\lVert\ell^{\ag}\rVert_\infty\leq\lVert\ell\rVert_\infty$.
  \item\label{prop:agPAC->PAC:flexible} If $\ell$ is flexible and $(\Sigma,\nu,G,\cN)$ witnesses its
    flexibility, then $\ell^{\ag}$ is also flexible with witness $(\Sigma,\nu,G,\cN)$.
  \item\label{prop:agPAC->PAC:symm} In the non-partite case, if $\ell$ is symmetric, then $\ell^{\ag}$ is
    symmetric.
  \item\label{prop:agPAC->PAC:learn} If $\cH$ is agnostically $k$-PAC learnable (with randomness,
    respectively) with respect to $\ell^{\ag}$, then $\cH$ is $k$-PAC learnable (with randomness,
    respectively) with respect to $\ell$; more precisely, any (randomized, respectively) agnostic
    $k$-PAC learner $\cA$ for $\cH$ with respect to $\ell^{\ag}$ is a (randomized, respectively) $k$-PAC
    learner for $\cH$ with respect to $\ell$ with $m^{\PAC}_{\cH,\ell,\cA}=m^{\agPAC}_{\cH,\ell^{\ag},\cA}$
    ($m^{\PACr}_{\cH,\ell,\cA}=m^{\agPACr}_{\cH,\ell^{\ag},\cA}$, respectively).
  \end{enumerate}
\end{proposition}

\begin{proof}
  Item~\ref{prop:agPAC->PAC:local} follows directly from the definition of locality (using a zero
  regularization term), item~\ref{prop:agPAC->PAC:bound} is obvious and
  item~\ref{prop:agPAC->PAC:flexible} also follows directly from the definition of flexibility.

  \medskip

  Item~\ref{prop:agPAC->PAC:symm} is easily checked: for $H\in\cH$, $x\in\cE_k(\Omega)$, $y\in\Lambda^{S_k}$
  and $\sigma\in S_k$, we have
  \begin{align*}
    \ell^{\ag}(H,\sigma^*(x),\sigma^*(y))
    & =
    \ell(\sigma^*(x),H^*_k(\sigma^*(x)),\sigma^*(y))
    =
    \ell(\sigma^*(x),\sigma^*(H^*_k(x)),\sigma^*(y))
    \\
    & =
    \ell(x,H^*_k(x),y)
    =
    \ell^{\ag}(H,x,y),
  \end{align*}
  where the second equality follows from Lemma~\ref{lem:F*Vequiv}.

  \medskip

  For item~\ref{prop:agPAC->PAC:learn}, it suffices to prove only the randomized case as the non-randomized
  case amounts to the case when $R_\cA\equiv 1$.

  Let $\Omega'$ be a Borel ($k$-partite, respectively) template and for $F\in\cF_k(\Omega,\Lambda)$ realizable,
  define $F'\in\cF_k(\Omega\otimes\Omega',\Lambda)$ by $F'(x,x')\df F(x)$ via the natural
  identification of $\cE_k(\Omega\otimes\Omega')$ with $\cE_k(\Omega)\times\cE_k(\Omega')$
  ($\cE_1(\Omega\otimes\Omega')$ with $\cE_1(\Omega)\times\cE_1(\Omega')$, respectively).

  Note that for $\mu\in\Pr(\Omega)$ and $\mu'\in\Pr(\Omega')$, we have
  \begin{align*}
    L_{\mu,\mu',F',\ell^{\ag}}(H) & = L_{\mu,F,\ell}(H),
  \end{align*}
  which in particular implies
  \begin{align*}
    \inf_{H\in\cH} L_{\mu,\mu',F',\ell^{\ag}}(H)
    & =
    \inf_{H\in\cH} L_{\mu,F,\ell}(H)
    =
    0,
  \end{align*}
  where the last equality follows since $F$ is realizable.

  Now note that for $x\in\cE_m(\Omega)$ and $x'\in\cE_m(\Omega')$, we have
  \begin{align*}
    (x, (F')^*_m(x,x')) & = (x, F^*_m(x)),
  \end{align*}
  so if $(\rn{x},\rn{x'})\sim(\mu\otimes\mu')^m$, $\cA$ is a randomized agnostic $k$-PAC learner with
  respect to $\ell^{\ag}$ and $\rn{b}$ picked uniformly at random in $[R_\cA(m)]$, then
  \begin{align*}
    \cA(\rn{x}, F^*_m(\rn{x}),\rn{b})
    & =
    \cA(\rn{x}, (F')^*_m(\rn{x},\rn{x'}),\rn{b}),
  \end{align*}
  so
  \begin{align*}
    \PP_{\rn{x}}[\EE_{\rn{b}}[L_{\mu,F,\ell}(\cA(\rn{x}, F^*_m(\rn{x})))] \leq \epsilon]
    & =
    \PP_{\rn{x},\rn{x'}}[\EE_{\rn{b}}[L_{\mu,\mu',F',\ell^{\ag}}(\cA(\rn{x}, (F')^*_m(\rn{x},\rn{x'})))]
      \leq \epsilon]
    \geq
    1 - \delta,
  \end{align*}
  where the inequality follows from the agnostic $k$-PAC learning guarantee of $\cA$ for
  $\ell^{\ag}$. Hence $\cA$ is a randomized $k$-PAC learner for $\cH$ with respect to $\ell$ with
  $m^{\PACr}_{\cH,\ell,\cA}=m^{\agPACr}_{\cH,\ell^{\ag},\cA}$.
\end{proof}

\subsection{What if there is no real correlation?}

Recall that high-arity PAC gains its advantage from structured correlation. We conclude this section
by pointing out that if no real correlation is present, high-arity PAC becomes equivalent to classic
PAC.

More formally, a $k$-ary ($k$-partite, respectively) hypothesis $F\in\cF_k(\Omega,\Lambda)$ is
called \emph{$(k-1)$-independent}, if $F$ depends only on the coordinate indexed by $[k]$ (by the
unique function $[k]\to[1]$, respectively), that is, $F$ factors as $F=F'\comp\pi$ for the
projection $\pi$ onto this coordinate. We say that $\cH\subseteq\cF_k(\Omega,\Lambda)$ is
\emph{$(k-1)$-independent} if all of its elements are $(k-1)$-independent.

Note that if $F\in\cF_k(\Omega,\Lambda)$ is $(k-1)$-independent and $\rn{x}\sim\mu^m$ for some
$\mu\in\Pr(\Omega)$, then the entries of $F^*_m(\rn{x})$ are independent. In particular, this
implies that $k$-PAC learnability of $(k-1)$-independent family $\cH$ trivially reduces to the PAC
learnability of $\cH'\df\{H' \mid H\in\cH\}$, save for the fact that an $[m]$-sample in the $k$-ary
case corresponds to a sample of size $\binom{m}{k}$ for $\cH'$ and an $([m],\ldots,[m])$-sample in
the $k$-partite case corresponds to a sample of size $m^k$ for $\cH'$.

It is also straightforward to check that the $\VCN_k$-dimension of a $(k-1)$-independent family
$\cH$ is the same as the Natarajan dimension of $\cH'$.

\section{Derandomization}
\label{sec:derand}

In this section we show how to remove randomness from (agnostic) $k$-PAC learners. For this, we will
also need the non-partite version of the empirical loss. Let us note that the empirical loss in the
non-partite case has a technicality that is not present in the partite case: we need to specify an
``order choice'' for each of the $k$-sets to compute the loss.

\begin{definition}[Empirical loss]
  Let $\Omega$ be a Borel template, let $\Lambda$ be a non-empty Borel space, let
  $\cH\subseteq\cF_k(\Omega,\Lambda)$ be a $k$-ary hypothesis class, let $V$ be a finite non-empty
  set, let $(x,y)\in\cE_V(\Omega)\times\Lambda^{(V)_k}$ be a $V$-sample and let $H\in\cH$.
  \begin{enumdef}
  \item A \emph{($k$-ary) order choice} for $V$ is a sequence $\alpha=(\alpha_U)_{U\in\binom{V}{k}}$ such
    that for each $U\in\binom{V}{k}$, $\alpha_U\in (V)_k$ is an injection with $\im(\alpha_U) =
    U$.
  \item For an order choice $\alpha$ for $V$, we let
    $b_\alpha\colon\Lambda^{(V)_k}\to(\Lambda^{S_k})^{\binom{V}{k}}$ be the natural
    Borel-isomorphism given by
    \begin{align}\label{eq:balpha}
      (b_\alpha(y)_U)_\pi & \df y_{\alpha_U\comp\pi}
      \qquad\left(y\in\Lambda^{(V)_k}, U\in\binom{V}{k}, \pi\in S_k\right).
    \end{align}
  \item The \emph{empirical loss} (or \emph{empirical risk}) of $H$ with respect to $(x,y)$, a
    $k$-ary loss function $\ell\colon\cE_k(\Omega)\times\Lambda^{S_k}\times\Lambda^{S_k}\to\RR_{\geq
    0}$ and an order choice $\alpha$ is
    \begin{align*}
      L^\alpha_{x,y,\ell}(H)
      & \df
      \frac{1}{\binom{\lvert V\rvert}{k}}
      \sum_{U\in\binom{V}{k}} \ell(\alpha_U^*(x), b_\alpha(H^*_V(x))_U, b_\alpha(y)_U).
    \end{align*}
  \item The \emph{empirical loss} (or \emph{empirical risk}) of $H$ with respect to $(x,y)$, a $k$-ary
    agnostic loss function $\ell\colon\cH\times\cE_k(\Omega)\times\Lambda^{S_k}\to\RR_{\geq 0}$ and an
    order choice $\alpha$ is
    \begin{align*}
      L^\alpha_{x,y,\ell}(H)
      & \df
      \frac{1}{\binom{\lvert V\rvert}{k}}
      \sum_{U\in\binom{V}{k}} \ell(H, \alpha_U^*(x), b_\alpha(y)_U).
  \end{align*}
  \end{enumdef}
\end{definition}

To prove the derandomization result (Proposition~\ref{prop:derand}), we first need a small lemma on
the concentration of the empirical loss.

\begin{lemma}[Empirical loss concentration]\label{lem:empconc}
  Let $\Omega$ be a Borel ($k$-partite, respectively) template, let $\Lambda$ be a non-empty Borel
  space, let $\cH\subseteq\cF_k(\Omega,\Lambda)$ be a $k$-ary ($k$-partite, respectively) hypothesis
  class and let $\mu\in\Pr(\Omega)$. Let also
  \begin{align*}
    K & \df
    \begin{dcases*}
      k^2, & in the non-partite case,\\
      k, & in the partite case.
    \end{dcases*}
  \end{align*}

  Then the following hold:
  \begin{enumerate}
  \item\label{lem:empconc:nonag} If $\ell$ is a $k$-ary ($k$-partite, respectively) loss function
    and $F\in\cF_k(\Omega,\Lambda)$, then
    \begin{align}\label{eq:empconc:nonag}
      \PP_{\rn{x}\sim\mu^m}[
        \lvert L^\alpha_{\rn{x},F^*_m(\rn{x}),\ell}(H) - L_{\mu,F,\ell}(H)\rvert
        \geq\epsilon
      ]
      & \leq
      2\exp\left(-\frac{\epsilon^2\cdot m}{2\cdot K\cdot\lVert\ell\rVert_\infty^2}\right)
    \end{align}
    for every $m\in\NN_+$, every $\epsilon > 0$, every $H\in\cH$ and every $k$-ary order choice
    $\alpha$ for $[m]$ (in the $k$-partite case, $\alpha$ is omitted and $L^\alpha$ is replaced with
    $L$).
  \item\label{lem:empconc:ag} If $\ell$ is a $k$-ary ($k$-partite, respectively) loss function,
    $\Omega'$ is a Borel ($k$-partite, respectively) template, $\mu'\in\Pr(\Omega)$ and
    $F\in\cF_k(\Omega\otimes\Omega',\Lambda)$, then
    \begin{align*}
      \PP_{(\rn{x},\rn{x'})\sim(\mu\otimes\mu')^m}[
        \lvert L^\alpha_{\rn{x},F^*_m(\rn{x},\rn{x'}),\ell}(H) - L_{\mu,\mu',F,\ell}(H)\rvert
        \geq\epsilon
      ]
      & \leq
      2\exp\left(-\frac{\epsilon^2\cdot m}{2\cdot K\cdot\lVert\ell\rVert_\infty^2}\right)
    \end{align*}
    for every $m\in\NN_+$, every $\epsilon > 0$, every $H\in\cH$ and every $k$-ary order choice
    $\alpha$ for $[m]$ (in the $k$-partite case, $\alpha$ is omitted and $L^\alpha$ is replaced with
    $L$).
  \end{enumerate}
\end{lemma}

\begin{proof}
  We start with item~\ref{lem:empconc:nonag} in the non-partite case. For each $i\in\{0,\ldots,m\}$,
  let $\cF_i$ be the $\sigma$-algebra generated by $(\rn{x}_A \mid A\in r(i))$ and let
  \begin{align}\label{eq:empconc:Z}
    \rn{Z}_i & \df \EE[L^\alpha_{\rn{x},F^*_m(\rn{x}),\ell}(H) \given \cF_i]
  \end{align}
  so that $(\rn{Z}_i)_{i=0}^m$ is a Doob martingale with respect to $(\cF_i)_{i=0}^m$. Note also
  that
  \begin{align*}
    \rn{Z}_0 & = \EE_{\rn{x}}[L^\alpha_{\rn{x},F^*_m(\rn{x}),\ell}(H)] = L_{\mu,F,\ell}(H),
    &
    \rn{Z}_m & = L^\alpha_{\rn{x},F^*_m(\rn{x}),\ell}(H),
  \end{align*}
  where the second equality of the former follows by linearity of expectation.

  On the other hand, since
  \begin{align*}
    L^\alpha_{\rn{x},F^*_m(\rn{x}),\ell}(H)
    & =
    \frac{1}{\binom{m}{k}}
    \sum_{U\in\binom{[m]}{k}} \ell(\alpha_U^*(\rn{x}), b_\alpha(H^*_m(\rn{x}))_U, b_\alpha(F^*_m(\rn{x}))_U),
  \end{align*}
  and for each $U\in\binom{[m]}{k}$ the random variable $\ell(\alpha_U^*(\rn{x}),
  b_\alpha(H^*_m(\rn{x}))_U, b_\alpha(F^*_m(\rn{x}))_U)$ is measurable with respect to the
  $\sigma$-algebra generated by $(\rn{x}_A \mid A\in r(U))$, it follows that
  \begin{align*}
    \lvert\rn{Z}_i - \rn{Z}_{i-1}\rvert
    & \leq
    \frac{\lVert\ell\rVert_ \infty}{\binom{m}{k}}\cdot
    \left\lvert\left\{A\in\binom{[m]}{k} \;\middle\vert\;
    i\in A\right\}\right\rvert
    =
    \frac{k\cdot\lVert\ell\rVert_\infty}{m},
  \end{align*}
  for every $i\in[m]$. So~\eqref{eq:empconc:nonag} follows by Azuma's Inequality.

  \medskip

  In the partite setting, for each $i\in [km]$, we let $\cF_i$ be the $\sigma$-algebra generated by
  the coordinates of $\rn{x}$ indexed by some $f\in r_k(m)$ with $\dom(f)\subseteq[\ceil{i/k}]$ and
  if $\ceil{i/k}\in\dom(f)$, then $f(\ceil{i/k})\leq ((i-1)\bmod k) + 1$. Define $\rn{Z}$
  by~\eqref{eq:empconc:Z} so that $(\rn{Z})_{i=0}^{km}$ is a Doob martingale with respect to
  $(\cF_i)_{i=0}^{km}$ and
  \begin{align*}
    \rn{Z}_0 & = \EE_{\rn{x}}[L_{\rn{x},F^*_m(\rn{x}),\ell}(H)] = L_{\mu,F,\ell}(H),
    &
    \rn{Z}_{km} & = L_{\rn{x},F^*_m(\rn{x}),\ell}(H),
  \end{align*}
  where the second equality of the former follows by linearity of expectation.

  Since
  \begin{align*}
    L_{\rn{x}, F^*_m(\rn{x}),\ell}(H)
    & =
    \frac{1}{m^k}
    \sum_{\alpha\in[m]^k}
    \ell(\alpha^*(\rn{x}), H^*_m(\rn{x})_\alpha, F^*_m(\rn{x})_\alpha),
  \end{align*}
  and for each $\alpha\in [m]^k$ the random variable $\ell(\alpha^*(\rn{x}), H^*_m(\rn{x})_\alpha,
  F^*_m(\rn{x})_\alpha)$ is measurable with respect to the $\sigma$-algebra generated by the
  coordinates of $\rn{x}$ indexed by some $f\in r_k(m)$ that is a restriction of $\alpha$, it
  follows that
  \begin{align*}
    \lvert\rn{Z}_i - \rn{Z}_{i-1}\rvert
    & \leq
    \frac{\lVert\ell\rVert_\infty}{m^k}\cdot
    \lvert\{\alpha\in [m]^k \mid \alpha(\ceil{i/k}) = ((i-1)\bmod k) + 1\}\rvert
    =
    \frac{\lVert\ell\rVert_\infty}{m}
  \end{align*}
  for every $i\in[km]$. So~\eqref{eq:empconc:nonag} follows by Azuma's Inequality.

  \medskip

  The proof of item~\ref{lem:empconc:ag} is analogous to that of item~\ref{lem:empconc:nonag} except
  that when defining $\cF_i$ one considers all coordinates of both $\rn{x}$ and $\rn{x'}$ indexed by
  the sets $A\in r(i)$ in the non-partite case and indexed by $f\in r_k(m)$ with
  $\dom(f)\subseteq[\ceil{i/k}]$ and if $\ceil{i/k}\in\dom(f)$, then $f(\ceil{i/k})\leq ((i-1)\bmod
  k) + 1$ in the partite case.
\end{proof}

We now turn to derandomization and the proof intuition will be given after the statement of the
proposition.

\begin{proposition}[Derandomization]\label{prop:derand}
  Let $\Omega$ be a Borel ($k$-partite, respectively) template, let $\Lambda$ be a
  non-empty Borel space and let $\cH\subseteq\cF_k(\Omega,\Lambda)$ be a $k$-ary ($k$-partite,
  respectively) hypothesis class. Let also
  \begin{align*}
    K & \df
    \begin{dcases*}
      k^2, & in the non-partite case,\\
      k, & in the partite case.
    \end{dcases*}
  \end{align*}

  Then the following hold:
  \begin{enumerate}
  \item If $\cH$ is $k$-PAC learnable with randomness with respect to a $k$-ary ($k$-partite,
    respectively) loss function $\ell$ with $\lVert\ell\rVert_\infty<\infty$, then $\cH$ is also
    $k$-PAC learnable with respect to $\ell$.

    More precisely, if $\cA$ is a randomized $k$-PAC learner for $\cH$, then there exists a $k$-PAC
    learner $\cA'$ for $\cH$ with
    \begin{align}\label{eq:derand:m}
      m^{\PAC}_{\cH,\ell,\cA'}(\epsilon,\delta)
      & \df
      M(\epsilon,\delta)
      +
      \ceil{
        \frac{2\cdot K\cdot\lVert\ell\rVert_\infty^2}{\xi(\epsilon,\delta)^2}\cdot
        \ln\left(
        \frac{2\cdot R_\cA(M(\epsilon,\delta))}{\xi(\epsilon,\delta)}
        \right)
      },
    \end{align}
    where $M(\epsilon,\delta)\df\ceil{m^{\PACr}_{\cH,\ell,\cA}(\xi(\epsilon,\delta),
      \xi(\epsilon,\delta))}$ and
    \begin{align}\label{eq:derand:xi}
      \xi(\epsilon,\delta)
      & \df
      \min\left\{\ceil{\frac{2}{\epsilon}}^{-1}, \ceil{\frac{2}{\delta}}^{-1}\right\}.
    \end{align}
  \item If $\cH$ is agnostically $k$-PAC learnable with randomness with respect to a $k$-ary
    ($k$-partite, respectively) agnostic loss function $\ell$ with $\lVert\ell\rVert_\infty<\infty$,
    then $\cH$ is also agnostically $k$-PAC learnable with respect to $\ell$.

    More precisely, if $\cA$ is a randomized agnostic $k$-PAC learner for $\cH$, then there exists an
    agnostic $k$-PAC learner $\cA'$ for $\cH$ with $m^{\agPAC}_{\cH,\ell,\cA'}(\epsilon,\delta)$ given
    by the right-hand side of~\eqref{eq:derand:m} but with $M(\epsilon,\delta)\df
    \ceil{m^{\agPACr}_{\cH,\ell,\cA}(\xi(\epsilon,\delta), \xi(\epsilon,\delta))}$ and $\xi$ is given
    by~\eqref{eq:derand:xi}.
  \end{enumerate}
\end{proposition}

Before we start the proof, let us remark that the formula~\eqref{eq:derand:m} above has a simpler
form when $\delta=\epsilon$, namely:
\begin{align}\label{eq:derand:msimple}
  \begin{multlined}
    m^{\PAC}_{\cH,\ell,\cA'}(\epsilon,\epsilon)
    =
    \ceil{m^{\PACr}_{\cH,\ell,\cA}\left(\frac{1}{\ceil{2/\epsilon}},\frac{1}{\ceil{2/\epsilon}}\right)}
    \\
    +
    \ceil{
        2\cdot K\cdot\lVert\ell\rVert_\infty^2\ceil{\frac{2}{\epsilon}}^2
        \cdot
        \ln\left(
        2\cdot\ceil{\frac{2}{\epsilon}}\cdot
          R_\cA\left(\ceil{
            m^{\PACr}_{\cH,\ell,\cA}\left(\frac{1}{\ceil{2/\epsilon}},\frac{1}{\ceil{2/\epsilon}}\right)
          }
          \right)
        \right)
      }.
  \end{multlined}
\end{align}

\begin{proofint}
  Let us first give an idea of the (non-randomized) learning algorithm $\cA'$ (the idea is the same in
  all settings).

  Since our setup does not care about the time-complexity of the algorithm $\cA'$ (but see also
  Remark~\ref{rmk:complexity} below), it is tempting to use standard brute-force derandomization: on
  an $[m]$-sample input $(x,y)$, we simply enumerate over all possible $b\in[R_\cA(m)]$, consider all
  hypotheses $H(b)\df \cA(x,y,b)$ and return the best possible $H(b)$. The only issue is that we do
  not have any way of checking which hypothesis $H(b)$ is the best one as we cannot compute the
  total loss directly. However, if we had access to a fresh new sufficiently large sample, then we
  could estimate the total loss of each $H(b)$ by their empirical loss.

  Slightly more formally, we assume we receive a sample of $m = m_1 + m_2$ points, we run $\cA$ on the
  first $m_1$ points by manually providing all possible sources of randomness $b\in [R_\cA(m_1)]$ and
  for each outcome $H(b)$ of $\cA$, we use the last $m_2$ points to estimate the total loss, returning
  the $H(b)$ that minimizes the empirical loss on the last $m_2$ points. If $m_2$ is sufficiently
  large, then putting together the fact that $\cA$ is a randomized agnostic $k$-PAC learner with the
  concentration Lemma~\ref{lem:empconc} will ensure that with high probability we output some
  hypothesis with small total loss.

  There is one more issue with the approach above: both $m_1$ and $m_2$ need to be sufficiently
  large in terms of $\epsilon$ and $\delta$, which are not available as an input to the
  algorithm. To address this, we use a small trick: when given input $m$, we find the largest
  $s\in\NN_+$ such that the approach above would work for $\epsilon=\delta=1/s$ and run it.
\end{proofint}

\begin{proof}
  All four cases, partite or not, agnostic or not, have completely analogous proofs, so we show only
  the case that would naively seem to be the most complex: the non-partite agnostic case. We also
  assume without loss of generality that $\lVert\ell\rVert_\infty > 0$ (otherwise the result is
  trivial).

  For each $m\in\NN$, we let $s(m)$ be the largest $s\in\NN_+$ such that
  \begin{align*}
    \ceil{m^{\agPACr}_{\cH,\ell,\cA}\left(\frac{1}{2s}, \frac{1}{2s}\right)}
    +
    \ceil{
      8\cdot K\cdot\lVert\ell\rVert_\infty^2\cdot s^2\cdot
      \ln\left(
      4\cdot s
      \cdot R_\cA\left(\ceil{m^{\agPACr}_{\cH,\ell,\cA}\left(\frac{1}{2s}, \frac{1}{2s}\right)}\right)
      \right)
    }
    & \leq
    m,
  \end{align*}
  setting $s(m)\df -\infty$ if no such $s$ exists. If $s(m)\neq-\infty$, we also let
  \begin{align}
    m_1(m)
    & \df
    \ceil{m^{\agPACr}_{\cH,\ell,\cA}\left(\frac{1}{2\cdot s(m)}, \frac{1}{2\cdot s(m)}\right)},
    \notag
    \\
    m_2(m)
    & \df
    m - m_1(m)
    \geq
    \ceil{
      8\cdot K\cdot\lVert\ell\rVert_\infty^2\cdot s(m)^2\cdot
      \ln(4\cdot s(m)\cdot R_\cA(m_1(m)))
    }
    >
    0.
    \label{eq:derand:m2}
  \end{align}
  
  We now define the learning algorithm $\cA'$ as follows. On input
  $(x,y)\in\cE_m(\Omega)\times\Lambda^{([m])_k}$, if $s(m)\neq-\infty$, then we let
  \begin{align}\label{eq:derand:b}
    b(x,y)
    & \df
    \argmin_{b\in [R_\cA(m_1(m))]} L^{\alpha(m)}_{\iota_2^*(x),\iota_2^*(y),\ell}(\cA(\iota_1^*(x),\iota_1^*(y),b)),
  \end{align}
  where $\alpha(m)$ is any fixed order choice for $[m_2(m)]$ and $\iota_i\colon [m_i(m)]\to [m]$
  ($i\in[2]$) are the injections given by
  \begin{align*}
    \iota_1(i) & \df i
    \qquad (i\in[m_1(m)]),
    \\
    \iota_2(i) & \df m_1(m) + i
    \qquad (i\in[m_2(m)]),
  \end{align*}
  and we let $\cA'(x,y)\df \cA(\iota_1^*(x),\iota_1^*(y),b(x,y))$. If $s(m)=-\infty$, we define $\cA$
  arbitrarily but measurably (say, let it be always equal to some fixed element of
  $\cH$). Furthermore, in~\eqref{eq:derand:b}, the argmin must be computed measurably, say, we can
  break ties by taking $b(x,y)$ smallest possible.

  In plain English, the algorithm $\cA'$ splits its input $[m]$-sample into an $[m_1(m)]$-sample
  $(x^1,y^1)\df(\iota_1^*(x),\iota_1^*(y))$ and an $[m_2(m)]$-sample
  $(x^2,y^2)\df(\iota_2^*(x),\iota_2^*(y))$, runs $\cA$ on $(x^1,y^1)$ with all possible sources of
  randomness $b\in R_\cA(m_1(m))$, to obtain $H(b)\in\cH$ and returns the $H(b)$ with least empirical
  loss of each $H(b)$ on $(x^2,y^2)$.

  Let us show that $\cA'$ is an agnostic $k$-PAC learner with $m^{\agPAC}_{\cH,\ell,\cA'}$ given by
  the right-hand side of~\eqref{eq:derand:m} with $M(\epsilon,\delta)\df
  \ceil{m^{\agPACr}_{\cH,\ell,\cA}(\xi(\epsilon,\delta), \xi(\epsilon,\delta))}$.

  Fix $\epsilon,\delta\in(0,1)$ and note that our definition ensures that
  \begin{align*}
    \xi(\epsilon,\delta)\leq\min\left\{\frac{\epsilon}{2},\frac{\delta}{2}\right\}
  \end{align*}
  so if $m\geq m^{\agPAC}_{\cH,\ell,\cA'}(\epsilon,\delta)$ is an integer, then for $s\df
  1/(2\xi(\epsilon,\delta))\in\NN_+$, we have
  \begin{align*}
    m
    & \geq
    m^{\agPAC}_{\cH,\ell,\cA'}(\epsilon,\delta)
    \\
    & \df
    M(\epsilon,\delta)
    +
    \ceil{
      \frac{2\cdot K\cdot\lVert\ell\rVert_\infty^2}{\xi(\epsilon,\delta)^2}\cdot
      \ln\left(
      \frac{2\cdot R_\cA(M(\epsilon,\delta))}{\xi(\epsilon,\delta)}
      \right)
    }
    \\
    & =
    \ceil{m^{\agPACr}_{\cH,\ell,\cA}\left(\frac{1}{2s}, \frac{1}{2s}\right)}
    +
    \ceil{
      8\cdot K\cdot\lVert\ell\rVert_\infty^2\cdot s^2\cdot
      \ln\left(
      4\cdot s
      \cdot R_\cA\left(\ceil{m^{\agPACr}_{\cH,\ell,\cA}\left(\frac{1}{2s}, \frac{1}{2s}\right)}\right)
      \right)
    },
  \end{align*}
  so we conclude that $s(m)\geq 1/(2\xi(\epsilon,\delta))$.

  Let $\Omega'$ be a Borel template, let $\mu'\in\Pr(\Omega')$, let
  $F\in\cF_k(\Omega\otimes\Omega',\Lambda)$, let $(\rn{x},\rn{x'})\sim(\mu\otimes\mu')^m$, let
  $\rn{y}\df F^*_m(\rn{x},\rn{x'})$ and for $i\in[2]$, let
  \begin{align*}
    \rn{x}^i & \df \iota_i^*(\rn{x}), &
    (\rn{x'})^i & \df \iota_i^*(\rn{x'}), &
    \rn{y}^i & \df \iota_i^*(\rn{y}).
  \end{align*}
  Lemma~\ref{lem:F*Vequiv} gives $\rn{y}^i = F^*_m(\rn{x}^i,(\rn{x'})^i)$.

  Note that $(\rn{x}^1,(\rn{x'})^1,\rn{y}^1)$ is independent from $(\rn{x}^2,(\rn{x'})^2,\rn{y}^2)$ as
  they correspond to disjoint sets of coordinates of $(\rn{x},\rn{x'},\rn{y})$. 

  Let $E(\rn{x}^1,(\rn{x'})^1)$ be the event
  \begin{align*}
    \frac{1}{R_\cA(m_1(m))}\sum_{b\in[R_\cA(m_1(m))]}
    L_{\mu,\mu',F,\ell}(\cA(\rn{x}^1),\rn{y}^1,b)
    & \leq
    \inf_{H\in\cH} L_{\mu,\mu',F,\ell}(H)
    + \frac{1}{2\cdot s(m)}.
  \end{align*}
  By definition of $m_1(m)$, we get
  \begin{align}\label{eq:derand:firstevent}
    \PP_{\rn{x},\rn{x'}}[E(\rn{x}^1,(\rn{x'})^1)]
    & \geq
    1 - \frac{1}{2\cdot s(m)}.
  \end{align}

  We now consider a fixed value $(x^1,(x')^1,y^1)$ of $(\rn{x}^1,(\rn{x'})^1,\rn{y})$ and a fixed
  $b\in [R_\cA(m_1(m))]$ and let $E_{x^1,(x')^1,y^1,b}(\rn{x}^2,(\rn{x'})^2)$ be the event
  \begin{align*}
    \lvert L^{\alpha(m)}_{\rn{x}^2,\rn{y}^2,\ell}(\cA(x^1,y^1,b)) - L_{\mu,\mu',F,\ell}(\cA(x^1,y^1,b))\rvert
    & \leq
    \frac{1}{2\cdot s(m)}.
  \end{align*}

  By Lemma~\ref{lem:empconc}\ref{lem:empconc:ag}, we have
  \begin{align*}
    \PP_{\rn{x}^2,(\rn{x'})^2}[E_{x^1,(x')^1,y^1,b}(\rn{x}^2,(\rn{x'})^2)]
    & \geq
    1 - 2\exp\left(-\frac{(1/(2\cdot s(m)))^2\cdot m_2(m)}{2\cdot K\cdot\lVert\ell\rVert_\infty^2}\right)
    \\
    & \geq
    1 - \frac{1}{2\cdot s(m)\cdot R_\cA(m_1(m))},
  \end{align*}
  where the second inequality follows by~\eqref{eq:derand:m2}. By the union bound, we have
  \begin{align}\label{eq:derand:secondevent}
    \PP_{\rn{x}^2,(\rn{x'})^2}\left[
      \bigcap_{b\in [R_\cA(m_1(m))]}E_{x^1,(x')^1,y^1,b}(\rn{x}^2,(\rn{x'})^2)
      \right]
    & \geq
    1 - \frac{1}{2\cdot s(m)}.
  \end{align}

  Since the event $E(\rn{x}^1,(\rn{x'})^1)$ in particular implies that there exists $\rn{b^*}\in
  [R_\cA(m_1(m))]$ such that
  \begin{align*}
    L_{\mu,\mu',F,\ell}(\cA(\rn{x}^1,\rn{y}^1),\rn{b^*})
    \leq
    \inf_{H\in\cH} L_{\mu,\mu',F,\ell}(H)
    + \frac{1}{2\cdot s(m)},
  \end{align*}
  it follows that within the event
  \begin{align*}
    E(\rn{x}^1,(\rn{x'})^1)
    \cap\bigcap_{b\in [R_\cA(m_1(m))]} E_{\rn{x}^1,(\rn{x'})^1,\rn{y}^1,b}(\rn{x}^2,(\rn{x'})^2),
  \end{align*}
  we have
  \begin{align*}
    L_{\mu,\mu',F,\ell}(\cA'(\rn{x},\rn{y}))
    & \leq
    L^{\alpha(m)}_{\rn{x}^2,\rn{y}^2,F,\ell}(\cA'(\rn{x},\rn{y}))
    +
    \frac{1}{2\cdot s(m)}
    \\
    & \leq
    L^{\alpha(m)}_{\rn{x}^2,\rn{y}^2,F,\ell}(\cA(\rn{x}^1,\rn{y}^1,\rn{b^*}))
    +
    \frac{1}{2\cdot s(m)}
    \\
    & \leq
    \inf_{H\in\cH} L_{\mu,\mu',F,\ell}(H) + 
    \frac{1}{s(m)}
    \\
    & \leq
    \inf_{H\in\cH} L_{\mu,\mu',F,\ell}(H) + 2\xi(\epsilon,\delta)
    \\
    & \leq
    \inf_{H\in\cH} L_{\mu,\mu',F,\ell}(H) + \epsilon.
  \end{align*}

  Finally, applying the union bound to~\eqref{eq:derand:firstevent}
  and~\eqref{eq:derand:secondevent}, we get
  \begin{align*}
    & \!\!\!\!\!\!
    \PP_{\rn{x},\rn{x'}}[L_{\mu,\mu',F,\ell}(\cA'(\rn{x},\rn{y}))
      \leq \inf_{H\in\cH} L_{\mu,\mu',F,\ell}(H) + \epsilon]
    \\
    & \geq
    \PP_{\rn{x},\rn{x'}}\left[
      E(\rn{x}^1,(\rn{x'})^1)
      \cap\bigcap_{b\in [R_\cA(m_1(m))]} E_{\rn{x}^1,(\rn{x'})^1,\rn{y}^1,b}(\rn{x}^2,(\rn{x'})^2)
      \right]
    \\
    & \geq
    1 - \frac{1}{s(m)}
    \geq
    1 - 2\xi(\epsilon,\delta)
    \geq
    1 - \delta,
  \end{align*}
  concluding the proof.
\end{proof}

\begin{remark}\label{rmk:complexity}
  Note that the reduction of Proposition~\ref{prop:derand} above is very inefficient in terms of
  computational resources: since we have to enumerate all sources of randomness, we incur at least a
  time-complexity overhead of $R_\cA(M(\epsilon,\delta))$. Now typically when we think of an
  randomized algorithm $\cA$ of polynomial-time (in $m$), we allow it to use polynomially many
  random bits, which amounts to allowing $R_\cA$ to be exponential in $m$. If that is the case, even
  if $\cA$ is polynomial-time and the original PAC guarantee is polynomial in $1/\epsilon$ and
  $1/\delta$, our reduction is still potentially exponential due to $R_\cA$ potentially being
  exponential. As a consolation prize, we at least have the guarantee that the sample size when
  $\epsilon=\delta$ remains polynomial in $1/\epsilon$ (see~\eqref{eq:derand:msimple}).
\end{remark}

\section{Partite versus non-partite}
\label{sec:part}

In this section we compare the non-partite and partite settings.

\subsection{Basic partization properties}

We start by proving some basic properties about some auxiliary objects related to the partization
construction of Definition~\ref{def:kpart}.

\begin{lemma}[Partization basics]\label{lem:kpartbasics}
  Let $\Omega$ be a Borel template, let $k\in\NN_+$, let $\Lambda$ be a non-empty Borel
  space. Then the following hold:
  \begin{enumerate}
  \item\label{lem:kpartbasics:phi} For $\mu\in\Pr(\Omega)$ and $m\in\NN$ the function
    $\phi_m\colon\cE_m(\Omega)\to\cE_{\floor{m/k}}(\Omega^{\kpart})$ given by
    \begin{align}\label{eq:kpartbasics:phi}
      \phi_m(x)_f
      & \df
      x_{\{(i-1)\floor{m/k} + f(i) \mid i\in\dom(f)\}}
      \qquad \left(f\in r_k\left(\floor{m/k}\right)\right).
    \end{align}
    is measure-preserving with respect to $\mu^m$ and $(\mu^{\kpart})^{\floor{m/k}}$. Furthermore,
    if $m$ is divisible by $k$, then $\phi_m$ is a measure-isomorphism.

    Moreover, we have $\phi_k^{-1} = \iota_{\kpart}$, where $\iota_{\kpart}$ is given
    by~\eqref{eq:iotakpart}.
  \item\label{lem:kpartbasics:Phi} For $m\in\NN$, $F\in\cF_k(\Omega,\Lambda)$ and
    $\Phi_m\colon\Lambda^{([m])_k}\to(\Lambda^{S_k})^{[\floor{m/k}]^k}$ given by
    \begin{align}\label{eq:kpartbasics:Phi}
      (\Phi_m(y)_\alpha)_\tau
      & \df
      y_{\beta_\alpha\comp\tau}
      \qquad \left(\alpha\in \left[\floor{m/k}\right]^k, \tau\in S_k\right),
    \end{align}
    where $\beta_\alpha\in([m])_k$ is given by
    \begin{align}
      \beta_\alpha(i)
      & \df
      (i-1)\floor{\frac{m}{k}} + \alpha(i)
      \qquad \left(\alpha\in \left[\floor{m/k}\right]^k, i\in[k]\right),
    \end{align}
    the diagram
    \begin{equation*}
      \begin{tikzcd}[column sep={2.5cm}]
        \cE_m(\Omega)
        \arrow[r, "F^*_m"]
        \arrow[d, "\phi_m"']
        &
        \Lambda^{([m])_k}
        \arrow[d, "\Phi_m"]
        \\
        \cE_{\floor{m/k}}(\Omega^{\kpart})
        \arrow[r, "(F^{\kpart})^*_{\floor{m/k}}"]
        &
        (\Lambda^{S_k})^{[\floor{m/k}]^k}
      \end{tikzcd}
    \end{equation*}
    commutes, where $\phi_m$ is given by~\eqref{eq:kpartbasics:phi}.
  \item\label{lem:kpartbasics:action} For the covariant action of $S_k$ on $\cE_1(\Omega^{\kpart})$
    given by
    \begin{align}\label{eq:kpartbasics:action}
      \sigma_*(x)_f & \df x_{f\comp\sigma\rest_{\sigma^{-1}(\dom(f))}}
      \qquad (x\in\cE_1(\Omega^{\kpart}), \sigma\in S_k, f\in r_k(1)),
    \end{align}
    the diagram
    \begin{equation*}
      \begin{tikzcd}
        \cE_1(\Omega^{\kpart})
        \arrow[r, "\iota_{\kpart}", shift left]
        \arrow[d, "\sigma^{-1}_*"']
        &
        \cE_k(\Omega)
        \arrow[l, "\phi_k", shift left]
        \arrow[d, "\sigma^*"]
        \\
        \cE_1(\Omega^{\kpart})
        \arrow[r, "\iota_{\kpart}", shift left]
        &
        \cE_k(\Omega)
        \arrow[l, "\phi_k", shift left]
      \end{tikzcd}
    \end{equation*}
    commutes for every $\sigma\in S_k$, where $\iota_{\kpart}$ is given by~\eqref{eq:iotakpart} and
    $\phi_k$ is given by~\eqref{eq:kpartbasics:phi}, that is, $\iota_{\kpart}$ and $\phi_k$ are
    $S_k$-equivariant.
  \end{enumerate}
\end{lemma}

\begin{proof}
  Throughout this proof, we use the notation $1^A$ for the unique function $A\to[1]$.

  The first part of item~\ref{lem:kpartbasics:phi} follows since $\phi_m$ is a projection map followed by
  relabeling of the coordinates and $(\mu^{\kpart})^{\floor{m/k}}$ is the corresponding marginal of
  $\mu^m$ (after relabeling of the coordinates) and when $m$ is divisible by $k$, $\phi_m$ is
  bijective (i.e., it is simply a relabeling of the coordinates).

  For the second part, since $\phi_k$ is a bijection, it suffices to show that
  $\iota_{\kpart}\comp\phi_k$ is the identity map in $\cE_k(\Omega)$, but indeed, note that for
  $x\in\cE_k(\Omega)$ and $A\in r(k)$, we have
  \begin{align*}
    \iota_{\kpart}(\phi_k(x))_A
    & =
    \phi_k(x)_{1^A}
    =
    x_{\{(i-1)\cdot 1 + 1 \mid i\in A\}}
    =
    x_A.
  \end{align*}

  \medskip

  For item~\ref{lem:kpartbasics:Phi}, note that for $x\in\cE_m(\Omega)$, $\alpha\in[\floor{m/k}]^k$ and
  $\tau\in S_k$, we have
  \begin{align*}
    (\Phi_m(F^*_m(x))_\alpha)_\tau
    & =
    F^*_m(x)_{\beta_\alpha\comp\tau}
    =
    F((\beta_\alpha\comp\tau)^*(x))
    =
    F(\tau^*(\beta_\alpha^*(x)))
    =
    F_k^*(\beta_\alpha^*(x))_\tau
  \end{align*}
  and
  \begin{align*}
    ((F^{\kpart})^*_m(\phi_m(x))_\alpha)_\tau
    & =
    F^{\kpart}(\alpha^*(\phi_m(x)))_\tau
    =
    F^*_k(\iota_{\kpart}(\alpha^*(\phi_m(x))))_\tau,
  \end{align*}
  so it suffices to show that $\iota_{\kpart}(\alpha^*(\phi_m(x))) = \beta_\alpha^*(x)$. For this,
  note that if $A\in r(k)$, then
  \begin{align*}
    \iota_{\kpart}(\alpha^*(\phi_m(x)))_A
    & =
    \alpha^*(\phi_m(x))_{1^A}
    =
    \phi_m(x)_{\alpha\rest_A}
    =
    x_{\{(i-1)\floor{m/k} + \alpha\rest_A(i) \mid i\in A\}}
    =
    x_{\beta_\alpha(A)}
    =
    \beta_\alpha^*(x)_A,
  \end{align*}
  as desired.

  \medskip

  Let us now prove item~\ref{lem:kpartbasics:action}. First note that since $\Omega^{\kpart}_A =
  \Omega_{\lvert A\rvert}$, the formula~\eqref{eq:kpartbasics:action} is well-defined (as
  $\lvert\dom(f\comp\sigma\rest_{\sigma^{-1}(\dom(f))})\rvert = \lvert\dom(f)\rvert$). Now
  covariance is easily checked: for $x\in\cE_1(\Omega^{\kpart})$ and $\sigma,\tau\in S_k$, we have
  \begin{align*}
    \tau_*(\sigma_*(x))_f
    & =
    \sigma_*(x)_{f\comp\tau\rest_{\tau^{-1}(\dom(f))}}
    =
    x_{f\comp\tau\rest_{\tau^{-1}(\dom(f))}\comp\sigma\rest_{\sigma^{-1}(\tau^{-1}(\dom(f)))}}
    \\
    & =
    x_{f\comp\tau\comp\sigma\rest_{(\tau\comp\sigma)^{-1}(\dom(f))}}
    =
    (\tau\comp\sigma)_*(x)_f.
  \end{align*}

  Since $\iota_{\kpart}$ and $\phi_k$ are inverses of each other by item~\ref{lem:kpartbasics:phi},
  it suffices to show only $S_k$-equivariance of $\iota_{\kpart}$. Again, this is easily checked:
  for $x\in\cE_1(\Omega^{\kpart})$, $\sigma\in S_k$ and $A\in r(k)$, we have
  \begin{align*}
    \sigma^*(\iota_{\kpart}(x))_A
    & =
    \iota_{\kpart}(x)_{\sigma(A)}
    =
    x_{1^{\sigma(A)}}
    =
    \sigma^{-1}_*(x)_{1^A}
    =
    \iota_{\kpart}(\sigma^{-1}_*(x))_A.
    \qedhere
  \end{align*}
\end{proof}

The next proposition says that rank is invariant under the injection $\cH\mapsto\cH^{\kpart}$.

\begin{proposition}[Rank invariance]\label{prop:kpartrk}
  Let $\Omega$ be a Borel template, let $k\in\NN_+$, let $\Lambda$ be a non-empty Borel
  space, let $H\in\cF_k(\Omega,\Lambda)$ be a $k$-ary hypothesis and let
  $\cH\subseteq\cF_k(\Omega,\Lambda)$ be a $k$-ary hypothesis class. Then the following hold:
  \begin{enumerate}
  \item We have $\rk(H)=\rk(H^{\kpart})$.
  \item We have $\rk(\cH)=\rk(\cH^{\kpart})$.
  \end{enumerate}
\end{proposition}

\begin{proof}
  Clearly, the assertion for $\cH$ follow from the assertion for $H$.
  
  For each $t\in[k]$, let $D_t\df\{f\in r_k(1) \mid \lvert\dom(f)\rvert=t\}$ and
  $\iota_{\kpart}^t\colon X_t^{D_t}\to X_t^{\binom{[k]}{t}}$ be given by
  \begin{align*}
    \iota_{\kpart}^t(x)_A & \df x_{1^A} \qquad \left(x\in X_t^{D_t}, A\in\binom{[k]}{t}\right),
  \end{align*}
  where $1^A$ is the unique function $A\to[1]$. Note that $\iota_{\kpart}$ can be decomposed as
  \begin{align*}
    \iota_{\kpart}(x)
    & =
    (\iota_{\kpart}^1(x^1),\ldots,\iota_{\kpart}^k(x^k))
    \qquad (x\in\cE_1(\Omega^{\kpart})),
  \end{align*}
  where $x^t$ is the projection of $x$ onto the coordinates indexed by $D_t$. Since
  \begin{align*}
    H^{\kpart}(x)
    & \df
    H^*_k(\iota_{\kpart}(x))
    =
    H^*_k(\iota_{\kpart}^1(x^1),\ldots,\iota_{\kpart}^k(x^k))
    \qquad
    (x\in\cE_1(\Omega^{\kpart})),
  \end{align*}
  the assertion on the rank of $H$ and $H^{\kpart}$ follows from the definition of $H^*_k$.
\end{proof}

The next proposition shows that the $\VCN_k$-dimension is also invariant under the injection
$\cH\mapsto\cH^{\kpart}$.

\begin{proposition}[$\VCN_k$-dimension invariance]\label{prop:kpartVCN}
  Let $\Omega$ be a Borel template, let $k\in\NN_+$, let $\Lambda$ be a non-empty Borel space and
  let $\cH\subseteq\cF_k(\Omega,\Lambda)$ be a $k$-ary hypothesis class. Then $\VCN_k(\cH) =
  \VCN_k(\cH^{\kpart})$.
\end{proposition}

\begin{proof}
  First note that the function $\iota_{\kpart}$ given by~\eqref{eq:iotakpart} can be decomposed as
  \begin{align*}
    \iota_{\kpart}(x,x') & = (\iota_{\kpart}^1(x),\iota_{\kpart}^2(x')),
  \end{align*}
  where
  \begin{align*}
    \iota_{\kpart}^1\colon & \prod_{f\in r_{k,[k-1]}} X_{\lvert\dom(f)\rvert}\to\cE_{k-1}(\Omega),
    \\
    \iota_{\kpart}^2\colon & \prod_{f\in r_k(1)\setminus r_{k,[k-1]}} X_{\lvert\dom(f)\rvert}
    \to\prod_{A\in r(k)\setminus r(k-1)} X_{\lvert A\rvert}
  \end{align*}
  are given by
  \begin{align*}
    \iota_{\kpart}^1(x)_A & \df x_{1^A}
    \qquad (A\in r(k-1)),
    \\
    \iota_{\kpart}^2(x')_A & \df x_{1^A}
    \qquad (A\in r(k)\setminus r(k-1)),
  \end{align*}
  $r_{k,A}\df\{f\in r_k(1) \mid \dom(f)\subseteq A\}$ is as in~\eqref{eq:rkA} and $1^A$ is the
  unique function $A\to[1]$.

  Note that both $\iota_{\kpart}^1$ and $\iota_{\kpart}^2$ are bijections and for every
  $x\in\cE_{k-1}$, we have
  \begin{align*}
    \cH(x)
    & =
    \{H^*_k(x,\place) \mid H\in\cH\}
    =
    \{H^{\kpart}((\iota_{\kpart}^1)^{-1}(x), (\iota_{\kpart}^2)^{-1}(\place)) \mid H\in\cH\}
    \\
    & =
    \{F\comp(\iota_{\kpart}^2)^{-1} \mid F\in\cH^{\kpart}((\iota_{\kpart}^1)^{-1}(x))\}.
  \end{align*}
  Since $\iota_{\kpart}^2$ is bijective, we conclude that
  \begin{align*}
    \Nat(\cH(x)) & = \Nat(\cH^{\kpart}((\iota_{\kpart}^1)^{-1}(x))) \leq \VCN_k(\cH^{\kpart}),
  \end{align*}
  from which we get $\VCN_k(\cH)\leq\VCN_k(\cH^{\kpart})$ by taking supremum over $x$.

  \medskip

  For the other inequality, first recall that $S_k$ has a covariant action on
  $\cE_1(\Omega^{\kpart})$ given by~\eqref{eq:kpartbasics:action}. Let $A\in\binom{[k]}{k-1}$ and
  let $\sigma\in S_k$ be such that $\sigma(A)=[k-1]$ and note that for $x\in\prod_{f\in r_{k,A}}
  X_{\lvert\dom(f)\rvert}$, we have
  \begin{align*}
    \cH^{\kpart}(x)
    & =
    \{H^{\kpart}(x,\place) \mid H\in\cH\}
    \\
    & =
    \{H^*_k(\iota_{\kpart}(x,\place)) \mid H\in\cH\}
    \\
    & =
    \{\sigma^*(H^*_k(\iota_{\kpart}(\sigma_*(x,\place)))) \mid H\in\cH\}
    \\
    & =
    \{\sigma^*(H^*_k(\iota_{\kpart}^1(x'), \iota_{\kpart}^2(\sigma_*(\place)))) \mid H\in\cH\}
    \\
    & =
    \{\sigma^*\comp F\comp\iota_{\kpart}^2\comp\sigma_* \mid F\in\cH(\iota_{\kpart}^1(x'))\},
  \end{align*}
  where the third equality follows since both $H^*_k$ and $\iota_{\kpart}$ are $S_k$-equivariant (by
  Lemmas~\ref{lem:F*Vequiv} and~\ref{lem:kpartbasics}\ref{lem:kpartbasics:action}, respectively) and
  $x'\in\prod_{f\in r_{k,[k-1]}} X_{\lvert\dom(f)\rvert}$ is given by
  \begin{align*}
    x'_f & \df x_{f\comp\sigma\rest_{\sigma^{-1}}(\dom(f))}
    \qquad (f\in r_{k,[k-1]}).
  \end{align*}

  Since $\sigma^*$, $\iota_{\kpart}^2$ and $\sigma_*$ are bijective, we conclude that
  \begin{align*}
    \Nat(\cH^{\kpart}(x))
    & =
    \Nat(\cH(\iota_{\kpart}^1(x')))
    \leq
    \VCN_k(\cH),
  \end{align*}
  from which we get $\VCN_k(\cH^{\kpart})\leq\VCN_k(\cH)$ by taking supremum over $x$ and $A$.
\end{proof}

\subsection{Partite learnability implies non-partite learnability}

The next proposition says that the injection $\cH\mapsto\cH^{\kpart}$ reduces (agnostic) $k$-PAC
learnability (with randomness) to its partite counterpart. However, it does not say anything about
the map \emph{preserving} learnability; this a priori is a much harder problem and will be
(partially) covered by Propositions~\ref{prop:kpart2} and~\ref{prop:kpart3}.

\begin{proposition}[Partite to non-partite]\label{prop:kpart}
  Let $\Omega$ be a Borel template, let $k\in\NN_+$, let $\Lambda$ be a non-empty Borel space
  and let $\cH\subseteq\cF_k(\Omega,\Lambda)$ be a $k$-ary hypothesis class. Then the following
  hold:
  \begin{enumerate}
  \item\label{prop:kpart:partPAC->PAC} If
    $\ell\colon\cE_k(\Omega)\times\Lambda^{S_k}\times\Lambda^{S_k}\to\RR_{\geq 0}$ is a $k$-ary loss
    function and $\cH^{\kpart}$ is $k$-PAC learnable (with randomness, respectively) with respect to
    $\ell^{\kpart}$, then $\cH$ is $k$-PAC learnable (with randomness, respectively) with respect to
    $\ell$.

    More precisely, if $\cA'$ is a $k$-PAC learner for $\cH^{\kpart}$, then $\cA$ defined by
    \begin{align}\label{eq:kpart:cA}
      \cA(x,y)
      & \df
      \cA'(\phi_m(x), \Phi_m(y))^{\kpart,-1}
      \qquad (x\in\cE_m(\Omega), y\in\Lambda^{([m])_k}, m\in\NN)
    \end{align}
    is a $k$-PAC learner for $\cH$ with $m^{\PAC}_{\cH,\ell,\cA} = k\cdot\ceil{m^{\PAC}_{\cH^{\kpart},\ell^{\kpart},\cA'}}$,
    where $\phi_m$ and $\Phi_m$ are given by~\eqref{eq:kpartbasics:phi}
    and~\eqref{eq:kpartbasics:Phi}, respectively.

    In the randomized case, if $\cA'$ is a randomized $k$-PAC learner for $\cH^{\kpart}$, then $\cA$
    defined by $R_\cA\df R_{\cA'}$ and
    \begin{align}\label{eq:kpart:cAr}
      \cA(x,y,b)
      & \df
      \cA'(\phi_m(x), \Phi_m(y), b)^{\kpart,-1}
      \qquad (x\in\cE_m(\Omega), y\in\Lambda^{([m])_k}, m\in\NN)
    \end{align}
    is a randomized $k$-PAC learner for $\cH$ with $m^{\PACr}_{\cH,\ell,\cA} =
    k\cdot\ceil{m^{\PACr}_{\cH^{\kpart},\ell^{\kpart},\cA'}}$.
  \item\label{prop:kpart:partagPAC->agPAC} If
    $\ell\colon\cH\times\cE_k(\Omega)\times\Lambda^{S_k}\to\RR_{\geq 0}$ is a $k$-ary agnostic loss
    function and $\cH^{\kpart}$ is agnostically $k$-PAC learnable (with randomness, respectively)
    with respect to $\ell^{\kpart}$, then $\cH$ is agnostically $k$-PAC learnable (with randomness,
    respectively) with respect to $\ell$.

    More precisely, if $\cA'$ is an agnostic $k$-PAC learner for
    $\cH^{\kpart}$, then $\cA$ defined by~\eqref{eq:kpart:cA} is an agnostic $k$-PAC learner for $\cH$
    with $m^{\agPAC}_{\cH,\ell,\cA} = k\cdot\ceil{m^{\agPAC}_{\cH^{\kpart},\ell^{\kpart},\cA'}}$.

    And in the randomized case, if $\cA'$ is a randomized agnostic $k$-PAC learner for $\cH^{\kpart}$,
    then $\cA$ defined by~\eqref{eq:kpart:cAr} is a randomized agnostic $k$-PAC learner for $\cH$ with
    $m^{\agPACr}_{\cH,\ell,\cA} = k\cdot\ceil{m^{\agPACr}_{\cH^{\kpart},\ell^{\kpart},\cA'}}$.
  \end{enumerate}
\end{proposition}

\begin{proof}
  We start with item~\ref{prop:kpart:partPAC->PAC}. It suffices to prove only the randomized case as
  the non-randomized case amounts to the case when $R_{\cA'}\equiv 1$.

  First note that for $F,H\in\cF_k(\Omega,\Lambda)$
  and $\mu\in\Pr(\Omega)$, by Lemma~\ref{lem:kpartbasics}\ref{lem:kpartbasics:phi}, we have
  \begin{equation}\label{eq:kpart:L}
    \begin{aligned}
      L_{\mu,F,\ell}(H)
      & =
      \EE_{\rn{x}\sim\mu^k}[\ell(\rn{x}, H^*_k(\rn{x}), F^*_k(\rn{x}))]
      \\
      & =
      \EE_{\rn{x}\sim\mu^k}[\ell^{\kpart}(\phi_k(\rn{x}), H^{\kpart}(\phi_k(\rn{x})), F^{\kpart}(\phi_k(\rn{x})))]
      \\
      & =
      \EE_{\rn{z}\sim(\mu^{\kpart})^1}[\ell^{\kpart}(\rn{z}, H^{\kpart}(\rn{z}), F^{\kpart}(\rn{z}))]
      \\
      & =
      L_{\mu^{\kpart},F^{\kpart},\ell^{\kpart}}(H^{\kpart}),
    \end{aligned}
  \end{equation}
  which in particular implies that if $F$ is realizable in $\cH$ with respect to $\ell$ and $\mu$,
  then $F^{\kpart}$ is realizable in $\cH^{\kpart}$ with respect to $\ell^{\kpart}$ and
  $\mu^{\kpart}$.

  Note now that for $\epsilon,\delta\in(0,1)$ and an integer $m\geq
  k\cdot\ceil{m^{\PAC}_{\cH^{\kpart},\ell^{\kpart},\cA'}(\epsilon,\delta)}$, we have
  \begin{align*}
    & \!\!\!\!\!\!
    \PP_{\rn{x}\sim\mu^m}[
      \EE_{\rn{b}\sim U(R_\cA(m))}[L_{\mu,F,\ell}(\cA(\rn{x}, F^*_m(\rn{x}),\rn{b}))]
      \leq\epsilon]
    \\
    & =
    \PP_{\rn{x}\sim\mu^m}[
      \EE_{\rn{b}\sim U(R_{\cA'}(m))}[
        L_{\mu^{\kpart},F^{\kpart},\ell^{\kpart}}(\cA'(\phi_m(\rn{x}),\Phi_m(F^*_m(\rn{x})),\rn{b}))
      ]
      \leq\epsilon]
    \\
    & =
    \PP_{\rn{x}\sim\mu^m}[
      \EE_{\rn{b}\sim U(R_{\cA'}(m))}[
        L_{\mu^{\kpart},F^{\kpart},\ell^{\kpart}}(\cA'(\phi_m(\rn{x}),(F^{\kpart})^*_{\floor{m/k}}(\phi_m(\rn{x})),\rn{b}))
      ]
      \leq\epsilon]
    \\
    & =
    \PP_{\rn{z}\sim(\mu^{\kpart})^{\floor{m/k}}}[
      \EE_{\rn{b}\sim U(R_{\cA'}(m))}[
        L_{\mu^{\kpart},F^{\kpart},\ell^{\kpart}}(\cA'(\rn{z},(F^{\kpart})^*_{\floor{m/k}}(\rn{z}),\rn{b}))
      ]
      \leq\epsilon]
    \\
    & \geq 1-\delta,
  \end{align*}
  where the first equality follows from the definition of $\cA$ and~\eqref{eq:kpart:L}, the second
  equality follows from Lemma~\ref{lem:kpartbasics}\ref{lem:kpartbasics:Phi}, the third equality
  follows since $\phi_m(\rn{x})\sim\rn{z}$ by Lemma~\ref{lem:kpartbasics}\ref{lem:kpartbasics:phi}
  and the inequality follows since $\floor{m/k}\geq
  m^{\PAC}_{\cH^{\kpart},\ell^{\kpart},\cA'}(\epsilon,\delta)$ and $\cA'$ is a randomized $k$-PAC
  learner for $\cH^{\kpart}$.

  Therefore $\cA$ is a randomized $k$-PAC learner for $\cH$ with $m^{\PACr}_{\cH,\ell,\cA} =
  k\cdot\ceil{m^{\PACr}_{\cH^{\kpart},\ell^{\kpart},\cA'}}$.

  \medskip

  We proceed to item~\ref{prop:kpart:partagPAC->agPAC}, whose proof is similar to the one of the
  previous item. Again, it suffices to show the randomized case.

  First, we note that for a Borel template $\Omega'$ and probability templates $\mu\in\Pr(\Omega)$
  and $\mu'\in\Pr(\Omega')$, we have $(\Omega\otimes\Omega')^{\kpart} =
  \Omega^{\kpart}\otimes(\Omega')^{\kpart}$, and $(\mu\otimes\mu')^{\kpart} =
  \mu^{\kpart}\otimes(\mu')^{\kpart}$.

  Let $\iota_{\kpart}^\Omega$, $\iota_{\kpart}^{\Omega'}$ and
  $\iota_{\kpart}^{\Omega\otimes\Omega'}$ be defined by~\eqref{eq:iotakpart} for the spaces
  $\Omega$, $\Omega'$ and $\Omega\otimes\Omega'$, respectively and note that
  \begin{align*}
    \iota_{\kpart}^{\Omega\otimes\Omega'}(x,x') & = (\iota_{\kpart}^\Omega(x),\iota_{\kpart}^{\Omega'}(x')).
  \end{align*}

  Similarly, letting $\phi_m^\Omega$, $\phi_m^{\Omega'}$ and $\phi_m^{\Omega\otimes\Omega'}$ be
  given by~\eqref{eq:kpartbasics:phi} for $\Omega$, $\Omega'$ and $\Omega\otimes\Omega'$, respectively, we
  have
  \begin{align*}
    \phi_m^{\Omega\otimes\Omega'}(x,x') & = (\phi_m^\Omega(x),\phi_m^{\Omega'}(x')).
  \end{align*}

  Note that if $F\in\cF_k(\Omega\otimes\Omega',\Lambda)$, then by
  Lemma~\ref{lem:kpartbasics}\ref{lem:kpartbasics:phi}, we have
  \begin{equation}\label{eq:kpart:agL}
    \begin{aligned}
      L_{\mu,\mu',F,\ell}(H)
      & =
      \EE_{(\rn{x},\rn{x'})\sim(\mu\otimes\mu')^k}[\ell(H,\rn{x},F_k^*(\rn{x},\rn{x'}))]
      \\
      & =
      \EE_{(\rn{x},\rn{x'})\sim(\mu\otimes\mu')^k}[
        \ell^{\kpart}(H^{\kpart},
        \phi_k^\Omega(\rn{x}),
        F^{\kpart}(\phi_k^\Omega(\rn{x}),\phi_k^{\Omega'}(\rn{x'})))
      ]
      \\
      & =
      \EE_{(\rn{z},\rn{z'})\sim((\mu\otimes\mu')^{\kpart})^1}[\ell^{\kpart}(H^{\kpart},\rn{z},F^{\kpart}(\rn{z},\rn{z'}))]
      \\
      & =
      L_{\mu^{\kpart},(\mu')^{\kpart},F^{\kpart},\ell^{\kpart}}(H^{\kpart}).
    \end{aligned}
  \end{equation}
  In particular, since $H\mapsto H^{\kpart}$ is a bijection between $\cH$ and $\cH^{\kpart}$, we
  have
  \begin{align*}
    \inf_{H\in\cH} L_{\mu,\mu',F,\ell}(H) & = \inf_{H\in\cH^{\kpart}} L_{\mu^{\kpart},(\mu')^{\kpart},F^{\kpart},\ell^{\kpart}}(H).
  \end{align*}
  Let $I$ be the infimum above, let $\epsilon,\delta\in(0,1)$ and let $m\geq
  k\cdot\ceil{m^{\agPAC}_{\cH^{\kpart},\ell^{\kpart},\cA'}(\epsilon,\delta)}$ be an integer and note
  that
  \begin{align*}
    & \!\!\!\!\!\!
    \PP_{(\rn{x},\rn{x'})\sim(\mu\otimes\mu')^m}[
      \EE_{\rn{b}}[
        L_{\mu,\mu',F,\ell}(\cA(\rn{x}, F^*_m(\rn{x},\rn{x'}),\rn{b}))
      ]
      \leq I + \epsilon
    ]
    \\
    & =
    \PP_{(\rn{x},\rn{x'})\sim(\mu\otimes\mu')^m}[
      \EE_{\rn{b}}[
        L_{\mu^{\kpart},(\mu')^{\kpart},F^{\kpart},\ell^{\kpart}}(\cA'(\phi_m^{\Omega}(\rn{x}),\Phi_m(F^*_m(\rn{x},\rn{x'})),\rn{b}))
      ]
      \leq I + \epsilon
    ]
    \\
    & =
    \begin{multlined}[t]
      \PP_{(\rn{x},\rn{x'})\sim(\mu\otimes\mu')^m}[
        \EE_{\rn{b}}[
          L_{\mu^{\kpart},(\mu')^{\kpart},F^{\kpart},\ell^{\kpart}}(
          \cA'(\phi_m^\Omega(\rn{x}),F^*_m(\phi_m^\Omega(\rn{x}),\phi_m^{\Omega'}(\rn{x'})),\rn{b}))
        ]
        \\
        \leq I + \epsilon
      ]
    \end{multlined}
    \\
    & =
    \PP_{(\rn{z},\rn{z'})\sim((\mu\otimes\mu')^{\kpart})^{\floor{m/k}}}[
      \EE_{\rn{b}}[
        L_{\mu^{\kpart},(\mu')^{\kpart},F^{\kpart},\ell^{\kpart}}(
        \cA'(\rn{z},F^*_m(\rn{x},\rn{z'}),\rn{b}))
      ]
      \leq I + \epsilon
    ]
    \\
    & \geq
    1 - \delta,
  \end{align*}
  where all $\rn{b}$ are picked uniformly at random in $[R_\cA(m)]=[R_{\cA'}(m)]$, the first equality
  follows from the definition of $\cA$ and~\eqref{eq:kpart:agL}, the second equality follows from
  Lemma~\ref{lem:kpartbasics}\ref{lem:kpartbasics:Phi}, the third equality follows since
  \begin{align*}
    (\phi_m^\Omega(\rn{x}),\phi_m^{\Omega'}(\rn{x'}))
    =
    \phi_m^{\Omega\otimes\Omega'}(\rn{x},\rn{x'})
    \sim
    (\rn{z},\rn{z'})
  \end{align*}
  by Lemma~\ref{lem:kpartbasics}\ref{lem:kpartbasics:phi} and the inequality follows since
  $\floor{m/k}\geq m^{\agPAC}_{\cH^{\kpart},\ell^{\kpart},\cA'}(\epsilon,\delta)$ and $\cA'$ is a
  randomized agnostic $k$-PAC learner for $\cH^{\kpart}$.

  Therefore $\cA$ is a randomized agnostic $k$-PAC learner for $\cH$ w.r.t.\ $\ell$ with
  $m^{\agPACr}_{\cH,\ell,\cA} = k\cdot\ceil{m^{\agPACr}_{\cH^{\kpart},\ell^{\kpart},\cA'}}$.
\end{proof}

\subsection{Non-partite agnostic learnability implies partite agnostic learnability}
\label{subsec:nonpart->part}

Before we proceed to the (partial) converses of Proposition~\ref{prop:kpart},
Propositions~\ref{prop:kpart2} and~\ref{prop:kpart3} below, we need the definition below of neutral
symbols and we need to recall the concept of a Markov kernel from probability theory.

\begin{definition}[Neutral symbols]\label{def:neutsymb}
  Let $k\in\NN_+$.
  \begin{enumdef}
  \item\label{def:neutsymb:nonpart} For a $k$-ary agnostic loss function
    $\ell\colon\cH\times\cE_k(\Omega)\times\Lambda^{S_k}\to\RR_{\geq 0}$, we say that
    $\bot\in\Lambda$ is a \emph{neutral symbol} for $\ell$ if for every $H,H'\in\cH$, every
    $x\in\cE_k(\Omega)$ and every $y,y'\in\Lambda^{S_k}$ with $\bot\in\im(y)$ and $\bot\in\im(y')$,
    we have
    \begin{align*}
      \ell(H,x,y) & = \ell(H',x,y').
    \end{align*}

    Note that when $\bot$ is a neutral symbol for $\ell$, we can define a function
    $\ell_\bot\colon\cE_k(\Omega)\to\RR_{\geq 0}$ by
    \begin{align}\label{eq:neutsymbellbot}
      \ell_\bot(x)
      & \df
      \ell(H,x,y)
      \qquad (x\in\cE_k(\Omega)),
    \end{align}
    where $H$ is any hypothesis in $\cH$ and $y$ is any point in $\Lambda^{S_k}$ with $\bot\in\im(y)$.
  \item For a $k$-partite agnostic loss function
    $\ell\colon\cH\times\cE_1(\Omega)\times\Lambda\to\RR_{\geq 0}$, we say that $\bot\in\Lambda$ is a
    \emph{neutral symbol} for $\ell$ if for every $H,H'\in\cH$ and every $x\in\cE_1(\Omega)$, we have
    \begin{align*}
      \ell(H,x,\bot) & = \ell(H',x,\bot)
    \end{align*}

    When $\bot$ is a neutral symbol for $\ell$, we can define a function
    $\ell_\bot\colon\cE_1(\Omega)\to\RR_{\geq 0}$ by
    \begin{align}\label{eq:partiteneutsymbellbot}
      \ell_\bot(x)
      & \df
      \ell(H,x,\bot)
      \qquad (x\in\cE_1(\Omega)),
    \end{align}
    where $H$ is any hypothesis in $\cH$.
  \end{enumdef}
\end{definition}

The intuition of a neutral symbol $\bot$ for $\ell$ is that if $\bot$ appears in the sample, then all
possible guesses for $H\in\cH$ give the same penalty under $\ell$, so $\bot$ does not interfere in the
learning task. Even though we could define the notion of a neutral symbol for non-agnostic loss
functions analogously, we skip it since we do not currently have any application for such a concept.

\begin{remark}\label{rmk:neutsymb}
  Let us note that if $\ell$ (partite or not) has a neutral symbol $\bot$, then $\ell$ is
  flexible. Indeed, by letting $\Sigma$ be such that every space of $\Sigma$ has a single point and
  letting $G\in\cF_k(\Omega\otimes\Sigma,\Lambda)$ be the function that is constant and equal to
  $\bot$, we get
  \begin{align*}
    \EE_{\rn{z}\sim\nu^k}[\ell(H,x,G^*_k(x,\rn{z}))]
    & =
    \ell_\bot(x)
    \qquad (H\in\cH, x\in\cE_k(\Omega))
  \end{align*}
  in the non-partite case and
  \begin{align*}
    \EE_{\rn{z}\sim\nu^1}[\ell(H,x,G(x,\rn{z}))]
    & =
    \ell_\bot(x)
    \qquad (H\in\cH, x\in\cE_1(\Omega))
  \end{align*}
  in the partite case. Since the right-hand sides of the equations above do not depend on $H$, we
  get~\eqref{eq:EEindepH} (\eqref{eq:partiteEEindepH}, in the partite case). The existence of the
  corresponding $\cN$ then follow by Remarks~\ref{rmk:flexibilityfinite}
  and~\ref{rmk:partiteflexibilityfinite}.

  We will see in Proposition~\ref{prop:neutsymb} below that when an agnostic loss function is
  bounded and flexible, we can add neutral symbols to it (i.e., increase $\Lambda$) without
  disrupting agnostic $k$-PAC learnability. Thus, for bounded loss functions, having a neutral
  symbol is morally equivalent to flexibility.
\end{remark}

Even though the definition of a neutral symbol will be extremely convenient for the proof of
Proposition~\ref{prop:kpart2} below, it has a major flaw: except in the trivial cases when $\Lambda$
has size $1$ or $\cH$ is empty, the agnostic $0/1$-loss does not have neutral symbols. However the
agnostic $0/1$-loss is bounded and flexible when $\Lambda$ is finite (see
Lemma~\ref{lem:flexibility}), which means that we will be able to produce a neutral symbol for it
with Proposition~\ref{prop:neutsymb} below.

\medskip

Recall that a \emph{Markov kernel} from a Borel space $\Omega=(X,\cB)$ to a Borel space
$\Lambda=(Y,\cB')$ is a function $\rho\colon\Omega\to\Pr(\Lambda)$ such that for every measurable set
$B'\in\cB'$, the function $X\ni x\mapsto\rho(x)(B')\in[0,1]$ is measurable (with respect to
$\cB$). Given further a probability measure $\mu\in\Pr(\Omega)$, \Caratheodory's Extension Theorem
implies that there exists a unique measure $\rho[\mu]$ on $\Omega\otimes\Lambda$ such that
\begin{align*}
  \rho[\mu](B\times B') & = \int_B \rho(x)(B')\ d\mu(x)
\end{align*}
for all measurable sets $B\in\cB$ and $B'\in\cB'$. In particular, if $(\rn{x},\rn{y})\sim\rho[\mu]$, then
\begin{align*}
  \PP[\rn{y}\in B'\given \rn{x}\in B]
  & =
  \frac{1}{\mu(B)}\int_B \rho(x)(B')\ d\mu(x).
\end{align*}
If we take $B$ to be smaller and smaller sets converging to $\{x\}$ and hope that some version of
Lebesgue Differentiation Theorem holds, we can intuitively interpret $\rho(x)(B')$ as the conditional
probability that $\rn{y}\in B'$ given that $\rn{x}=x$ (this justifies another name used for Markov
kernels: regular conditional probabilities). This is formalized by the Disintegration Theorem below
(for a proof, see e.g.~\cite[Theorem~10.6.6 and Corollary~10.6.7]{Bog07}).

\begin{theorem}[Disintegration Theorem]
  \label{thm:disintegration}
  Let $\Omega=(X,\cB)$ and $\Lambda=(Y,\cB')$ be Borel spaces, let $\nu\in\Pr(\Lambda)$, let
  $\pi\colon\Lambda\to\Omega$ be a measurable function and let $\mu\df\pi_*(\nu)\in\Pr(\Omega)$ be
  the pushforward measure. Then there exists a Markov kernel $\rho\colon\Omega\to\Pr(\Lambda)$ such
  that the set
  \begin{align*}
    R & \df \{x\in\Omega \mid \rho(x)(\pi^{-1}(x)) = 1\}
  \end{align*}
  has $\mu$-measure $1$ and for every measurable function $\psi\colon\Lambda\to [0,\infty]$, we have
  \begin{align*}
    \int_Y \psi(y)\ d\nu(y)
    & =
    \int_R \int_{\pi^{-1}(x)} \psi(y)\ d\rho(x)(y) \ d\mu(x).
  \end{align*}

  In particular, the marginal of $\rho[\mu]$ on $\Lambda$ is $\nu$, that is, we have
  \begin{align*}
    \nu(B') & = \rho[\mu](X\times B') = \int_X \rho(x)(B')\ d\mu(x)
  \end{align*}
  for every $B'\in\cB'$.
\end{theorem}

The next proposition says that the map $\cH\mapsto\cH^{\kpart}$ preserves agnostic $k$-PAC
learnability with randomness as long as the agnostic loss function is symmetric, bounded and has a
neutral symbol. We will see in Proposition~\ref{prop:neutsymb} below that if $\ell$ does not have
neutral symbols, then they can be artificially produced as long as $\ell$ is at least flexible.

\begin{proposition}[Non-partite to partite with neutral symbols]\label{prop:kpart2}
  Let $\Omega$ be a Borel template, let $k\in\NN_+$, let $\Lambda$ be a non-empty Borel space, let
  $\cH\subseteq\cF_k(\Omega,\Lambda)$ be a $k$-ary hypothesis class and let
  $\ell\colon\cH\times\cE_k(\Omega)\times\Lambda^{S_k}\to\RR_{\geq 0}$ be a symmetric $k$-ary
  agnostic loss function with $\lVert\ell\rVert_\infty < \infty$. Suppose further that $\ell$ has a
  neutral symbol.
  
  If $\cH$ is agnostically $k$-PAC learnable with randomness with respect to $\ell$, then
  $\cH^{\kpart}$ is agnostically $k$-PAC learnable with randomness with respect to $\ell^{\kpart}$;
  more precisely, if $\cA$ is a randomized agnostic $k$-PAC learner for $\cH$, then there exists a
  randomized $k$-PAC learner $\cA'$ for $\cH^{\kpart}$ with
  \begin{gather*}
    R_{\cA'}(m)
    \df
    R_\cA(m)\cdot m!\cdot\prod_{i=1}^k \binom{k}{i}^{2\binom{m}{i}}
    \leq
    R_\cA(m)\cdot m!\cdot k^{2\cdot k\cdot m^k},
    \\
    m^{\agPACr}_{\cH^{\kpart},\ell^{\kpart},\cA'}(\epsilon,\delta)
    \df
    m^{\agPACr}_{\cH,\ell,\cA}\left(\frac{p\cdot\epsilon}{2},\widetilde{\delta}_\ell(\epsilon,\delta)\right),
  \end{gather*}
  where
  \begin{align}
    \widetilde{\delta}_\ell(\epsilon,\delta)
    & \df
    \min\left\{\frac{\epsilon\delta}{2\lVert\ell\rVert_\infty}, \frac{1}{2}\right\},
    \label{eq:kpart2:wdelta}
    \\
    p & \df \prod_{i=1}^k \binom{k}{i}^{-2\binom{k}{i}} \geq \frac{1}{2^{k\cdot 2^k}}.
    \label{eq:kpart2:p}
  \end{align}
\end{proposition}

\begin{proofint}
  Since this is arguably the most technical result of this article, let us first give an idea of the
  proof in a simpler case: suppose that $\cH$ is agnostically $k$-PAC learnable (without randomness)
  by some agnostic $k$-PAC learner $\cA$ and we want to show that $\cH^{\kpart}$ is (non-agnostically)
  $k$-PAC learnable with randomness. In fact, let us think of the even simpler case where our
  hypotheses are $k$-hypergraphs, i.e., $\Omega_i$ has a single point for every $i\geq 2$ (see
  Remark~\ref{rmk:rk1}), $\Lambda=\{0,1\}$ and every $H\in\cH$ is $S_k$-invariant in the sense that
  $H^*_k(x)=H^*_k(\sigma^*(x))$ for every $\sigma\in S_k$ (this is not the same as the
  $S_k$-equivariance that always holds, cf.~Lemma~\ref{lem:F*Vequiv}). This means we can think of
  our hypotheses as $S_k$-invariant functions $F\colon\Omega_1^k\to\{0,1\}$ and the $S_k$-invariance
  allows us to think of the partite versions $F^{\kpart}\colon\Omega_1^k\to\{0,1\}^{S_k}$ as just
  functions $\Omega_1^k\to\{0,1\}$ since
  \begin{align*}
    F^{\kpart}(x)_\sigma
    & =
    F^*_k(\iota_{\kpart}(x))_\sigma
    =
    \sigma^*(F^*_k(\iota_{\kpart}(x)))_{\id_k}
    \\
    & =
    F^*_k(\sigma^*(\iota_{\kpart}(x)))_{\id_k}
    =
    F^*_k(\iota_{\kpart}(x))_{\id_k}.
  \end{align*}
  We will also be ignoring higher-order variables throughout this proof intuition, with the
  exception of item~\ref{it:final} below.

  Assume we are trying to (non-agnostically) learn some (realizable) $k$-partite $k$-hypergraph
  $F\colon\Omega_1^k\to\{0,1\}$ under some $\mu\in\Pr(\Omega^{\kpart})$. When our algorithm $\cA'$
  receives an $([m],\ldots,[m])$-sample $(\rn{x},\rn{y})$ of the form $\rn{x}\sim\mu^m$ and
  $\rn{y}=F^*_m(\rn{x})$, we would like to pass it to $\cA$ and leverage the agnostic $k$-PAC
  learning guarantee of $\cA$ to claim $\cA'$ is a $k$-PAC learner. However, several complications
  arise:
  \begin{enumerate}
  \item\label{it:bot} If we try to pass all $mk$ points $(\rn{x}_{i\mapsto j})_{i\in[k], j\in[m]}$,
    where $i\mapsto j$ denotes the unique function $\{i\}\to [m]$ that maps $i$ to $j$
    ($\rn{x}_{i\mapsto j}$ is the $j$th point in the $i$th part), then we only know information
    about $k$-sets that are transversal to the $k$-partition, that is, our $\rn{y}$ has edge
    information only about $k$-sets of the form $\{\rn{x}_{1\mapsto j_1},\ldots,\rn{x}_{k\mapsto
      j_k}\}$ for some $j\in[m]^k$; we are missing information about $k$-sets that contain more than
    one point in one of the parts, that is, $k$-sets of the form $\{\rn{x}_{i_1\mapsto
      j_1},\ldots,\rn{x}_{i_k\mapsto j_k}\}$ where the $(i_t,j_t)$ are pairwise distinct, but the
    $i_t$ are not pairwise distinct.

    The way we are going to address this is to use the neutral symbol $\bot$ to mean ``I
    don't know'' and label all the missing tuples with $\bot$.
  \item\label{it:permute} Naively using the $\bot$ trick of item~\ref{it:bot} amounts to defining a
    (non-partite) $[mk]$-sample $(\rn{\widetilde{x}},\rn{\widetilde{y}})$ by
    \begin{align*}
      \rn{\widetilde{x}}_v & \df \rn{x}_{\ceil{v/k}\mapsto ((v-1)\bmod k) + 1}
      \qquad (v\in[mk]),
    \end{align*}
    and for every $\alpha\in([mk])_k$, letting
    \begin{align*}
      \rn{\widetilde{y}}_\alpha
      & \df
      \begin{dcases*}
        \rn{y}_{(((\alpha(\tau(t))-1)\bmod k) + 1)_{t=1}^k}, & if $t\mapsto \ceil{\alpha(\tau(t))/k}$
        is increasing for $\tau\in S_k$,
        \\
        \bot, & otherwise,
      \end{dcases*}
    \end{align*}
    and running $\cA'(\rn{\widetilde{x}},\rn{\widetilde{y}})$.

    For simplicity, in this proof intuition, we assume that $F$ never outputs $\bot$, so that the
    first case in the above will never produce $\bot$. However, we point out that this will not be
    an issue in the actual proof: $F$ could output $\bot$ and the proof still goes through.

    To leverage the agnostic $k$-PAC learning guarantee, we need to claim that
    $(\rn{\widetilde{x}},\rn{\widetilde{y}})$ is distributed as
    $(\rn{z},\widehat{F}^*_{mk}(\rn{z},\rn{z'}))$ where
    $(\rn{z},\rn{z'})\sim(\widehat{\mu}\otimes\widehat{\mu'})^m$ for some \emph{non-partite}
    $\widehat{\mu}\in\Pr(\Omega)$ and $\widehat{\mu'}\in\Pr(\Omega')$ (where $\Omega'$ is some
    (non-partite) Borel template) and some $\widehat{F}\in\cF_k(\Omega\otimes\Omega',\{0,1\})$. By
    Remark~\ref{rmk:exchangeable}, if that is to hold, at the very least we need
    $(\rn{\widetilde{x}},\rn{\widetilde{y}})$ to be local and exchangeable (more precisely, the
    finite marginal of a local and exchangeable distribution). Unfortunately, this is obviously
    false: the first $m$ points of $\rn{x}$ are from part $1$, the next $m$ points are from part $2$
    and so on, which means that we are guaranteed that all ``I don't know'' values $\bot$ happen in
    a non-exchangeable manner: $\rn{y}_{(1,m+1,\ldots,(k-1)m+1)}$ is never $\bot$ and
    $\rn{y}_{(1,2,\ldots,k)}$ is always $\bot$ if $m > k$.

    This is where our partite algorithm $\cA$ needs to start using randomness: one way we can
    enforce exchangeability is by letting $\cA$ uniformly permute the points before passing them to
    $\cA'$, that is, if $\rn{\sigma}$ is picked uniformly at random in $S_{mk}$, then we run instead
    $\cA'(\rn{\sigma}^*(\rn{\widetilde{x}}),\rn{\sigma}^*(\rn{\widetilde{y}}))$. If now
    $(\rn{\sigma}^*(\rn{\widetilde{x}}),\rn{\sigma}^*(\rn{\widetilde{y}}))$ has a distribution of
    the form $(\rn{z},\widehat{F}^*_{mk}(\rn{z},\rn{z'}))$, then the agnostic $k$-PAC learning
    guarantee of $\cA'$ will translate to a $k$-PAC learning guarantee for $\cA'$ in expected value
    over the randomness $\rn{\sigma}$.
  \item\label{it:coinflip} If we naively use the trick of item~\ref{it:permute}, even though the
    distribution of $(\rn{\sigma}^*(\rn{\widetilde{x}}),\rn{\sigma}^*(\rn{\widetilde{y}}))$ is
    indeed exchangeable, it is not local as we are guaranteed that each part has exactly $m$ points:
    say, the event
    \begin{align*}
      \forall \alpha\in ([m(k-1)])_k, \rn{\sigma}^*(\rn{\widetilde{y}})_\alpha = \bot
    \end{align*}
    implies the event
    \begin{align*}
      \forall \alpha\in (\{m(k-1)+1,\ldots,m\})_k, \rn{\sigma}^*(\rn{\widetilde{y}})_\alpha = \bot,
    \end{align*}
    so these are not independent even though they concern marginals on disjoint sets.

    To address this, instead of producing an $[mk]$-sample from our $([m],\ldots,[m])$-sample along
    with an extra permutation randomness, we produce an $[m]$-sample with the help of even extra
    randomness. For each point $v\in[m]$, we first use an extra $\rn{u}_v$ picked uniformly at
    random in $[k]$ (independently from other $\rn{u}_{v'}$) that determines which part the $v$th
    point is going to come from, that is, we let instead
    \begin{align*}
      \rn{x}^{\rn{u}}_v & \df \rn{x}_{\rn{u}_v\mapsto v}
      \qquad (v\in[m]),
    \end{align*}
    and for every $\alpha\in([m])_k$, let
    \begin{align*}
      \rn{y}^{\rn{u}}_\alpha
      & \df
      \begin{dcases*}
        \rn{y}_{\alpha\comp\tau}, & if $t\mapsto\rn{u}_{\tau(t)}$
        is increasing for $\tau\in S_k$,
        \\
        \bot, & otherwise.
      \end{dcases*}
    \end{align*}
    We still permute the points afterwards by passing
    $(\rn{\sigma}^*(\rn{x}^{\rn{u}}),\rn{\sigma}^*(\rn{x}^{\rn{u}}))$ to $\cA$ for $\rn{\sigma}$
    picked uniformly at random in $S_m$ (independently from $\rn{u}$).
    
    (Note that with positive probability, we could be passing points all in the same part, which
    would mean that $\rn{y}$ is the ``I don't know'' value $\bot$ everywhere, but this will not be
    an issue as we will know the precise probability that this happens.)
  \item\label{it:agnostic} It turns out that the trick of item~\ref{it:coinflip} is enough to
    guarantee that $(\rn{\sigma}^*(\rn{x}^{\rn{u}}),\rn{\sigma}^*(\rn{y}^{\rn{u}}))$ is local and
    exchangeable, that is, it is of the form $(\rn{z},F^*_m(\rn{z},\rn{z'}))$ with
    $(\rn{z},\rn{z'})\sim(\widehat{\mu}\otimes\widehat{\mu'})^m$ for some
    $\widehat{\mu}\in\Pr(\Omega)$, some $\widehat{\mu'}\in\Pr(\Omega')$ (some Borel template
    $\Omega'$) and some $\widehat{F}\in\cF_k(\Omega\otimes\Omega',\{0,1\})$, but we need to
    understand exactly what such a $(\widehat{\mu},\widehat{\mu'},\widehat{F})$ looks like as we
    will need to leverage the agnostic $k$-PAC learning guarantee of $\cA$ for
    $(\widehat{\mu},\widehat{\mu'},\widehat{F})$ to give a $k$-PAC learning guarantee (in expected
    value over $(\rn{u},\rn{\sigma})$) for $\cA'$ for $\mu,F$.

    Note that since $\mu$ is in the partite setting, we potentially have $k$ different measures
    $(\mu_{\{i\}})_{i\in[k]}$ on each copy of $\Omega_1$ and the non-partite setting cannot simulate
    that with a single measure over $\Omega_1$ and no extra help, but let us consider first the
    simpler case when all $\mu_{\{i\}}$ are the same measure $\nu$. In this simple case, it is very
    easy to see that $\widehat{\mu}=\nu$, that is, the entries of $\rn{\sigma}^*(\rn{x}^{\rn{u}})$
    are i.i.d.\ distributed according to $\nu$. However, note that we cannot simply ignore $\Omega'$
    and take $\widehat{F}\in\cF_k(\Omega,\{0,1\})$ as, e.g., the non-partite version $F^{\kpart,-1}$
    of $F$ does not assign the ``I don't know'' value $\bot$ to any $k$-set; in fact, even worse:
    the same $k$-set can be assigned different values, depending on which parts the points from the
    $k$-set comes from. This is why we need the power of agnostic: we can let $\Omega'$ encode the
    parts by taking $\Omega'_1\df[k]$, let $\widehat{\mu'}$ be the uniform measure on $[k]$ and let
    $\widehat{F}$ be given by
    \begin{align*}
      & \!\!\!\!\!\!
      \widehat{F}((z_1,\ldots,z_k),(z'_1,\ldots,z'_k))
      \\
      & \df
      \begin{dcases*}
        F(z_{\tau(1)},\ldots,z_{\tau(k)}),
        & if $t\mapsto z'_{\tau(t)}$ is increasing for $\tau\in S_k$,
        \\
        \bot, & otherwise.
      \end{dcases*}
    \end{align*}
    One can then check that if $(\rn{z},\rn{z'})\sim(\widehat{\mu}\otimes\widehat{\mu'})^m$, then
    $(\rn{\sigma}^*(\rn{x}^{\rn{u}}),\rn{\sigma}^*(\rn{y}^{\rn{u}}))\sim (\rn{z},
    \widehat{F}^*_m(\rn{z},\rn{z'}))$.
  \item The only issue with item~\ref{it:agnostic} is that it heavily relied on the fact that all
    $\mu_{\{i\}}$ were the same measure. In the general case, these measures are potentially
    different and it is easy to see that the entries of $\rn{\sigma}^*(\rn{x}^{\rn{u}})$ are
    i.i.d.\ distributed according to the average $\widehat{\mu}\df(1/k)\sum_{i\in[k]}\mu_{\{i\}}$ of
    the measures. However, the definitions of $\widehat{F}$ and $\widehat{\mu'}$ are not so
    simple. For example, if the supports of the $\mu_{\{i\}}$ are pairwise disjoint, then once we
    sample $\rn{z}\sim\widehat{\mu}^m$, we immediately know which part each point belongs to.

    Instead, we want $\widehat{\mu}$ and $\widehat{F}$ to take into account the ``conditional
    probability that $\rn{z}_v$ comes from part $i$ given that $\rn{z}_v = z_v$'', which might be
    nonsensical if the condition has zero measure.

    This is where the Disintegration Theorem, Theorem~\ref{thm:disintegration}, comes into play: we
    define $\nu$ as the unique probability distribution on $\Omega_1\times[k]$ such that
    \begin{align*}
      \nu(W\times\{j\}) & \df \mu_{\{j\}}(W) \qquad (W\subseteq\Omega_1, j\in[k])
    \end{align*}
    and let $\rho\colon\Omega_1\to\Pr([k])$ be the Markov kernel given by
    Theorem~\ref{thm:disintegration} applied to $\nu$ and the projection
    $\pi\colon\Omega_1\times[k]\to\Omega_1$, so that intuitively $\rho(x)(\{j\})$ is the conditional
    probability that $\rn{j}=j$ given $\rn{x}=x$ when $(\rn{x},\rn{j})\sim\nu$. By noting that
    $\pi_*\nu = \widehat{\mu}$, it follows that $\rho(z)(\{j\})$ is the desired notion of
    ``conditional probability that $\rn{z}$ comes from part $j$ given $\rn{z}=z$''.

    Finally, to encode this within $\widehat{F}$ and $\widehat{\mu'}$, we take $\Omega'_1\df[0,1)$
      and $\widehat{\mu'}\df\lambda$ as the Lebesgue measure and set
    \begin{align*}
      & \!\!\!\!\!\!
      \widehat{F}((z_1,\ldots,z_k),(z'_1,\ldots,z'_k))
      \\
      & \df
      \begin{dcases*}
        F(z_{\tau(1)},\ldots,z_{\tau(k)}),
        & if $t\mapsto j_{z_{\tau(t)}}(z'_{\tau(t)})$ is increasing for $\tau\in S_k$,
        \\
        \bot, & otherwise,
      \end{dcases*}
    \end{align*}
    where $j_w(w')$ is the unique element of $[k]$ such that
    \begin{align*}
      \rho(w)([j_w(w')-1]) & \leq w' < \rho(w)([j_w(w')])
    \end{align*}
    that is, when we partition $[0,1)$ into intervals of measures $(\rho(w)(\{j\}))_{j=1}^k$, in
      this order, the point $w'$ is in the $j_w(w')$th of them.

    We will then check that if $(\rn{z},\rn{z'})\sim(\widehat{\mu}\otimes\widehat{\mu'})^m$, then
    $(\rn{\sigma}^*(\rn{x}^{\rn{u}}),\rn{\sigma}^*(\rn{y}^{\rn{u}}))\sim (\rn{z},
    \widehat{F}^*_m(\rn{z},\rn{z'}))$.
  \item So far our plan is to set
    \begin{align*}
      \cA'(x,y,\sigma,u) & \df \cA(\sigma^*(x^u),\sigma^*(y^u))^{\kpart}
    \end{align*}
    that is, we use the sources of randomness $\sigma\in S_m$ and $u\in[k]^m$ to make the
    construction of item~\ref{it:coinflip}, run $\cA$ on the result $(\sigma^*(x^u),\sigma^*(y^u))$ to
    obtain some $H$ and return the its partite version $H^{\kpart}$.

    One final ingredient is needed to leverage the agnostic $k$-PAC learnability guarantee of $\cA$
    for $(\widehat{F},\widehat{\mu},\widehat{\mu'})$ to yield a $k$-PAC learnability guarantee for
    $\cA'$ for $(F,\mu)$ in expected value over $(\rn{\sigma},\rn{u})$: we need to know how the total
    loss of $L_{\mu,F}(H^{\kpart})$ compares to the total loss
    $L_{\widehat{\mu},\widehat{\mu'},\widehat{F}}(H)$. In fact, we will need to know how these two
    quantities compare for every $H\in\cH$ not just the ones returned by the algorithm as the
    agnostic $k$-PAC learnability guarantee is not that the total loss is small with high
    probability, but rather that one gets a total loss that is very close to the best possible
    $H\in\cH$ with high probability.

    This is where we observe yet another property of our construction: if we sample
    $(\rn{z},\rn{z'})\sim(\widehat{\mu}\otimes\widehat{\mu'})^k$ and condition on the event
    \begin{align}\label{eq:condition}
      \widehat{F}(\rn{z},\rn{z'})\neq\bot,
    \end{align}
    then there must exist a unique permutation $\rn{\tau}$ in $S_k$ such that $t\mapsto
    j_{\rn{z}_{\rn{\tau}(t)}}(\rn{z'}_{\rn{\tau}(t)})$ is increasing, which means that the point
    \begin{align*}
      (\rn{\tau}^*(\rn{z}), \widehat{F}^*(\rn{\tau}^*(\rn{z},\rn{z'})))
    \end{align*}
    naturally corresponds to the partite point
    \begin{align}\label{eq:phiPhi}
      (\phi_k(\rn{\tau}^*(\rn{z})), \Phi_k(\widehat{F}^*_k(\rn{\tau}^*(\rn{z},\rn{z'})))),
    \end{align}
    where $\phi_k$ and $\Phi_k$ are given by~\eqref{eq:kpartbasics:phi}
    and~\eqref{eq:kpartbasics:Phi}, respectively. We will then check that the conditional
    distribution of~\eqref{eq:phiPhi} given the event~\eqref{eq:condition} is the same as the
    (unconditional) distribution of
    \begin{align*}
      (\rn{x}, F(\rn{x}))
    \end{align*}
    when $\rn{x}\sim\mu^1$. Since $\ell$ is symmetric, the permutation $\rn{\tau}$ does not affect
    it, so a direct consequence of the above is
    \begin{align}\label{eq:LL}
      L_{\widehat{\mu},\widehat{\mu'},\widehat{F},\ell}(H)
      & =
      (1-p)\cdot C_{\ell,\bot,F} + 
      p\cdot L_{\mu,F,\ell^{\kpart}}(H^{\kpart}),
    \end{align}
    where $p\df\PP_{\rn{z},\rn{z'}}[\widehat{F}^*_k(\rn{z},\rn{z'})\neq\bot]$ and
    \begin{align*}
      C_{\ell,\bot,F}
      & \df
      \EE_{(\rn{z},\rn{z'})\sim(\widehat{\mu}\otimes\widehat{\mu'})^k}[
        \ell_\bot(\rn{z}) \given
        \widehat{F}(\rn{z},\rn{z'})
        =
        \bot
      ],
    \end{align*}
    and $\ell_\bot$ is as in the definition of neutral symbol (see~\eqref{eq:neutsymbellbot} in
    Definition~\ref{def:neutsymb:nonpart}). It is not too hard to compute $p=k^{-k}\cdot k!$ (this
    is not the same as~\eqref{eq:kpart2:p} because of item~\ref{it:final} below).

    To put everything together: since $F$ is realizable for $\mu$, a direct consequence
    of~\eqref{eq:LL} is that
    \begin{align}\label{eq:infLL}
      \inf_{H\in\cH} L_{\widehat{\mu},\widehat{\mu'},\widehat{F},\ell}(H)
      & =
      (1-p)\cdot C_{\ell,\bot,F},
    \end{align}
    so the fact that $\cA$ is an agnostic $k$-PAC learner guarantees that if
    \begin{align*}
      m
      \geq
      m^{\agPAC}_{\cH,\ell,\cA}\left(\frac{p\cdot\epsilon}{2},\widetilde{\delta}_\ell(\epsilon,\delta)\right)
    \end{align*}
    is an integer, $\rn{x}\sim\mu^m$ and $\rn{\sigma}$ and $\rn{u}$ are picked uniformly and
    independently (and independently from $\rn{x}$) in $S_m$ and $[k]^m$, respectively, then since
    $(\rn{\sigma}^*(\rn{x}^{\rn{u}}),\rn{\sigma}^*(F^*_m(\rn{x})^{\rn{u}}))$ has the same
    distribution as $(\rn{z},\widehat{F}^*_m(\rn{z},\rn{z'}))$ for
    $(\rn{z},\rn{z'})\sim(\widehat{\mu}\otimes\widehat{\mu'})^m$, by~\eqref{eq:LL}
    and~\eqref{eq:infLL}, we have
    \begin{align*}
      & \!\!\!\!\!\!
      \PP_{\rn{x},\rn{\sigma},\rn{u}}\biggl[
        L_{\mu,F,\ell^{\kpart}}(\cA'(\rn{x},\rn{\sigma},\rn{u}))
        \leq\frac{\epsilon}{2}
        \biggr]
      \\
      & =
      \PP_{\rn{z},\rn{z'}}\biggl[
        L_{\widehat{\mu},\widehat{\mu'},\widehat{F},\ell}\Bigl(
        \cA(\rn{z},\widehat{F}^*_m(\rn{z},\rn{z'}))
        \Bigl)
        \leq
        (1-p)\cdot C_{\ell,\bot,F}
        +
        \frac{p\cdot\epsilon}{2}
        \biggr]
      \\
      & \geq 1 - \widetilde{\delta}_\ell(\epsilon,\delta).
    \end{align*}

    With a simple Markov's Inequality plus boundedness argument, and by our choice of
    $\widetilde{\delta}_\ell(\epsilon,\delta)$, the above then implies
    \begin{align*}
      \PP_{\rn{x}}\biggl[
        \EE_{\rn{\sigma},\rn{u}}\Bigl[
          L_{\mu,F,\ell^{\kpart}}\bigl(
          \cA'(\rn{x},\rn{\sigma},\rn{u})
          \bigr)
          \Bigr]
        \leq\epsilon
        \biggr]
      & \geq 1 - \delta.
    \end{align*}
  \item\label{it:final} Five final complications were not covered in this proof intuition:

    First, the adversary's chosen $k$-partite hypothesis $F$ could also output the neutral symbol
    $\bot$. It turns out that this is not a problem if we keep track of when $\bot$ was produced
    from $F$ and when it was produced artificially as an ``I don't know'' symbol.

    Second, we could be dealing with not necessarily symmetric hypotheses (e.g., directed
    $k$-hypergraphs as opposed to $k$-hypergraphs). This is not a big deal, we will only need to be
    slightly more careful with the equivariance of our definitions. In particular, this will require
    one extra coordinate $\rn{w}$ on the non-partite agnostic side that will enforce us to say the
    ``I don't know'' value if the $k$-tuple is not in increasing order, thus the value of $p$ will
    get slightly worse: $p=k^{-k}$.

    Third, the proposition only assumes that $\cH$ is agnostically $k$-PAC learnable with
    randomness, that is, the non-partite agnostic $k$-PAC learner $\cA$ also uses randomness. Again,
    this is not a big deal since our partite $\cA'$ will already use randomness anyways, so it can
    also receive the randomness of $\cA$ (along with the extra needed $\rn{\sigma}$ and $\rn{u}$) and
    pass it directly to $\cA$.

    Fourth, our hypothesis class is not necessarily of rank at most $1$. This means that all ideas
    above need to be extended to cover higher-order variables. The extension is completely
    analogous, but it gives a slightly worse value for $p$:
    \begin{align}\label{eq:kpart2:pnonag}
      p & = \prod_{i=1}^k \binom{k}{i}^{-\binom{k}{i}}.
    \end{align}

    Fifth, the proposition actually wants agnostic $k$-PAC learnability of $\cH^{\kpart}$ instead of
    $k$-PAC learnability of $\cH^{\kpart}$, which means that the adversary does not play partite
    $(\mu,F)$, but rather partite $(\mu,\mu',F)$. It turns out that we can treat the extra variables
    coming from $\mu'$ in the same way as we treat those from $\mu$ and the analogue
    of~\eqref{eq:LL} will be strong enough to carry out the proof even without the realizability
    assumption. However, since the trick will need to be applied twice, once for $\mu$ and once for
    $\mu'$, this will make the value of $p$ the square of~\eqref{eq:kpart2:pnonag}:
    \begin{align*}
      p & = \prod_{i=1}^k \binom{k}{i}^{-2\binom{k}{i}},
    \end{align*}
    which finally justifies~\eqref{eq:kpart2:p}.\qedhere
  \end{enumerate}
\end{proofint}

\begin{remark}\label{rmk:sepexch->exch}
  Before we start the formal proof, let us point out what is going on from the point of view of
  exchangeability theory. What we are showing is that for every local separately exchangeable
  distribution $(\rn{x},\rn{y})$ on
  $\cE_{\NN_+,\ldots,\NN_+}(\Omega^{\kpart})\times\Lambda^{\NN_+^k}$, there exists a local
  exchangeable distribution $(\rn{\widehat{x}},\rn{\widehat{y}})$ on
  $\cE_{\NN_+}(\Omega)\times\Lambda^{(\NN_+)_k}$ and functions
  $h_m\colon\cE_m(\Omega^{\kpart})\times\Lambda^{[m]^k}\times
  [R(m)]\to\cE_m(\Omega)\times\Lambda^{([m])_k}$ ($m\in\NN$) and where
  \begin{align*}
    R(m) & \df m!\cdot\prod_{i=1}^k\binom{k}{i}^{2\binom{m}{i}}
  \end{align*}
  such that
  \begin{enumerate}
  \item if $\rn{b}$ is picked uniformly in $[R(m)]$, independently from $(\rn{x},\rn{y})$, then
    \begin{align*}
      h_m((\rn{x},\rn{y})\rest_m, \rn{b}) & \sim (\rn{\widehat{x}},\rn{\widehat{y}})\rest_m,
    \end{align*}
    where all $\place\rest_m$ denote the natural projections.
  \item we have
    \begin{align*}
      \PP[\rn{\widehat{y}}_{(1,\ldots,k)} \neq \bot]
      & =
      p\cdot\PP[\rn{y}_{(1,\ldots,1)}\neq\bot],
    \end{align*}
    where
    \begin{align*}
      p
      \df
      \prod_{i=1}^k \binom{k}{i}^{-2\binom{k}{i}},
    \end{align*}
    i.e., we are not allowed to cheat by putting $\bot$ everywhere with probability $1$.
  \end{enumerate}
\end{remark}

Finally, we also point out that when the hypothesis class $\cH$ is known to have rank at most $r$,
then clearly there is no need to randomize variables indexed by sets of size larger than $r$, which
immediately reduces the amount of randomness needed by the learning algorithm. Furthermore, we
believe that if $\cH$ has rank at most $1$, then there is also no need for the order randomization,
reducing the amount of randomness by a further $m!$ factor. However, even though the argument for
rank at most $1$ without the order randomness has the same ideas as the main argument, the technical
details are sufficiently different from the main argument that it would require carefully repeating
all calculations, so we do not do it here.

\begin{proof}
  Without loss of generality, we assume $\lVert\ell\rVert_\infty > 0$ (otherwise the result is
  trivial).
  
  By Proposition~\ref{prop:ho}, we may assume that $\cA$ does not depend on coordinates indexed by
  sets of size greater than $k$ and to simplify notation, throughout this proof we will simply
  ignore all such coordinates as any arbitrary (but measurable) definition of them is easily checked
  to work. We will also use the notation $r(m,k)\df\{C\in r(m)\mid\lvert C\rvert\leq k\}$.

  We start with some technicalities. First, since
  \begin{align*}
    R_{\cA'}(m)
    & =
    R_\cA(m)\cdot m!\cdot\prod_{i=1}^k \binom{k}{i}^{2\binom{m}{i}}
    =
    R_\cA(m)\cdot m!\cdot\left(\prod_{C\in r(m,k)} \binom{k}{\lvert C\rvert}\right)^2
  \end{align*}
  for each $m\in\NN$, by using an appropriate (fixed) bijection, we may assume that $\cA'$ is of the
  form
  \begin{multline*}
    \cA'\colon
    \bigcup_{m\in\NN}
    \left(
    \cE_m(\Omega^{\kpart})\times(\Lambda^{S_k})^{[m]^k}\times[R_\cA(m)]
    \times S_m
    \times\!\!\prod_{C\in r(m,k)} \binom{[k]}{\lvert C\rvert}
    \times\!\!\prod_{C\in r(m,k)} \binom{[k]}{\lvert C\rvert}
    \right)
    \\
    \longrightarrow
    \cH^{\kpart}.
  \end{multline*}

  Second, to show that $\cA'$ is a randomized agnostic $k$-PAC learner for $\cH^{\kpart}$ with
  respect to $\ell^{\kpart}$, we need to show that~\eqref{eq:partagkPAC} holds for every
  $\epsilon,\delta\in(0,1)$, every $\mu\in\Pr(\Omega^{\kpart})$, every Borel $k$-partite template
  $\Omega'$, every $\mu'\in\Pr(\Omega')$, every
  $F\in\cF_k(\Omega^{\kpart}\otimes\Omega',\Lambda^{S_k})$ and every integer $m\geq
  m^{\agPACr}_{\cH^{\kpart},\ell^{\kpart},\cA'}(\epsilon,\delta)$. By using Borel-isomorphisms, we
  may suppose without loss of generality that $\Omega'_{C_1} = \Omega'_{C_2}$ whenever $\lvert
  C_1\rvert=\lvert C_2\rvert$, which in particular means that $\Omega' = \widehat{\Omega}^{\kpart}$
  for some Borel template $\widehat{\Omega}$. To ease notation, let us use $\Omega'$ for
  $\widehat{\Omega}$, so $(\Omega')^{\kpart}$ is the previous $\Omega'$ and we have
  $\mu'\in\Pr((\Omega')^{\kpart})$, $(\Omega\otimes\Omega')^{\kpart} =
  \Omega^{\kpart}\otimes(\Omega')^{\kpart}$ and
  $F\in\cF_k((\Omega\otimes\Omega')^{\kpart},\Lambda^{S_k})$.

  Finally, let $\bot\in\Lambda$ be a neutral symbol for $\ell$ and let
  $\ell_\bot\colon\cE_k(\Omega)\to\RR_{\geq 0}$ be the corresponding function given
  by~\eqref{eq:neutsymbellbot} in Definition~\ref{def:neutsymb:nonpart}.

  Let us now give an informal description of the proof strategy: the algorithm $\cA'$ will take an
  input
  \begin{align*}
    (x,y,b,\sigma,U,U') & \in
    \cE_m(\Omega^{\kpart})\times(\Lambda^{S_k})^{m^k}\times[R_\cA(m)]
    \times S_m
    \times\!\!\prod_{C\in r(m,k)} \binom{[k]}{\lvert C\rvert}
    \times\!\!\prod_{C\in r(m,k)} \binom{[k]}{\lvert C\rvert}
  \end{align*}
  pass the source of randomness $b$ directly to $\cA$, but use $(\sigma,U,U')$ to decide which
  $[m]$-sample $(x^{\sigma,U},y^{\sigma,U,U'})$ to pass to $\cA$. The idea then is to show that when
  $\cA'$ is attempting to agnostically learn
  $F\in\cF_k((\Omega\otimes\Omega')^{\kpart},\Lambda^{S_k})$ under $\mu\in\Pr(\Omega^{\kpart})$ and
  $\mu'\in\Pr((\Omega')^{\kpart})$, $\cA'$ is essentially simulating how $\cA$ would agnostically learn
  some $\widehat{F}\in\cF_k(\Omega\otimes\Omega'',\Lambda')$ under some
  $\widehat{\mu}\in\Pr(\Omega)$ and some $\widehat{\mu''}\in\Pr(\Omega'')$ for some suitably defined
  $\widehat{F}$, $\Omega''$, $\widehat{\mu}$ and $\widehat{\mu''}$. This means that we will want
  these objects to have the following properties:
  \begin{enumerate}[label={\Roman*.}, ref={(\Roman*)}]
  \item\label{it:kpart2:loss} The total losses are comparable: there exists a constant
    $C_{\ell,\bot,F}\in\RR_{\geq 0}$ such that for every $H\in\cH$, we have
    \begin{align*}
      L_{\widehat{\mu},\widehat{\mu''},\widehat{F},\ell}(H)
      & =
      (1-p)\cdot C_{\ell,\bot,F} +
      p\cdot L_{\mu,\mu',F,\ell^{\kpart}}(H^{\kpart}),
    \end{align*}
    where $p\in(0,1)$ is given by~\eqref{eq:kpart2:p}.
  \item\label{it:kpart2:dist} The distribution of $[m]$-samples is correct: if $(\rn{x},
    \rn{x'})\sim(\mu\otimes\mu')^m$ and $(\rn{\sigma},\rn{U},\rn{U'})$ is picked uniformly at random
    in $S_m\times\prod_{C\in r(m,k)} \binom{[k]}{\lvert C\rvert}\times\prod_{C\in r(m,k)}
    \binom{[k]}{\lvert C\rvert}$, independently from $\rn{x}$ and $\rn{x'}$; and
    $(\rn{\widehat{x}},\rn{\widehat{x''}})\sim (\widehat{\mu}\otimes\widehat{\mu''})^m$, then
    \begin{align*}
      (\pi_{m,k}(\rn{x}^{\rn{\sigma},\rn{U}}), F^*_m(\rn{x},\rn{x'})^{\rn{\sigma},\rn{U},\rn{U'}})
      & \sim
      (\pi_{m,k}(\rn{\widehat{x}}), \widehat{F}^*_m(\rn{\widehat{x}},\rn{\widehat{x''}})),
    \end{align*}
    where $\pi_{m,k}$ is the projection onto the coordinates indexed by $r(m,k)$.
  \end{enumerate}

  Let us now formalize the ideas above. We start by letting $\Upsilon$ be the Borel template given
  by $\Upsilon_i\df[0,1)$, equipped with the usual Borel $\sigma$-algebra and we let
    $\Omega''\df\Omega'\otimes\Upsilon\otimes\Upsilon\otimes\Upsilon$. It will also be convenient to
    have the Borel templates $\Sigma$ and $\Sigma'$ given by
    \begin{align*}
      \Sigma_i & \df \Omega_i\otimes\left[\binom{k}{i}\right] \qquad (i\in[k]),\\
      \Sigma'_i & \df \Omega'_i\otimes\left[\binom{k}{i}\right] \qquad (i\in[k]),
    \end{align*}
    where $[n]$ is equipped with the discrete $\sigma$-algebra. (For well-definedness, when $i > k$,
    we set $\Sigma_i\df\Omega_i\otimes[1]$ and $\Sigma'_i\df\Omega'_i\otimes[1]$.)

  For each $i\in[k]$, let us enumerate the elements of $\binom{[k]}{i}$ as
  $B^i_1,\ldots,B^i_{\binom{k}{i}}$ and recalling that $\Omega^{\kpart}_C=\Omega_{\lvert C\rvert}$
  ($C\in r(k)$), we let $\nu\in\Pr(\Sigma)$ and $\nu'\in\Pr(\Sigma')$ be the unique probability
  templates such that
  \begin{align}\label{eq:nunu'}
    \nu_i(W\times J) & = \binom{k}{i}^{-1}\sum_{j\in J} \mu_{B^i_j}(W), &
    \nu'_i(W'\times J) & = \binom{k}{i}^{-1}\sum_{j\in J} \mu'_{B^i_j}(W')
  \end{align}
  for every $i\in[k]$, every measurable $W\subseteq\Omega_i$, every measurable
  $W'\subseteq\Omega'_i$ and every $J\subseteq[\binom{k}{i}]$. (For well-definedness, when $i > k$,
  we define $\nu_i$ and $\nu'_i$ arbitrarily.)

  We also let $\pi_i\colon\Sigma_i\to\Omega_i$ and $\pi'_i\colon\Sigma'_i\to\Omega'_i$ be the
  natural projections and let $\widehat{\mu}\in\Pr(\Omega)$ and $\widehat{\mu'}\in\Pr(\Omega')$ be
  the probability templates obtained as the pushforward measures:
  \begin{align*}
    \widehat{\mu}_i & \df (\pi_i)_*(\nu_i) \qquad (i\in\NN_+),\\
    \widehat{\mu'}_i & \df (\pi'_i)_*(\nu_i) \qquad (i\in\NN_+).
  \end{align*}
  Equivalently, we have
  \begin{align}\label{eq:kpart2:widehatmu}
    \widehat{\mu}_i(W) & = \binom{k}{i}^{-1}\cdot\sum_{j=1}^{\binom{k}{i}} \mu_{B^i_j}(W), &
    \widehat{\mu'}_i(W') & = \binom{k}{i}^{-1}\cdot\sum_{j=1}^{\binom{k}{i}} \mu_{B^i_j}(W')
  \end{align}
  for every $i\in[k]$, every measurable $W\subseteq\Omega_i$ and every measurable
  $W'\subseteq\Omega'_i$.

  Let $\lambda\in\Pr(\Upsilon)$ denote the probability template that is the Lebesgue measure on
  $[0,1)$ (which we also denote by $\lambda$ by abuse of notation) in all coordinates and recalling
    that $\Omega''=\Omega'\otimes\Upsilon\otimes\Upsilon\otimes\Upsilon$, let
    $\widehat{\mu''}\in\Pr(\Omega'')$ be given by
    $\widehat{\mu''}\df\widehat{\mu'}\otimes\lambda\otimes\lambda\otimes\lambda$.

  For $i\in[k]$, $\theta\in\Pr([\binom{k}{i}])$ and $j\in[\binom{k}{i}]$, we define the interval
  \begin{align*}
    K^\theta_j & \df \Bigl[\theta\bigl([j-1]\bigr), \theta\bigl([j]\bigr)\Bigr)
  \end{align*}
  so that $(K^\theta_j \mid j\in[\binom{k}{i}])$ is a measurable partition of $[0,1)$ with
  \begin{align}\label{eq:lambdaK}
    \lambda(K^\theta_j) & = \theta(\{j\}) \qquad\left(j\in\left[\binom{k}{i}\right]\right).
  \end{align}
  For $z\in [0,1)$, we let $j_\theta(z)\in[\binom{k}{i}]$ be the unique integer with $z\in
    K^\theta_{j_\theta(z)}$.

  Given $\theta\in\prod_{C\in r(m,k)}\Pr([\binom{k}{\lvert C\rvert}])$ and $z\in\cE_m(\Upsilon)$, we
  let $j_\theta(z)\in\prod_{C\in r(m,k)}[\binom{k}{\lvert C\rvert}]$ be given by
  \begin{align*}
    j_\theta(z)_C
    & \df
    j_{\theta_C}(z_C)
    \qquad (C\in r(m,k)).
  \end{align*}

  For $\tau\in S_k$, we let
  \begin{align*}
    \cJ_\tau
    & \df
    \left\{j\in\prod_{C\in r(k)} \left[\binom{k}{\lvert C\rvert}\right]
    \;\middle\vert\;
    \forall C\in r(k), B^{\lvert C\rvert}_{j_C} = \tau^{-1}(C)
    \right\}.
  \end{align*}
  By inspecting the condition above when $\lvert C\rvert=1$, one clearly sees that the $\cJ_\tau$
  are pairwise disjoint and $\lvert\cJ_\tau\rvert=1$.

  Recalling that $\widehat{\mu}_i = (\pi_i)_*(\nu_i)$ for $i\in[k]$, where
  $\pi_i\colon\Sigma_i\to\Omega_i$ is the projection, by Theorem~\ref{thm:disintegration}, there
  exists a Markov kernel $\rho_i\colon\Omega_i\to\Pr(\Sigma_i)$ such that the set
  \begin{equation}\label{eq:Ri}
    \begin{aligned}
      R_i
      & \df
      \{x\in\Omega_i \mid \rho_i(x)(\pi_i^{-1}(x)) = 1\}
      \\
      & =
      \left\{x\in\Omega_i \;\middle\vert\;
      \rho_i(x)\left(\left\{(x,j) \middle\vert\; j\in \left[\binom{k}{i}\right]\right\}\right)
      = 1\right\}
    \end{aligned}
  \end{equation}
  has $\widehat{\mu}_i$-measure $1$ and such that for every measurable function
  $\psi\colon\Sigma_i\to[0,\infty]$, we have
  \begin{align*}
    \int_{X_i\times[\binom{k}{i}]} \psi(x,j) \ d\nu_i(x,j)
    & =
    \int_{R_i} \int_{\pi_i^{-1}(x)} \psi(x',j)\ d\rho(x)(x',j)\ d\widehat{\mu}_i(x).
  \end{align*}
  In particular, if $V\subseteq X_i$ is measurable and $J\subseteq[\binom{k}{i}]$, then
  \begin{align}\label{eq:disintegrated}
    \nu_i(V\times J)
    & =
    \int_{V\cap R_i} \rho(x)(X_i\times J)\ d\widehat{\mu}_i(x).
  \end{align}

  Note that for $x\in R_i$, since
  \begin{align*}
    \pi_i^{-1}(x) & = \left\{(x,j) \;\middle\vert\; j\in\left[\binom{k}{i}\right]\right\}
  \end{align*}
  has $\rho_i(x)$-measure $1$, we can define a probability measure $\eta_i(x)\in\Pr([\binom{k}{i}])$
  by
  \begin{align}\label{eq:etai}
    \eta_i(x)(J) & \df \rho_i(x)(\{(x,j) \mid j\in J\}).
  \end{align}
  We also define $\eta_i(x)\in\Pr([\binom{k}{i}])$ when $x\in\Omega_i\setminus R_i$ arbitrarily (but
  measurably in $x$, say, make it the Dirac delta concentrated on $1$).

  For $x\in\cE_k(\Omega)$, we let $\eta(x)\in\prod_{C\in r(k)}\Pr([\binom{k}{\lvert C\rvert}])$ be
  given by
  \begin{align*}
    \eta(x)_C & \df \eta_{\lvert C\rvert}(x_C)
    \qquad (C\in r(k)).
  \end{align*}

  We also define $(R'_i,\rho'_i,\eta'_i)$ analogously in terms of $\widehat{\mu}'_i$ and $\pi'_i$
  and define $\eta'$ analogously from the $\eta'_i$.

  For $w\in\cE_m(\Upsilon)$, with pairwise distinct coordinates, we let $\sigma(w)\in S_m$ be the
  unique permutation such that
  \begin{align*}
    w_{\{\sigma(w)(1)\}} < w_{\{\sigma(w)(2)\}} < \cdots < w_{\{\sigma(w)(m)\}}.
  \end{align*}
  We also define $\sigma(w)\in S_m$ arbitrarily (but measurably) when $w$ has repeated coordinates
  (note that this is set of $\lambda^m$-measure zero).

  Recalling that $\Omega''=\Omega'\otimes\Upsilon\otimes\Upsilon\otimes\Upsilon$, we define
  $\widehat{F}\in\cF_k(\Omega\otimes\Omega'',\Lambda')$ by
  \begin{align*}
    \widehat{F}(\widehat{x},\widehat{x'},z,z',w)
    & \df
    \begin{dcases*}
      F(\phi_k(\sigma(w)^*(\widehat{x},\widehat{x'})))_{\sigma(w)^{-1}},
      & if $j_{\eta(\widehat{x})}(z)=j_{\eta(\widehat{x'})}(z')\in\cJ_{\sigma(w)}$,
      \\
      \bot, & otherwise,
    \end{dcases*}
  \end{align*}
  for every $\widehat{x}\in\cE_k(\Omega)$, every $\widehat{x'}\in\cE_k(\Omega')$ and every
  $z,z',w\in\cE_k(\Upsilon)$, where
  $\phi_k\colon\cE_k(\Omega\otimes\Omega')\to\cE_1((\Omega\otimes\Omega')^{\kpart})$ is given
  by~\eqref{eq:kpartbasics:phi}.

  Given $\tau\in S_m$ and $\alpha\in([m])_k$, we let
  \begin{align}\label{eq:taualpha}
    \tau_\alpha
    & \df
    \alpha^{-1}\comp\tau\comp\iota_{\im(\tau^{-1}\comp\alpha),m},
  \end{align}
  where for $C\subseteq [m]$, $\iota_{C,m}$ is the unique increasing function $[\lvert
    C\rvert]\to[m]$ with $\im(\iota_{C,m})=C$. Equivalently, $\tau_\alpha$ is the unique permutation
  in $S_k$ such that $\tau^{-1}\comp\alpha\comp\tau_\alpha$ is increasing (as it is equal to
  $\iota_{\im(\tau^{-1}\comp\alpha),m}$).

  \medskip

  Let us now work toward proving the promised items~\ref{it:kpart2:loss}
  and~\ref{it:kpart2:dist}. The former item is shown in Claim~\ref{clm:dist}\ref{clm:dist:loss} and
  the latter item is shown as the combination of Claim~\ref{clm:dist}\ref{clm:dist:tilde} and
  Claim~\ref{clm:dist2}\ref{clm:dist2:sample}. Our first claim concerns several equivariance
  properties that our definitions satisfy.

  \begin{claim}\label{clm:equiv}
    The following hold for every $(\widehat{x},\widehat{x'},z,z',w)\in
    \cE_m(\Omega\otimes\Omega'\otimes\Upsilon\otimes\Upsilon\otimes\Upsilon)$:
    \begin{enumerate}
    \item\label{clm:equiv:j} For $\alpha\colon [m']\to[m]$ injective, we have
      \begin{align*}
        j_{\eta(\alpha^*(\widehat{x}))}(\alpha^*(z)) & = \alpha^*(j_{\eta(\widehat{x})}(z)),\\
        j_{\eta'(\alpha^*(\widehat{x'}))}(\alpha^*(z')) & = \alpha^*(j_{\eta(\widehat{x'})}(z')),
      \end{align*}
      where for $t\in\prod_{C\in r(m,k)} [\binom{k}{\lvert C\rvert}]$, the element
      $\alpha^*(t)\in\prod_{C\in r(m',k)} [\binom{k}{\lvert C\rvert}]$ is given by
      \begin{align*}
        \alpha^*(t)_C & \df t_{\alpha(C)}\qquad (C\in r(m',k)).
      \end{align*}
    \item\label{clm:equiv:cJ} For every $j\in\prod_{C\in r(k)} [\binom{k}{\lvert C\rvert}]$ and
      every $\tau,\gamma\in S_k$, we have
      \begin{align*}
        \tau^*(j)\in\cJ_\gamma
        & \iff
        j\in\cJ_{\tau\comp\gamma}.
      \end{align*}
    \item\label{clm:equiv:w} For $\alpha\colon [m']\to[m]$ injective, if all coordinates of $w$ are
      pairwise distinct, then
      \begin{align*}
        \sigma(\alpha^*(w)) & = \sigma(w)_\alpha,
      \end{align*}
      where
      $\sigma(w)_\alpha\df\alpha^{-1}\comp\sigma(w)\comp\iota_{\im(\sigma(w)^{-1}\comp\alpha),m}$ is
      given by~\eqref{eq:taualpha}. In particular, if $m=m'$, then
      \begin{align*}
        \sigma(\alpha^*(w)) & = \alpha^{-1}\comp\sigma(w).
      \end{align*}
    \item\label{clm:equiv:jcJw} If $m=k$, all coordinates of $w$ are pairwise distinct and $\tau\in
      S_k$, then
      \begin{align*}
        j_{\eta(\tau^*(\widehat{x}))}(\tau^*(z))=j_{\eta'(\tau^*(\widehat{x'}))}(\tau^*(z'))\in\cJ_{\sigma(\tau^*(w))}
        & \iff
        j_{\eta(\widehat{x})}(z)=j_{\eta'(\widehat{x'})}(z')\in\cJ_{\sigma(w)}.
      \end{align*}
    \item\label{clm:equiv:widehatF*k} If $m=k$ and all coordinates of $w$ are pairwise distinct,
      then
      \begin{align*}
        \widehat{F}^*_k(\widehat{x},\widehat{x'},z,z',w)
        & =
        \begin{dcases*}
          (\sigma(w)^{-1})^*(F(\phi_k(\sigma(w)^*(\widehat{x},\widehat{x'})))),
          & if $j_{\eta(\widehat{x})}(z)=j_{\eta(\widehat{x'})}(z')\in\cJ_{\sigma(w)}$,
          \\
          \bot^{S_k}, & otherwise,
        \end{dcases*}
      \end{align*}
      where in the second case $\bot^{S_k}$ denotes the element of $\Lambda^{S_k}$ that is $\bot$ everywhere.
    \end{enumerate}
  \end{claim}

  \begin{proofof}{Claim~\ref{clm:equiv}}
    For item~\ref{clm:equiv:j}, note that for $C\in r(m',k)$, we have
    \begin{align*}
      j_{\eta(\alpha^*(\widehat{x}))}(\alpha^*(z))_C
      & =
      j_{\eta_{\lvert C\rvert}(\widehat{x}_{\alpha(C)})}(z_{\alpha(C)})
      =
      j_{\eta(\widehat{x})}(z)_{\alpha(C)}
      =
      \alpha^*(j_{\eta(\widehat{x})}(z))_C.
    \end{align*}
    The other assertion is proved analogously.

    \medskip

    For item~\ref{clm:equiv:cJ}, note that
    \begin{align*}
      \tau^*(j)\in\cJ_\gamma
      & \iff
      \forall C\in r(k), B^{\lvert C\rvert}_{\tau(C)} = \gamma^{-1}(C)
      \\
      & \iff
      \forall D\in r(k), B^{\lvert D\rvert}_D = \gamma^{-1}(\tau^{-1}(D))
      \iff
      j\in\cJ_{\tau\comp\gamma},
    \end{align*}
    where the second equivalence follows from the change of variables $D\df\tau(C)$.

    \medskip

    For the first assertion of item~\ref{clm:equiv:w}, note that for $i\in [m']$, we have
    \begin{align*}
      \alpha^*(w)_{\{\sigma(w)_\alpha(i)\}}
      & =
      \alpha^*(w)_{\{(\alpha^{-1}\comp\sigma(w)\comp\iota_{\im(\sigma(w)^{-1}\comp\alpha),m})(i)\}}
      =
      w_{\{(\sigma(w)\comp\iota_{\im(\sigma(w)^{-1}\comp\alpha),m})(i)\}},
    \end{align*}
    and since $\iota_{\im(\sigma(w)^{-1}\comp\alpha),m}$ is increasing, by the definition of
    $\sigma(w)$, we get
    \begin{align*}
      \alpha^*(w)_{\{\sigma(w)_\alpha(1)\}}
      <
      \alpha^*(w)_{\{\sigma(w)_\alpha(2)\}}
      <
      \cdots
      <
      \alpha^*(w)_{\{\sigma(w)_\alpha(m')\}}
    \end{align*}
    from which we get $\sigma(\alpha^*(w))=\sigma(w)_\alpha$.

    The second assertion follows since when $m'=m$, we have
    $\sigma(w)_\alpha=\alpha^{-1}\comp\sigma(w)$ as $\iota_{\im(\sigma(w)^{-1}\comp\alpha),m} =
    \id_m$.

    \medskip
    
    For item~\ref{clm:equiv:jcJw}, by items~\ref{clm:equiv:j}, \ref{clm:equiv:cJ}
    and~\ref{clm:equiv:w}, we have
    \begin{align*}
      j_{\eta(\tau^*(\widehat{x}))}(\tau^*(z))=j_{\eta'(\tau^*(\widehat{x'}))}(\tau^*(z'))\in\cJ_{\sigma(\tau^*(w))}
      & \iff
      \tau^*(j_{\eta(\widehat{x})}(z))=\tau^*(j_{\eta'(\widehat{x'})}(z'))\in\cJ_{\tau^{-1}\comp\sigma(w)}
      \\
      & \iff
      j_{\eta(\widehat{x})}(z)=j_{\eta'(\widehat{x'})}(z')\in\cJ_{\sigma(w)},
    \end{align*}
    as desired.

    \medskip

    For item~\ref{clm:equiv:widehatF*k}, note that for $\tau\in S_k$ we have
    \begin{align*}
      & \!\!\!\!\!\!
      \widehat{F}^*_k(\widehat{x},\widehat{x'},z,z',w)_\tau
      \\
      & =
      \widehat{F}(\tau^*(\widehat{x},\widehat{x'},z,z',w))
      \\
      & =
      \begin{dcases*}
        F(\phi_k(\sigma(\tau^*(w))^*(\tau^*(\widehat{x},\widehat{x'}))))_{\sigma(\tau^*(w))^{-1}},
        & if
        $j_{\eta(\tau^*(\widehat{x}))}(\tau^*(z))=j_{\eta(\tau^*(\widehat{x'}))}(\tau^*(z'))\in\cJ_{\sigma(\tau^*(w))}$,
        \\
        \bot, & otherwise,
      \end{dcases*}
      \\
      & =
      \begin{dcases*}
        F(\phi_k((\tau^{-1}\comp\sigma(w))^*(\tau^*(\widehat{x},\widehat{x'}))))_{\sigma(w)^{-1}\comp\tau},
        & if $j_{\eta(\widehat{x})}(z)=j_{\eta(\widehat{x'})}(z')\in\cJ_{\sigma(w)}$,
        \\
        \bot, & otherwise,
      \end{dcases*}
      \\
      & =
      \begin{dcases*}
        (\sigma(w)^{-1})^*(F(\phi_k(\sigma(w)^*(\widehat{x},\widehat{x'}))))_\tau,
        & if $j_{\eta(\widehat{x})}(z)=j_{\eta(\widehat{x'})}(z')\in\cJ_{\sigma(w)}$,
        \\
        \bot, & otherwise,
      \end{dcases*}
    \end{align*}
    where the third equality follows from items~\ref{clm:equiv:w} and~\ref{clm:equiv:jcJw}.
  \end{proofof}

  \begin{claim}\label{clm:dist}
    Let $(\rn{\widehat{x}},\rn{\widehat{x'}},\rn{z},\rn{z'},\rn{w})\sim
    (\widehat{\mu}\otimes\widehat{\mu'}\otimes\lambda\otimes\lambda\otimes\lambda)^m$, let
    \begin{align*}
      \rn{\widehat{j}} & \df j_{\eta(\rn{\widehat{x}})}(\rn{z}), &
      \rn{\widehat{j'}} & \df j_{\eta(\rn{\widehat{x'}})}(\rn{z'}), &
      \rn{\widehat{\sigma}} & \df \sigma(\rn{w}).
    \end{align*}
    Since $\rn{\widehat{j}}$ and $\rn{\widehat{j'}}$ are elements of $\prod_{C\in
      r(m,k)}[\binom{k}{\lvert C\rvert}]$, we extend them by letting
    $\rn{\widehat{j}}_C\df\rn{\widehat{j'}}_C\df 1$ when $C\in r(m)$ with $\lvert C\rvert > k$.

    Let further
    $((\rn{\widetilde{x}},\rn{\widetilde{j}}),(\rn{\widetilde{x'}},\rn{\widetilde{j'}}))\sim
    (\nu\otimes\nu')^m$ and let $\rn{\widetilde{\sigma}}$ be picked uniformly at random in $S_m$,
    independently from
    $((\rn{\widetilde{x}},\rn{\widetilde{j}}),(\rn{\widetilde{x'}},\rn{\widetilde{j'}}))$.

    Finally, let
    \begin{align*}
      \rn{\widehat{y}}
      & \df
      \widehat{F}^*_m(\rn{\widehat{x}},\rn{\widehat{x'}},\rn{z},\rn{z'},\rn{w}),
    \end{align*}
    and for every $\alpha\in([m])_k$, let
    \begin{align*}
      \rn{\widetilde{y}}_\alpha
      & \df
      \begin{dcases*}
        F(\phi_k(\rn{\widetilde{\sigma}}_\alpha^*(\alpha^*(
        \rn{\widetilde{x}},\rn{\widetilde{x'}}))))_{\rn{\widetilde{\sigma}}_\alpha^{-1}},
        & if $\alpha^*(\rn{\widetilde{j}})=\alpha^*(\rn{\widetilde{j'}})\in
        \cJ_{\rn{\widetilde{\sigma}}_\alpha}$,
        \\
        \bot, & otherwise.
      \end{dcases*}
    \end{align*}
    
    Then the following hold:
    \begin{enumerate}
    \item\label{clm:dist:tilde} 
      We have
      \begin{align}\label{eq:dist:tilde}
        (\rn{\widehat{x}},\rn{\widehat{j}},\rn{\widehat{x'}},\rn{\widehat{j'}},
        \rn{\widehat{\sigma}},\rn{\widehat{y}})
        & \sim
        (\rn{\widetilde{x}},\rn{\widetilde{j}},\rn{\widetilde{x'}},\rn{\widetilde{j'}},
        \rn{\widetilde{\sigma}},\rn{\widetilde{y}}).
      \end{align}
    \item\label{clm:dist:inv} If $m = k$, then the distribution of
      $(\rn{\widetilde{x}},\rn{\widetilde{j}},\rn{\widetilde{x'}},\rn{\widetilde{j'}})$ is
      $S_k$-invariant.
    \item\label{clm:dist:p} If $m=k$, then
      \begin{align*}
        \PP_{\rn{\widehat{x}},\rn{\widehat{x'}},\rn{z},\rn{z'},\rn{w}}[
          \rn{\widehat{j}}=\rn{\widehat{j'}}\in\cJ_{\rn{\widehat{\sigma}}}
        ] & = p,
      \end{align*}
      where
      \begin{align*}
        p & \df \prod_{i=1}^k \binom{k}{i}^{-2\binom{k}{i}}
      \end{align*}
      is as in~\eqref{eq:kpart2:p}.
    \item\label{clm:dist:phi} If $m=k$, then the conditional distribution of
      $\phi_k(\rn{\widehat{\sigma}}^*(\rn{\widehat{x}},\rn{\widehat{x'}}))$ given the event
      $\rn{\widehat{j}}=\rn{\widehat{j'}}\in\cJ_{\rn{\widehat{\sigma}}}$ is $(\mu\otimes\mu')^1$.
    \item\label{clm:dist:loss} For every $H\in\cH$, we have
      \begin{align*}
        L_{\widehat{\mu},\widehat{\mu''},\widehat{F},\ell}(H)
        & =
        (1-p)\cdot C_{\ell,\bot,F} + p\cdot L_{\mu,\mu',F,\ell^{\kpart}}(H^{\kpart}),
      \end{align*}
      where
      \begin{align}\label{eq:kpart2:CellbotF}
        C_{\ell,\bot,F}
        & \df
        \EE_{\rn{\widehat{x}},\rn{\widehat{x'}},\rn{z},\rn{z'},\rn{w}}[
          \ell_\bot(\rn{\widehat{x}})
          \given
          \rn{\widehat{j}}=\rn{\widehat{j'}}\notin\cJ_{\rn{\widehat{\sigma}}}
        ]
      \end{align}
      when $m=k$.
    \item\label{clm:dist:infloss} We have
      \begin{align*}
        \inf_{H\in\cH} L_{\widehat{\mu},\widehat{\mu''},\widehat{F},\ell}(H)
        & =
        (1-p)\cdot C_{\ell,\bot,F} +
        p\cdot\inf_{H\in\cH^{\kpart}} L_{\mu,\mu',F,\ell^{\kpart}}(H^{\kpart}),
      \end{align*}
      where $C_{\ell,\bot,F}$ is given by~\eqref{eq:kpart2:CellbotF}.
    \end{enumerate}
  \end{claim}

  \begin{proofof}{Claim~\ref{clm:dist}}
    For item~\ref{clm:dist:tilde}, let us first prove that
    \begin{align}\label{eq:dist:tildeeasy}
      (\rn{\widehat{x}},\rn{\widehat{j}},\rn{\widehat{x'}},\rn{\widehat{j'}},\rn{\widehat{\sigma}})
      & \sim
      (\rn{\widetilde{x}},\rn{\widetilde{j}},\rn{\widetilde{x'}},\rn{\widetilde{j'}},
      \rn{\widetilde{\sigma}}).
    \end{align}

    It is straightforward to check that $\rn{\widehat{\sigma}}\df\sigma(\rn{w})$ has the same
    distribution as $\rn{\widetilde{\sigma}}$ (i.e., the uniform distribution on $S_m$). Since
    $\rn{w}$ is independent from
    $(\rn{\widehat{x}},\rn{\widehat{j}},\rn{\widehat{x'}},\rn{\widehat{j'}})$ and
    $\rn{\widetilde{\sigma}}$ is independent from
    $(\rn{\widetilde{x}},\rn{\widetilde{j}},\rn{\widetilde{x'}},\rn{\widetilde{j'}})$, it suffices
    to show
    \begin{align*}
      (\rn{\widehat{x}},\rn{\widehat{j}},\rn{\widehat{x'}},\rn{\widehat{j'}})
      & \sim
      (\rn{\widetilde{x}},\rn{\widetilde{j}},\rn{\widetilde{x'}},\rn{\widetilde{j'}}).
    \end{align*}
    In turn, since $\rn{\widehat{j}}$ is $(\rn{\widehat{x}},\rn{z})$-measurable and $\rn{\widehat{j'}}$ is
    $(\rn{\widehat{x'}},\rn{z'})$-measurable, it suffices to show
    \begin{align*}
      (\rn{\widehat{x}},\rn{\widehat{j}})
      & \sim
      (\rn{\widetilde{x}},\rn{\widetilde{j}}),
      &
      (\rn{\widehat{x'}},\rn{\widehat{j'}})
      & \sim
      (\rn{\widetilde{x'}},\rn{\widetilde{j'}}).
    \end{align*}
    We prove only the former as the latter has an analogous proof. Since
    $(\rn{\widetilde{x}},\rn{\widetilde{j}})$ has distribution $\nu^m$, it suffices to show that
    \begin{align}\label{eq:dist:tildeprob}
      \PP_{\rn{\widehat{x}},\rn{z}}[
        \forall C\in r(m), (\rn{\widehat{x}}_C\in V_C\land\rn{\widehat{j}}_C\in J_C)
      ]
      & =
      \prod_{C\in r(m)} \nu_{\lvert C\rvert}(V_C\times J_C)
    \end{align}
    for every measurable $V_C\subseteq X_{\lvert C\rvert}$ and every non-empty
    $J_C\subseteq[\max\{1,\binom{k}{\lvert C\rvert}\}]$.

    But note that
    \begin{align*}
      & \!\!\!\!\!\!
      \PP_{\rn{\widehat{x}},\rn{z}}[
        \forall C\in r(m), (\rn{\widehat{x}}_C\in V_C\land\rn{\widehat{j}}_C\in J_C)
      ]
      \\
      & =
      \PP_{\rn{\widehat{x}},\rn{z}}\left[
        (\forall C\in r(m), \rn{\widehat{x}}_C\in V_C)
        \land\left(
        \forall C\in r(m,k),
        \rn{z}_C\in\bigcup_{j\in J_C} K^{\eta_{\lvert C\rvert}(\rn{\widehat{x}}_C)}_j
        \right)
        \right]
      \\
      & =
      \prod_{C\in r(m,k)}
      \PP_{(\rn{\xi},\rn{\zeta})\sim\widehat{\mu}_{\lvert C\rvert}\otimes\lambda}\left[
        \rn{\xi}\in V_C
        \land
        \rn{\zeta}\in\bigcup_{j\in J_C} K^{\eta_{\lvert C\rvert}(\rn{\xi})}_j
        \right]
      \cdot
      \prod_{\substack{C\in r(m)\\\lvert C\rvert > k}}\widehat{\mu}_{\lvert C\rvert}(V_C)
      \\
      & =
      \prod_{C\in r(m,k)}
      \EE_{\rn{\xi}\sim\widehat{\mu}_{\lvert C\rvert}}[
        \One[\rn{\xi}\in V_C]
        \cdot
        \eta_{\lvert C\rvert}(\rn{\xi})(J_C)
      ]
      \cdot
      \prod_{\substack{C\in r(m)\\\lvert C\rvert > k}}\nu_{\lvert C\rvert}(V_C\times J_C),
    \end{align*}
    where the second equality follows since the coordinates of $\rn{\widehat{x}}$ and $\rn{z}$ are
    mutually independent and the third equality follows from~\eqref{eq:lambdaK}.

    For a fixed $C\in r(m,k)$, note that
    \begin{align*}
      \EE_{\rn{\xi}\sim\widehat{\mu}_{\lvert C\rvert}}[
        \One[\rn{\xi}\in V_C]
        \cdot
        \eta_{\lvert C\rvert}(\rn{\xi})(J_C)
      ]
      & =
      \int_{X_{\lvert C\rvert}}
      \One_{V_C}(\xi)\cdot
      \rho_{\lvert C\rvert}(\xi)(\{\xi\}\times J_C)
      \ d\widehat{\mu}_{\lvert C\rvert}(\xi)
      \\
      & =
      \int_{V_C\cap R_{\lvert C\rvert}}
      \rho_{\lvert C\rvert}(\xi)(X_{\lvert C\rvert}\times J_C)
      \ d\widehat{\mu}_{\lvert C\rvert}(\xi)
      \\
      & =
      \nu_{\lvert C\rvert}(V_C\times J_C),
    \end{align*}
    where the first equality follows from~\eqref{eq:etai}, the second equality follows since
    $\widehat{\mu}_i(R_i)=1$ and when $x\in R_i$, the measure $\rho_i(x)$ is supported on
    $\{x\}\times[\binom{k}{i}]$ (see~\eqref{eq:Ri}) and the third equality follows
    from~\eqref{eq:disintegrated}. Therefore~\eqref{eq:dist:tildeprob} holds
    and~\eqref{eq:dist:tildeeasy} follows.

    Let us now show the upgraded version of~\eqref{eq:dist:tildeeasy},
    equation~\eqref{eq:dist:tilde}.

    First note that for $\alpha\in([m])_k$, since with probability $1$ all coordinates of $\rn{w}$
    are distinct, by Claim~\ref{clm:equiv}\ref{clm:equiv:w}, with probability $1$, we have
    \begin{align*}
      \sigma(\alpha^*(\rn{w})) & = \sigma(\rn{w})_\alpha = \rn{\widehat{\sigma}}_\alpha.
    \end{align*}
    Using item~\ref{clm:equiv:j} of the same claim, we also conclude that
    \begin{align*}
      j_{\eta(\alpha^*(\rn{\widehat{x}}))}(\alpha^*(\rn{\widehat{z}}))=
      j_{\eta'(\alpha^*(\rn{\widehat{x'}}))}(\alpha^*(\rn{\widehat{z'}}))\in\cJ_{\sigma(\alpha^*(\rn{w}))}
      & \iff
      \alpha^*(\rn{\widehat{j}})=\alpha^*(\rn{\widehat{j'}})\in\cJ_{\rn{\widehat{\sigma}}_\alpha}
    \end{align*}
    with probability $1$.

    Therefore, we have the following alternative formula that holds with probability $1$ for
    $\rn{\widehat{y}}\df\widehat{F}^*_m(\rn{\widehat{x}},\rn{\widehat{x'}},\rn{z},\rn{z'},\rn{w})$:
    \begin{align*}
      \rn{\widehat{y}}_\alpha
      & =
      \begin{dcases*}
        F(\phi_k(\rn{\widehat{\sigma}}_\alpha^*(\alpha^*(
        \rn{\widehat{x}},\rn{\widehat{x'}}))))_{\rn{\widehat{\sigma}}_\alpha^{-1}},
        & if $\alpha^*(\rn{\widehat{j}})=\alpha^*(\rn{\widehat{j'}})\in
        \cJ_{\rn{\widehat{\sigma}}_\alpha}$,
        \\
        \bot, & otherwise.
      \end{dcases*}
    \end{align*}

    Now~\eqref{eq:dist:tilde} follows from~\eqref{eq:dist:tildeeasy} and the definition of
    $\rn{\widetilde{y}}$.

    \medskip

    For item~\ref{clm:dist:inv}, by item~\ref{clm:dist:tilde}, for every $\tau\in S_k$, we have
    \begin{align*}
      \tau^*(\rn{\widetilde{x}},\rn{\widetilde{j}},\rn{\widetilde{x'}},\rn{\widetilde{j'}})
      & \sim
      \tau^*(\rn{\widehat{x}},\rn{\widehat{j}},\rn{\widehat{x'}},\rn{\widehat{j'}})
      \\
      & =
      (\tau^*(\rn{\widehat{x}}), j_{\eta(\rn{\tau^*(\widehat{x}}))}(\tau^*(\rn{z})),
      \tau^*(\rn{\widehat{x'}}), j_{\eta'(\rn{\tau^*(\widehat{x'}}))}(\rn{z'}))
      \\
      & \sim
      (\rn{\widehat{x}},\rn{\widehat{j}},\rn{\widehat{x'}},\rn{\widehat{j'}})
      \\
      & \sim
      (\rn{\widetilde{x}},\rn{\widetilde{j}},\rn{\widetilde{x'}},\rn{\widetilde{j'}}),
    \end{align*}
    where the equality follows by Claim~\ref{clm:equiv}\ref{clm:equiv:j} and the second
    distributional equality follows since the distribution of
    $(\rn{\widehat{x}},\rn{\widehat{z}},\rn{\widehat{x'}},\rn{\widehat{z'}})$ is $S_k$-invariant.

    \medskip

    For item~\ref{clm:dist:p}, by item~\ref{clm:dist:tilde}, it suffices to show
    \begin{align*}
      \PP_{\rn{\widetilde{x}},\rn{\widetilde{j}},\rn{\widetilde{x'}},\rn{\widetilde{j'}},\rn{\widetilde{\sigma}}}[
        \rn{\widetilde{j}}=\rn{\widetilde{j'}}\in\cJ_{\rn{\widetilde{\sigma}}}
      ]
      & =
      p.
    \end{align*}
    In turn, since $\rn{\widetilde{\sigma}}$ is uniformly distributed in $S_k$ and is independent
    from $(\rn{\widetilde{x}},\rn{\widetilde{j}},\rn{\widetilde{x'}},\rn{\widetilde{j'}})$, it
    suffices to show that for every $\sigma\in S_k$, we have
    \begin{align}\label{eq:dist:p}
      \PP_{\rn{\widetilde{x}},\rn{\widetilde{j}},\rn{\widetilde{x'}},\rn{\widetilde{j'}}}[
        \rn{\widetilde{j}}=\rn{\widetilde{j'}}\in\cJ_\sigma
      ]
      & =
      p.
    \end{align}

    By Claim~\ref{clm:equiv}\ref{clm:equiv:cJ}, we have
    $\rn{\widetilde{j}}=\rn{\widetilde{j'}}\in\cJ_\sigma$ if and only if
    $\sigma^*(\rn{\widetilde{j}})=\sigma^*(\rn{\widetilde{j'}})\in\cJ_{\id_k}$ and since
    $(\rn{\widetilde{x}},\rn{\widetilde{j}},\rn{\widetilde{x'}},\rn{\widetilde{j'}})$ is
    $S_k$-invariant by item~\ref{clm:dist:inv}, it suffices to show~\eqref{eq:dist:p} only when
    $\sigma=\id_k$.

    But note that
    \begin{align*}
      \PP_{\rn{\widetilde{x}},\rn{\widetilde{j}},\rn{\widetilde{x'}},\rn{\widetilde{j'}},\rn{\widetilde{\sigma}}}[
        \rn{\widetilde{j}}=\rn{\widetilde{j'}}\in\cJ_{\id_k}
      ]
      & =
      \PP_{\rn{\widetilde{x}},\rn{\widetilde{j}},\rn{\widetilde{x'}},\rn{\widetilde{j'}},\rn{\widetilde{\sigma}}}[
        \forall C\in r(k), (\rn{\widetilde{j}}_C = \rn{\widetilde{j'}}_C
        \land B^{\lvert C\rvert}_{\rn{\widetilde{j}}_C} = C)
      ]
      \\
      & =
      \prod_{C\in r(k)} \binom{k}{\lvert C\rvert}^{-2}
      =
      p,
    \end{align*}
    as desired.

    \medskip

    Let us now prove item~\ref{clm:dist:phi}. By item~\ref{clm:dist:tilde}, it suffices to show that
    the conditional distribution of
    $\phi_k(\rn{\widetilde{\sigma}}^*(\rn{\widetilde{x}},\rn{\widetilde{x'}}))$ given the event
    $\rn{\widetilde{j}}=\rn{\widetilde{j'}}\in\cJ_{\rn{\widetilde{\sigma}}}$ is
    $(\mu\otimes\mu')^1$. Since $\rn{\widetilde{\sigma}}$ is uniformly distributed in $S_k$ and is
    independent from
    $(\rn{\widetilde{x}},\rn{\widetilde{j}},\rn{\widetilde{x'}},\rn{\widetilde{j'}})$, it suffices
    to show that for every $\sigma\in S_k$, the conditional distribution of
    $(\phi_k(\sigma^*(\rn{\widetilde{x}},\rn{\widetilde{x'}})))$ given the event
    $\rn{\widetilde{j}}=\rn{\widetilde{j'}}\in\cJ_\sigma$ is $(\mu\otimes\mu')^1$.

    By Claim~\ref{clm:equiv}\ref{clm:equiv:cJ},
    $\rn{\widetilde{j}}=\rn{\widetilde{j'}}\in\cJ_\sigma$ if and only if
    $\sigma^*(\rn{\widetilde{j}})=\sigma^*(\rn{\widetilde{j'}})\in\cJ_{\id_k}$ and since
    $(\rn{\widetilde{x}},\rn{\widetilde{j}},\rn{\widetilde{x'}},\rn{\widetilde{j'}})$ is
    $S_k$-invariant by item~\ref{clm:dist:inv}, it suffices show the assertion only when
    $\sigma=\id_k$.

    In turn, it suffices to show that
    \begin{align*}
      \PP_{\rn{\widetilde{x}},\rn{\widetilde{j}},\rn{\widetilde{x'}},\rn{\widetilde{j'}}}[
        \forall f\in r_k(1), \phi_k(\rn{\widetilde{x}},\rn{\widetilde{x'}})_f \in V_f\times V'_f
        \given
        \rn{\widetilde{j}}=\rn{\widetilde{j'}}\in\cJ_{\id_k}
      ]
      & =
      \prod_{f\in r_k(1)} (\mu_{\dom(f)}(V_f)\cdot\mu'_{\dom(f)}(V'_f))
    \end{align*}
    for all measurable $V_f\subseteq\Omega_{\lvert\dom(f)\rvert}$ and measurable
    $V'_f\subseteq\Omega'_{\lvert\dom(f)\rvert}$.

    First note that for $f\in r_k(1)$, we have
    \begin{align*}
      \{(i-1) + f(i) \mid i\in\dom(f)\} & = \dom(f),
    \end{align*}
    so we get
    \begin{align*}
      \phi_k(\rn{\widetilde{x}},\rn{\widetilde{x'}})_f \in V_f\times V'_f
      & \iff
      \rn{\widetilde{x}}_{\dom(f)}\in V_f\land\rn{\widetilde{x'}}_{\dom(f)}\in V'_f,
    \end{align*}
    hence for the unique element $j^{\id_k}$ of $\cJ_{\id_k}$, we have
    \begin{align*}
      & \!\!\!\!\!\!
      \PP_{\rn{\widetilde{x}},\rn{\widetilde{j}},\rn{\widetilde{x'}},\rn{\widetilde{j'}}}[
        \forall f\in r_k(1), \phi_k(\rn{\widetilde{x}},\rn{\widetilde{x'}})_f \in V_f\times V'_f
        \given
        \rn{\widetilde{j}}=\rn{\widetilde{j'}}\in\cJ_{\id_k}
      ]
      \\
      & =
      \PP_{\rn{\widetilde{x}},\rn{\widetilde{j}},\rn{\widetilde{x'}},\rn{\widetilde{j'}}}[
        \forall f\in r_k(1), \phi_k(\rn{\widetilde{x}},\rn{\widetilde{x'}})_f \in V_f\times V'_f
        \given
        \rn{\widetilde{j}}=\rn{\widetilde{j'}}=j^{\id_k}
      ]
      \\
      & =
      \prod_{f\in r_k(1)}
      \binom{k}{\lvert\dom(f)\rvert}^2
      \bigl(\nu_{\lvert\dom(f)\rvert}(V_f\times\{j^{\id_k}_{\dom(f)}\})
      \cdot \nu'_{\lvert\dom(f)\rvert}(V_f\times\{j^{\id_k}_{\dom(f)}\})\bigr)
      \\
      & =
      \prod_{f\in r_k(1)}
      \bigl(\mu_{\dom(f)}(V_f)\cdot\mu'_{\dom(f)}(V'_f)\bigr),
    \end{align*}
    as desired.

    \medskip

    For item~\ref{clm:dist:loss}, let $m=k$ and note that since $\bot$ is a neutral symbol, we have
    \begin{align*}
      & \!\!\!\!\!\!
      \EE_{\rn{\widehat{x}},\rn{\widehat{x'}},\rn{z},\rn{z'},\rn{w}}[
        \ell(H,\rn{\widehat{x}},\widehat{F}^*_k(\rn{\widehat{x}},\rn{\widehat{x'}},\rn{z},\rn{z'},\rn{w}))
        \given
        \rn{\widehat{j}}=\rn{\widehat{j'}}\notin\cJ_{\rn{\widehat{\sigma}}}
      ]
      \\
      & =
      \EE_{\rn{\widehat{x}},\rn{\widehat{x'}},\rn{z},\rn{z'},\rn{w}}[
        \ell_\bot(\rn{\widehat{x}})
        \given
        \rn{\widehat{j}}=\rn{\widehat{j'}}\notin\cJ_{\rn{\widehat{\sigma}}}
      ]
      =
      C_{\ell,\bot,F},
    \end{align*}
    so we get
    \begin{align*}
      & \!\!\!\!\!\!
      L_{\widehat{\mu},\widehat{\mu''},\widehat{F},\ell}(H)
      \\
      & =
      \EE_{\rn{\widehat{x}},\rn{\widehat{x'}},\rn{z},\rn{z'},\rn{w}}[
        \ell(H,\rn{\widehat{x}},\widehat{F}(\rn{\widehat{x}},\rn{\widehat{x'}},\rn{z},\rn{z'},\rn{w}))
      ]
      \\
      & =
      (1-p)\cdot C_{\ell,\bot,F}
      +
      p\cdot\EE_{\rn{\widehat{x}},\rn{\widehat{x'}},\rn{z},\rn{z'},\rn{w}}[
        \ell(H,\rn{\widehat{x}},
        (\rn{\widehat{\sigma}}^{-1})^*(F(\phi_k(\rn{\widehat{\sigma}}^*(\rn{\widehat{x}},\rn{\widehat{x'}})))))
        \given
        \rn{\widehat{j}}=\rn{\widehat{j'}}\in\cJ_{\rn{\widehat{\sigma}}}
      ]
      \\
      & =
      (1-p)\cdot C_{\ell,\bot,F}
      +
      p\cdot\EE_{\rn{\widehat{x}},\rn{\widehat{x'}},\rn{z},\rn{z'},\rn{w}}[
        \ell(H,\rn{\widehat{\sigma}}^*(\rn{\widehat{x}}),
        F(\phi_k(\rn{\widehat{\sigma}}^*(\rn{\widehat{x}},\rn{\widehat{x'}}))))
        \given
        \rn{\widehat{j}}=\rn{\widehat{j'}}\in\cJ_{\rn{\widehat{\sigma}}}
      ]
      \\
      & =
      (1-p)\cdot C_{\ell,\bot,F}
      +
      p\cdot\EE_{(\rn{x},\rn{x'})\sim(\mu\otimes\mu')^1}[
        \ell^{\kpart}(H^{\kpart},\rn{x}, F(\rn{x},\rn{x'}))
      ]
      \\
      & =
      (1-p)\cdot C_{\ell,\bot,F} + p\cdot L_{\mu,\mu',F,\ell^{\kpart}}(H^{\kpart}),
    \end{align*}
    where the second equality follows from item~\ref{clm:dist:p} and
    Claim~\ref{clm:equiv}\ref{clm:equiv:widehatF*k}, the third equality follows since $\ell$ is
    symmetric and the fourth equality follows by item~\ref{clm:dist:phi} and the definition of
    $\ell^{\kpart}$.

    \medskip

    Finally, item~\ref{clm:dist:infloss} follows directly from item~\ref{clm:dist:loss} and the fact
    that $\cH\ni H\mapsto H^{\kpart}\in\cH^{\kpart}$ is a bijection.
  \end{proofof}

  We now define the algorithm $\cA'$: for an input
  \begin{align*}
    (x,y,b,\sigma,U,U') & \in \cE_m(\Omega^{\kpart})\times(\Lambda^{S_k})^{m^k}\times[R_\cA(m)]
    \times S_m
    \times\!\!\prod_{C\in r(m,k)} \binom{[k]}{\lvert C\rvert}
    \times\!\!\prod_{C\in r(m,k)} \binom{[k]}{\lvert C\rvert},
  \end{align*}
  we let
  \begin{align*}
    \cA'(x,y,b,\sigma,U,U')
    & \df
    \cA(x^{\sigma,U},y^{\sigma,U,U'},b)^{\kpart},
  \end{align*}
  where
  \begin{align*}
    x^{\sigma,U}_C & \df x_{\sigma\comp\iota_{\sigma^{-1}(C),m}\comp\iota_{U_C,k}^{-1}}
    \qquad (C\in r(m,k)),
    \\
    y^{\sigma,U,U'}_\alpha
    & \df
    \begin{dcases*}
      (y_{\alpha\comp\sigma_\alpha})_{\sigma_\alpha^{-1}}, & if $\alpha\in\cG(\sigma,U,U')$,\\
      \bot, & otherwise,
    \end{dcases*}
    & (\alpha\in([m])_k),
    \\
    \cG(\sigma,U,U')
    & \df
    \{\alpha\in([m])_k \mid
    \forall C\in r(k), U_{\alpha(C)} = U'_{\alpha(C)} = \sigma_\alpha^{-1}(C)\}.
  \end{align*}

  We will also need the analogous definition
  \begin{align*}
    (x')^{\sigma,U'}_C & \df x'_{\sigma\comp\iota_{\sigma^{-1}(C),m}\comp\iota_{U'_C,k}^{-1}}
    \qquad (C\in r(m,k))
  \end{align*}
  for $x'\in\cE_m((\Omega')^{\kpart})$.

  \begin{claim}\label{clm:dist2}
    Let $((\rn{\widetilde{x}},\rn{\widetilde{j}}),(\rn{\widetilde{x'}},\rn{\widetilde{j'}}))\sim
    (\nu\otimes\nu')^m$, let $\rn{\widetilde{\sigma}}$ be picked uniformly at random in $S_m$,
    independently from
    $((\rn{\widetilde{x}},\rn{\widetilde{j}}),(\rn{\widetilde{x'}},\rn{\widetilde{j'}}))$ and for
    every $\alpha\in([m])_k$, let
    \begin{align*}
      \rn{\widetilde{y}}_\alpha
      & \df
      \begin{dcases*}
        F(\phi_k(\rn{\widetilde{\sigma}}_\alpha^*(\alpha^*(
        \rn{\widetilde{x}},\rn{\widetilde{x'}}))))_{\rn{\widetilde{\sigma}}_\alpha^{-1}},
        & if $\alpha^*(\rn{\widetilde{j}})=\alpha^*(\rn{\widetilde{j'}})\in
        \cJ_{\rn{\widetilde{\sigma}}_\alpha}$,
        \\
        \bot, & otherwise.
      \end{dcases*}
    \end{align*}

    Let also $\rn{\widetilde{U}}$ and $\rn{\widetilde{U'}}$ be the random elements in $\prod_{C\in
      r(m,k)}\binom{[k]}{\lvert C\rvert}$ given by
    \begin{align*}
      \rn{\widetilde{U}}_C & \df B^{\lvert C\rvert}_{\rn{\widetilde{j}}_C} \qquad (C\in r(m,k)),\\
      \rn{\widetilde{U'}}_C & \df B^{\lvert C\rvert}_{\rn{\widetilde{j'}}_C} \qquad (C\in r(m,k)).
    \end{align*}

    Let further $(\rn{x},\rn{x'})\sim(\mu\otimes\mu')^m$ and let $(\rn{\sigma},\rn{U},\rn{U'})$ be
    picked uniformly at random in $S_m\times\prod_{C\in r(m,k)} \binom{[k]}{\lvert
      C\rvert}\times\prod_{C\in r(m,k)} \binom{[k]}{\lvert C\rvert}$, independently from $\rn{x}$
    and $\rn{x'}$. Finally, let $\rn{y}\df F^*_m(\rn{x},\rn{x'})$.

    Then the following hold:
    \begin{enumerate}
    \item\label{clm:dist2:jcG} For every $\alpha\in([m])_k$, we have
      \begin{align*}
        \alpha^*(\rn{\widetilde{j}})=\alpha^*(\rn{\widetilde{j'}})\in
        \cJ_{\rn{\widetilde{\sigma}}_\alpha}
        & \iff
        \alpha\in\cG(\rn{\widetilde{\sigma}},\rn{\widetilde{U}},\rn{\widetilde{U'}}).
      \end{align*}
    \item\label{clm:dist2:sample} We have
      \begin{align}\label{eq:dist2:sample}
        (\pi_{m,k}(\rn{\widetilde{x}},\rn{\widetilde{x'}}),
        \rn{\widetilde{\sigma}},\rn{\widetilde{U}},\rn{\widetilde{U'}},\rn{\widetilde{y}})
        & \sim
        (\pi_{m,k}(\rn{x}^{\rn{\sigma},\rn{U}},(\rn{x'})^{\rn{\sigma},\rn{U'}}),
        \rn{\sigma},\rn{U},\rn{U'},\rn{y}^{\rn{\sigma},\rn{U},\rn{U'}}),
      \end{align}
      where $\pi_{m,k}$ is the projection onto the coordinates indexed by $r(m,k)$.
    \end{enumerate}
  \end{claim}

  \begin{proofof}{Claim~\ref{clm:dist2}}
    For item~\ref{clm:dist2:jcG}, note that
    \begin{align*}
      \alpha^*(\rn{\widetilde{j}})=\alpha^*(\rn{\widetilde{j'}})\in
      \cJ_{\rn{\widetilde{\sigma}}_\alpha}
      & \iff
      \forall C\in r(k),
      B^{\lvert C\rvert}_{\rn{\widetilde{j}}_{\alpha(C)}}
      = B^{\lvert C\rvert}_{\rn{\widetilde{j'}}_C}
      = \rn{\widetilde{\sigma}}_\alpha^{-1}(C)
      \\
      & \iff
      \forall C\in r(k),
      \rn{\widetilde{U}}_{\alpha(C)} = \rn{\widetilde{U'}}_{\alpha(C)}
      = \rn{\widetilde{\sigma}}_\alpha^{-1}(C)
      \iff
      \alpha\in\cG(\rn{\widetilde{\sigma}},\rn{\widetilde{U}},\rn{\widetilde{U'}}).
    \end{align*}

    \medskip

    For item~\ref{clm:dist2:sample}, let us first prove that
    \begin{align}\label{eq:dist2:sampleeasy}
      (\pi_{m,k}(\rn{\widetilde{x}},\rn{\widetilde{x'}}),
      \rn{\widetilde{\sigma}},\rn{\widetilde{U}},\rn{\widetilde{U'}})
      & \sim
      (\pi_{m,k}(\rn{x}^{\rn{\sigma},\rn{U}},(\rn{x'})^{\rn{\sigma},\rn{U'}}),
      \rn{\sigma},\rn{U},\rn{U'}).
    \end{align}

    Since $\rn{\widetilde{\sigma}}$ is independent from $(\rn{\widetilde{x}},\rn{\widetilde{x'}})$
    and $\rn{\sigma}$ is independent from $(\rn{x},\rn{x'})$ and both are uniformly distributed on
    $S_m$, it suffices to show that for every $\sigma\in S_k$, we have
    \begin{align*}
      (\pi_{m,k}(\rn{\widetilde{x}},\rn{\widetilde{x'}}),\rn{\widetilde{U}},\rn{\widetilde{U'}})
      & \sim
      (\pi_{m,k}(\rn{x}^{\sigma,\rn{U}},(\rn{x'})^{\sigma,\rn{U'}}),\rn{U},\rn{U'}).
    \end{align*}
    Since $(\rn{\widetilde{x}},\rn{\widetilde{U}})$ is independent from
    $(\rn{\widetilde{x'}},\rn{\widetilde{U'}})$ and $(\rn{x}^{\sigma,\rn{U}},\rn{U})$ is independent
    from $((\rn{x'})^{\sigma,\rn{U'}},\rn{U'})$, it suffices to show
    \begin{align*}
      (\pi_{m,k}(\rn{\widetilde{x}}),\rn{\widetilde{U}})
      & \sim
      (\pi_{m,k}(\rn{x}^{\sigma,\rn{U}}),\rn{U}),
      &
      (\pi_{m,k}(\rn{\widetilde{x'}}),\rn{\widetilde{U'}})
      & \sim
      (\pi_{m,k}((\rn{x'})^{\sigma,\rn{U'}}),\rn{U'}).
    \end{align*}
    We show only the former as the latter has an analogous proof.

    First, for $C\in r(m,k)$ and $W\in\binom{[k]}{\lvert C\rvert}$, let $f_{C,W}\colon W\to[m]$ be
    given by
    \begin{align*}
      f_{C,W} & \df \sigma\comp\iota_{\sigma^{-1}(C),m}\comp\iota_{W,k}^{-1}
    \end{align*}
    so that $\im(f_{C,W})=C$.

    Define $\Psi\colon r(m,k)\times\prod_{C\in r(m,k)}\binom{[k]}{\lvert C\rvert}\to r_k(m)$ by
    \begin{align*}
      \Psi(C,U) & \df f_{C,U_C}
    \end{align*}
    and note that since $\im(f_{C,U_C})=C$, it follows that if $C,C'\in r(m,k)$ are distinct and
    $U,U'\in\prod_{D\in r(m,k)}\binom{[k]}{\lvert D\rvert}$ are any elements (distinct or not), then
    $\Psi(C,U)\neq\Psi(C',U')$.

    Since for every $C\in r(m,k)$, we have
    \begin{align*}
      \rn{x}^{\sigma,\rn{U}}_C
      & =
      \rn{x}_{f_{C,\rn{U}_C}}
      =
      \rn{x}_{\Psi(C,\rn{U})},
    \end{align*}
    it follows that the coordinates of $\rn{x}^{\sigma,\rn{U}}$ are mutually independent (as they
    correspond to different coordinates of $\rn{x}$). Since $\rn{x}^{\sigma,\rn{U}}_C$ is
    $((\rn{x}_{f_{C,W}} \mid W\in\binom{[k]}{\lvert C\rvert}),\rn{U}_C)$-measurable and the
    coordinates of $(\rn{\widetilde{x}},\rn{\widetilde{U}})$ are also mutually independent, it
    suffices to show for every $C\in r(m,k)$, we have
    \begin{align*}
      (\rn{\widetilde{x}}_C,\rn{\widetilde{U}}_C) & \sim (\rn{x}^{\sigma,\rn{U}}_C,\rn{U}_C).
    \end{align*}
    By tracking definitions, it suffices to show that
    \begin{align*}
      \PP_{\rn{x},\rn{U}}[
        \rn{x}^{\sigma,\rn{U}}_C\in V
        \land\rn{U}_C = B^{\lvert C\rvert}_j
      ]
      & =
      \binom{k}{\lvert C\rvert}^{-1}\cdot\mu_{B^{\lvert C\rvert}_j}(V)
      \qquad
      (= \nu_{\lvert C\rvert}(V\times\{j\}))
    \end{align*}
    for every measurable $V\subseteq\Omega_{\lvert C\rvert}$ and every $j\in[\binom{k}{\lvert C\rvert}]$.

    But indeed, note that
    \begin{align*}
      \PP_{\rn{x},\rn{U}}[
        \rn{x}^{\sigma,\rn{U}}_C\in V
        \land\rn{U}_C = B^{\lvert C\rvert}_j
      ]
      & =
      \binom{k}{\lvert C\rvert}^{-1}\cdot
      \PP_{\rn{x},\rn{U}}[
        \rn{x}^{\sigma,\rn{U}}_C\in V
        \given
        \rn{U}_C = B^{\lvert C\rvert}_j
      ]
      \\
      & =
      \binom{k}{\lvert C\rvert}^{-1}\cdot
      \PP_{\rn{x}}[\rn{x}_{f_{C,B^{\lvert C\rvert}_j}}\in V]
      \\
      & =
      \binom{k}{\lvert C\rvert}^{-1}\cdot
      \mu_{\dom(f_{C,B^{\lvert C\rvert}_j})}(V)
      \\
      & =
      \binom{k}{\lvert C\rvert}^{-1}\cdot\mu_{B^{\lvert C\rvert}_j}(V),
    \end{align*}
    as desired. This concludes the proof of~\eqref{eq:dist2:sampleeasy}.

    Let us now show the upgraded version of~\eqref{eq:dist2:sampleeasy},
    equation~\eqref{eq:dist2:sample}.

    By item~\ref{clm:dist2:jcG}, we have
    \begin{align*}
      \rn{\widetilde{y}}_\alpha
      & \df
      \begin{dcases*}
        F(\phi_k(\rn{\widetilde{\sigma}}_\alpha^*(\alpha^*(
        \rn{\widetilde{x}},\rn{\widetilde{x'}}))))_{\rn{\widetilde{\sigma}}_\alpha^{-1}},
        & if $\alpha\in\cG(\rn{\widetilde{\sigma}},\rn{\widetilde{U}},\rn{\widetilde{U'}})$,
        \\
        \bot, & otherwise.
      \end{dcases*}
    \end{align*}

    We now define a random element $\rn{\upsilon}$ using the same formula above but replacing
    $(\rn{\widetilde{x}},\rn{\widetilde{x'}},\rn{\widetilde{\sigma}},\rn{\widetilde{U}},\rn{\widetilde{U'}})$
    with $(\rn{x}^{\rn{\sigma},\rn{U}},(\rn{x'})^{\rn{\sigma},\rn{U'}},\rn{\sigma},\rn{U},\rn{U'})$,
    that is, for every $\alpha\in([m])_k$, let
    \begin{align*}
      \rn{\upsilon}_\alpha & \df
      \begin{dcases*}
        F(\phi_k(\rn{\sigma}_\alpha^*(\alpha^*(
        \rn{x}^{\rn{\sigma},\rn{U}},(\rn{x'})^{\rn{\sigma},\rn{U'}}))))_{\rn{\sigma}_\alpha^{-1}},
        & if $\alpha\in\cG(\rn{\sigma},\rn{U},\rn{U'})$,
        \\
        \bot, & otherwise.
      \end{dcases*}
    \end{align*}

    By~\eqref{eq:dist2:sampleeasy}, to show~\eqref{eq:dist2:sample} it suffices to show that
    $\rn{\upsilon}_\alpha = \rn{y}^{\rn{\sigma},\rn{U},\rn{U'}}_\alpha$ for every $\alpha\in([m])_k$.

    Clearly, when $\alpha\notin\cG(\rn{\sigma},\rn{U},\rn{U'})$, then both $\rn{\upsilon}_\alpha$
    and $\rn{y}^{\rn{\sigma},\rn{U},\rn{U'}}_\alpha$ take the value $\bot$. Suppose then that
    $\alpha\in\cG(\rn{\sigma},\rn{U},\rn{U'})$ and note that
    \begin{align*}
      \rn{y}^{\rn{\sigma},\rn{U},\rn{U'}}_\alpha
      & =
      (\rn{y}_{\alpha\comp\rn{\sigma}_\alpha})_{\rn{\sigma}_\alpha^{-1}}
      =
      F((\alpha\comp\rn{\sigma}_\alpha)^*(\rn{x},\rn{x'}))_{\rn{\sigma}_\alpha^{-1}},
      \\
      \rn{\upsilon}_\alpha
      & =
      F(\phi_k(\rn{\sigma}_\alpha^*(\alpha^*(
      \rn{x}^{\rn{\sigma},\rn{U}},(\rn{x'})^{\rn{\sigma},\rn{U'}}))))_{\rn{\sigma}_\alpha^{-1}},
    \end{align*}
    so it suffices to prove that
    \begin{align}\label{eq:alphasigmaalpha}
      (\alpha\comp\rn{\sigma}_\alpha)^*(\rn{x},\rn{x'})
      & =
      \phi_k(\rn{\sigma}_\alpha^*(\alpha^*(\rn{x}^{\rn{\sigma},\rn{U}},(\rn{x'})^{\rn{\sigma},\rn{U'}}))).
    \end{align}

    Let $1^C\in r_k(1)$ be the unique function $C\to[1]$ and let us compute the coordinate
    corresponding to $1^C$ of both sides of~\eqref{eq:alphasigmaalpha}. The coordinate $1^C$ of the
    left-hand side of~\eqref{eq:alphasigmaalpha} is:
    \begin{align}\label{eq:dist2:xx'1C}
      (\alpha\comp\rn{\sigma}_\alpha)^*(\rn{x},\rn{x'})_{1^C}
      & =
      (\rn{x},\rn{x'})_{\alpha\comp\rn{\sigma}_\alpha\rest_C}
      =
      (\rn{x},\rn{x'})_{\rn{\sigma}\comp\iota_{\im(\rn{\sigma}^{-1}\comp\alpha),m}\rest_C},
    \end{align}
    where the second equality follows from~\eqref{eq:taualpha}. 

    The coordinate $1^C$ of the right-hand side of~\eqref{eq:alphasigmaalpha} is:
    \begin{equation}\label{eq:dist2:phik1C}
      \begin{aligned}
        & \!\!\!\!\!\!
        \phi_k(\rn{\sigma}_\alpha^*(\alpha^*(\rn{x}^{\rn{\sigma},\rn{U}},(\rn{x'})^{\rn{\sigma},\rn{U'}})))_{1^C}
        \\
        & =
        (\alpha\comp\rn{\sigma}_\alpha)^*(\rn{x}^{\rn{\sigma},\rn{U}},(\rn{x'})^{\rn{\sigma},\rn{U'}})_C
        \\
        & =
        (\rn{x}^{\rn{\sigma},\rn{U}},(\rn{x'})^{\rn{\sigma},\rn{U'}})_{(\alpha\comp\rn{\sigma}_\alpha)(C)}
        \\
        & =
        (\rn{x}_{\rn{\sigma}\comp\iota_{\rn{\sigma}^{-1}((\alpha\comp\rn{\sigma}_\alpha)(C)),m}\comp\iota_{\rn{U}_{(\alpha\comp\rn{\sigma}_\alpha)(C)},k}^{-1}},
        \rn{x'}_{\rn{\sigma}\comp\iota_{\rn{\sigma}^{-1}((\alpha\comp\rn{\sigma}_\alpha)(C)),m}\comp\iota_{\rn{U'}_{(\alpha\comp\rn{\sigma}_\alpha)(C)},k}^{-1}}).
      \end{aligned}
    \end{equation}

    Since $\alpha\in\cG(\rn{\sigma},\rn{U},\rn{U'})$, we have
    \begin{align*}
      \rn{U}_{(\alpha\comp\rn{\sigma}_\alpha)(C)}
      & =
      \rn{U'}_{(\alpha\comp\rn{\sigma}_\alpha)(C)}
      =
      \rn{\sigma}_\alpha^{-1}(\rn{\sigma}_\alpha(C))
      =
      C,
    \end{align*}
    so the coordinate of both $\rn{x}$ and $\rn{x'}$ in the last expression
    of~\eqref{eq:dist2:phik1C} is
    \begin{align*}
      \rn{\sigma}\comp\iota_{(\rn{\sigma}^{-1}\comp\alpha\comp\rn{\sigma}_\alpha)(C),m}\comp\iota_{C,k}^{-1}.
    \end{align*}
    Comparing the above with~\eqref{eq:dist2:xx'1C}, it suffices to show
    \begin{align}\label{eq:dist2:iota}
      \iota_{(\rn{\sigma}^{-1}\comp\alpha\comp\rn{\sigma}_\alpha)(C),m}\comp\iota_{C,k}^{-1}
      & =
      \iota_{\im(\rn{\sigma}^{-1}\comp\alpha),m}\rest_C.
    \end{align}

    But this equality is easily checked: both sides are increasing functions of the form $C\to[m]$,
    so it suffices to check that their image match. The image of the right-hand side
    of~\eqref{eq:dist2:iota} is clearly $\iota_{\im(\rn{\sigma}^{-1}\comp\alpha),m}(C)$. For the
    left-hand side, first note that~\eqref{eq:taualpha} gives
    \begin{align*}
      (\rn{\sigma}^{-1}\comp\alpha\comp\rn{\sigma}_\alpha)
      & =
      \iota_{\im(\rn{\sigma}^{-1}\comp\alpha),m},
    \end{align*}
    so the left-hand side of~\eqref{eq:dist2:iota} is
    \begin{align*}
      \iota_{\iota_{\im(\rn{\sigma}^{-1}\comp\alpha),m}(C),m}\comp\iota_{C,k}^{-1}
    \end{align*}
    whose image is clearly $\iota_{\im(\rn{\sigma}^{-1}\comp\alpha),m}(C)$, as desired.
  \end{proofof}

  We can finally conclude the proof of the proposition. Let $\epsilon,\delta\in(0,1)$ and let
  \begin{align*}
    m
    & \geq
    m^{\agPACr}_{\cH^{\kpart},\ell^{\kpart},\cA'}(\epsilon,\delta)
    \df
    m^{\agPACr}_{\cH,\ell,\cA}\left(\frac{p\cdot\epsilon}{2},\widetilde{\delta}_\ell(\epsilon,\delta)\right)
  \end{align*}
  be an integer, where
  \begin{align*}
    \widetilde{\delta}_\ell(\epsilon,\delta)
    & \df
    \min\left\{\frac{\epsilon\delta}{2\lVert\ell\rVert_\infty}, \frac{1}{2}\right\}.
  \end{align*}
  is given by~\eqref{eq:kpart2:wdelta}. (The only reason for taking minimum with $1/2$ is to ensure
  that the number above is in $(0,1)$.)

  Let $(\rn{x},\rn{x'})\sim(\mu\otimes\mu')^m$ and let $(\rn{\sigma},\rn{U},\rn{U'})$ be picked
  uniformly at random in $S_m\times\prod_{C\in r(m,k)} \binom{[k]}{\lvert C\rvert}\times\prod_{C\in
    r(m,k)} \binom{[k]}{\lvert C\rvert}$, independently from $\rn{x}$ and $\rn{x'}$. Let also
  $\rn{y}\df F^*_m(\rn{x},\rn{x'})$. Let further
  $(\rn{\widehat{x}},\rn{\widehat{x'}},\rn{z},\rn{z'},\rn{w})\sim
  (\widehat{\mu}\otimes\widehat{\mu'}\otimes\lambda\otimes\lambda\otimes\lambda)^m$ and
  $\rn{\widehat{y}}\df\widehat{F}^*_m(\rn{\widehat{x}},\rn{\widehat{x'}},\rn{z},\rn{z'},\rn{w})$. Finally,
  let $\rn{b}$ be picked uniformly in $[R_\cA(m)]=[R_{\cA'}(m)]$, independently from all previous
  random elements and let
  \begin{align*}
    I & \df \inf_{H\in\cH^{\kpart}} L_{\mu,\mu',F,\ell^{\kpart}}(H), &
    \widehat{I} & \df \inf_{H\in\cH} L_{\widehat{\mu},\widehat{\mu''},\widehat{F},\ell}(H).
  \end{align*}

  Recall that by Claim~\ref{clm:dist}\ref{clm:dist:infloss}, we have
  \begin{align*}
    \widehat{I} & = (1-p)\cdot C_{\ell,\bot,F} + p\cdot I.
  \end{align*}

  We claim that
  \begin{align}\label{eq:kpart2:PP}
    \PP_{\rn{x},\rn{x'},\rn{\sigma},\rn{U},\rn{U'}}\biggl[
      \EE_{\rn{b}}\Bigl[
        L_{\mu,\mu',F,\ell^{\kpart}}\bigl(
        \cA'(\rn{x},\rn{y},\rn{b},\rn{\sigma},\rn{U},\rn{U'})
        \bigr)
        \Bigr]
      \leq I + \frac{\epsilon}{2}
      \biggr]
    & \geq
    1 - \widetilde{\delta}_\ell(\epsilon,\delta).
  \end{align}
  Indeed, note that
  \begin{align*}
    & \!\!\!\!\!\!
    \PP_{\rn{x},\rn{x'},\rn{\sigma},\rn{U},\rn{U'}}\biggl[
      \EE_{\rn{b}}\Bigl[
        L_{\mu,\mu',F,\ell^{\kpart}}\bigl(
        \cA'(\rn{x},\rn{y},\rn{b},\rn{\sigma},\rn{U},\rn{U'})
        \bigr)
        \Bigr]
      \leq I + \frac{\epsilon}{2}
      \biggr]
    \\
    & =
    \PP_{\rn{x},\rn{x'},\rn{\sigma},\rn{U},\rn{U'}}\biggl[
      \EE_{\rn{b}}\Bigl[
        L_{\mu,\mu',F,\ell^{\kpart}}\bigl(
        \cA(\rn{x}^{\rn{\sigma},\rn{U}},\rn{y}^{\rn{\sigma},\rn{U},\rn{U'}},\rn{b})^{\kpart}
        \bigr)
        \Bigr]
      \leq I + \frac{\epsilon}{2}
      \biggr]
    \\
    & =
    \PP_{\rn{\widehat{x}},\rn{\widehat{x'}},\rn{z},\rn{z'},\rn{w}}\biggl[
      \EE_{\rn{b}}\Bigl[
        L_{\mu,\mu',F,\ell^{\kpart}}\bigl(
        \cA(\rn{\widehat{x}},\rn{\widehat{y}},\rn{b})^{\kpart}
        \bigr)
        \Bigr]
      \leq I + \frac{\epsilon}{2}
      \biggr]
    \\
    & =
    \PP_{\rn{\widehat{x}},\rn{\widehat{x'}},\rn{z},\rn{z'},\rn{w}}\biggl[
      \EE_{\rn{b}}\Bigl[
        L_{\widehat{\mu},\widehat{\mu''},\widehat{F},\ell}\bigl(
        \cA(\rn{\widehat{x}},\rn{\widehat{y}},\rn{b})
        \bigr)
        \Bigr]
      \leq \widehat{I} + \frac{p\cdot\epsilon}{2}
      \biggr]
    \\
    & \geq 1 - \widetilde{\delta}_\ell(\epsilon,\delta),
  \end{align*}
  where the first equality follows from the definition of $\cA'$, the second equality follows from
  Claim~\ref{clm:dist}\ref{clm:dist:tilde} and Claim~\ref{clm:dist2}\ref{clm:dist2:sample}, the
  third equality follows from Claim~\ref{clm:dist}, items~\ref{clm:dist:loss}
  and~\ref{clm:dist:infloss} and the inequality follows since $\cA$ is a randomized agnostic $k$-PAC
  learner for $\cH$ and $m\geq
  m^{\agPACr}_{\cH,\ell,\cA}(p\cdot\epsilon/2,\widetilde{\delta}_\ell(\epsilon,\delta))$. Thus~\eqref{eq:kpart2:PP}
  holds.

  Note now that we can rewrite~\eqref{eq:kpart2:PP} as
  \begin{align*}
    \widetilde{\delta}_\ell(\epsilon,\delta)
    & \geq
    \PP_{\rn{x},\rn{x'},\rn{\sigma},\rn{U},\rn{U'}}\biggl[
      \EE_{\rn{b}}\Bigl[
        L_{\mu,\mu',F,\ell^{\kpart}}\bigl(
        \cA'(\rn{x},\rn{y},\rn{b},\rn{\sigma},\rn{U},\rn{U'})
        \bigr)
        \Bigr]
      > I + \frac{\epsilon}{2}
      \biggr]
    \\
    & =
    \EE_{\rn{x},\rn{x'}}\Bigl[
      \PP_{\rn{\sigma},\rn{U},\rn{U'}}\bigl[
        E(\rn{x},\rn{x'},\rn{\sigma},\rn{U},\rn{U'})
        \bigr]
      \Bigr],
  \end{align*}
  where $E(\rn{x},\rn{x'},\rn{\sigma},\rn{U},\rn{U'})$ is the event
  \begin{align*}
    \EE_{\rn{b}}\Bigl[
      L_{\mu,\mu',F,\ell^{\kpart}}\bigl(
      \cA'(\rn{x},\rn{y},\rn{b},\rn{\sigma},\rn{U},\rn{U'})
      \bigr)
      \Bigr]
    > I + \frac{\epsilon}{2}.
  \end{align*}

  By Markov's Inequality, we get
  \begin{align}\label{eq:kpart2:Markov}
    \PP_{\rn{x},\rn{x'}}\biggl[
      \PP_{\rn{\sigma},\rn{U},\rn{U'}}\bigl[
        E(\rn{x},\rn{x'},\rn{\sigma},\rn{U},\rn{U'})
        \bigr]
      > \frac{\epsilon}{2\lVert\ell\rVert_\infty}
      \biggr]
    & \leq
    \frac{2\cdot\lVert\ell\rVert_\infty\cdot\widetilde{\delta}_\ell(\epsilon,\delta)}{\epsilon}
    \leq
    \delta.
  \end{align}

  Since the total loss is bounded by $\lVert\ell\rVert_\infty$, we have the following implication:
  \begin{align*}
    & \!\!\!\!\!\!
    \PP_{\rn{\sigma},\rn{U},\rn{U'}}\bigl[
        E(\rn{x},\rn{x'},\rn{\sigma},\rn{U},\rn{U'})
        \bigr]
    \leq\frac{\epsilon}{2\lVert\ell\rVert_\infty}
    \\
    & \implies
    \EE_{\rn{\sigma},\rn{U},\rn{U'},\rn{b}}\Bigl[
      L_{\mu,\mu',F,\ell^{\kpart}}\bigl(
      \cA'(\rn{x},\rn{y},\rn{b},\rn{\sigma},\rn{U},\rn{U'})
      \bigr)
      \Bigr]
    \leq
    I + \frac{\epsilon}{2} + \frac{\epsilon}{2\lVert\ell\rVert_\infty}\cdot\lVert\ell\rVert_\infty
    =
    I + \epsilon.
  \end{align*}

  Applying the contra-positive of the above to~\eqref{eq:kpart2:Markov} gives
  \begin{align*}
    \PP_{\rn{x},\rn{x'}}\biggl[
      \EE_{\rn{\sigma},\rn{U},\rn{U'},\rn{b}}\Bigl[
        L_{\mu,\mu',F,\ell^{\kpart}}\bigl(
        \cA'(\rn{x},\rn{y},\rn{b},\rn{\sigma},\rn{U},\rn{U'})
        \bigr)
        \Bigr]
      > I + \epsilon
      \biggr]
    & \leq
    \delta.
  \end{align*}
  Therefore, $\cA'$ is a randomized agnostic $k$-PAC learner for $\cH^{\kpart}$.
\end{proof}

\begin{remark}
  Even though the upper bound of the extra randomness needed in Proposition~\ref{prop:kpart2} looks
  pretty bad ($m!\cdot k^{2\cdot k\cdot m^k}$), we remind the reader that $R_\cA$ encodes the size
  of the space used for randomness. This means that if our randomness were to be encoded by uniform
  random bits, then we would need at most $O(k\cdot m^k\cdot\ln(k) + m\cdot\ln(m))$ such random bits
  (some care is needed here since the randomness used in the reduction is not necessarily a power of
  $2$, but one can solve this by adjusting $\delta$). Once this is observed, it is then easy to
  check that the reduction above is a randomized polynomial-time reduction.
\end{remark}

Proposition~\ref{prop:neutsymb} below shows how to produce a neutral symbol for a bounded flexible
agnostic loss function without disrupting agnostic $k$-PAC learnability. This in particular shows
that flexibility is morally equivalent to the existence of neutral symbols (see
Remark~\ref{rmk:neutsymb}). Before we prove it, let us finally give an intuition for the definition
of flexibility as it will serve as an intuition for the proof of Proposition~\ref{prop:neutsymb}
below.

\begin{remark}\label{rmk:flexibilityintuition}
  If $(\Sigma,\nu,G,\cN)$ witnesses the flexibility of a $k$-ary agnostic loss function
  $\ell\colon\cH\times\cE_k(\Omega)\times\Lambda^{S_k}\to\RR_{\geq 0}$, then for a carefully
  constructed $k$-ary agnostic loss function $\ell^{\Sigma,\nu,G,\bot}$ that has a neutral symbol
  $\bot$, an agnostic adversary is able to use the extra Borel template $\Sigma$ and $G$ to sample
  points in $\Lambda$ that effectively simulate the neutral symbol $\bot$ so that the total loss
  with respect to $\ell^{\Sigma,\nu,G,\bot}$ of the sample that can use $\bot$ is equal to the total
  loss with respect to $\ell$ of the simulated sample.

  On the other hand, the existence of the function $\cN$ and the condition that for every
  $x\in\cE_m(\Omega)$, if $\rn{z}\sim\nu^m$ and $\rn{b}$ is picked uniformly in $[R_\cN(m)]$, then
  $G^*_m(x,\rn{z}) \sim \cN(x,\rn{b})$ implies that a randomized algorithm is also able to simulate
  the neutral symbol $\bot$ using its finite source of randomness. This will allow us to perform a
  trick similar to that of Proposition~\ref{prop:kpart2}: a randomized algorithm that needs to
  agnostically learn with respect to $\ell^{\Sigma,\nu,G,\bot}$ can simulate all occurrences of
  $\bot$ with $\cN$ using some extra randomness (exactly a multiplicative factor $R_\cN(m)$ of extra
  randomness) and pass the simulation to a randomized $k$-PAC learner with respect to $\ell$.
\end{remark}

\begin{proposition}[Creating neutral symbols]\label{prop:neutsymb}
  Let $k\in\NN_+$, let $\Omega$ be a Borel ($k$-partite, respectively) template, let $\Lambda$ be a
  non-empty Borel space, let $\cH\subseteq\cF_k(\Omega,\Lambda)$ be a $k$-ary ($k$-partite,
  respectively) hypothesis class and let $\ell$ be a $k$-ary ($k$-partite, respectively) agnostic
  loss function. Suppose that $\ell$ is flexible and let $(\Sigma,\nu,G,\cN)$ witness its
  flexibility.

  Define $\Lambda'\df\Lambda\cup\{\bot\}$, where $\bot$ is a new element and $\Lambda'$ is equipped
  with the co-product $\sigma$-algebra. Let $\cH'$ be $\cH$ when we increase the codomain of its
  elements to $\Lambda'$; formally, let
  \begin{align*}
    \cH' & \df \{\iota\comp H \mid H\in\cH\},
  \end{align*}
  where $\iota\colon\Lambda\to\Lambda'$ is the inclusion map. Let also $\iota_\cH\colon\cH\to\cH'$
  be the bijection given by $\iota_\cH(H)\df\iota\comp H$. We also view $\Lambda^{([m])_k}$
  ($\Lambda^{[m]^k}$, respectively) as a subset of $(\Lambda')^{([m])_k}$ ($(\Lambda')^{[m]^k}$,
  respectively) naturally.

  In the non-partite case, let $\ell^{\Sigma,\nu,G,\bot}\colon\cH'\times\cE_k(\Omega)\times(\Lambda')^{S_k}\to\RR_{\geq 0}$
  be given by
  \begin{align*}
    \ell^{\Sigma,\nu,G,\bot}(H,x,y)
    & \df
    \begin{dcases*}
      \ell(\iota_\cH^{-1}(H),x,y), & if $\bot\notin\im(y)$,\\
      \ell^{\Sigma,\nu,G}(x), & if $\bot\in\im(y)$,
    \end{dcases*}
    \qquad (H\in\cH', x\in\cE_k(\Omega), y\in(\Lambda')^{S_k}),
  \end{align*}
  where $\ell^{\Sigma,\nu,G}\colon\cE_k(\Omega)\to\RR_{\geq 0}$ is given
  by~\eqref{eq:flexibleellSigmanuG} as per the definition of flexibility.

  In the partite case, let
  $\ell^{\Sigma,\nu,G,\bot}\colon\cH'\times\cE_1(\Omega)\times\Lambda'\to\RR_{\geq 0}$ be given by
  \begin{align*}
    \ell^{\Sigma,\nu,G,\bot}(H,x,y)
    & \df
    \begin{dcases*}
      \ell(\iota_\cH^{-1}(H),x,y), & if $y\neq\bot$,\\
      \ell^{\Sigma,\nu,G}(x), & if $y=\bot$,
    \end{dcases*}
    \qquad (H\in\cH', x\in\cE_1(\Omega), y\in\Lambda'),
  \end{align*}
  where $\ell^{\Sigma,\nu,G}\colon\cE_1(\Omega)\to\RR_{\geq 0}$ is given
  by~\eqref{eq:partiteflexibleellSigmanuG} as per the definition of flexibility.

  Then the following hold:
  \begin{enumerate}
  \item\label{prop:neutsymb:bounded} We have
    $\lVert\ell^{\Sigma,\nu,G,\bot}\rVert_\infty=\lVert\ell\rVert_\infty$.
  \item\label{prop:neutsymb:symm} In the non-partite case, if $\ell$ is symmetric, then so is
    $\ell^{\Sigma,\nu,G,\bot}$.
  \item\label{prop:neutsymb:ext} $\bot$ is a neutral symbol for $\ell^{\Sigma,\nu,G,\bot}$ with
    $\ell^{\Sigma,\nu,G,\bot}_\bot=\ell^{\Sigma,\nu,G}$ (see~\eqref{eq:flexibleellSigmanuG}
    and~\eqref{eq:neutsymbellbot} for the non-partite case and~\eqref{eq:partiteflexibleellSigmanuG}
    and~\eqref{eq:partiteneutsymbellbot} for the partite case).
  \item\label{prop:neutsymb:agPAC} Both in the non-partite and partite cases, if $\ell$ is bounded
    and $\cA$ is a randomized $k$-PAC learner for $\cH$ with respect to $\ell$, then there exists a
    randomized $k$-PAC learner $\cA'$ for $\cH'$ with respect to $\ell^{\Sigma,\nu,G,\bot}$ with
    \begin{align*}
      R_{\cA'}(m) & \df R_\cA(m)\cdot R_\cN(m),
      &
      m^{\agPACr}_{\cH',\ell^{\Sigma,\nu,G,\bot},\cA'}
      & \df
      m^{\agPACr}_{\cH,\ell,\cA}\left(\frac{\epsilon}{2}, \widetilde{\delta}_\ell(\epsilon,\delta)\right),
    \end{align*}
    where
    \begin{align}\label{eq:neutsymbdelta}
      \widetilde{\delta}_\ell(\epsilon,\delta)
      & \df
      \min\left\{\frac{\epsilon\delta}{2\lVert\ell\rVert_\infty}, \frac{1}{2}\right\}.
    \end{align}
  \end{enumerate}
\end{proposition}

Before we start the proof, we remind the reader that the converse of item~\ref{prop:neutsymb:agPAC}
is covered by Proposition~\ref{prop:cod}\ref{prop:cod:agPAC} and does not require
flexibility. Furthermore, even though we did not state the version of item~\ref{prop:neutsymb:agPAC}
without randomness, the derandomization of Proposition~\ref{prop:derand} allows us to remove
randomness afterward.

\begin{proof}  
  Item~\ref{prop:neutsymb:bounded} is obvious.

  \medskip

  Item~\ref{prop:neutsymb:symm} is easily checked: let $H\in\cH$, $x\in\cE_k(\Omega)$,
  $y\in(\Lambda')^{S_k}$ and $\sigma\in S_k$.

  If $\bot\notin\im(y)$, then
  \begin{align*}
    \ell^{\Sigma,\nu,G,\bot}(H,\sigma^*(x),\sigma^*(y))
    & =
    \ell(\iota_\cH^{-1}(H),\sigma^*(x),\sigma^*(y))
    =
    \ell(\iota_\cH^{-1}(H),x,y)
    =
    \ell^{\Sigma,\nu,G,\bot}(H,x,y),
  \end{align*}
  where the second equality follows since $\ell$ is symmetric.

  If $\bot\in\im(y)$, then
  \begin{align*}
    \ell^{\Sigma,\nu,G,\bot}(H,\sigma^*(x),\sigma^*(y))
    & =
    \ell^{\Sigma,\nu,G}(\sigma^*(x))
    =
    \EE_{\rn{z}\sim\nu^k}[\ell(H,\sigma^*(x),G^*_k(\sigma^*(x),\rn{z}))]
    \\
    & =
    \EE_{\rn{z}\sim\nu^k}[\ell(H,x,G^*_k(x,(\sigma^{-1})^*(\rn{z})))]
    =
    \EE_{\rn{z}\sim\nu^k}[\ell(H,x,G^*_k(x,\rn{z}))]
    \\
    & =
    \ell^{\Sigma,\nu,G}(x)
    =
    \ell^{\Sigma,\nu,G,\bot}(H,x,y),
  \end{align*}
  where the third equality follows since $\ell$ is symmetric and the fourth equality follows since
  $(\sigma^{-1})^*(\rn{z})\sim\rn{z}$.

  \medskip

  Item~\ref{prop:neutsymb:ext} follows immediately from the definition of neutral symbol.

  \medskip

  We now prove item~\ref{prop:neutsymb:agPAC} in the non-partite case. Assume
  $\lVert\ell\rVert_\infty > 0$ (otherwise the result is trivial), let $\Omega'$ be a Borel
  template, let $\mu'\in\Pr(\Omega')$ and let $F'\in\cF_k(\Omega\otimes\Omega',\Lambda')$. Define
  $F\in\cF_k(\Omega\otimes\Omega'\otimes\Sigma,\Lambda)$ by
  \begin{align*}
    F(x,x',z)
    & \df
    \begin{dcases*}
      F'(x,x'), & if $\bot\notin\im((F')^*_k(x,x'))$,\\
      G(x,z), & otherwise.
    \end{dcases*}
    \qquad (x\in\cE_k(\Omega), x'\in\cE_k(\Omega'), z\in\cE_k(\Sigma)).
  \end{align*}

  Similarly to the proof of Proposition~\ref{prop:kpart2}, since $R_{\cA'}(m) = R_\cA(m)\cdot
  R_\cN(m)$ for each $m\in\NN$, by using an appropriate (fixed) bijection, we may assume that $\cA'$
  is of the form
  \begin{align*}
    \cA'\colon\bigcup_{m\in\NN}(\cE_m(\Omega)\times(\Lambda')^{([m])_k}\times[R_\cA(m)]\times[R_\cN(m)])
    \longrightarrow
    \cH'.
  \end{align*}
  We then define $\cA'$ as follows: for an input
  \begin{align*}
    (x,y,b,\widetilde{b})\in\cE_m(\Omega)\times(\Lambda')^{([m])_k}\times[R_\cA(m)]\times[R_\cN(m)],
  \end{align*}
  we let
  \begin{align*}
    \cA'(x,y,b,\widetilde{b})
    & \df
    \iota_\cH(\cA(x,y^{x,\widetilde{b}},b)),
  \end{align*}
  where
  \begin{align*}
    y^{x,\widetilde{b}}_\alpha
    & \df
    \begin{dcases*}
      y_\alpha,
      & if $y_\beta\neq\bot$ for every $\beta\in([m])_k$ with $\im(\alpha)=\im(\beta)$,
      \\
      \cN(x,\widetilde{b})_\alpha, & otherwise,
    \end{dcases*}
    \qquad (\alpha\in([m])_k).
  \end{align*}

  By assumption of flexibility, we know that for every $m\in\NN$, if
  $(\rn{x},\rn{x'},\rn{z})\sim(\mu\otimes\mu'\otimes\nu)^m$ and $(\rn{b},\rn{\widetilde{b}})$ is
  picked uniformly at random in $[R_\cA(m)]\times[R_\cN(m)]$ independently from
  $(\rn{x},\rn{x'},\rn{z})$, then
  \begin{align*}
    (\rn{x},\rn{x'},G^*_m(\rn{x},\rn{z}),\rn{b})
    & \sim
    (\rn{x},\rn{x'},\cN(\rn{x},\rn{\widetilde{b}}),\rn{b}),
  \end{align*}
  which in particular implies
  \begin{align}\label{eq:neutsymbdist}
    (\rn{x}, F^*_m(\rn{x},\rn{x'},\rn{z}), \rn{b})
    & \sim
    (\rn{x}, (F')^*_m(\rn{x},\rn{x'})^{\rn{x},\rn{\widetilde{b}}}, \rn{b}).
  \end{align}

  Note now that if $m=k$, $H\in\cH$, $E(\rn{x},\rn{x'})$ denotes the event
  $\bot\notin\im((F')^*_k(\rn{x},\rn{x'}))$ and $\overline{E}(\rn{x},\rn{x'})$ denotes its
  complement, then
  \begin{equation}\label{eq:neutsymbLL}
    \begin{aligned}
      L_{\mu,\mu',F',\ell^{\Sigma,\nu,G,\bot}}(\iota_\cH(H))
      & =
      \begin{multlined}[t]
        \PP_{\rn{x},\rn{x'}}[E(\rn{x},\rn{x'})]
        \cdot\EE_{\rn{x},\rn{x'}}[\ell(H,\rn{x},(F')^*_k(\rn{x},\rn{x'}))\given E(\rn{x},\rn{x'})]
        \\
        +
        \PP_{\rn{x},\rn{x'}}[\overline{E}(\rn{x},\rn{x'})]
        \cdot\EE_{\rn{x},\rn{x'}}[\ell^{\Sigma,\nu,G}(\rn{x})\given \overline{E}(\rn{x},\rn{x'})]
      \end{multlined}
      \\
      & =
      \begin{multlined}[t]
        \PP_{\rn{x},\rn{x'},\rn{z}}[E(\rn{x},\rn{x'})]
        \cdot\EE_{\rn{x},\rn{x'},\rn{z}}[\ell(H,\rn{x}, (F')^*_k(\rn{x},\rn{x'}))\given
          E(\rn{x},\rn{x'})]
        \\
        +
        \PP_{\rn{x},\rn{x'},\rn{z}}[\overline{E}(\rn{x},\rn{x'})]
        \cdot\EE_{\rn{x},\rn{x'},\rn{z}}[\ell(H,\rn{x}, G^*_k(\rn{x},\rn{z}))\given
          \overline{E}(\rn{x},\rn{x'})]
      \end{multlined}
      \\
      & =
      \EE_{\rn{x},\rn{x'},\rn{z}}[\ell(H,\rn{x}, F^*_k(\rn{x},\rn{x'},\rn{z}))]
      \\
      & =
      L_{\mu,\mu'\otimes\nu,F,\ell}(H),
    \end{aligned}
  \end{equation}
  where the first equality follows from the definition of $\ell^{\Sigma,\nu,G,\bot}$, the second
  equality follows from the definition of $\ell^{\Sigma,\nu,G}$ (see~\eqref{eq:flexibleellSigmanuG})
  and the third equality follows from the definition of $F$. This along with the fact that
  $\iota_\cH$ is bijective implies
  \begin{align*}
    \inf_{H\in\cH'} L_{\mu,\mu',F',\ell^{\Sigma,\nu,G,\bot}}(H) & = \inf_{H\in\cH} L_{\mu,\mu'\otimes\nu,F,\ell}(H).
  \end{align*}
  Let $I$ be the infimum above.

  Now fix $\epsilon,\delta\in(0,1)$, let
  \begin{align*}
    m
    & \geq
    m^{\agPACr}_{\cH',\ell^{\Sigma,\nu,G,\bot},\cA'}(\epsilon,\delta)
    \df
    m^{\agPACr}_{\cH,\ell,\cA}\left(\frac{\epsilon}{2}, \widetilde{\delta}_\ell(\epsilon,\delta)\right)
  \end{align*}
  be an integer, where
  \begin{align*}
    \widetilde{\delta}_\ell(\epsilon,\delta)
    & \df
    \min\left\{\frac{\epsilon\delta}{2\lVert\ell\rVert_\infty}, \frac{1}{2}\right\}
  \end{align*}
  is given by~\eqref{eq:neutsymbdelta} (the only reason for taking minimum with $1/2$ is to ensure
  that the number above is in $(0,1)$) and note that
  \begin{equation}\label{eq:neutsymbPP}
    \begin{aligned}
      & \!\!\!\!\!\!
      \PP_{\rn{x},\rn{x'},\rn{\widetilde{b}}}\biggl[
        \EE_{\rn{b}}\Bigl[
          L_{\mu,\mu',F',\ell^{\Sigma,\nu,G,\bot}}\bigl(
          \cA'(\rn{x}, (F')^*_m(\rn{x},\rn{x'}), \rn{b}, \rn{\widetilde{b}})
          \bigr)
          \Bigr]
        \leq I + \frac{\epsilon}{2}
        \biggr]
      \\
      & =
      \PP_{\rn{x},\rn{x'},\rn{\widetilde{b}}}\biggl[
        \EE_{\rn{b}}\Bigl[
          L_{\mu,\mu',F,\ell}\bigl(
          \cA(\rn{x}, (F')^*_m(\rn{x},\rn{x'})^{\rn{x},\rn{\widetilde{b}}}, \rn{b})
          \bigr)
          \Bigr]
        \leq I + \frac{\epsilon}{2}
        \biggr]
      \\
      & =
      \PP_{\rn{x},\rn{x'},\rn{z}}\biggl[
        \EE_{\rn{b}}\Bigl[
          L_{\mu,\mu'\otimes\nu,F,\ell}\bigl(
          \cA(\rn{x}, F^*_m(\rn{x},\rn{x'},\rn{z}), \rn{b})
          \bigr)
          \Bigr]
        \leq I + \frac{\epsilon}{2}
        \biggr]
      \\
      & \geq 1 - \widetilde{\delta}_\ell(\epsilon,\delta),
    \end{aligned}
  \end{equation}
  where the first equality follows from the definition of $\cA'$ and~\eqref{eq:neutsymbLL}, the
  second equality follows from~\eqref{eq:neutsymbdist} and the inequality follows since $\cA$ is a
  randomized agnostic $k$-PAC learner for $\cH$ and $m\geq
  m^{\agPACr}_{\cH,\ell,\cA}(\epsilon/2,\widetilde{\delta}_\ell(\epsilon,\delta))$.

  Now the same Markov's Inequality plus boundedness argument of the end of the proof of
  Proposition~\ref{prop:kpart2} finishes the job: \eqref{eq:neutsymbPP} can be rewritten as
  \begin{align*}
    \widetilde{\delta}_\ell(\epsilon,\delta)
    & \geq
    \PP_{\rn{x},\rn{x'},\rn{\widetilde{b}}}\biggl[
        \EE_{\rn{b}}\Bigl[
          L_{\mu,\mu',F',\ell^{\Sigma,\nu,G,\bot}}\bigl(
          \cA'(\rn{x}, (F')^*_m(\rn{x},\rn{x'}), \rn{b}, \rn{\widetilde{b}})
          \bigr)
          \Bigr]
        > I + \frac{\epsilon}{2}
        \biggr]
    \\
    & =
    \EE_{\rn{x},\rn{x'}}\Biggl[
      \PP_{\rn{\widetilde{b}}}\biggl[
        \EE_{\rn{b}}\Bigl[
          L_{\mu,\mu',F',\ell^{\Sigma,\nu,G,\bot}}\bigl(
          \cA'(\rn{x}, (F')^*_m(\rn{x},\rn{x'}), \rn{b}, \rn{\widetilde{b}})
          \bigr)
          \Bigr]
        > I + \frac{\epsilon}{2}
        \biggr]
      \Biggr],
  \end{align*}
  which by Markov's Inequality gives
  \begin{equation}\label{eq:neutsymbMarkov}
    \begin{aligned}
      & \!\!\!\!\!\!
      \PP_{\rn{x},\rn{x'}}\Biggl[
        \PP_{\rn{\widetilde{b}}}\biggl[
          \EE_{\rn{b}}\Bigl[
            L_{\mu,\mu',F',\ell^{\Sigma,\nu,G,\bot}}\bigl(
            \cA'(\rn{x}, (F')^*_m(\rn{x},\rn{x'}), \rn{b}, \rn{\widetilde{b}})
            \bigr)
            \Bigr]
          > I + \frac{\epsilon}{2}
          \biggr]
        >
        \frac{\epsilon}{2\lVert\ell\rVert_\infty}
        \Biggr]
      \\
      & \leq
      \frac{2\cdot\lVert\ell\rVert_\infty\cdot\widetilde{\delta}_\ell(\epsilon,\delta)}{\epsilon}
      \leq
      \delta.
    \end{aligned}
  \end{equation}

  Since the total loss is bounded by $\lVert\ell\rVert_\infty$ (by item~\ref{prop:neutsymb:bounded}),
  we have the following implication:
  \begin{align*}
    & \!\!\!\!\!\!
    \PP_{\rn{\widetilde{b}}}\biggl[
      \EE_{\rn{b}}\Bigl[
          L_{\mu,\mu',F',\ell^{\Sigma,\nu,G,\bot}}\bigl(
          \cA'(\rn{x}, (F')^*_m(\rn{x},\rn{x'}), \rn{b}, \rn{\widetilde{b}})
          \bigr)
          \Bigr]
      > I + \frac{\epsilon}{2}
      \biggr]
    \leq \frac{\epsilon}{2\lVert\ell\rVert_\infty}
    \\
    & \implies
    \EE_{\rn{\widetilde{b}}}\biggl[
      \EE_{\rn{b}}\Bigl[
          L_{\mu,\mu',F',\ell^{\Sigma,\nu,G,\bot}}\bigl(
          \cA'(\rn{x}, (F')^*_m(\rn{x},\rn{x'}), \rn{b}, \rn{\widetilde{b}})
          \bigr)
          \Bigr]
      \biggr]
    \leq
    I + \frac{\epsilon}{2} + \frac{\epsilon}{2\lVert\ell\rVert_\infty}\cdot\lVert\ell\rVert_\infty
    =
    I + \epsilon,
  \end{align*}
  so applying the contra-positive of the above to~\eqref{eq:neutsymbMarkov}, we get
  \begin{align*}
    \PP_{\rn{x},\rn{x'}}\biggl[
      \EE_{\rn{b},\rn{\widetilde{b}}}\Bigl[
        L_{\mu,\mu',F',\ell^{\Sigma,\nu,G,\bot}}\bigl(
        \cA'(\rn{x}, (F')^*_m(\rn{x},\rn{x'}), \rn{b}, \rn{\widetilde{b}})
        \bigr)
        \Bigr]
      > I + \epsilon
      \biggr]
    & \leq \delta.
  \end{align*}
  Therefore, $\cA'$ is a randomized agnostic $k$-PAC learner for $\cH'$ with respect to
  $\ell^{\Sigma,\nu,G,\bot}$.

  \medskip

  The proof of item~\ref{prop:neutsymb:agPAC} in the partite case is analogous to the non-partite
  case using the following definitions instead: we set
  \begin{align*}
    F(x,x',z)
    & \df
    \begin{dcases*}
      F'(x,x'), & if $F'(x,x')\neq\bot$,\\
      G(x,z), & otherwise.
    \end{dcases*}
  \end{align*}
  and for $m\in\NN$, $x\in\cE_m(\Omega)$, $y\in(\Lambda')^{[m]^k}$, $b\in [R_\cA(m)]$ and
  $\widetilde{b}\in [R_\cN(m)]$, we set
  \begin{align*}
    \cA'(x,y,b,\widetilde{b})
    & \df
    \iota_{\cH}(\cA(x,y^{x,\widetilde{b}},b)),
  \end{align*}
  where
  \begin{align*}
    y^{x,\widetilde{b}}_\alpha
    & \df
    \begin{dcases*}
      y_\alpha, & if $y_\alpha\neq\bot$,\\
      \cN(x,\widetilde{b})_\alpha, & otherwise,
    \end{dcases*}
    \qquad (\alpha\in[m]^k).
    \qedhere
  \end{align*}
\end{proof}

We can finally put together Propositions~\ref{prop:cod}, \ref{prop:kpart2} and~\ref{prop:neutsymb} to
show that the operation $\cH\mapsto\cH^{\kpart}$ preserves agnostic $k$-PAC learnability with
randomness as long as the agnostic loss function is symmetric, bounded and flexible.

\begin{proposition}[Non-partite to partite with flexibility]\label{prop:kpart3}
  Let $\Omega$ be a Borel template, let $k\in\NN_+$, let $\Lambda$ be a non-empty Borel space, let
  $\cH\subseteq\cF_k(\Omega,\Lambda)$ be a $k$-ary hypothesis class and let
  $\ell\colon\cH\times\cE_k(\Omega)\times\Lambda^{S_k}\to\RR_{\geq 0}$ be a symmetric
  $k$-ary agnostic loss function with $\lVert\ell\rVert_\infty < \infty$. Suppose further that
  $\ell$ is flexible and let $(\Sigma,\nu,G,\cN)$ witness its flexibility.
  
  If $\cH$ is agnostically $k$-PAC learnable with randomness with respect to $\ell$, then
  $\cH^{\kpart}$ is agnostically $k$-PAC learnable with randomness with respect to $\ell^{\kpart}$;
  more precisely, if $\cA$ is a randomized agnostic $k$-PAC learner for $\cH$, then there exists a
  randomized $k$-PAC learner $\cA'$ for $\cH^{\kpart}$ with
  \begin{gather*}
    R_{\cA'}(m)
    \df
    R_\cA(m)\cdot R_\cN(m)\cdot m!\cdot\prod_{i=1}^k\binom{k}{i}^{2\binom{m}{i}}
    \leq
    R_\cA(m)\cdot R_\cN(m)\cdot m!\cdot k^{2\cdot k\cdot m^k},
    \\
    m^{\agPACr}_{\cH^{\kpart},\ell^{\kpart},\cA'}(\epsilon,\delta)
    \df
    m^{\agPACr}_{\cH,\ell,\cA}\left(\frac{p\cdot\epsilon}{2},\widehat{\delta}_\ell(\epsilon,\delta)\right),
  \end{gather*}
  where
  \begin{align}
    \widehat{\delta}_\ell(\epsilon,\delta)
    & \df
    \min\left\{
    \frac{p\cdot\epsilon^2\cdot\delta}{8\lVert\ell\rVert_\infty^2},
    \frac{p\cdot\epsilon}{8\lVert\ell\rVert_\infty},
    \frac{1}{2}
    \right\},
    \\
    p & \df \prod_{i=1}^k \binom{k}{i}^{-2\binom{k}{i}} \geq \frac{1}{2^{k\cdot 2^k}}.
    \label{eq:kpart3:p}
  \end{align}
\end{proposition}

\begin{proof}
  Let $\Lambda'\df\Lambda\cup\{\bot\}$ be equipped with co-product $\sigma$-algebra, let $\cH'$ be
  the version of $\cH$ where we increase the codomain of all its elements to $\Lambda'$ (without
  changing them as functions) and let $\ell^{\Sigma,\nu,G,\bot}$ be as in
  Proposition~\ref{prop:neutsymb}, so that item~\ref{prop:neutsymb:agPAC} of the proposition
  guarantees the existence of a randomized agnostic $k$-PAC learner $\widetilde{\cA}$ for $\cH'$
  with respect to $\ell^{\Sigma,\nu,G,\bot}$ with
  \begin{align*}
    R_{\widetilde{\cA}}(m)
    & \df
    R_\cA(m)\cdot R_\cN(m),
    &
    m^{\agPACr}_{\cH',\ell^{\Sigma,\nu,G,\bot},\widetilde{\cA}}(\epsilon,\delta)
    & \df
    m^{\agPACr}_{\cH,\ell,\cA}\left(\frac{\epsilon}{2}, \widetilde{\delta}_\ell(\epsilon,\delta)\right),
  \end{align*}
  where
  \begin{align}\label{eq:kpart3:wdelta}
    \widetilde{\delta}_\ell(\epsilon,\delta)
    & \df
    \min\left\{\frac{\epsilon\delta}{2\lVert\ell\rVert_\infty}, \frac{1}{2}\right\}.
  \end{align}

  Since items~\ref{prop:neutsymb:bounded}, \ref{prop:neutsymb:symm} and~\ref{prop:neutsymb:ext} of
  Proposition~\ref{prop:neutsymb} ensure that $\ell^{\Sigma,\nu,G,\bot}$ satisfies
  $\lVert\ell^{\Sigma,\nu,G,\bot}\rVert_\infty=\lVert\ell\rVert_\infty$, is symmetric and has $\bot$ as
  a neutral symbol, by Proposition~\ref{prop:kpart2}, there exists a randomized agnostic $k$-PAC
  learner $\widehat{\cA}$ for $(\cH')^{\kpart}$ with respect to $(\ell^{\Sigma,\nu,G,\bot})^{\kpart}$ with
  \begin{gather*}
    R_{\widehat{\cA}}(m)
    \df
    R_{\widetilde{\cA}}(m)\cdot m!\cdot\prod_{i=1}^k\binom{k}{i}^{2\binom{m}{i}}
    =
    R_\cA(m)\cdot R_\cN(m)\cdot m!\cdot\prod_{i=1}^k\binom{k}{i}^{2\binom{m}{i}}
    ,
    \\
    m^{\agPACr}_{(\cH')^{\kpart},(\ell^{\Sigma,\nu,G,\bot})^{\kpart},\widehat{\cA}}(\epsilon,\delta)
    \df
    m^{\agPACr}_{\cH',\ell^{\Sigma,\nu,G,\bot},\widetilde{\cA}}\left(
    \frac{p\cdot\epsilon}{2},\widetilde{\delta}_\ell(\epsilon,\delta)
    \right),
  \end{gather*}
  where $p$ is given by~\eqref{eq:kpart3:p} and $\widetilde{\delta}_\ell(\epsilon,\delta)$ is given
  by~\eqref{eq:kpart3:wdelta} (again).

  We now note that $(\cH')^{\kpart}$ can alternatively be obtained from $\cH^{\kpart}$ by increasing
  the codomain of all its elements from $\Lambda^{S_k}$ to $(\Lambda')^{S_k}$ and
  $(\ell^{\Sigma,\nu,G,\bot})^{\kpart}$ is an extension of $\ell^{\kpart}$, so
  Proposition~\ref{prop:cod}\ref{prop:cod:agPAC} guarantees the existence of a randomized agnostic
  $k$-PAC learner $\cA'$ with
  \begin{align*}
    R_{\cA'}(m)
    & \df
    R_{\widehat{\cA}}(m)
    =
    R_\cA(m)\cdot R_\cN(m)\cdot m!\cdot\prod_{i=1}^k\binom{k}{i}^{2\binom{m}{i}},
    \\
    m^{\agPACr}_{\cH^{\kpart},\ell^{\kpart},\cA'}(\epsilon,\delta)
    & \df
    m^{\agPACr}_{(\cH')^{\kpart},(\ell^{\Sigma,\nu,G,\bot})^{\kpart},\widehat{\cA}}(\epsilon,\delta)
    =
    m^{\agPACr}_{\cH,\ell,\cA}\left(\frac{p\cdot\epsilon}{4}, \widehat{\delta}_\ell(\epsilon,\delta)\right),
  \end{align*}
  where
  \begin{align*}
    \widehat{\delta}_\ell(\epsilon,\delta)
    & \df
    \widetilde{\delta}_\ell\left(
    \frac{p\cdot\epsilon}{2}, \widetilde{\delta}_\ell(\epsilon,\delta)
    \right)
    =
    \min\left\{
    \frac{p\cdot\epsilon^2\cdot\delta}{8\lVert\ell\rVert_\infty^2},
    \frac{p\cdot\epsilon}{8\lVert\ell\rVert_\infty},
    \frac{1}{2}
    \right\},
  \end{align*}
  as desired.
\end{proof}

\section{Uniform convergence and $\VCN_k$-dimension}
\label{sec:UC}

In this section, we will show that finite $\VCN_k$-dimension in the partite setting implies uniform
convergence, which in turn implies agnostic $k$-PAC learnability. We start with the easier second
implication.

\subsection{Uniform convergence implies agnostic learnability}

The next lemma is the high-arity ($k$-partite) analogue of the standard argument from classic PAC
learning theory that shows that empirical risk minimizers are good learners in the presence of
uniform convergence.

\begin{lemma}[Representative samples]\label{lem:repr}
  Let $\Omega$ be a Borel $k$-partite template, let $\Lambda$ be a non-empty Borel space, let
  $\cH\subseteq\cF_k(\Omega,\Lambda)$, let $\ell$ be a $k$-partite agnostic loss function and
  suppose $\cA$ is an empirical risk minimizer for $\ell$.

  If $m\in\NN_+$ and $(x,y)\in\cE_m(\Omega)$ is $\epsilon/2$-representative with respect to $\cH$, 
  $\mu\in\Pr(\Omega)$, $\mu'\in\Pr(\Omega')$, $F\in\cF_k(\Omega\otimes\Omega',\Lambda)$ and $\ell$,
  then
  \begin{align*}
    L_{\mu,\mu',F,\ell}(\cA(x,y))
    & \leq
    \inf_{H\in\cH} L_{\mu,\mu',F,\ell}(H) + \epsilon.
  \end{align*}
\end{lemma}

\begin{proof}
  Follows from
  \begin{align*}
    L_{\mu,\mu',F,\ell}(\cA(x,y))
    & \leq
    L_{x,y,\ell}(\cA(x,y)) + \frac{\epsilon}{2}
    =
    \inf_{H\in\cH} L_{x,y,\ell}(H) + \frac{\epsilon}{2}
    \leq
    \inf_{H\in\cH} L_{\mu,\mu',F,\ell}(H) + \epsilon,
  \end{align*}
  where the first and last inequalities are due to $\epsilon/2$-representativeness and the second
  inequality is due to $\cA$ being an empirical risk minimizer.
\end{proof}

The next proposition is the high-arity ($k$-partite) analogue of the classic PAC learning theory
argument showing that uniform convergence implies agnostic $k$-PAC learnability (when empirical risk
minimizers exist).

\begin{proposition}[Uniform convergence implies agnostic learnability]\label{prop:partUC->partagPAC}
  If $\cH\subseteq\cF_k(\Omega,\Lambda)$ is a $k$-partite hypothesis class with uniform convergence
  property with respect to $\ell\colon\cH\times\cE_1(\Omega)\times\Lambda\to\RR_{\geq 0}$ and
  (almost) empirical risk minimizers exist (see Remark~\ref{rmk:empriskmin}), then $\cH$ is
  agnostically $k$-PAC learnable with respect to $\ell$. More precisely, any empirical risk
  minimizer $\cA$ for $\ell$ is an agnostic $k$-PAC learner for $\cH$ with
  \begin{align*}
    m^{\agPAC}_{\cH,\ell,\cA}(\epsilon,\delta) & = m^{\UC}_{\cH,\ell}\left(\frac{\epsilon}{2},\delta\right).
  \end{align*}
\end{proposition}

\begin{proof}
  Let $m\geq m^{\UC}_{\cH,\ell}(\epsilon/2,\delta)$ be an integer and suppose that $\cA$ is an
  empirical risk minimizer for $\ell$. Then if $\mu\in\Pr(\Omega)$, $\Omega'$ is a Borel $k$-partite
  template, $\mu'\in\Pr(\Omega')$ and $F\in\cF_k(\Omega\otimes\Omega',\Lambda)$ and
  $(\rn{x},\rn{x'})\sim(\mu\otimes\mu')^m$, then with probability at least $1-\delta$,
  $(\rn{x},F^*_m(\rn{x},\rn{x'}))$ is $\epsilon/2$-representative with respect to $\cH$, $\mu$,
  $\mu'$, $F$ and $\ell$. By Lemma~\ref{lem:repr}, for all such $(\rn{x},\rn{x'})$, we have
  \begin{align*}
    L_{\mu,\mu',F,\ell}(\cA(\rn{x},F^*_m(\rn{x},\rn{x'})))
    & <
    \inf_{H\in\cH} L_{\mu,\mu',F,\ell}(H) + \epsilon.
  \end{align*}
  so we conclude that
  \begin{align*}
    \PP_{(\rn{x},\rn{x'})\sim(\mu\otimes\mu')^m}\bigl[
      L_{\mu,\mu',F,\ell}(\cA(\rn{x},F^*_m(\rn{x},\rn{x'})))
      \leq \inf_{H\in\cH} L_{\mu,\mu',F,\ell}(H) + \epsilon
    \bigr]
    & \geq 1-\delta,
  \end{align*}
  as desired.
\end{proof}

\subsection{$\VCN_k$-dimension and growth function}
\label{subsec:growth}

Let us recall the analogue of the Sauer--Shelah--Perles Lemma for the Natarajan dimension.

\begin{lemma}[\protect{\cite{Nat89}}]\label{lem:Nat}
  If $\cF$ is a collection of functions of the form $X\to Y$ with finite Natarajan dimension, then
  for every finite $V\subseteq X$, we have
  \begin{equation}\label{eq:Nat}
    \begin{aligned}
      \lvert\cF_V\rvert
      & \leq
      (\lvert V\rvert+1)_{\min\{\Nat(\cF_V),\lvert V\rvert+1\}}\cdot\binom{\lvert Y\rvert}{2}^{\Nat(\cF_V)}
      \\
      & \leq
      (\lvert V\rvert+1)^{\Nat(\cF)}\cdot\binom{\lvert Y\rvert}{2}^{\Nat(\cF)},
    \end{aligned}
  \end{equation}
  where
  \begin{align*}
    \cF_V & \df \{F\rest_V \mid F\in\cF\}.
  \end{align*}
\end{lemma}

\begin{proof}
  The second inequality follows since $\Nat(\cF_V)\leq\Nat(\cF)$ and $(n+1)_{\min\{m,n+1\}}\leq
  (n+1)^m$ for every $n,m\in\NN$.

  We prove the first inequality by induction in $n\df\lvert V\rvert$.

  First note that if $\Nat(\cF_V)=0$, then $\lvert\cF_V\rvert\leq 1$, so~\eqref{eq:Nat}
  follows. In particular, this covers the case when $n = 0$.

  Suppose then that $\Nat(\cF_V)\geq 1$, which in particular means that $\lvert\cF_V\rvert\geq 2$,
  hence $n\geq 1$.

  Let $v\in V$, let $V'\df V\setminus\{v\}$ and for every $y\in Y$ and every $F\colon V'\to Y$, let
  $F_y\colon V\to Y$ be the extension of $F$ that maps $v$ to $y$.

  For every $\{y_0,y_1\}\in\binom{Y}{2}$, let
  \begin{align*}
    \cF^{\{y_0,y_1\}}
    & \df
    \{F\colon V'\to Y \mid F_{y_0},F_{y_1}\in\cF_V\}.
  \end{align*}

  Note that
  \begin{align}\label{eq:Natbound}
    \lvert\cF_V\rvert
    & \leq
    \lvert\cF_{V'}\rvert + \sum_{\{y_1,y_2\}\in\binom{Y}{2}}\lvert\cF^{\{y_0,y_1\}}\rvert.
  \end{align}

  Clearly $\Nat(\cF_{V'})\leq\Nat(\cF_V)$. For the other families, we claim that
  $\Nat(\cF^{\{y_0,y_1\}})\leq\Nat(\cF_V)-1$ for every $\{y_0,y_1\}\in\binom{Y}{2}$. More
  specifically, we claim that if $\cF^{\{y_0,y_1\}}$ shatters $A\subseteq V'$, then $\cF_V$ shatters
  $A\cup\{v\}$. Indeed, if $f_0,f_1\colon A\to Y$ and $(F_U)_{U\subseteq A}\in
  (\cF^{\{y_0,y_1\}})^{\cP(A)}$ witness the shattering of $A$, then for the extensions
  $\widehat{f}_0,\widehat{f}_1\colon A\cup\{v\}\to Y$ such that $\widehat{f}_i(v) = y_i$, we clearly
  have $\widehat{f}_0(a)\neq\widehat{f}_1(a)$ for every $a\in A\cup\{v\}$ and
  \begin{align*}
    (F_U)_{y_0}(a) & = \widehat{f}_{\One[a\in U]}(a), &
    (F_U)_{y_1}(a) & = \widehat{f}_{\One[a\in U\cup\{v\}]}(a)
  \end{align*}
  for every $a\in A\cup\{v\}$ and every $U\subseteq A$. Since by the definition of
  $\cF^{\{y_0,y_1\}}$, all $(F_U)_{y_i}$ are in $\cF_V$, it follows that $\cF_V$ shatters
  $A\cup\{v\}$, hence $\Nat(\cF^{\{y_0,y_1\}})\leq\Nat(\cF_V)-1$.

  By inductive hypothesis, \eqref{eq:Natbound} becomes
  \begin{align*}
    \lvert\cF_V\rvert
    & \leq
    (n)_{\min\{\Nat(\cF_V),n\}}\cdot\binom{\lvert Y\rvert}{2}^{\Nat(\cF_V)}
    +
    \binom{\lvert Y\rvert}{2}\cdot
    (n)_{\min\{\Nat(\cF_V)-1,n\}}\cdot\binom{\lvert Y\rvert}{2}^{\Nat(\cF_V)-1}
    \\
    & \leq
    (n+1)_{\min\{\Nat(\cF_V),n+1\}}\cdot\binom{\lvert Y\rvert}{2}^{\Nat(\cF_V)},
  \end{align*}
  where the last inequality follows from
  \begin{align*}
    (n)_{\min\{t,n\}} + (n)_{\min\{t-1,n\}}
    & \leq
    (n)_{\min\{t-1,n\}}\cdot (n+1)
    =
    (n+1)_{\min\{t,n+1\}},
  \end{align*}
  for every $n,t\in\NN_+$.
\end{proof}

We now give the correct analogues of concepts related to $\VCN_k$-dimension.

\begin{definition}[Growth function]\label{def:growth}
  Let $\Omega$ be a Borel $k$-partite template, $\Lambda$ be a non-empty Borel space and
  $\cH\subseteq\cF_k(\Omega,\Lambda)$ be a $k$-partite hypothesis class.

  The \emph{growth function} of $\cH$ is the function $\tau^k_\cH\colon\NN\to\NN\cup\{\infty\}$
  defined as
  \begin{align*}
    \tau^k_\cH(m)
    & \df
    \sup_{x,V} \lvert\{F\rest_V \mid F\in\cH(x)\}\rvert,
  \end{align*}
  where $\cH(x)$ is given by~\eqref{eq:partitecHx} and the supremum is over all $x\in\prod_{f\in
    r_{k,A}} X_{\dom(f)}$ for some $A\in\binom{[k]}{k-1}$ and over all
  \begin{align*}
    V\subseteq\prod_{f\in r_k(1)\setminus r_{k,A}} X_{\dom f}
  \end{align*}
  of size $m$, where $r_{k,A}$ is given by~\eqref{eq:rkA}.
\end{definition}

The next lemma is the high-arity ($k$-partite) analogue of the Sauer--Shelah--Perles Lemma.

\begin{lemma}[$\VCN_k$-dimension controls growth function]\label{lem:VCNk}
  Let $\Omega$ be a Borel $k$-partite template, $\Lambda$ be a finite non-empty Borel space and
  $\cH\subseteq\cF_k(\Omega,\Lambda)$ be a $k$-partite hypothesis class with finite
  $\VCN_k$-dimension. Then
  \begin{align*}
    \tau^k_\cH(m)
    & \leq
    (m+1)_{\min\{\VCN_k(\cH),m+1\}}\cdot\binom{\lvert\Lambda\rvert}{2}^{\VCN_k(\cH)}
    \\
    & \leq
    (m+1)^{\VCN_k(\cH)}\cdot\binom{\lvert\Lambda\rvert}{2}^{\VCN_k(\cH)}
  \end{align*}
  for every $m\in\NN_+$.
\end{lemma}

\begin{proof}
  For $A\in\binom{[k]}{k-1}$ and $x\in\prod_{f\in r_{k,A}} X_{\dom(f)}$, let
  \begin{align*}
    \tau^k_\cH(m,x) & \df \sup_V \lvert\{F\rest_V \mid F\in\cH(x)\}\rvert,
  \end{align*}
  where the supremum is over all
  \begin{align*}
    V\subseteq\prod_{f\in r_k(1)\setminus r_{k,A}} X_{\dom f}
  \end{align*}
  of size $m$, so that
  \begin{align*}
    \tau^k_\cH(m)
    & =
    \sup_{\substack{A\in\binom{[k]}{k-1}\\ x\in\prod_{f\in r_{k,A}} X_{\dom(f)}}}\tau^k_\cH(m,x).
  \end{align*}

  Since $\Nat(\cH(x))\leq\VCN_k(\cH)<\infty$, by Lemma~\ref{lem:Nat}, we have
  \begin{align*}
    \tau^k_\cH(m,x)
    & \leq
    (m+1)_{\min\{\VCN_k(\cH),m+1\}}\cdot\binom{\lvert\Lambda\rvert}{2}^{\VCN_k(\cH)}
    \\
    & \leq
    (m+1)^{\VCN_k(\cH)}\cdot\binom{\lvert\Lambda\rvert}{2}^{\VCN_k(\cH)},
  \end{align*}
  so the result follows by taking supremum over $A\in\binom{[k]}{k-1}$ and $x\in\prod_{f\in r_{k,A}}
  X_{\dom(f)}$.
\end{proof}

\subsection{Finite $\VCN_k$-dimension implies uniform convergence}

To leverage Lemma~\ref{lem:VCNk} to show that finite $\VCN_k$-dimension implies uniform convergence,
it will be convenient to have partial versions of the empirical losses so that we can interpolate
between the total loss and the empirical loss with an inductive argument.

\begin{definition}[Partially empirical losses]
  Let $\Omega$ and $\Omega'$ be Borel $k$-partite templates, let $\Lambda$ be a non-empty Borel
  space, let $\cH\subseteq\cF_k(\Omega,\Lambda)$ be a $k$-partite hypothesis class, let $H\in\cH$,
  let $\mu\in\Pr(\Omega)$, let $\mu'\in\Pr(\Omega')$, let $F\in\cF_k(\Omega\otimes\Omega',\Lambda)$
  and let $\ell\colon\cH\times\cE_1(\Omega)\times\Lambda\to\RR_{\geq 0}$ be a $k$-partite agnostic
  loss function.
  \begin{enumdef}
  \item For $m\in\NN_+$ and $p\in\{0,\ldots,k\}$, we decompose the space $\cE_m(\Omega)$ as $\cE_m(\Omega)
    = \cE_m^p(\Omega)\times\overline{\cE}_m^p(\Omega)$, where
    \begin{align*}
      \cE_m^p(\Omega)
      & \df
      \prod_{\substack{f\in r_k(m)\\\dom(f)\subseteq [p]}} X_{\dom(f)},
      &
      \overline{\cE}_m^p(\Omega)
      & \df
      \prod_{\substack{f\in r_k(m)\\\dom(f)\not\subseteq [p]}} X_{\dom(f)}.
    \end{align*}
    Note that when $p=0$, the first product above is empty, so the space has only one element: the
    empty function. A similar situation happens for the second product when $p=k$.
  \item We let $\mu^{m,p}\in\Pr(\cE_m^p(\Omega))$ and
    $\overline{\mu}^{m,p}\in\Pr(\overline{\cE}_m^p(\Omega))$ be the product measures given by
    \begin{align*}
      \mu^{m,p}
      & \df
      \bigotimes_{\substack{f\in r_k(m)\\\dom(f)\subseteq [p]}} \mu_{\dom(f)},
      &
      \overline{\mu}^{m,p}
      & \df
      \bigotimes_{\substack{f\in r_k(m)\\\dom(f)\not\subseteq [p]}} \mu_{\dom(f)}.
    \end{align*}
  \item For $x\in\cE_m^p(\Omega)$ and $x'\in\cE_m^p(\Omega')$, the \emph{$p$-partially empirical
    loss} (or \emph{$p$-partially empirical risk}) of $H$ with respect to $x$, $x'$, $\mu$, $\mu'$,
    $F$ and $\ell$ is
    \begin{align*}
      L^p_{x,x',\mu,\mu',F,\ell}(H)
      & \df
      \EE_{(\rn{z},\rn{z'})\sim\overline{\mu\otimes\mu'}^{m,p}}[
        L_{(x,\rn{z}), F^*_m((x,\rn{z}),(x',\rn{z'})),\ell}(H)
      ].
    \end{align*}
  \end{enumdef}
\end{definition}

\begin{remark}\label{rmk:partemp}
  Note that by linearity of expectation and Fubini's Theorem, we have
  \begin{align*}
    \EE_{(\rn{x},\rn{x'})\sim(\mu\otimes\mu')^{m,p}}[L^p_{\rn{x},\rn{x'},\mu,\mu',F,\ell}(H)]
    & =
    L_{\mu,\mu',F,\ell}(H).
  \end{align*}
  In particular, for the unique points $x\in\cE_m^0(\Omega)$ and
  $x'\in\cE_m^0(\Omega')$, we have
  \begin{align*}
    L^0_{x,x',\mu,\mu',F,\ell}(H) & = L_{\mu,\mu',F,\ell}(H).
  \end{align*}

  Finally, for every $x\in\cE_m(\Omega)=\cE_m^k(\Omega)$ and every
  $x'\in\cE_m(\Omega')=\cE_m^k(\Omega')$, we have
  \begin{align*}
    L^k_{x,x',\mu,\mu',F,\ell}(H) & = L_{x,F^*_m(x,x'),\ell}(H)
  \end{align*}
\end{remark}

Before we proceed, we need a couple of calculation lemmas.

\begin{lemma}\label{lem:logcalcs}
  For every $x\geq x_0 > 1$, we have
  \begin{align*}
    \min\left\{\frac{\ln\ln x_0}{\ln x_0},0\right\}\cdot\ln x
    & \leq
    \ln\ln x
    \leq
    \frac{\ln x}{e}.
  \end{align*}
\end{lemma}

\begin{proof}
  For the first inequality, note that if $x_0\geq e$, then the left-hand side is $0$ and
  $\ln\ln x\geq 0$, so we may suppose that $x_0 < e$. In this case, it suffices to show that the
  function
  \begin{align*}
    f(x) & \df \frac{\ln\ln x}{\ln x}
  \end{align*}
  defined for $x\geq x_0$ attains its minimum at $x_0$. For this, we compute its derivative:
  \begin{align*}
    f'(x) & = \frac{1 - \ln\ln x}{x(\ln x)^2}
  \end{align*}
  and note that the only critical point of $f$ is at $x = e^e$.

  Since
  \begin{align*}
    f(e^e) & = \frac{1}{e} > 0, &
    \lim_{x\to\infty} f(x) & = 0, &
    f(x_0) & = \frac{\ln\ln x_0}{\ln x_0} < 0,
  \end{align*}
  the inequality follows.

  \medskip

  The second inequality is equivalent to $\ln x \leq x^{1/e}$, so it suffices to show that the function
  \begin{align*}
    g(x) & \df x^{1/e} - \ln x
  \end{align*}
  defined for $x\geq x_0$ is non-negative. We will show that $g$ is non-negative even extending its
  definition for $x\geq 1$. For this, we compute its derivative:
  \begin{align*}
    g'(x) & = \frac{x^{1/e - 1}}{e} - \frac{1}{x} = \frac{x^{1/e}/e - 1}{x} 
  \end{align*}
  and note that the only critical point of $g$ is at $x = e^e$.

  Since
  \begin{align*}
    g(e^e) & = 0, &
    g(1) & = 1, &
    \lim_{x\to\infty} g(x) & = \infty,
  \end{align*}
  the inequality follows.
\end{proof}

The next lemma is a slightly more streamlined version of~\cite[Lemma~A.3]{SB14}.

\begin{lemma}\label{lem:gaussiantail}
  Let $\rn{X}$ be a non-negative real-valued random variable and suppose $a\geq x_0 > 1$ and $b > 0$ are
  such that for every $t > 0$, we have
  \begin{align*}
    \PP[\rn{X} > t] & \leq a\exp(-bt^2).
  \end{align*}
  Then
  \begin{align*}
    \EE[\rn{X}]
    & \leq
    \left(\sqrt{1 + \max\left\{-\frac{\ln\ln x_0}{\ln x_0}, 0\right\}}
    + \frac{1}{2\sqrt{1-1/e}}\right)
    \cdot
    \sqrt{\frac{\ln a}{b}}.
  \end{align*}

  In particular, if $x_0 = 2$, then
  \begin{align*}
    \EE[\rn{X}]
    & \leq
    c\cdot\sqrt{\frac{\ln a}{b}},
  \end{align*}
  where
  \begin{align}\label{eq:gaussiantail:c}
    c
    & \df
    \sqrt{1 - \frac{\ln\ln 2}{\ln 2}}  + \frac{1}{2\sqrt{1-1/e}}
    \in
    (1.865, 1.866).
  \end{align}
\end{lemma}

\begin{proof}
  Let $u > 0$ to be picked later and note that by layer cake representation, we have
  \begin{align*}
    \EE[\rn{X}]
    & =
    \int_0^\infty \PP[\rn{X} > t]\ dt
    \leq
    \int_0^u \PP[\rn{X} > t]\ dt
    + \int_u^\infty \frac{t}{u}\cdot\PP[\rn{X} > t]\ dt
    \\
    & \leq
    u + \frac{a}{u}\int_u^\infty t\cdot\exp(-b t^2)\ dt
    =
    u + \frac{a}{2bu}\exp(-bu^2).
  \end{align*}

  Taking
  \begin{align*}
    u & \df \sqrt{\frac{\ln a - \ln\ln a}{b}},
  \end{align*}
  we get
  \begin{align*}
    \EE[\rn{X}]
    & \leq
    \sqrt{\frac{\ln a - \ln\ln a}{b}} + \frac{\ln a}{2\sqrt{b(\ln a-\ln\ln a)}}
  \end{align*}
  and the result follows by applying Lemma~\ref{lem:logcalcs}.
\end{proof}

The next lemma is responsible for interpolating between the total loss and the empirical loss. More
precisely, under the hypothesis that $\ell$ is local, the lemma (implicitly) computes lower bounds
in terms of the growth function for the probability that a random $([m],\ldots,[m])$-sample will
have:
\begin{enumerate}
\item its $p$-partially empirical loss close to its $(p-1)$-partially empirical loss;
\item its empirical loss close to its total loss.
\end{enumerate}

\begin{lemma}[Partially empirical losses interpolation]\label{lem:partemp}
  Let $\Omega$ and $\Omega'$ be Borel $k$-partite templates, let $\Lambda$ be a non-empty Borel
  space, let $\cH\subseteq\cF_k(\Omega,\Lambda)$ be a non-empty $k$-partite hypothesis class, let
  $\mu\in\Pr(\Omega)$, let $\mu'\in\Pr(\Omega')$, let $F\in\cF_k(\Omega\otimes\Omega',\Lambda)$, let
  $\ell\colon\cH\times\cE_1(\Omega)\times\Lambda\to\RR_{\geq 0}$ be a local $k$-partite agnostic
  loss function and let $m\in\NN_+$.

  We also let
  \begin{align}\label{eq:partemp:c}
    c
    & \df
    \sqrt{1 - \frac{\ln\ln 2}{\ln 2}} + \frac{1}{2\sqrt{1-1/e}}
    \in
    (1.865, 1.866)
  \end{align}
  and for every $p\in[k]$, we let $\psi_{m,p}\colon\cE_m^p(\Omega)\to\cE_m^{p-1}(\Omega)$ and
  $\psi'_{m,p}\colon\cE_m^p(\Omega')\to\cE_m^{p-1}(\Omega')$ be the projection maps.

  Then the following hold:
  \begin{enumerate}
  \item\label{lem:partemp:partial} For every $\delta\in(0,1)$ and every $p\in[k]$, we have
    \begin{multline*}
      \PP_{(\rn{x},\rn{x'})\sim(\mu\otimes\mu')^{m,p}}\Biggl[
          \sup_{H\in\cH}
          \bigl\lvert L^p_{\rn{x},\rn{x'},\mu,\mu',F,\ell}(H)
          - L^{p-1}_{\psi_{m,p}(\rn{x}),\psi'_{m,p}(\rn{x'}),\mu,\mu',F,\ell}(H)\bigr\rvert
          \\
          >
          \frac{c}{\delta}\cdot
          \sqrt{\frac{\ln(2\tau^k_\cH(2m))\cdot 2\cdot\lVert\ell\rVert_\infty^2}{m}}
          \Biggr]
        \leq
        \delta.
    \end{multline*}
  \item\label{lem:partemp:total} For every $\delta\in(0,1)$, we have
    \begin{multline*}
      \PP_{(\rn{x},\rn{x'})\sim(\mu\otimes\mu')^m}\Biggl[
        \sup_{H\in\cH}
        \bigl\lvert 
        L_{\rn{x}, F^*_m(\rn{x},\rn{x'}), \ell}(H)
        - L_{\mu,\mu',F,\ell}(H)
        \bigr\rvert
        \\
        >
        \frac{c\cdot k^2}{\delta}\cdot
        \sqrt{\frac{\ln(2\tau^k_\cH(2m))\cdot 2\cdot\lVert\ell\rVert_\infty^2}{m}}
        \Biggr]
      \leq
      \delta.
    \end{multline*}
  \end{enumerate}
\end{lemma}

\begin{proof}
  Note first that since $\cH$ is non-empty, we have $\tau^k_\cH\geq 1$.

  We start with item~\ref{lem:partemp:partial}. Let $(\rn{x},\rn{x'})\sim(\mu\otimes\mu')^{m,p}$ and
  let us decompose $\rn{x}$ and $\rn{x'}$ as
  \begin{align*}
    \rn{x} & = (\rn{w},\rn{z}), &
    \rn{x'} & = (\rn{w'}, \rn{z'}),
  \end{align*}
  where $\rn{w}\df\psi_{m,p}(\rn{x})$ and $\rn{w'}\df\psi'_{m,p}(\rn{x'})$ are random elements in
  $\cE_m^{p-1}(\Omega)$ and $\cE_m^{p-1}(\Omega')$, respectively, and $\rn{z}$ and $\rn{z'}$
  contain the coordinates indexed by $f\in r_k(m)$ with $p\in\dom(f)\subseteq [p]$.

  By Markov's Inequality, it suffices to prove that
  \begin{align}\label{eq:partemp}
    \EE_{\rn{w},\rn{w'},\rn{z},\rn{z'}}\bigl[
      \sup_{H\in\cH}
      \lvert L^p_{(\rn{w},\rn{z}),(\rn{w'},\rn{z'}),\mu,\mu',F,\ell}(H)
      - L^{p-1}_{\rn{w},\rn{w'},\mu,\mu',F,\ell}(H)\rvert
    \bigr]
    & \leq
    c\cdot
    \sqrt{\frac{\ln(2\tau^k_\cH(2m))\cdot 2\cdot\lVert\ell\rVert_\infty^2}{m}}.
  \end{align}

  Let $(\rn{\widehat{z}},\rn{\widehat{z'}})$ be picked independently from all previous random
  elements with the same distribution as $(\rn{z},\rn{z'})$ and let
  $(\rn{u},\rn{u'})\sim\overline{\mu\otimes\mu'}^{m,p}$ be picked independently from all previous
  random elements.

  Finally, we let
  \begin{align*}
    \rn{y}
    & \df
    F^*_m((\rn{w},\rn{z},\rn{u}), (\rn{w'},\rn{z'},\rn{u'})),
    &
    \rn{\widehat{y}}
    & \df
    F^*_m((\rn{w},\rn{\widehat{z}},\rn{u}), (\rn{w'},\rn{\widehat{z'}},\rn{u'})).
  \end{align*}

  Note that these choices guarantee the following equality of distributions over
  $\cE_m(\Omega\otimes\Omega')\times\Lambda^{[m]^k}$:
  \begin{align*}
    ((\rn{w},\rn{z},\rn{u}), (\rn{w'},\rn{z'},\rn{u'}),\rn{y})
    & \sim
    ((\rn{w},\rn{\widehat{z}},\rn{u}), (\rn{w'},\rn{\widehat{z'}},\rn{u'}),\rn{\widehat{y}}).
  \end{align*}
  Now, by the definition of $L^p$ and $L^{p-1}$, note that
  \begin{align*}
    L^p_{(\rn{w},\rn{z}),(\rn{w'},\rn{z'}),\mu,\mu',F,\ell}(H)
    & =
    \EE_{\rn{u},\rn{u'}}[L_{(\rn{w},\rn{z},\rn{u}),\rn{y},\ell}(H)],
    \\
    L^{p-1}_{\rn{w},\rn{w'},\mu,\mu',F,\ell}(H)
    & =
    \EE_{\rn{\widehat{z}},\rn{\widehat{z'}},\rn{u},\rn{u'}}[L_{(\rn{w},\rn{\widehat{z}},\rn{u}),\rn{\widehat{y}},\ell}(H)],
  \end{align*}
  so by linearity of expectation, we have
  \begin{align*}
    L^p_{(\rn{w},\rn{z}),(\rn{w'},\rn{z'}),\mu,\mu',F,\ell}(H)
    - L^{p-1}_{\rn{w},\rn{w'},\mu,\mu',F,\ell}(H)
    & =
    \EE_{\rn{\widehat{z}},\rn{\widehat{z'}},\rn{u},\rn{u'}}[
      L_{(\rn{w},\rn{z},\rn{u}),\rn{y},\ell}(H)
      - L_{(\rn{w},\rn{\widehat{z}},\rn{u}),\rn{\widehat{y}},\ell}(H)
    ].
  \end{align*}

  Using the above along with the facts that the absolute value of the expectation is at most the
  expectation of the absolute value and the supremum of expectation is at most the expectation of
  the suprema, we get
  \begin{equation}\label{eq:partemp2}
    \begin{aligned}
      & \!\!\!\!\!\!
      \EE_{\rn{w},\rn{w'},\rn{z},\rn{z'}}\bigl[
        \sup_{H\in\cH}
        \lvert L^p_{(\rn{w},\rn{z}),(\rn{w'},\rn{z'}),\mu,\mu',F,\ell}(H)
        - L^{p-1}_{\rn{w},\rn{w'},\mu,\mu',F,\ell}(H)\rvert
      \bigr]
      \\
      & \leq
      \EE_{\rn{w},\rn{w'},\rn{z},\rn{z'},\rn{\widehat{z}},\rn{\widehat{z'}},\rn{u},\rn{u'}}\bigl[
        \sup_{H\in\cH}
        \lvert
        L_{(\rn{w},\rn{z},\rn{u}),\rn{y},\ell}(H)
        - L_{(\rn{w},\rn{\widehat{z}},\rn{u}),\rn{\widehat{y}},\ell}(H)
        \rvert
      \bigr]
      \\
      & =
      \begin{aligned}[t]
        \frac{1}{m^k}
        \EE_{\rn{w},\rn{w'},\rn{z},\rn{z'},\rn{\widehat{z}},\rn{\widehat{z'}},\rn{u},\rn{u'}}\biggl[
          \sup_{H\in\cH}\Bigl\lvert
          \sum_{\alpha\in [m]^k}
          \bigl(
          &
          \ell(H, \alpha^*(\rn{w},\rn{z},\rn{u}), \rn{y}_\alpha)
          \\
          &
          - \ell(H, \alpha^*(\rn{w},\rn{\widehat{z}},\rn{u}), \rn{\widehat{y}}_\alpha)
          \bigr)
          \Bigr\rvert
          \biggr],
      \end{aligned}
    \end{aligned}
  \end{equation}
  where the equality follows from the definition of empirical loss in
  Definition~\ref{def:emploss:ag}.
  
  Note that the difference between the two computations of $\ell$ in the last expression is solely
  based on the coordinates of $\rn{z},\rn{z'},\rn{\widehat{z}},\rn{\widehat{z'}}$. Recall that these
  random variables contain the coordinates indexed by $f\in r_k(m)$ with
  $p\in\dom(f)\subseteq[p]$. We further decompose these as
  \begin{align*}
    \rn{z} & = (\rn{z}^i)_{i=1}^m, &
    \rn{z'} & = ((\rn{z'})^i)_{i=1}^m, &
    \rn{\widehat{z}} & = (\rn{\widehat{z}}^i)_{i=1}^m, &
    \rn{\widehat{z'}} & = ((\rn{\widehat{z'}})^i)_{i=1}^m,
  \end{align*}
  where $\rn{z}^i,(\rn{z'})^i,\rn{\widehat{z}}^i,(\rn{\widehat{z'}})^i$ contain all coordinates of the
  respective variables indexed by $f\in r_k(m)$ with $p\in\dom(f)\subseteq[p]$ and $f(p)=i$.

  Note also that for $\alpha\in[m]^k$, the inner expression in the last line of~\eqref{eq:partemp2}
  depends only on the coordinates of
  $\rn{w},\rn{w'},\rn{z},\rn{z'},\rn{\widehat{z}},\rn{\widehat{z'}},\rn{u},\rn{u'}$ indexed by $f\in
  r_k(m)$ such that $\alpha$ is an extension of $f$; in particular, for the coordinates of
  $\rn{z},\rn{z'},\rn{\widehat{z}},\rn{\widehat{z'}}$, it only depends on
  $(\rn{z}^i,(\rn{z'})^i,\rn{\widehat{z}}^i,(\rn{\widehat{z'}})^i)$ for $i=\alpha(p)$.

  This prompts us to compute the last line of~\eqref{eq:partemp2} in a different manner: let
  $\rn{\sigma}$ be picked uniformly at random in $\{-1,1\}^m$, independently from all previous
  random variables and for each $i\in[m]$, we let
  \begin{align*}
    \rn{Z}^i & \df 
    \begin{dcases*}
      \rn{z}^i, & if $\rn{\sigma}_i = 1$,\\
      \rn{\widehat{z}}^i, & if $\rn{\sigma}_i = -1$,
    \end{dcases*}
    &
    \rn{\widehat{Z}}^i & \df 
    \begin{dcases*}
      \rn{\widehat{z}}^i, & if $\rn{\sigma}_i = 1$,\\
      \rn{z}^i, & if $\rn{\sigma}_i = -1$,
    \end{dcases*}
    \\
    (\rn{Z'})^i & \df 
    \begin{dcases*}
      (\rn{z'})^i, & if $\rn{\sigma}_i = 1$,\\
      (\rn{\widehat{z'}})^i, & if $\rn{\sigma}_i = -1$,
    \end{dcases*}
    &
    (\rn{\widehat{Z'}})^i & \df 
    \begin{dcases*}
      (\rn{\widehat{z'}})^i, & if $\rn{\sigma}_i = 1$,\\
      (\rn{z'})^i, & if $\rn{\sigma}_i = -1$,
    \end{dcases*}
  \end{align*}
  that is, when $\rn{\sigma}_i=1$, $(\rn{Z}^i,\rn{\widehat{Z}^i})$ is
  $(\rn{z}^i,\rn{\widehat{z}}^i)$ and when $\rn{\sigma}_i=-1$, then $(\rn{Z}^i,\rn{\widehat{Z}}^i)$
  is swapped: $(\rn{\widehat{z}}^i,\rn{z}^i)$ (and analogously for
  $((\rn{Z'})^i,(\rn{\widehat{Z'}})^i)$). We also let $\rn{Z}\df(\rn{Z}^i)_{i=1}^m$ and
  define $\rn{Z'},\rn{\widehat{Z}},\rn{\widehat{Z'}}$ analogously.

  By further defining
  \begin{align*}
    \rn{Y}
    & \df
    F^*_m((\rn{w},\rn{Z},\rn{u}), (\rn{w'},\rn{Z'},\rn{u'})),
    &
    \rn{\widehat{Y}}
    & \df
    F^*_m((\rn{w},\rn{\widehat{Z}},\rn{u}), (\rn{w'},\rn{\widehat{Z'}},\rn{u'})),
  \end{align*}
  we observe that
  \begin{align*}
    \rn{Y}_\alpha & =
    \begin{dcases*}
      \rn{y}_\alpha, & if $\rn{\sigma}_{\alpha(p)}=1$,\\
      \rn{\widehat{y}}_\alpha, & if $\rn{\sigma}_{\alpha(p)}=-1$,
    \end{dcases*}
    &
    \rn{\widehat{Y}}_\alpha & =
    \begin{dcases*}
      \rn{\widehat{y}}_\alpha, & if $\rn{\sigma}_{\alpha(p)}=1$,\\
      \rn{y}_\alpha, & if $\rn{\sigma}_{\alpha(p)}=-1$.
    \end{dcases*}
  \end{align*}

  Since the distributions of $\rn{z},\rn{z'},\rn{\widehat{z}},\rn{\widehat{z'}}$ are product
  distributions (and they are independent from $(\rn{w},\rn{w'},\rn{u},\rn{u'})$), we have the
  following equality of distributions:
  \begin{align*}
    (\rn{w},\rn{z},\rn{\widehat{z}},\rn{u},
    \rn{w'},\rn{z'},\rn{\widehat{z'}},\rn{u'},
    \rn{y},\rn{\widehat{y}})
    & \sim
    (\rn{w},\rn{Z},\rn{\widehat{Z}},\rn{u},
    \rn{w'},\rn{Z'},\rn{\widehat{Z'}},\rn{u'},
    \rn{Y},\rn{\widehat{Y}}),
  \end{align*}
  which means that the last line of~\eqref{eq:partemp2} is equal to
  \begin{align*}
    \frac{1}{m^k}
    \EE_{\rn{w},\rn{w'},\rn{z},\rn{z'},\rn{\widehat{z}},\rn{\widehat{z'}},\rn{u},\rn{u'}}\Biggl[
      \EE_{\rn{\sigma}}\biggl[
        \sup_{H\in\cH}\Bigl\lvert
        \sum_{\alpha\in [m]^k}
        \bigl(
        &
        \ell(H, \alpha^*(\rn{w},\rn{Z},\rn{u}), \rn{Y}_\alpha)
        \\
        &
        - \ell(H, \alpha^*(\rn{w},\rn{\widehat{Z}},\rn{u}), \rn{\widehat{Y}}_\alpha)
        \bigr)
        \Bigr\rvert
        \biggr]
      \Biggr].
  \end{align*}

  We now partition the set of $\alpha\in[m]^k$ based on $\alpha\rest_{[k]\setminus\{p\}}$: for each
  $\beta\in [m]^{[k]\setminus\{p\}}$ and $i\in [m]$, we let $\alpha_{\beta,i}\in [m]^k$ be the
  extension of $\beta$ that maps $p$ to $i$.

  By using triangle inequality, the fact that the supremum of the sum is at most the sum of the
  suprema and linearity of expectation, we get
  \begin{equation}\label{eq:partemp3}
    \begin{aligned}
      & \!\!\!\!\!\!
      \EE_{\rn{w},\rn{w'},\rn{z},\rn{z'}}\bigl[
        \sup_{H\in\cH}
        \lvert L^p_{(\rn{w},\rn{z}),(\rn{w'},\rn{z'}),\mu,\mu',F,\ell}(H)
        - L^{p-1}_{\rn{w},\rn{w'},\mu,\mu',F,\ell}(H)\rvert
      \bigr]
      \\
      & \leq
      \begin{aligned}[t]
        \frac{1}{m^k}
        \sum_{\beta\in [m]^{[k]\setminus\{p\}}}
        \EE_{\rn{w},\rn{w'},\rn{z},\rn{z'},\rn{\widehat{z}},\rn{\widehat{z'}},\rn{u},\rn{u'}}\Biggl[
          \EE_{\rn{\sigma}}\biggl[
            \sup_{H\in\cH}\Bigl\lvert
            \sum_{i=1}^m
            \bigl(
            &
            \ell(H, \alpha_{\beta,i}^*(\rn{w},\rn{Z},\rn{u}), \rn{Y}_{\alpha_{\beta,i}})
            \\
            &
            - \ell(H, \alpha_{\beta,i}^*(\rn{w},\rn{\widehat{Z}},\rn{u}), \rn{\widehat{Y}}_{\alpha_{\beta,i}})
            \bigr)
            \Bigr\rvert
            \biggr]
          \Biggr]
      \end{aligned}
      \\
      & =
      \begin{aligned}[t]
        \frac{1}{m^k}
        \sum_{\beta\in [m]^{[k]\setminus\{p\}}}
        \EE_{\rn{w},\rn{w'},\rn{z},\rn{z'},\rn{\widehat{z}},\rn{\widehat{z'}},\rn{u},\rn{u'}}\Biggl[
          \EE_{\rn{\sigma}}\biggl[
            \sup_{H\in\cH}\Bigl\lvert
            \sum_{i=1}^m
            \rn{\sigma}_i\cdot
            \bigl(
            &
            \ell(H, \alpha_{\beta,i}^*(\rn{w},\rn{z},\rn{u}), \rn{y}_{\alpha_{\beta,i}})
            \\
            &
            - \ell(H, \alpha_{\beta,i}^*(\rn{w},\rn{\widehat{z}},\rn{u}), \rn{\widehat{y}}_{\alpha_{\beta,i}})
            \bigr)
            \Bigr\rvert
            \biggr]
          \Biggr],
      \end{aligned}
    \end{aligned}
  \end{equation}
  where the equality follows since $\ell(H, \alpha_{\beta,i}^*(\rn{w},\rn{Z},\rn{u}),
  \rn{Y}_{\alpha_{\beta,i}}) - \ell(H, \alpha_{\beta,i}^*(\rn{w},\rn{\widehat{Z}},\rn{u}),
  \rn{\widehat{Y}}_{\alpha_{\beta,i}})$ depends only on the
  $(\rn{Z}^j,(\rn{Z'})^j,\rn{\widehat{Z}}^j,(\rn{\widehat{Z'}})^j)$ with $j=i$, so $\rn{\sigma}_i$
  swaps the two losses corresponding to $\alpha_{\beta,i}$ when $\rn{\sigma}_i=-1$, which swaps the
  sign of the expression when computed from $(\rn{z},\rn{z'},\rn{\widehat{z}},\rn{\widehat{z'}})$.

  Note now that for each $\beta\in [m]^{[k]\setminus\{p\}}$, the expectation above only depends on
  coordinates of $(\rn{w},\rn{w'},\rn{z},\rn{z'},\rn{\widehat{z}},\rn{\widehat{z'}},\rn{u},\rn{u'})$
  indexed by $f\in r_k(m)$ such that $f\rest_{\dom(f)\setminus\{p\}} =
  \beta\rest_{\dom(f)\setminus\{p\}}$. Since all these are product distributions, by relabeling
  coordinates the expression under the sum indexed by $\beta$ is independent of $\beta$, so the last
  line of~\eqref{eq:partemp3} is equal to
  \begin{equation}\label{eq:partemp4}
    \begin{aligned}
      \frac{1}{m}
      \EE_{\rn{w},\rn{w'},\rn{z},\rn{z'},\rn{\widehat{z}},\rn{\widehat{z'}},\rn{u},\rn{u'}}\Biggl[
        \EE_{\rn{\sigma}}\biggl[
          \sup_{H\in\cH}\Bigl\lvert
          \sum_{i=1}^m
          \rn{\sigma}_i\cdot
          \bigl(
          &
          \ell(H, \gamma_i^*(\rn{w},\rn{z},\rn{u}), \rn{y}_{\gamma_i})
          \\
          &
          - \ell(H, \gamma_i^*(\rn{w},\rn{\widehat{z}},\rn{u}), \rn{\widehat{y}}_{\gamma_i})
          \bigr)
          \Bigr\rvert
          \biggr]
        \Biggr],
    \end{aligned}
  \end{equation}
  where $\gamma_i\df\alpha_{\beta_1,i}$ for the constant $1$ function
  $\beta_1\in[m]^{[k]\setminus\{p\}}$.

  As expected, most of the coordinates of the variables in the outer expectation
  in~\eqref{eq:partemp4} do not appear at all in the inner expression, so we can further simplify
  this computation as follows: recall the notation from~\eqref{eq:rkA} for
  \begin{align*}
    r_{k,[k]\setminus\{p\}}
    & \df
    \{f\in r_k(1) \mid \dom(f)\subseteq [k]\setminus\{p\}\}
  \end{align*}
  and let $\rn{\xi}$ and $\rn{\xi'}$ be sampled from
  \begin{align*}
    \bigotimes_{f\in r_{k,[k]\setminus\{p\}}} \mu_{\dom(f)}, &
    \bigotimes_{f\in r_{k,[k]\setminus\{p\}}} \mu'_{\dom(f)},
  \end{align*}
  respectively, independently from all previous random variables and independently from each
  other. We also let
  $\rn{\zeta}_1,\ldots,\rn{\zeta}_m,\rn{\widehat{\zeta}}_1,\ldots,\rn{\widehat{\zeta}}_m$ be i.i.d.,
  independent from all previous random variables with
  \begin{align*}
    \rn{\zeta}_1 & \sim \bigotimes_{f\in r_k(1)\setminus r_{k,[k]\setminus\{p\}}} \mu_{\dom(f)}
  \end{align*}
  and define
  $\rn{\zeta'}_1,\ldots,\rn{\zeta'}_m,\rn{\widehat{\zeta'}}_1,\ldots,\rn{\widehat{\zeta'}}_m$
  analogously from $\mu'$ (independently from $\rn{\zeta}$ as well).

  We also let
  \begin{align*}
    \rn{\upsilon}_i
    & \df
    F((\rn{\xi},\rn{\zeta}_i),(\rn{\xi'},\rn{\zeta'}_i)),
    &
    \rn{\widehat{\upsilon}}_i
    & \df
    F((\rn{\xi},\rn{\widehat{\zeta}}_i),(\rn{\xi'},\rn{\widehat{\zeta'}}_i).
  \end{align*}
  Then~\eqref{eq:partemp4} is equal to
  \begin{align}\label{eq:partemp5}
    \frac{1}{m}
    \EE_{\rn{\xi},\rn{\xi'},\rn{\zeta},\rn{\zeta'},\rn{\widehat{\zeta}},\rn{\widehat{\zeta'}}}\Biggl[
      \EE_{\rn{\sigma}}\biggl[
        \sup_{H\in\cH}\Bigl\lvert
        \sum_{i=1}^m
        \rn{\sigma}_i\cdot
        \bigl(
        \ell(H, (\rn{\xi},\rn{\zeta}_i), \rn{\upsilon}_i))
        - \ell(H, (\rn{\xi},\rn{\widehat{\zeta}}_i), \rn{\widehat{\upsilon}}_i)
        \bigr)
        \Bigr\rvert
        \biggr]
      \Biggr].
  \end{align}

  Let us now fix a value
  $(\xi,\xi',\zeta,\zeta',\widehat{\zeta},\widehat{\zeta'},\upsilon,\widehat{\upsilon})$ of
  $(\rn{\xi},\rn{\xi'},\rn{\zeta},\rn{\zeta'},\rn{\widehat{\zeta}},\rn{\widehat{\zeta'}},\rn{\upsilon},\rn{\widehat{\upsilon}})$
  and analyze the inner expectation of the above. First, since $\ell$ is local, there exist a
  regularization term $r\colon\cH\to\RR$ and a (non-agnostic) loss function
  $\ell_r\colon\cE_1(\Omega)\times\Lambda\times\Lambda\to\RR_{\geq 0}$ such that
  \begin{align*}
    \ell(H,x,y) & = \ell_r(x, H(x), y) + r(H)
    \qquad (H\in\cH, x\in\cE_1(\Omega), y\in\Lambda),
  \end{align*}
  so we have
  \begin{align*}
    & \!\!\!\!\!\!
    \EE_{\rn{\sigma}}\biggl[
      \sup_{H\in\cH}\Bigl\lvert
      \sum_{i=1}^m
      \rn{\sigma}_i\cdot
      \bigl(
      \ell(H, (\xi,\zeta_i), \upsilon_i)
      - \ell(H, (\xi,\widehat{\zeta}_i), \widehat{\upsilon}_i)
      \bigr)
      \Bigr\rvert
      \biggr]
    \\
    & =
    \EE_{\rn{\sigma}}\biggl[
      \sup_{H\in\cH}\Bigl\lvert
      \sum_{i=1}^m
      \rn{\sigma}_i\cdot
      \bigl(
      \ell_r((\xi,\zeta_i), H(\xi,\zeta_i), \upsilon_i)
      - \ell_r((\xi,\widehat{\zeta}_i), H(\xi,\widehat{\zeta}_i), \widehat{\upsilon}_i)
      \bigr)
      \Bigr\rvert
      \biggr]
    \\
    & =
    \EE_{\rn{\sigma}}\biggl[
      \sup_{G\in\cH(\xi)}\Bigl\lvert
      \sum_{i=1}^m
      \rn{\sigma}_i\cdot
      \bigl(
      \ell_r((\xi,\zeta_i), G(\zeta_i), \upsilon_i)
      - \ell_r((\xi,\widehat{\zeta}_i), G(\widehat{\zeta}_i), \widehat{\upsilon}_i)
      \bigr)
      \Bigr\rvert
      \biggr],
  \end{align*}
  where $\cH(\xi)$ is given by~\eqref{eq:partitecHx}.

  We now upper bound the expected value above by bounding its tail probabilities. First note that
  since $\rn{\sigma}$ is picked uniformly at random in $\{-1,1\}^m$, each term in the sum above is
  bounded above in absolute value by $\lVert\ell\rVert_\infty$ and has expected value $0$, so for
  each $G\in\cH(\xi)$, Hoeffding's Inequality gives
  \begin{equation}\label{eq:Hoeffding}
    \begin{aligned}
      & \!\!\!\!\!\!
      \PP_{\rn{\sigma}}\biggl[
        \Bigl\lvert
        \sum_{i=1}^m
        \rn{\sigma}_i\cdot
        \bigl(
        \ell_r((\xi,\zeta_i), G(\zeta_i), \upsilon_i)
        - \ell_r((\xi,\widehat{\zeta}_i), G(\widehat{\zeta}_i), \widehat{\upsilon}_i)
        \bigr)
        \Bigr\rvert
        \geq
        t
        \biggr]
      \\
      & \leq
      2\exp\left(-\frac{2t^2}{m\cdot 4\cdot\lVert\ell\rVert_\infty^2}\right)
    \end{aligned}
  \end{equation}
  for every $t > 0$.

  On the other hand, we note that the expression in the above only depends on $G$ via the
  restriction $G\rest_V$, where $V\df\{\zeta_i \mid i\in[m]\}\cup\{\widehat{\zeta}_i \mid
  i\in[m]\}$, which is a set of cardinality at most $2m$. By the definition of growth function (see
  Definition~\ref{def:growth}), there are at most $\tau^k_\cH(2m)$ many such $G\rest_V$, so we can
  apply the union bound to~\eqref{eq:Hoeffding} to get
  \begin{align*}
    & \!\!\!\!\!\!
    \PP_{\rn{\sigma}}\biggl[
      \sup_{G\in\cH(\xi)}
      \Bigl\lvert
      \sum_{i=1}^m
      \rn{\sigma}_i\cdot
      \bigl(
      \ell_r((\xi,\zeta_i), G(\zeta_i), \upsilon_i)
      - \ell_r((\xi,\widehat{\zeta}_i), G(\widehat{\zeta}_i), \widehat{\upsilon}_i)
      \bigr)
      \Bigr\rvert
      \geq
      t
      \biggr]
    \\
    & \leq
    2\tau^k_\cH(2m)\exp\left(-\frac{t^2}{2\cdot m\cdot\lVert\ell\rVert_\infty^2}\right)
  \end{align*}
  for every $t > 0$.

  Since $\tau^k_\cH(2m)\geq 1$ and recalling that $c$  defined in~\eqref{eq:partemp:c} precisely
  matches~\eqref{eq:gaussiantail:c}, Lemma~\ref{lem:gaussiantail} gives
  \begin{align*}
    & \!\!\!\!\!\!
    \EE_{\rn{\sigma}}\biggl[
      \sup_{G\in\cH(\xi)}\Bigl\lvert
      \sum_{i=1}^m
      \rn{\sigma}_i\cdot
      \bigl(
      \ell_r((\xi,\zeta_i), G(\zeta_i), \upsilon_i)
      - \ell_r((\xi,\widehat{\zeta}_i), G(\widehat{\zeta}_i), \widehat{\upsilon}_i)
      \bigr)
      \Bigr\rvert
      \biggr]
    \\
    & \leq
    c\cdot
    \sqrt{\ln(2\tau^k_\cH(2m))\cdot 2\cdot m\cdot\lVert\ell\rVert_\infty^2},
  \end{align*}
  which when plugged back in~\eqref{eq:partemp5} gives
  \begin{align*}
    & \!\!\!\!\!\!
    \EE_{\rn{w},\rn{w'},\rn{z},\rn{z'}}[
      \sup_{H\in\cH}
      \lvert L^p_{(\rn{w},\rn{z}),(\rn{w'},\rn{z'}),\mu,\mu',F,\ell}(H)
      - L^{p-1}_{\rn{w},\rn{w'},\mu,\mu',F,\ell}(H)\rvert
    ]
    \\
    & \leq
    \frac{c}{m}
    \cdot
    \sqrt{\ln(2\tau^k_\cH(2m))\cdot 2\cdot m\cdot\lVert\ell\rVert_\infty^2}
    \\
    & =
    c\cdot\sqrt{\frac{\ln(2\tau^k_\cH(2m))\cdot 2\cdot\lVert\ell\rVert_\infty^2}{m}},
  \end{align*}
  proving~\eqref{eq:partemp} and concluding the proof of item~\ref{lem:partemp:partial}.

  \medskip

  Finally, let us prove item~\ref{lem:partemp:total}. For this, we apply
  item~\ref{lem:partemp:partial} for each $p\in[k]$ with $\delta/k$ in place of $\delta$ and use the
  union bound to conclude that with probability at least $1-\delta$, we have
  \begin{align*}
    & \!\!\!\!\!\!
    \max_{p\in[k]}
    \sup_{H\in\cH}
    \lvert L^p_{\pi_{m,p}(\rn{x}),\pi'_{m,p}(\rn{x'}),\mu,\mu',F,\ell}(H)
    - L^{p-1}_{\pi_{m,p-1}(\rn{x}),\pi'_{m,p-1}(\rn{x'}),\mu,\mu',F,\ell}(H)\rvert
    \\
    & \leq
    \frac{c\cdot k}{\delta}\cdot
    \sqrt{\frac{\ln(2\tau^k_\cH(2m))\cdot 2\cdot\lVert\ell\rVert_\infty^2}{m}},
  \end{align*}
  where $\pi_{m,p}\colon\cE_m(\Omega)\to\cE_m^p(\Omega)$ and
  $\pi'_{m,p}\colon\cE_m(\Omega')\to\cE_m^p(\Omega')$ are the projection maps.

  Since by Remark~\ref{rmk:partemp} we have
  \begin{align*}
    L^0_{\pi_{m,0}(\rn{x}),\pi'_{m,0}(\rn{x'}),\mu,\mu',F,\ell}(H)
    & =
    L_{\mu,\mu',F,\ell}(H),
    \\
    L^k_{\pi_{m,k}(\rn{x}),\pi'_{m,k}(\rn{x'}),\mu,\mu',F,\ell}(H)
    =
    L^k_{\rn{x},\rn{x'},\mu,\mu',F,\ell}(H)
    & =
    L_{\rn{x},F^*_m(\rn{x},\rn{x'}),\ell}(H),
  \end{align*}
  applying triangle inequality, we conclude that
  \begin{align*}
    \sup_{H\in\cH}
    \lvert
    L_{\rn{x}, F^*_m(\rn{x},\rn{x'}), \ell}(H)
    - L_{\mu,\mu',F,\ell}(H)
    \rvert
    & \leq
    \frac{c\cdot k^2}{\delta}\cdot
    \sqrt{\frac{\ln(2\tau^k_\cH(2m))\cdot 2\cdot\lVert\ell\rVert_\infty^2}{m}}
  \end{align*}
  with probability at least $1-\delta$, as desired.
\end{proof}

It remains to use Lemma~\ref{lem:partemp} to show that finite $\VCN_k$-dimension implies uniform
convergence. For this, we need the next calculation lemma, which is a tiny improvement
of~\cite[Lemma~A.2]{SB14}.

\begin{lemma}\label{lem:logcalcs2}
  If $a\geq 1/2$, $b\geq 0$ and
  \begin{align*}
    x & \geq \frac{2e}{e-1}\cdot a\ln(2a) + 2b \qquad ({}\leq 1.582\cdot a\ln(2a) + 2b),
  \end{align*}
  then $x\geq a\ln x+b$.
\end{lemma}

\begin{proof}
  It suffices to show $x\geq 2a\ln x$ and $x\geq 2b$. Since $a\geq 1/2$, we have $\ln(2a)\geq 0$, hence
  $x\geq 2b$. To show $x\geq 2a\ln x$, it suffices to show that the function
  \begin{align*}
    f(x) & \df x - 2a\ln x
  \end{align*}
  defined for
  \begin{align*}
    x & \geq \frac{2e}{e-1}\cdot a\ln(2a) + 2b,
  \end{align*}
  is non-negative. We will show that $f$ is non-negative even extending its definition for $x\geq
  2e/(e-1)\cdot a\ln(2a)$. For this, we compute its derivative:
  \begin{align*}
    f'(x) & = 1 - \frac{2a}{x}
  \end{align*}
  and note that the only critical point of $f$ is potentially at $x = 2a$, if $2a$ belongs to the
  domain, that is, if $2a\geq 2e/(e-1)\cdot a\ln(2a)$. But if this is the case, we get
  \begin{align*}
    f(2a) & \geq \left(\frac{e}{e-1} - 1\right)\cdot 2a\ln(2a) \geq 0.
  \end{align*}

  On the other hand, since $\lim_{x\to\infty} f(x) = \infty$, it suffices to show that
  $f(2e/(e-1)\cdot a\ln(2a))$ is non-negative. But indeed, note that
  \begin{align*}
    f\left(\frac{2e}{e-1}\cdot a\ln(2a)\right)
    & =
    \frac{2e}{e-1}\cdot a\ln(2a) - 2a\ln\left(\frac{2e}{e-1}\cdot a\ln(2a)\right)
    \\
    & =
    \left(\frac{e}{e-1} - 1\right)\cdot 2a\ln(2a) - 2a\ln(\ln(2a)^{e/(e-1)})
    \\
    & \geq
    \left(\frac{e}{e-1} - 1 - \frac{1}{e-1}\right)\cdot 2a\ln(2a)
    =
    0,
  \end{align*}
  where the inequality follows from Lemma~\ref{lem:logcalcs}.
\end{proof}

We now show that finite $\VCN_k$-dimension implies the uniform convergence property\footnote{As
  pointed out in the introduction, this result was discovered in the special case of hypergraphs by
  Livni--Mansour in a prior work when studying graph-based discriminators~\cite{LM19a,LM19b}. In
  fact, their bounds on uniform convergence are slightly better than the ones we obtain
  here.}. Together with Proposition~\ref{prop:partUC->partagPAC}, this shows that finite
$\VCN_k$-dimension implies agnostic $k$-PAC learnability.

\begin{proposition}[Finite $\VCN_k$-dimension implies uniform convergence]\label{prop:VCNdim->UC}
  Let $\Omega$ be a Borel $k$-partite template, let $\Lambda$ be a finite non-empty Borel space, let
  $\cH\subseteq\cF_k(\Omega,\Lambda)$ be a $k$-partite hypothesis class with $\VCN_k(\cH) < \infty$,
  let
  \begin{align}\label{eq:VCNdim->UC:c}
    c
    & \df
    \sqrt{1 - \frac{\ln\ln 2}{\ln 2}} + \frac{1}{2\sqrt{1-1/e}}
    \in
    (1.865, 1.866),
  \end{align}
  let $\ell$ be a local $k$-partite agnostic loss function with $\lVert\ell\rVert_\infty <
  \infty$ and let
  \begin{align*}
    B_\ell
    & \df
    \max\left\{\frac{1}{2\sqrt{2}\cdot c},\lVert\ell\rVert_\infty\right\}
    \leq
    \max\{0.380,\lVert\ell\rVert_\infty\}.
  \end{align*}

  Then $\cH$ has the uniform convergence property with respect to $\ell$ with associated function
  \begin{align*}
    m^{\UC}_{\cH,\ell}(\epsilon,\delta)
    & \df
    \begin{multlined}[t]
      \frac{4\cdot c^2\cdot k^4\cdot B_\ell^2}{\delta^2\cdot\epsilon^2}\cdot
      \Biggl(
      \frac{e}{e-1}\cdot\VCN_k(\cH)\cdot\ln
      \frac{8\cdot c^2\cdot k^4\cdot B_\ell^2\cdot\VCN_k(\cH)}{\delta^2\cdot\epsilon^2}
      \\
      +
      \ln 2
      +
      \VCN_k(\cH)\cdot\ln\binom{\lvert\Lambda\rvert}{2}
      \Biggr)
      +
      \frac{1}{2}
    \end{multlined}
    \\
    & \leq
    \begin{multlined}[t]
      13.918\cdot
      \frac{k^4\cdot B_\ell^2}{\delta^2\cdot\epsilon^2}
      \cdot\Biggl(
      1.582\cdot\VCN_k(\cH)\ln
      \frac{27.836\cdot k^4\cdot B_\ell^2\cdot\VCN_k(\cH)}{\delta^2\cdot\epsilon^2}
      \\
      +
      0.694
      +
      \VCN_k(\cH)\cdot\ln\binom{\lvert\Lambda\rvert}{2}
      \Biggr)
      + 0.5,
    \end{multlined}
  \end{align*}
  where $0\ln 0$ should be interpreted as $0$.
\end{proposition}

\begin{proof}
  The assertion is trivial when $\VCN_k(\cH)=0$, so we suppose $\VCN_k(\cH)\geq 1$. This immediately
  forces $\cH$ to be non-empty and $\lvert\Lambda\rvert\geq 2$.

  Note that the definition of $c$ in~\eqref{eq:VCNdim->UC:c} matches~\eqref{eq:partemp:c}. By
  applying Lemma~\ref{lem:partemp}\ref{lem:partemp:total}, we know that for every $m\in\NN_+$, every
  Borel $k$-partite template $\Omega'$, every $\mu\in\Pr(\Omega)$, every $\mu'\in\Pr(\Omega')$ and
  every $F\in\cF_k(\Omega\otimes\Omega',\Lambda)$, if $(\rn{x},\rn{x'})\sim(\mu\otimes\mu')^m$, with
  probability at least $1-\delta$, we have
  \begin{align*}
    \sup_{H\in\cH}
    \lvert
    L_{\rn{x}, F^*_m(\rn{x},\rn{x'}), \ell}(H)
    - L_{\mu,\mu',F,\ell}(H)
    \rvert
    & \leq
    \frac{c\cdot k^2}{\delta}\cdot
    \sqrt{\frac{\ln(2\tau^k_\cH(2m))\cdot 2\cdot \lVert\ell\rVert_\infty^2}{m}},
  \end{align*}
  so it suffices to show that if $m\geq m^{\UC}_{\cH,\ell}(\epsilon,\delta)$, then the right-hand
  side of the above is at most $\epsilon$.

  By Lemma~\ref{lem:VCNk} and since $\lVert\ell\rVert_\infty\leq B_\ell$, the right-hand side of the
  above is at most
  \begin{align*}
    \frac{c\cdot k^2}{\delta}\cdot
    \sqrt{\frac{\ln(2\cdot(2m+1)^{\VCN_k(\cH)}\cdot\binom{\lvert\Lambda\rvert}{2}^{\VCN_k(\cH)})
        \cdot 2\cdot B_\ell^2}{m}},
  \end{align*}
  so it suffices to show that this is at most $\epsilon$, which in turn, this is equivalent to
  showing that
  \begin{align}\label{eq:mcondition}
    m
    & \geq
    \frac{2\cdot c^2\cdot k^4\cdot B_\ell^2}{\delta^2\cdot\epsilon^2}\cdot
    \left(
    \ln 2
    + \VCN_k(\cH)\cdot
    \left(\ln (2m+1) + \ln\binom{\lvert\Lambda\rvert}{2}\right)
    \right).
  \end{align}

  Let now
  \begin{align*}
    x
    & \df
    2m+1,
    \\
    a
    & \df
    \frac{4\cdot c^2\cdot k^4\cdot B_\ell^2\cdot\VCN_k(\cH)}{\delta^2\cdot\epsilon^2}
    \geq
    \frac{1}{2},
    \\
    b
    & \df
    \frac{4\cdot c^2\cdot k^4\cdot B_\ell^2}{\delta^2\cdot\epsilon^2}
    \cdot\left(
    \ln 2 + \VCN_k(\cH)\cdot\ln\binom{\lvert\Lambda\rvert}{2}
    \right)
    + 1,
  \end{align*}
  so that~\eqref{eq:mcondition} is equivalent to $x\geq a\ln x + b$, which by
  Lemma~\ref{lem:logcalcs2} follows from $x\geq 2e\cdot a\ln(2a)/(e-1) + 2b$. In turn, this is
  equivalent to
  \begin{align*}
    m
    & \geq
    \left(\frac{2e}{e-1}\cdot a\ln(2a) + 2b\right)
    \\
    & =
    \begin{multlined}[t]
      \frac{4\cdot c^2\cdot k^4\cdot B_\ell^2}{\delta^2\cdot\epsilon^2}\cdot
      \Biggl(
      \frac{e}{e-1}\cdot\VCN_k(\cH)\cdot\ln
      \frac{8\cdot c^2\cdot k^4\cdot B_\ell^2\cdot\VCN_k(\cH)}{\delta^2\cdot\epsilon^2}
      \\
      +
      \ln 2
      +
      \VCN_k(\cH)\cdot\ln\binom{\lvert\Lambda\rvert}{2}
      \Biggr)
      +
      \frac{1}{2}
    \end{multlined}
    \\
    & =
    m^{\UC}_{\cH,\ell}(\epsilon,\delta),
  \end{align*}
  so the result follows.
\end{proof}

\section{Non-learnability}
\label{sec:nonlearn}

In this section, we show that partite $k$-PAC learnable hypothesis classes of rank at most $1$ must
have finite $\VCN_k$-dimension. This is done by the contra-positive: for every $k$-partite
hypothesis class $\cH$ of rank at most $1$ and infinite $\VCN_k$-dimension, there exists
$\mu\in\Pr(\Omega)$ and $F\in\cF_k(\Omega,\Lambda)$ realizable that no algorithm can learn. To do
so, we first restate and prove a more quantified version of the no-free-lunch
Theorem~\cite[Theorem~5.1]{SB14} from classic PAC learning theory.

\begin{lemma}[No-free-lunch Theorem]\label{lem:nonlearn}
  Let $\Omega=(X,\cB)$ and $\Lambda=(Y,\cB')$ be non-empty Borel spaces, let
  $\cF\subseteq\cF_1(\Omega,\Lambda)$ be a collection of measurable functions such that
  $\Nat(\cF)\geq d > 0$, let
  \begin{align*}
    \cA\colon\bigcup_{m\in\NN} (X^m\times Y^m)\to \cF_1(\Omega,\Lambda)
  \end{align*}
  be a learning algorithm (allowed to output any measurable function $X\to Y$), let $X'\subseteq X$
  be such that $\lvert X'\rvert = d$ and $\cF$ shatters $X'$, let $\mu\in\Pr(\Omega)$ be the
  probability measure on $\Omega$ that is the uniform probability measure on $X'$ and let
  $\ell\colon X\times Y\times Y\to\RR_{\geq 0}$ be a separated $1$-ary loss function with
  $\lVert\ell\rVert_\infty<\infty$. Finally, let $B_\ell\geq\lVert\ell\rVert_\infty$ be a real
  number.

  Then for every $m\in\NN$, there exists $F\in\cF$ such that $L_{\mu,F,\ell}(F)=0$ and
  \begin{align}\label{eq:nofreelunch}
    \PP_{\rn{x}\sim\mu^m}[L_{\mu,F,\ell}(\cA(\rn{x},(F(\rn{x}_i))_{i=1}^m)) > \epsilon]
    & \geq
    \frac{1}{B_\ell - \epsilon}
    \cdot\left(\frac{s(\ell)}{2}\left(1 - \frac{m}{d}\right) - \epsilon\right)
  \end{align}
  for every $\epsilon\in(0,B_\ell)$.
\end{lemma}

\begin{proof}
  By possibly restricting all elements of $\cF$ to $X'$, we may assume that $X=X'$.
  
  Since $\cF$ shatters $X$, there exist functions $f_0,f_1\colon X\to Y$ and $F_B\in\cF$
  ($B\subseteq X$) such that for every $x\in X$, we have $f_0(x)\neq f_1(x)$ and for every $x\in X$
  and every $B\subseteq X$, we have $F_B(x) = f_{\One[x\in B]}(x)$.

  Our first objective is to show that
  \begin{align}\label{eq:EEnofreelunch}
    \max_{B\subseteq X} \EE_{\rn{x}\sim\mu^m}[L_{\mu,F_B,\ell}(\cA(\rn{x},(F_B(\rn{x}_i))_{i=1}^m))]
    \geq
    \frac{s(\ell)}{2}\left(1-\frac{m}{d}\right).
  \end{align}

  First note that
  \begin{equation}\label{eq:maxEE}
    \begin{aligned}
      \max_{B\subseteq X} \EE_{\rn{x}\sim\mu^m}[L_{\mu,F_B,\ell}(\cA(\rn{x},(F_B(\rn{x}_i))_{i=1}^m))]
      & \geq
      \frac{1}{2^d}\sum_{B\subseteq X}
      \EE_{\rn{x}\sim\mu^m}[L_{\mu,F_B,\ell}(\cA(\rn{x},(F_B(\rn{x}_i))_{i=1}^m))]
      \\
      & \geq
      \min_{x\in X^m}
      \frac{1}{2^d}\sum_{B\subseteq X}
      L_{\mu,F_B,\ell}(\cA(x,(F_B(x_i))_{i=1}^m)).
    \end{aligned}
  \end{equation}

  For every $x\in X^m$ and every $B\subseteq X$, let
  \begin{align*}
    H^\cA_{B,x}
    & \df
    \cA(x, (F_B(x_i))_{i=1}^m)
    =
    \cA(x, (f_{\One[x_i\in B]}(x_i))_{i=1}^m),
    \\
    M_x & \df X\setminus\im(x)
  \end{align*}
  and note that if $B,B'\subseteq X$ are such that $B\cap\im(x)=B'\cap\im(x)$, then
  $H^\cA_{B,x}=H^\cA_{B',x}$.
  
  Since $\mu$ is the uniform probability measure, we have
  \begin{align*}
    L_{\mu,F_B,\ell}(\cA(x,(F_B(x_i))_{i=1}^m))
    & =
    L_{\mu,F_B,\ell}(H^\cA_{B,x})
    =
    \frac{1}{d}\sum_{z\in X}\ell(z, H^\cA_{B,x}(z), F_B(z))
    \\
    & \geq
    \frac{1}{d}\sum_{z\in M_x}\ell(z, H^\cA_{B,x}(z), F_B(z))
    =
    \frac{1}{d}\sum_{z\in M_x}\ell(z, H^\cA_{B,x}(z), f_{\One[z\in B]}(z)).
  \end{align*}
  Note that the last expression only depends on $B$ via $B\cap(\im(x)\cup\{z\})$. This prompts us to
  decompose $B$ as $B=B_0\cup B_1\cup B_2$, where $B_0\df B\cap\im(x)$, $B_1\df B\cap\{z\}$ and
  $B_2\df B\cap (M_x\setminus\{z\})$, so that the term under the minimum in~\eqref{eq:maxEE} can be
  bounded as
  \begin{align*}
    \frac{1}{2^d}\sum_{B\subseteq X} L_{\mu,F_B,\ell}(\cA(x,(F_B(x_i))_{i=1}^m))
    & \geq
    \frac{1}{d\cdot 2^d}\sum_{z\in M_x}
    \sum_{\substack{B_0\subseteq\im(x)\\B_1\subseteq\{z\}\\B_2\subseteq M_x\setminus\{z\}}}
    \ell(z, H^\cA_{B_0,x}(z),f_{\One[z\in B_1]}(z))
    \\
    & \geq
    \frac{1}{d\cdot 2^d}\sum_{z\in M_x}
    \sum_{\substack{B_0\subseteq\im(x)\\B_2\subseteq M_x\setminus\{z\}}}
    s(\ell)
    \\
    & =
    \frac{\lvert M_x\rvert\cdot s(\ell)}{2\cdot d}
    \geq
    \frac{s(\ell)}{2}\left(1-\frac{m}{d}\right)
  \end{align*}
  where the second inequality follows since $f_0(z)\neq f_1(z)$ so at least one of the two possible
  choices of $B_1$ satisfies $H^\cA_{B_0,x}(z)\neq f_{\One[z\in B_1]}(z)$ and the last inequality
  follows since $\lvert M_x\rvert\geq d-m$. Hence~\eqref{eq:EEnofreelunch} follows.

  Since $F_B\in\cF$ for every $B\subseteq X$, we conclude that there exists $F\in\cF$ such that
  \begin{align*}
    \EE_{\rn{x}\sim\mu^m}[L_{\mu,F,\ell}(\cA(\rn{x},(F(\rn{x}_i))_{i=1}^m))]
    \geq
    \frac{s(\ell)}{2}\left(1-\frac{m}{d}\right).
  \end{align*}
  Since the expression under the expectation above is upper bounded by $\lVert\ell\rVert_\infty\leq B_\ell$, by
  Markov's Inequality, we have
  \begin{align*}
    \PP_{\rn{x}\sim\mu^m}[L_{\mu,F,\ell}(\cA(\rn{x},(F(\rn{x}_i))_{i=1}^m)) \leq \epsilon]
    & \leq
    \frac{1}{B_\ell - \epsilon}
    \cdot
    \left(B_\ell - \frac{s(\ell)}{2}\left(1-\frac{m}{d}\right)\right),
  \end{align*}
  from which~\eqref{eq:nofreelunch} follows. Finally, since $F\in\cF$ and $\ell$ is separated, we
  have $L_{\mu,F,\ell}(F)=0$.
\end{proof}

We can now use Lemma~\ref{lem:nonlearn} to show that infinite $\VCN_k$-dimension implies non-$k$-PAC
learnability for $k$-partite hypothesis classes of rank at most $1$.

\begin{proposition}[Learnability implies finite $\VCN_k$-dimension]\label{prop:partkPAC->VCN}
  Let $\Omega$ be a Borel $k$-partite template, let $\Lambda$ be a non-empty Borel space, let $\ell$
  be a separated $k$-partite loss function, let $\cH\subseteq\cF_k(\Omega,\Lambda)$ be a $k$-partite
  hypothesis class with $\VCN_k(\cH)\geq d$ for some $d\in\NN_+$ and $\rk(\cH)\leq 1$ and let $\cA$ be a
  learning algorithm for the full hypothesis class $\cF_k(\Omega,\Lambda)$.

  Then there exist $\mu\in\Pr(\Omega)$ and $F\in\cH$ such that the following hold:
  \begin{enumerate}
  \item\label{prop:partkPAC->VCN:realizable} We have $L_{\mu,F,\ell}(F) = 0$, so $F$ is realizable in $\cH$.
  \item If $\lVert\ell\rVert_\infty\leq B_\ell$ for some real number $B_\ell < \infty$, then for
    every $\epsilon\in(0,B_\ell)$ and $m\in\NN$, we have
    \begin{align*}
      \PP_{\rn{x}\sim\mu^m}[L_{\mu,F,\ell}(\cA(\rn{x}, F^*_m(\rn{x}))) > \epsilon]
      & \geq
      \frac{1}{B_\ell - \epsilon}
      \cdot\left(\frac{s(\ell)}{2}\left(1 - \frac{m}{d}\right) - \epsilon\right).
    \end{align*}
  \end{enumerate}

  In particular, if $\VCN_k(\cH)=\infty$, then $\cH$ is not $k$-PAC learnable with respect to $\ell$.
\end{proposition}

\begin{proof}
  Note that $\VCN_k(\cH) > 0$ implies $\cH$ is not empty.
  
  It is clear that item~\ref{prop:partkPAC->VCN:realizable} follows from the fact that $\ell$ is
  separated.

  Let us now assume that $\ell$ is bounded and $\lVert\ell\rVert_\infty\leq B_\ell$.
  
  Since $\VCN_k(\cH)\geq d$, there exists $A\in\binom{[k]}{k-1}$ and $z^0\in\prod_{f\in r_{k,A}}
  X_{\dom(f)}$ such that $\Nat(\cH(z^0))\geq d$. Let $a$ be the unique element in $[k]\setminus A$
  and for every $C\in r(k)$, let $1^C\in r_k(1)$ be the unique function $C\to[1]$.

  Let
  \begin{align*}
    R & \df r(k)\setminus(\{\{a\}\}\cup r_{k,A}), &
    X' & \df X_{\{a\}},
  \end{align*}
  and note that since $\rk(\cH)\leq 1$, every $H\in\cH(z^0)$ factors as $H = H'\comp\pi$ for some
  $H'\colon X'\to\Lambda$, where $\pi\colon\prod_{f\in r_k\setminus r_{k,A}} X_{\dom(f)}\to X'$ is
  the projection map onto the coordinate indexed by $1^{\{a\}}$. For each $C\in R$, fix a point
  $z^C\in X_C$.

  We also let $\cH'\df\{H' \mid H\in\cH(z^0)\}$ and note that $\Nat(\cH')=\Nat(\cH(z^0))\geq
  d$. Thus, there exists $\widehat{X}\subseteq X'$ with $\lvert\widehat{X}\rvert=\Nat(\cH')\geq d$.

  Given an element $w\in (X')^m$, we define a point $\widehat{w}\in\cE_m(\Omega)$ by
  \begin{align*}
    \widehat{w}_f & \df
    \begin{dcases*}
      z^0_{1^{\dom(f)}}, & if $\dom(f)\subseteq A$,\\
      w_i, & if $\dom(f)=\{a\}$ and $f(a)=i$,\\
      z^C, & if $\dom(f)=C\in R$,
    \end{dcases*}
    \qquad (f\in r_k(m))
  \end{align*}
  and given an element $u\in\Lambda^m$, we define a point $\widehat{u}\in\Lambda^{[m]^k}$ by
  \begin{align*}
    \widehat{u}_\alpha & \df u_{\alpha(a)}
    \qquad (\alpha\in[m]^k).
  \end{align*}

  We also define a (classic PAC) loss function $\ell'\colon X'\times\Lambda\times\Lambda\to\RR_{\geq
    0}$ by
  \begin{align*}
    \ell'(w,u,u')
    & \df
    \ell(\widehat{w},u,u').
  \end{align*}
  Clearly $\ell'$ inherits separation from $\ell$ and we have
  $\lVert\ell'\rVert_\infty\leq\lVert\ell\rVert_\infty\leq B_\ell$ and $s(\ell')\geq s(\ell)$.

  We also define a learning algorithm
  \begin{align*}
    \cA'\colon\bigcup_{m\in\NN} ((X')^m\times \Lambda^m)\to\cF_1(X',\Lambda)
  \end{align*}
  by
  \begin{align*}
    \cA'(w,u)
    & \df
    \cA(\widehat{w},\widehat{u})(z^0,\place)'
    \qquad
    (w\in (X')^m, u\in\Lambda^m),
  \end{align*}
  that is, the algorithm $\cA'$ runs $\cA$ on $(\widehat{w},\widehat{u})$ to obtain some $G\in\cH$,
  it considers $H\df G(z^0)\in\cH(z^0)$ and returns $H'\in\cH'$ such that $H$ factors as $H =
  H'\comp\pi$.

  Let now $\mu\in\Pr(\Omega)$ be given by
  \begin{align*}
    \mu_C & \df
    \begin{dcases*}
      \delta_{z^0_{1^C}}, & if $C\subseteq A$,\\
      \delta_{z^C}, & if $C\in R$,\\
      U(\widehat{X}), & if $C=\{a\}$,
    \end{dcases*}
  \end{align*}
  where $\delta_t$ is the Dirac delta concentrated on $t$ and $U(\widehat{X})$ is the uniform
  probability measure on $\widehat{X}$.

  Let also $\widehat{F}\in\cH'$ be given by Lemma~\ref{lem:nonlearn} applied to $\cH'$,
  $\widehat{X}$, $\cA'$ and $\ell'$ and let $F\in\cH$ be any element that witnesses
  $\widehat{F}\in\cH'$ (that is, we have $\widehat{F} = F(z^0,\place)'$).

  Note that for every $H\in\cH$, we have $L_{\mu,F,\ell}(H) =
  L_{U(\widehat{X}),\widehat{F},\ell'}(H(z^0,\place)')$, which in particular implies that
  $L_{\mu,F,\ell}(F) = L_{U(\widehat{X}),\widehat{F},\ell'}(\widehat{F}) = 0$.

  Fix $\epsilon\in(0,B_\ell)$ and $m\in\NN$. Let $\rn{x}\sim\mu^m$ and $\rn{y}\df
  F^*_m(\rn{x})$. Let also $\rn{w}\sim U(\widehat{X})^m$ and $\rn{u}\df
  (\widehat{F}(\rn{w}_i))_{i=1}^m$ and note that
  \begin{align*}
    (\widehat{\rn{w}}, \widehat{\rn{u}}) & \sim (\rn{x},\rn{y}),
  \end{align*}
  so we have
  \begin{align*}
    \PP_{\rn{x}}\Bigl[
      L_{\mu,F,\ell}\bigl(\cA(\rn{x},\rn{y})\bigr)
      > \epsilon
      \Bigr]
    & =
    \PP_{\rn{x}}\Bigl[
      L_{U(\widehat{X}),\widehat{F},\ell'}\bigl(\cA(\rn{x},\rn{y})(z^0,\place)'\bigr)
      > \epsilon
    \Bigr]
    \\
    & =
    \PP_{\rn{w}}\Bigl[
      L_{U(\widehat{X}),\widehat{F},\ell'}\bigl(\cA(\widehat{\rn{w}},\widehat{\rn{u}})(z^0,\place)')
    ]
    \\
    & =
    \PP_{\rn{w}}\Bigl[
      L_{U(\widehat{X}),\widehat{F},\ell'}\bigl(\cA'(\rn{w},\rn{u})\bigr)
      \Bigr]
    \\
    & \geq
    \frac{1}{B_\ell - \epsilon}
    \cdot\left(\frac{s(\ell)}{2}\left(1 - \frac{m}{d}\right) - \epsilon\right),
  \end{align*}
  where the inequality follows from Lemma~\ref{lem:nonlearn} (and since $s(\ell')\geq s(\ell)$).

  \medskip

  The final assertion that if $\VCN_k(\cH)=\infty$ then $\cH$ is not $k$-PAC learnable with respect
  to $\ell$ follows trivially from the first assertion if $\ell$ is bounded. When $\ell$ is not
  bounded, we can simply replace it with the obviously bounded $k$-partite loss function
  \begin{align*}
    \ell'(x,y,y') & \df \min\{\ell(x,y,y'), 1\}
    \qquad (x\in\cE_1(\Omega), y,y'\in\Lambda)
  \end{align*}
  and note that $\ell'$ inherits separation from $\ell$. It is also straightforward to check that
  $k$-PAC learnability of $\cH$ with respect to $\ell$ implies $k$-PAC learnability of $\cH$ with
  respect to $\ell'$, but $\cH$ is not $k$-PAC learnable with respect to $\ell'$ from the previous
  case.
\end{proof}

\section{Proofs of main theorems}
\label{sec:equivthms}

In this section we put together the results of the previous sections, by proving our main theorems
of Section~\ref{sec:mainthms}, which make explicit all
interesting equivalences that can be derived from Figure~\ref{fig:roadmap}. The statements are
repeated in this section for the reader's convenience.

The first theorem starts with a hypothesis class in the non-partite setting, the second theorem
concerns only the partite setting and the third concerns all forms of agnostic learning (without any
assumption on rank). We also remind the reader that if the hypothesis classes $\cH$ come from
$k$-hypergraphs or structures in a finite relational language (see Remarks~\ref{rmk:rk1}
and~\ref{rmk:partiterk1}), then $\rk(\cH)\leq 1$ and if $\ell$ is the (partite or not, agnostic or
not) $0/1$-loss, then all relevant hypotheses of symmetry, separation and boundedness hold.

\thmkPAC*

\begin{proof}
  By Proposition~\ref{prop:agPAC->PAC}\ref{prop:agPAC->PAC:local}, we know that $\ell^{\ag}$ is
  local. Furthermore, by items~\ref{prop:agPAC->PAC:bound}, \ref{prop:agPAC->PAC:flexible}
  and~\ref{prop:agPAC->PAC:symm} of the same proposition, we know that $\ell^{\ag}$ inherits
  flexibility, boundedness and symmetry from $\ell$. By Proposition~\ref{prop:kpartrk}, we also know
  that $\rk(\cH^{\kpart})=\rk(\cH)\leq 1$.

  Then we have the following correspondence between
  propositions and proofs of implications of the theorem (see also Figure~\ref{fig:roadmap}):
  \begin{description}
  \item[Proposition~\ref{prop:agPAC->PAC}\ref{prop:agPAC->PAC:learn}.]
    \ref{thm:kPAC:agPAC}$\implies$\ref{thm:kPAC:PAC},
    \ref{thm:kPAC:agPACr}$\implies$\ref{thm:kPAC:PACr},
    \ref{thm:kPAC:agPACkpart}$\implies$\ref{thm:kPAC:PACkpart},
    \ref{thm:kPAC:agPACrkpart}$\implies$\ref{thm:kPAC:PACrkpart}.
  \item[Proposition~\ref{prop:derand}.]
    \ref{thm:kPAC:agPACr}$\implies$\ref{thm:kPAC:agPAC},
    \ref{thm:kPAC:agPACrkpart}$\implies$\ref{thm:kPAC:agPACkpart},
    \ref{thm:kPAC:PACrkpart}$\implies$\ref{thm:kPAC:PACkpart}.
    \ref{thm:kPAC:PACr}$\implies$\ref{thm:kPAC:PAC}.
  \item[Proposition~\ref{prop:kpartVCN}.] \ref{thm:kPAC:VCN}$\iff$\ref{thm:kPAC:VCNkpart}.
  \item[Proposition~\ref{prop:kpart}.] \ref{thm:kPAC:agPACkpart}$\implies$\ref{thm:kPAC:agPAC},
    \ref{thm:kPAC:agPACrkpart}$\implies$\ref{thm:kPAC:agPACr},
    \ref{thm:kPAC:PACkpart}$\implies$\ref{thm:kPAC:PAC},
    \ref{thm:kPAC:PACrkpart}$\implies$\ref{thm:kPAC:PACr}.
  \item[Proposition~\ref{prop:kpart3}.] \ref{thm:kPAC:agPACr}$\implies$\ref{thm:kPAC:agPACrkpart}.
  \item[Proposition~\ref{prop:partUC->partagPAC}.]
    \ref{thm:kPAC:UC}$\implies$\ref{thm:kPAC:agPACkpart}.
  \item[Proposition~\ref{prop:VCNdim->UC}.] \ref{thm:kPAC:VCNkpart}$\implies$\ref{thm:kPAC:UC}.
  \item[Proposition~\ref{prop:partkPAC->VCN}.]
    \ref{thm:kPAC:PACkpart}$\implies$\ref{thm:kPAC:VCNkpart}.
  \end{description}

  Also note that the implications~\ref{thm:kPAC:agPAC}$\implies$\ref{thm:kPAC:agPACr},
  \ref{thm:kPAC:agPACkpart}$\implies$\ref{thm:kPAC:agPACrkpart},
  \ref{thm:kPAC:PACkpart}$\implies$\ref{thm:kPAC:PACrkpart}
  and~\ref{thm:kPAC:PAC}$\implies$\ref{thm:kPAC:PACr} are trivial from definitions.
\end{proof}

\thmkPACkpart*

\begin{proof}
  By Proposition~\ref{prop:agPAC->PAC}\ref{prop:agPAC->PAC:local}, we know that $\ell^{\ag}$ is
  local. Furthermore, by item~\ref{prop:agPAC->PAC:bound} of the same proposition, we know that
  $\ell^{\ag}$ inherits boundedness from $\ell$.

  Then we have the following correspondence between propositions and proofs of implications of the
  theorem (see also the last two rows of Figure~\ref{fig:roadmap}):
  \begin{description}
  \item[Proposition~\ref{prop:agPAC->PAC}\ref{prop:agPAC->PAC:learn}.]
    \ref{thm:kPACkpart:agPAC}$\implies$\ref{thm:kPACkpart:PAC},
    \ref{thm:kPACkpart:agPACr}$\implies$\ref{thm:kPACkpart:PACr}.
  \item[Proposition~\ref{prop:derand}.]
    \ref{thm:kPACkpart:agPACr}$\implies$\ref{thm:kPACkpart:agPAC},
    \ref{thm:kPACkpart:PACr}$\implies$\ref{thm:kPACkpart:PAC}.
  \item[Proposition~\ref{prop:partUC->partagPAC}.]
    \ref{thm:kPACkpart:UC}$\implies$\ref{thm:kPACkpart:agPAC}.
  \item[Proposition~\ref{prop:VCNdim->UC}.] \ref{thm:kPACkpart:VCN}$\implies$\ref{thm:kPACkpart:UC}.
  \item[Proposition~\ref{prop:partkPAC->VCN}.]
    \ref{thm:kPACkpart:PAC}$\implies$\ref{thm:kPACkpart:VCN}.
  \end{description}

  Also note that the implications~\ref{thm:kPACkpart:agPAC}$\implies$\ref{thm:kPACkpart:agPACr} and
  \ref{thm:kPACkpart:PAC}$\implies$\ref{thm:kPACkpart:PACr} are trivial from definitions.
\end{proof}

\thmagkPAC*

\begin{proof}
  We have the following correspondence between propositions and proofs of implications of the
  theorem (see also the third column of Figure~\ref{fig:roadmap}):
  \begin{description}
  \item[Proposition~\ref{prop:derand}.]
    \ref{thm:agkPAC:agPACr}$\implies$\ref{thm:agkPAC:agPAC} and
    \ref{thm:agkPAC:agPACrkpart}$\implies$\ref{thm:agkPAC:agPACkpart}.
  \item[Proposition~\ref{prop:kpart}.] \ref{thm:agkPAC:agPACkpart}$\implies$\ref{thm:agkPAC:agPAC}
    and \ref{thm:agkPAC:agPACrkpart}$\implies$\ref{thm:agkPAC:agPACr}.
  \item[Proposition~\ref{prop:kpart3}.] \ref{thm:agkPAC:agPACr}$\implies$\ref{thm:agkPAC:agPACrkpart}.
  \end{description}

  Also note that the implications~\ref{thm:agkPAC:agPAC}$\implies$\ref{thm:agkPAC:agPACr}
  and~\ref{thm:agkPAC:agPACkpart}$\implies$\ref{thm:agkPAC:agPACrkpart} are trivial from definitions.
\end{proof}

\section{Higher-order variables help learnability}
\label{sec:highorder}

In this section we show that the assumption of $\rk(\cH)\leq 1$ is indeed necessary for the
equivalence of Theorem~\ref{thm:kPACkpart}. In particular, in the presence of higher-order variables
finite $\VCN_k$-dimension is sufficient but not necessary for learnability. To show this, we will
present a $2$-partite hypothesis class $\cH$ such that $\VCN_2(\cH)=\infty$ (even in a strong sense:
there exists $x$ with $\cH(x)$ shattering an infinite set), but is such that $\cH$ is $2$-PAC
learnable with respect to any separated and bounded loss function.

\begin{proposition}[Higher-order variables help learnability]\label{prop:infVCNpart}
  Let $\Omega$ be the Borel $2$-partite template given by $\Omega_{\{1\}}\df[1]$,
  $\Omega_{\{2\}}\df\Omega_{\{1,2\}}\df\NN$, all equipped with the discrete $\sigma$-algebra. Let also
  $\Lambda=\{0,1\}$ be equipped with the discrete $\sigma$-algebra and for
  every $V\subseteq\NN$, let $H_V\in\cF_2(\Omega,\Lambda)$ be given by
  \begin{align*}
    H_V(x) & \df \One[x_{1^{\{2\}}}=x_{1^{\{1,2\}}}\in V],
    \qquad (x\in\cE_1(\Omega))
  \end{align*}
  where $1^B$ is the unique function $B\to[1]$. Let also $\cH\df\{H_V \mid V\subseteq\NN\}$, let
  $\ell$ be a separated and bounded $2$-partite loss function and let
  $B_\ell\df\max\{\lVert\ell\rVert_\infty,1\}$. Then the following hold:
  \begin{enumerate}
  \item\label{prop:infVCNpart:VCN} $\VCN_2(\cH)=\infty$. In fact, for the unique point
    $x\in\Omega_{\{1\}}$, $\cH(x)$ shatters an infinite set.
  \item\label{prop:infVCNpart:learn} If (almost) empirical risk minimizers for $\ell$ exist (see
    Remark~\ref{rmk:empriskmin}), then $\cH$ is $2$-PAC learnable with respect to $\ell$. More
    precisely, any empirical risk minimizer $\cA$ for $\ell$ is a $2$-PAC learner for $\cH$ with
    respect to $\ell$ with
    \begin{align*}
      m^{\PAC}_{\cH,\ell,\cA}(\epsilon,\delta)
      & \df
      \begin{multlined}[t]
        \frac{2\cdot B_\ell}{\epsilon}\cdot
        \Biggl(
        m'(\epsilon,\delta) + \ln\frac{2\cdot B_\ell}{\delta'(\delta)\cdot\epsilon}
        \\
        + \sqrt{
          2\cdot m'(\epsilon,\delta)\cdot\ln\frac{2\cdot B_\ell}{\delta'(\delta)\cdot\epsilon}
          + \left(\ln\frac{2\cdot B_\ell}{\delta'(\delta)\cdot\epsilon}\right)^2
        }
        \Biggr),
      \end{multlined}
    \end{align*}
    where
    \begin{align*}
      m'(\epsilon,\delta)
      & \df
      \sqrt{
        \frac{
          \ln(2\cdot B_\ell/(\epsilon\cdot\delta'(\delta)))
        }{
          \ln(2\cdot B_\ell/(2\cdot B_\ell - \epsilon))
        }
      },
      \\
      \delta'(\delta)
      & \df
      1 - \sqrt{1-\delta}\in(0,1).
    \end{align*}
  \end{enumerate}
\end{proposition}

\begin{proof}
  Item~\ref{prop:infVCNpart:VCN} is easy: for the unique point $x\in\Omega_{\{1\}}$, it is straightforward
  to check that $\cH(x)$ shatters the infinite set
  \begin{align*}
    \left\{x\in \prod_{U\in\{\{2\},\{1,2\}\}} \NN
    \;\middle\vert\;
    x_{\{2\}} = x_{\{1,2\}}
    \right\}.
  \end{align*}

  \medskip

  For item~\ref{prop:infVCNpart:learn}, we assume $\lVert\ell\rVert_\infty > 0$ (otherwise the
  result is trivial), let $\epsilon,\delta\in(0,1)$ and $m\geq
  m^{\PAC}_{\cH,\ell,\cA}(\epsilon,\delta)$ be an integer. Let also $\mu\in\Pr(\Omega)$, let
  $F\in\cF_2(\Omega,\Lambda)$ be realizable with respect to $\mu$ and let $\rn{x}\sim\mu^m$. Let
  further
  \begin{align*}
    V_F & \df \{i\in\NN \mid F(1,i,i)=1\},
  \end{align*}
  where in the above, the triple $(1,i,i)$ should be interpreted as the element $z$ of
  $\cE_1(\Omega)$ with $(z_{1^{\{1\}}},z_{1^{\{2\}}},z_{1^{\{1,2\}}})=(1,i,i)$.

  Note that since $F$ is realizable, $\ell$ is separated and $\cE_2(\Omega)$ is countable, we have
  $F(x) = H_{V_F}(x)$ for $\mu^2$-almost every $x\in\cE_2(\Omega)$. Let now $\rn{V}$ be the random
  subset of $\NN$ such that
  \begin{align*}
    \cA(\rn{x},F^*_m(\rn{x})) & = H_{\rn{V}}.
  \end{align*}
  Finally, let
  \begin{align*}
    \rn{W} & \df \{i\in\NN \mid \exists\alpha\in [m]^2, \alpha^*(\rn{x}) = (1, i, i)\}.
  \end{align*}

  Since $\ell$ is separated, it follows that for every $V\subseteq\NN$, we have
  \begin{align*}
    L_{\rn{x},F^*_m(\rn{x}),\ell}(H_V) = 0 & \iff V\cap \rn{W} = V_F\cap \rn{W},
  \end{align*}
  and since $\cA$ is an empirical risk minimizer, we conclude that
  \begin{align}\label{eq:VWVFW}
    \rn{V}\cap\rn{W} = V_F\cap\rn{W}.
  \end{align}

  For every $i\in\NN$, let
  \begin{align*}
    p_i & \df \mu_{\{2\}}(\{i\}), &
    q_i & \df \mu_{\{1,2\}}(\{i\}).
  \end{align*}
  Let us order the $p_i$'s in non-increasing order: $p_{i_1}\geq p_{i_2}\geq\cdots$; note that we
  can get this with order type $\omega$ because for each $a > 0$, there are finitely many (in fact,
  at most $1/a$) indices $i\in\NN$ with $p_i\geq a$. Let $t_\epsilon\in\NN$ be the smallest integer
  with
  \begin{align*}
    \sum_{j > t_\epsilon} p_{i_j}\cdot q_{i_j}
    & \leq
    \frac{\epsilon}{2\cdot\lVert\ell\rVert_\infty}.
  \end{align*}

  We claim that $t_\epsilon < 2\cdot\lVert\ell\rVert_\infty/\epsilon$ and for every
  $j\in[t_\epsilon]$, we have $p_{i_j} > \epsilon/(2\cdot\lVert\ell\rVert_\infty)$. If
  $t_\epsilon=0$, both assertions trivially hold, so suppose that $t_\epsilon\geq 1$ and note that
  the minimality of $t_\epsilon$ gives
  \begin{align*}
    \frac{\epsilon}{2\cdot\lVert\ell\rVert_\infty}
    & <
    \sum_{j \geq t_\epsilon} p_{i_j}\cdot q_{i_j}
    \leq
    p_{i_{t_\epsilon}}\cdot\sum_{j\geq t_\epsilon} q_{i_j}
    \leq
    p_{i_{t_\epsilon}},
  \end{align*}
  so we get $p_{i_j}\geq p_{i_{t_\epsilon}} > \epsilon/(2\cdot\lVert\ell\rVert_\infty)$ for every
  $j\in[t_\epsilon]$. On the other hand, we have
  \begin{align*}
    1
    & \geq
    \sum_{j\leq t_\epsilon} p_{i_j}
    \geq
    t_\epsilon\cdot p_{i_{t_\epsilon}}
    >
    t_\epsilon\cdot\frac{\epsilon}{2\cdot\lVert\ell\rVert_\infty},
  \end{align*}
  so we conclude that $t_\epsilon < 2\cdot\lVert\ell\rVert_\infty/\epsilon$.

  We now order the $q_{i_1},\ldots,q_{i_{t_\epsilon}}$ in non-increasing order: $q_{i_{j_1}}\geq
  q_{i_{j_2}}\geq\cdots\geq q_{i_{j_{t_\epsilon}}}$ and let $u_\epsilon\leq t_\epsilon$ be the
  smallest integer with
  \begin{align*}
    \sum_{v = u_\epsilon + 1}^{t_\epsilon} p_{i_{j_v}}\cdot q_{i_{j_v}}
    & \leq
    \frac{\epsilon}{2\cdot\lVert\ell\rVert_\infty}.
  \end{align*}

  We claim that for every $v\in[u_\epsilon]$, we have $q_{i_{j_v}} > \epsilon/(2\cdot\lVert\ell\rVert_\infty)$. If
  $u_\epsilon=0$, then this holds vacuously, so suppose $u_\epsilon\geq 1$ and note that the
  minimality of $u_\epsilon$ gives
  \begin{align*}
    \frac{\epsilon}{2\cdot\lVert\ell\rVert_\infty}
    & <
    \sum_{v = u_\epsilon}^{t_\epsilon} p_{i_{j_v}}\cdot q_{i_{j_v}}
    \leq
    q_{i_{j_{u_\epsilon}}}\cdot\sum_{v = u_\epsilon}^{t_\epsilon} p_{i_{j_v}}
    \leq
    q_{i_{j_{u_\epsilon}}},
  \end{align*}
  so we have $q_{i_{j_v}}\geq q_{i_{j_{u_\epsilon}}} > \epsilon/(2\cdot\lVert\ell\rVert_\infty)$ for every
  $v\in[u_\epsilon]$.

  Let us now abbreviate $m'\df m'(\epsilon,\delta)$ and $\delta'\df\delta'(\delta)$. For each
  $j\in[t_\epsilon]$, let $E_j(\rn{x})$ be the event $\rn{n}_j\geq m'$, where
  \begin{align}\label{eq:nj}
    \rn{n}_j & \df \lvert\{a\in [m] \mid \rn{x}_{2\mapsto a} = i_j\}\rvert,
  \end{align}
  where in the above $2\mapsto a$ denotes the function $\{2\}\to[m]$ that maps $2$ to $a$. We also
  let $E(\rn{x})\df\bigcap_{j\in[t_\epsilon]} E_j(\rn{x})$. Since $\rn{n}_j$ has binomial
  distribution $\Bi(m,p_{i_j})$, by Chernoff's Bound, we have
  \begin{align*}
    \PP_{\rn{x}}[E_j(\rn{x})]
    & \geq
    1 - \exp\left(-\frac{(p_{i_j}\cdot m - m')^2}{2\cdot p_{i_j}\cdot m}\right)
    \geq
    1 - \exp\left(
    - \frac{B_\ell}{\epsilon\cdot m}\cdot
    \left(\frac{\epsilon\cdot m}{2\cdot B_\ell} - m'\right)^2
    \right),
  \end{align*}
  where the second inequality follows since $p_{i_j} >
  \epsilon/(2\cdot\lVert\ell\rVert_\infty)\geq\epsilon/(2\cdot B_\ell)$ and $\epsilon\cdot m/(2\cdot
  B_\ell)\geq m'$ (as $B_\ell\df\max\{\lVert\ell\rVert_\infty,1\}$ and from $m\geq
  m^{\PAC}_{\cH,\ell,\cA}(\epsilon,\delta)$).

  Note now that
  \begin{align*}
    & \!\!\!\!\!\!
    t_\epsilon\cdot
    \exp\left(
    - \frac{B_\ell}{\epsilon\cdot m}\cdot
    \left(\frac{\epsilon\cdot m}{2\cdot B_\ell} - m'\right)^2
    \right)
    \\
    & <
    \frac{2\cdot B_\ell}{\epsilon}\cdot
    \exp\left(-\frac{B_\ell}{\epsilon\cdot m}\cdot
    \left(
    \frac{\epsilon^2\cdot m^2}{4\cdot B_\ell^2}
    - \frac{\epsilon\cdot m\cdot m'}{B_\ell}
    + (m')^2
    \right)
    \right)
  \end{align*}
  and note that the right-hand side of the above is at most $\delta'$ if and only if
  \begin{align*}
    \frac{\epsilon^2}{4\cdot B_\ell^2}\cdot m^2
    - \frac{\epsilon}{B_\ell}\cdot\left(
    m' + \ln\frac{2\cdot B_\ell}{\delta'\cdot\epsilon}
    \right)\cdot m
    + (m')^2
    \geq
    0.
  \end{align*}

  By analyzing the roots of the above as a polynomial of $m$, since
  \begin{align*}
    m
    & \geq
    m^{\PAC}_{\cH,\ell,\cA}(\epsilon,\delta)
    \df
    \frac{2\cdot B_\ell}{\epsilon}\cdot
    \left(
    m' + \ln\frac{2\cdot B_\ell}{\delta'\cdot\epsilon}
    + \sqrt{
      2\cdot m'\cdot\ln\frac{2\cdot B_\ell}{\delta'\cdot\epsilon}
      + \left(\ln\frac{2\cdot B_\ell}{\delta'\cdot\epsilon}\right)^2
    }
    \right),
  \end{align*}
  we conclude that
  \begin{align*}
    t_\epsilon\cdot
    \exp\left(
    - \frac{B_\ell}{\epsilon m}\cdot
    \left(\frac{\epsilon\cdot m}{2\cdot B_\ell} - m'\right)^2
    \right)
    & \leq
    \delta',
  \end{align*}
  so by the union bound, we get
  \begin{align}\label{eq:PE}
    \PP_{\rn{x}}[E(\rn{x})]
    & \geq
    1 - \sum_{j\in[t_\epsilon]}(1-\PP_{\rn{x}}[E_j(\rn{x})])
    \geq
    1 - \delta'.
  \end{align}

  For each $v\in[u_\epsilon]$, let $E'_v(\rn{x})$ be the event $i_{j_v}\in\rn{W}$. We also
  let $E'(\rn{x})\df\bigcap_{v\in[u_\epsilon]} E'_v(\rn{x})$. Recalling the definition of
  $\rn{n}_j$ in~\eqref{eq:nj}, note that
  \begin{align*}
    \PP_{\rn{x}}[E'_v(\rn{x})] & = 1 - (1 - q_{i_{j_v}})^{m\cdot\rn{n}_{j_v}},
  \end{align*}
  and since $q_{i_{j_v}}>\epsilon/(2\cdot\lVert\ell\rVert_\infty)\geq\epsilon/(2 B_\ell)$ and
  $m\geq m'$, we get
  \begin{align*}
    \PP_{\rn{x}}[E'_v(\rn{x}) \given E(\rn{x})]
    & \geq
    1 - \left(1 - \frac{\epsilon}{2\cdot B_\ell}\right)^{(m')^2}.
  \end{align*}

  Note now that
  \begin{align*}
    u_\epsilon\cdot\left(1 - \frac{\epsilon}{2\cdot B_\ell}\right)^{(m')^2}
    & <
    \frac{2\cdot B_\ell}{\epsilon}\cdot
    \left(1 - \frac{\epsilon}{2\cdot B_\ell}\right)^{(m')^2}
  \end{align*}
  and note that the right-hand side of the above is at most $\delta'$ if and only if
  \begin{align*}
    m'
    & \geq
    \sqrt{
      \frac{
        \ln(2\cdot B_\ell/(\epsilon\cdot\delta'))
      }{
        \ln(2\cdot B_\ell/(2\cdot B_\ell - \epsilon))
      }
    },
  \end{align*}
  and since $m'$ is defined precisely as the right-hand side of the above, we conclude that
  \begin{align*}
    u_\epsilon\cdot\left(1 - \frac{\epsilon}{2\cdot B_\ell}\right)^{(m')^2}
    & \leq
    \delta',
  \end{align*}
  so by the union bound, we get
  \begin{align*}
    \PP_{\rn{x}}[E'(\rn{x})\given E(\rn{x})]
    & \geq
    1 - \sum_{v\in[u_\epsilon]}(1-\PP_{\rn{x}}[E'_v(\rn{x})])
    \geq
    1 - \delta',
  \end{align*}
  which along with~\eqref{eq:PE} gives
  \begin{align*}
    \PP_{\rn{x}}[E'(\rn{x})]
    & \geq
    \PP_{\rn{x}}[E'(\rn{x})\cap E(\rn{x})]
    \geq
    (1 - \delta')^2
    =
    1 - \delta.
  \end{align*}

  Finally, note that within the event $E'(\rn{x})$, we have
  \begin{align*}
    L_{\mu,F,\ell}(\cA(\rn{x},F^*_m(\rn{x})))
    & =
    \PP_{\rn{z}\sim\mu^1}[\ell(\rn{z}, H_{\rn{V}}(\rn{z}), F(\rn{z}))]
    =
    \sum_{i\in\NN} p_i\cdot q_i\cdot\ell(\rn{z}, \One[i\in\rn{V}], \One[i\in V_F])
    \\
    & \leq
    \lVert\ell\rVert_\infty\cdot
    \sum_{i\in\NN\setminus\rn{W}} p_i\cdot q_i
    \leq
    \lVert\ell\rVert_\infty\cdot
    \sum_{i\in\NN\setminus\{i_{j_v} \mid v\in[u_\epsilon]\}} p_i\cdot q_i
    \\
    & =
    \lVert\ell\rVert_\infty\cdot
    \left(
    \sum_{j > t_\epsilon} p_{i_j}\cdot q_{i_j}
    + \sum_{v = u_\epsilon + 1}^{t_\epsilon} p_{i_{j_v}}\cdot q_{i_{j_v}}
    \right)
    \leq
    \epsilon,
  \end{align*}
  where the second equality follows since $\ell$ is separated, the first inequality follows
  from~\eqref{eq:VWVFW} and the fact that $\ell$ is separated and the second inequality follows
  since within the event $E'(\rn{x})$, we have $i_{j_v}\in\rn{W}$ for every $v\in[u_\epsilon]$.

  Thus, we conclude that
  \begin{align*}
    \PP_{\rn{x}}[L_{\mu,F,\ell}(\cA(\rn{x},F^*_m(\rn{x}))) \leq \epsilon]
    & \geq
    \PP_{\rn{x}}[E'(\rn{x})]
    \geq
    1 - \delta,
  \end{align*}
  so $\cA$ is a $2$-PAC learner for $\cH$ with respect to $\ell$.
\end{proof}

\begin{remark}\label{rmk:infVCNpart}
  One might wonder if the result of Proposition~\ref{prop:infVCNpart} can be extended to the
  non-partite setting, i.e., in the presence of higher-order variables, does there exists a
  non-partite $2$-ary hypothesis class with infinite $\VCN_2$-dimension that is still $2$-PAC
  learnable?

  To obtain such $2$-ary hypothesis class, one could hope to apply Propositions~\ref{prop:kpartVCN}
  and~\ref{prop:kpart} to transfer $2$-PAC learnability and infinite $\VCN_2$-dimension to the
  non-partite setting. Unfortunately, it is easy to see that the $\cH$ constructed in
  Proposition~\ref{prop:infVCNpart} is \emph{not} in the image of the operation $\place^{\kpart[2]}$
  of Definition~\ref{def:kpart:cH}.

  However, if we let instead $\Omega$ be the Borel $2$-partite template given by
  $\Omega_{\{1\}}\df\Omega_{\{2\}}\df\Omega_{\{1,2\}}\df\NN$ and $\Lambda\df\{0,1\}^{S_2}$, all
  equipped with the discrete $\sigma$-algebra and let $\cH'\df\{H'_V \mid V\subseteq\NN_+\}$, where
  \begin{align*}
    H'_V(x)_\tau
    & \df
    \One[x_{1^{\{\tau(1)\}}}=0\land x_{1^{\{\tau(2)\}}} = x_{1^{\{1,2\}}}\in V]
    \qquad (x\in\cE_1(\Omega), \tau\in S_2),
  \end{align*}
  then it is not too hard to adapt the proof of Proposition~\ref{prop:infVCNpart} to $\cH'$ with
  some case analysis over the values $\mu_{\{1\}}(\{0\})$ and $\mu_{\{2\}}(\{0\})$. On the other
  hand, it is straightforward to check that $\cH'$ is in the image of the operation
  $\place^{\kpart[2]}$, so one can use Propositions~\ref{prop:kpartVCN} and~\ref{prop:kpart} to
  transfer the result to the non-partite setting.
\end{remark}

\section{What about non-learnability in the non-partite?}
\label{sec:nonlearnnonpart}

In this section we discuss why the proof of Proposition~\ref{prop:partkPAC->VCN} does not naively
extend to the non-partite setting even under the assumption of rank at most $1$, but we give some
evidence that (non-agnostic) non-partite learnability fits into the cycle of the main theorem, as
mentioned in Section~\ref{sec:exp}. Throughout this section, we will focus on the case of graphs,
that is, our Borel template $\Omega$ has $\Omega_i$ being a single point for every $i\geq 2$, we
have $\Lambda=\{0,1\}$ and all of our hypotheses $H\in\cH\subseteq\cF_2(\Omega,\Lambda)$ are
symmetric in the sense that $H(x) = H(\sigma^*(x))$ for every $\sigma\in S_2$ and irreflexive in the
sense that $H(x)=0$ whenever $x_{\{1\}}=x_{\{2\}}$. Thus, we will simplify notation by dropping
higher-order variables.

Assume our family of graphs $\cH$ satisfies $\VCN_2(\cH)=\infty$. This means that there are vertices
$x\in\Omega_1$ such that the family $\cH(x)$ of potential neighborhoods of $x$ in $\cH$ shatter
arbitrarily large sets as we vary $x$. Let us assume for a moment that there exists $x$ such that
$\cH(x)$ actually shatters an infinite set. What we would like to claim is that if we can $2$-PAC
learn $\cH$, then we would be able to PAC learn $\cH(x)$, which would contradict classic PAC
theory.

The families of measures that for which it is too hard to PAC learn $\cH(x)$ are uniform measures in
arbitrarily large subsets shattered by $\cH(x)$ (see Lemma~\ref{lem:nonlearn}), so intuitively a
natural choice of $\mu\in\Pr(\Omega)$ for which it would be too hard to $2$-PAC learn $\cH$ is to
take the average of the Dirac delta distribution concentrated on $x$ and a uniform distribution on
an arbitrarily large subset $U$ shattered by $\cH(x)$. This way, we hope to claim that to $2$-PAC
learn $\mu$, we need to at the very least learn how the neighborhood of $x$ behaves on $U$, which
would seem hard because the behavior of $\cH(x)$ on $U$ is not PAC learnable. The problem is that
the $2$-ary algorithm also gets access to how edges behave inside $U$ and these could reveal
information about how $x$ connects to $U$. The example below illustrates how this can happen.

\begin{example}\label{ex:distgraphs}
  Let $\Omega_1\df\ZZ$. The family of \emph{distance graphs on $\ZZ$} is $\cH^{\dist}\df\{G_A \mid
  A\subseteq\NN_+\}$, where $G_A\in\cF_2(\Omega,\{0,1\})$ is given by
  \begin{align*}
    G_A(x,y) & \df \One[\lvert x-y\rvert\in A]
    \qquad (x,y\in\ZZ),
  \end{align*}
  that is, we connect two vertices exactly when their distance is in $A$.
\end{example}

It is straightforward to check that for every $x\in\ZZ$, $\cH^{\dist}(x)$ shatters the infinite set
$\{y\in\ZZ \mid y > x\}$. Let us now see how our intuition of reducing $2$-PAC learnability of
$\cH^{\dist}$ to PAC learnability of $\cH^{\dist}(x)$ fares: suppose $\mu$ is the average of the
Dirac delta concentrated on $0$ and the uniform distribution on $[2n]$. PAC learnability theory
applied to $\cH^{\dist}(0)$ says that the only way to learn $\cH^{\dist}(0)$ is to see almost all
points of $[2n]$ in the sample. However, note that if in our sample we see the vertices
$0,n,n+1,\ldots,2n$, then we know enough information to deduce an $F\in\cH^{\dist}$ with zero total
loss: this is because the only relevant distances for this are $0,1,\ldots,2n$ and they are all
present in the pairs of $0,n,n+1,\ldots,2n$.

This example shows that the reduction of $2$-PAC learnability to PAC learnability cannot be done
naively: we need to find a large set $U_*$ that is shattered by $\cH^{\dist}(x)$ but whose internal
edges do not give any extra information about how $x$ relates to $U_*$. For the particular case of
$\cH^{\dist}$, this is very easy to do: take $U_*\df\{4^n \mid n\in\NN_+\}$ and note that for
$u,v\in U_*$, we have $\lvert u-v\rvert\notin U_*$ (since in base $2$, $\lvert u-v\rvert$ has at
least two digits $1$, whereas all elements of $U_*$ have exactly one digit $1$ in base $2$). Thus
$\cH^{\dist}$ is not $2$-PAC learnable due to the measures that are averages of the Dirac delta on
$0$ and the uniform distribution on an arbitrarily large subset of $U_*$.

One could wonder if it is not possible to slightly change Example~\ref{ex:distgraphs} spreading the
information further so as to produce a family of graphs with infinite $\VCN_2$-dimension that is
actually $2$-PAC learnable. This leads to the following definition:
\begin{definition}[Partition families]
  A family of graphs $\cH$ on a non-empty Borel space $\Omega_1$ is a \emph{partition family} if
  there exists a countable measurable partition $(P_i)_{i\in I}$ of $\binom{\Omega_1}{2}$ such that
  \begin{align*}
    \cH & = \{G_A \mid A\subseteq I\},
  \end{align*}
  where
  \begin{align*}
    G_A(x,y) & \df \One[\exists i\in A, \{x,y\}\in P_i]
    \qquad (x,y\in\Omega_1, A\subseteq I).
  \end{align*}
\end{definition}
Clearly $\cH^{\dist}$ is a partition family with the partition given by the distance.

The rest of this section is devoted to proving that for partition families, finite
$\VCN_2$-dimension does indeed characterize $2$-PAC learnability. More specifically, we will show
that for a partition family $\cH$, if $\VCN_2(\cH)=\infty$, then we can always find arbitrarily
large sets $U_*\subseteq\Omega_1$ along with $x\in\Omega_1$ such that $\cH(x)$ shatters $U_*$ and
internal edges of $U_*$ give no useful information about how $x$ relates to $U_*$. Let us give some
intuition on what ``useful information'' means. For $x\in\Omega_1$ and $U\subseteq\Omega_1$, let
\begin{align*}
  C^\cH_1(x,U) & \df \{i\in I \mid \exists u\in U, \{x,u\}\in P_i\},\\
  C^\cH_2(U) & \df \{i\in I \mid \exists u,v\in U, \{u,v\}\in P_i\}.
\end{align*}
The best case scenario would be if we could find $x$ and $U_*$ such that $\cH(x)$ shatters $U_*$ and
$C^\cH_1(x,U_*)\cap C^\cH_2(x,U_*)=\varnothing$ so that internal edges of $U_*$ do not give any information
at all about how $x$ relates to $U_*$. However, the next example shows that infinite
$\VCN_2$-dimension does not guarantee the existence of such $x$ and $U_*$.

\begin{example}\label{ex:maxgraphs}
  Let $\Omega_1=\NN$, consider the partition $(P_i)_{i\in\NN_+}$ of $\binom{\NN}{2}$ given by
  \begin{align*}
    P_i
    & \df
    \left\{\{x,y\}\in\binom{\NN}{2} \;\middle\vert\;
    \max\{x,y\} = i\right\}
    \qquad (i\in\NN_+),
  \end{align*}
  and let $\cH^{\max}$ be the partition family of graphs associated with $(P_i)_{i\in\NN}$.
\end{example}

It is straightforward to check that if $\cH^{\max}(x)$ shatters a set $U\subseteq\NN$ of size at
least $2$, then we must have $x < \min(U\setminus\{\min(U)\})$ (i.e., $x$ can only be greater or
equal to at most one element of $U$) and if this holds, then we have
\begin{align*}
  C^{\cH^{\max}}_1(x,U) & \supseteq U\setminus\{\min(U)\}, &
  C^{\cH^{\max}}_2(x,U) & = U\setminus\{\min(U)\},
\end{align*}
so these are never disjoint when $\lvert U\rvert\geq 2$. However, even though such sets $U$ give
information about how $x$ relates to $U$, they do not give useful information: if we know whether
$\{u,v\}\subseteq U$ is an edge of $H\in\cH^{\max}$, this only reveals whether $\{x,\max\{u,v\}\}$ is
an edge of $H$, which we would already know since we already see $u$ and $v$ in our sample.

So instead, we would be content in finding $U$ and $x$ such that $\cH(x)$ shatters $U$ and whenever
$\{u,v\}\in\binom{U}{2}\cap P_i$, at least one of the following holds:
\begin{enumerate*}[label={(\roman*)}]
\item $i\notin C^\cH_1(x,U)$,
\item $\{x,u\}\in P_i$,
\item $\{x,v\}\in P_i$.
\end{enumerate*}
In other words, the pair $\{u,v\}$ does not reveal any information about how $x$ relates to
$U\setminus\{u,v\}$.

We now encode this idea into a Ramsey-theoretic problem: suppose first that $\VCN_2(\cH)=\infty$
because there exists $x\in\Omega_1$ such that $\cH(x)$ shatters an infinite set $V$. Then we can
define two functions $f_1\colon V\to I$ and $f_2\colon\binom{V}{2}\to I$ by letting $f_1(u)$ and
$f_2(\{u,v\})$ be the unique elements of $I$ such that
\begin{align*}
  \{x,u\} & \in P_{f_1(u)}, &
  \{u,v\} & \in P_{f_2(\{u,v\})}.
\end{align*}
Since $V$ is shattered by $\cH(x)$, it follows that $f_1$ is injective. Thus, our objective would be
simply to find an infinite set $U\subseteq V$ such that for every $\{u,v\}\in\binom{U}{2}$, we have
$f_2(\{u,v\})\notin f_1(U)$ or $f_2(\{u,v\})\in\{f_1(u),f_1(v)\}$ ($U$ is automatically shattered by
$\cH(x)$ since $U\subseteq V$).

As expected, this intuition is slightly too simplistic because $\VCN_2(\cH)$ does not imply that
there exists $x$ such that $\cH(x)$ shatters an infinite set, in fact, it does not even guarantee
that $\cH(x)$ shatters arbitrarily large sets for the same $x$ as $\VCN_2(\cH)$ is defined as the
supremum of $\Nat(\cH(x))$ when $x$ varies. So instead, we need the following finite version of the
Ramsey-theoretic problem above:
\begin{lemma}\label{lem:finiteRamseylike}
  For every $n\in\NN$, there exists $\rho=\rho(n)\in\NN$ such that for every set $I$, every
  injection $f_1\colon[\rho]\to I$ and every $f_2\colon\binom{[\rho]}{2}\to I$, there exists
  $U\in\binom{[\rho]}{n}$ such that for every $\{u,v\}\in\binom{U}{2}$, we have $f_2(\{u,v\})\notin
  f_1(U)$ or $f_2(\{u,v\})\in\{f_1(u),f_1(v)\}$.

  In fact, one can take
  \begin{align}\label{eq:finiteRamseylike:rho}
    \rho(n)
    & \df
    \begin{dcases*}
      n, & if $n\leq 2$,\\
      \frac{(n)_3}{2} + 3, & if $n\geq 3$.
    \end{dcases*}
  \end{align}
\end{lemma}

\begin{proof}
  The result for $n\leq 2$ is obvious, so suppose $n\geq 3$ and let $\rho\df\rho(n)\df(n)_3/2 +
  3\geq 6$ be given by~\eqref{eq:finiteRamseylike:rho} (note that $\rho$ is indeed an integer as
  $(n)_3$ is divisible by $2$). Given an injection $f_1\colon[\rho]\to I$ and a function
  $f_2\colon\binom{[\rho]}{2}\to I$, we define the sets
  \begin{align*}
    \cB_A
    & \df
    \left\{U\in\binom{[\rho]}{n} \;\middle\vert\;
    A\subseteq U\land f_2(A)\in f_1(U\setminus A)\right\}
    \qquad \left(A\in\binom{[\rho]}{2}\right).
  \end{align*}

  Note that if $U\in\binom{[\rho]}{n}\setminus\bigcup_{A\in\binom{[\rho]}{2}}\cB_A$, then for every
  $\{u,v\}\in\binom{U}{2}$, we have $f_2(\{u,v\})\notin f_1(U)$ or
  $f_2(\{u,v\})\in\{f_1(u),f_1(v)\}$ (as $f_1$ is injective). Thus, it suffices to show that the set
  $\binom{[\rho]}{n}\setminus\bigcup_{A\in\binom{[\rho]}{2}}\cB_A$ is non-empty, which we prove with
  a simple counting argument.

  First note that for every $\{u,v\}\in\binom{[\rho]}{2}$, if $U\in\cB_{\{u,v\}}$, then
  $f_2(\{u,v\})$ must be in the image of $f_1$, $U$ must contain the set
  $\{f_1^{-1}(f_2(\{u,v\})),u,v\}$ and this set must have size $3$, so we conclude that
  \begin{align*}
    \lvert\cB_{\{u,v\}}\rvert & \leq \binom{\rho-3}{n-3}.
  \end{align*}

  Thus, we get
  \begin{align*}
    \left\lvert
    \binom{[\rho]}{n}
    \middle\backslash
    \bigcup_{A\in\binom{[\rho]}{2}}\cB_A
    \right\rvert
    & \geq
    \binom{\rho}{n}
    -
    \binom{\rho}{2}\cdot\binom{\rho-3}{n-3}
    =
    \binom{\rho-3}{n-3}\cdot
    (\rho)_2\cdot
    \left(
    \frac{\rho-2}{(n)_3} -\frac{1}{2}
    \right),
  \end{align*}
  which is positive since $\rho=(n)_3/2 + 3$.
\end{proof}

We now leverage Lemma~\ref{lem:finiteRamseylike} to prove that any partition family of graphs with
infinite $\VCN_2$-dimension is not $2$-PAC learnable.

\begin{proposition}[Learnability for partition families]\label{prop:pf}
  Let $\cH$ be a partition family of graphs over $\Omega_1$, let $\ell$ be a separated $2$-ary loss
  function and suppose $\VCN_2(\cH)=\infty$. Then $\cH$ is not $2$-PAC learnable with respect to
  $\ell$.
\end{proposition}

\begin{proof}
  Since all $\Omega_i$ have a single element for $i\geq 2$, we will drop higher-order variables from
  the notation. By possibly replacing $\ell$ with the obviously bounded $2$-ary loss function
  \begin{align*}
    \widehat{\ell}(x,y,y') & \df \min\{\ell(x,y,y'),1\}
    \qquad (x\in\cE_2(\Omega), y,y'\in\{0,1\}^{S_2}),
  \end{align*}
  we may without loss of generality assume that $\ell$ is bounded (as $2$-PAC learnability with
  respect to $\ell$ implies $2$-PAC learnability with respect to $\widehat{\ell}$). Since all of our
  hypotheses $H$ are graphs, we will make a small abuse of notation and think of $H^*_m$ as a
  function $\cE_m(\Omega)\to\{0,1\}^{\binom{[m]}{2}}$ instead of
  $\cE_m(\Omega)\to\{0,1\}^{([m])_2}$. In particular, $H^*_2\colon\cE_2(\Omega)\to\{0,1\}$.

  Let $(P_i)_{i\in I}$ be the partition of $\binom{\Omega_1}{2}$ corresponding to the partition family
  $\cH$ so that $\cH = \{G_B \mid B\subseteq I\}$, where
  \begin{align*}
    G_B(x,y) & \df \One[\exists i\in B, \{x,y\}\in P_i]
    \qquad (x,y\in\Omega_1, B\subseteq I).
  \end{align*}
  Let us define a function $\chi_2\colon\Omega_1^2\to I\cup\{\bot\}$ (where $\bot\notin I$) by
  letting $\chi_2(x,x)=\bot$ for every $x\in\Omega_1$ and for $x_1,x_2\in\Omega_1$ distinct, letting
  $\chi_2(x_1,x_2)$ be the unique element of $I$ such that $\{x_1,x_2\}\in P_{\chi_2(x_1,x_2)}$, so
  we get
  \begin{align*}
    G_B(x,y) & = \One[\chi_2(x,y)\in B]
    \qquad (x,y\in\Omega_1, B\subseteq I).
  \end{align*}

  Suppose toward a contradiction that $\cH$ is $2$-PAC learnable with respect to $\ell$, let $\cA$ be
  a $2$-PAC learning algorithm for $\cH$ with respect to $\ell$, let $d\in\NN$ be a non-negative
  integer to be picked later and let $\rho(d)$ be given by Lemma~\ref{lem:finiteRamseylike}.

  Since $\VCN_2(\cH)=\infty$, there exists $z_*\in\Omega_1$ and a
  set $V\in\binom{\Omega_1}{\rho(d)}$ such that $\cH(z_*)$ shatters $V$; note that this in
  particular implies that $z_*\notin V$. We then define a function $\chi_1\colon\Omega_1\to
  I\cup\{\bot\}$ by letting $\chi_1(z_*)=\bot$ and for $x\in\Omega_1\setminus\{z_*\}$, we let
  $\chi_1(x)$ be the unique element of $I$ such that $\{z_*,x\}\in P_{\chi_1(x)}$.

  Let us enumerate $V$ as $v_1,\ldots,v_{\rho(d)}$. Define the functions $f_1\colon[\rho(d)]\to I$
  and $f_2\colon\binom{[\rho(d)]}{2}\to I$ by
  \begin{align*}
    f_1(j) & \df \chi_1(v_j) \qquad (j\in[\rho(d)]),\\
    f_2(\{j_1,j_2\}) & \df \chi_2(v_{j_1},v_{j_2})
    \qquad \left(\{j_1,j_2\}\in\binom{[\rho(d)]}{2}\right).
  \end{align*}

  By Lemma~\ref{lem:finiteRamseylike}, there exists $U\in\binom{[\rho(d)]}{d}$ such that for every
  $\{j_1,j_2\}\in\binom{U}{2}$, we have $f_2(\{j_1,j_2\})\notin f_1(U)$ or
  $f_2(\{j_1,j_2\})\in\{f_1(j_1),f_1(j_2)\}$.

  Let
  \begin{align*}
    I' & \df f_1(U), &
    \widetilde{\cH} & \df \{G_B \mid B\subseteq I'\},\\
    V' & \df \{v_j \mid j\in U\}, &
    \cH' & \df \{G_B\down \mid B\subseteq I'\},
  \end{align*}
  where $H\down\df H^*_2(z_*)\rest_{V'}$. Clearly $\lvert V'\rvert=d$. Note also that since $\cH(z_*)$
  shatters $V$ and $V'\subseteq V$, the family of functions $\cH'$ is simply the full family
  $\{0,1\}^{V'}$.

  Note also that the guarantees of Lemma~\ref{lem:finiteRamseylike} for $U$ imply that for every
  $x_1,x_2\in V'$, we have $\chi_2(x_1,x_2)\notin\chi_1(V')$ or
  $\chi_2(x_1,x_2)\in\{\chi_1(x_1),\chi_1(x_2)\}$.

  We now define a function $g\colon (V')^2\to [2]\cup\{\bot\}$ by
  \begin{align*}
    g(x_1,x_2) & \df
    \begin{dcases*}
      t, & if $\chi_2(x_1,x_2)=\chi_1(x_t)$ for $t\in[2]$,\\
      \bot, & otherwise,
    \end{dcases*}
    \qquad (x_1,x_2\in V').
  \end{align*}

  Given $F\in\cH'$, that is, $F\colon V'\to\{0,1\}$, we let
  \begin{align*}
    B_F & \df \{\chi_1(x) \mid x\in V'\land F(x)=1\}.
  \end{align*}

  Our next objective is to use the $2$-PAC learning algorithm $\cA$ for $\cH$ to produce a randomized
  learning algorithm
  \begin{align*}
    \cA'\colon\bigcup_{m\in\NN}((V')^m\times\{0,1\}^m\times [2^m]) \to \{0,1\}^{V'}
  \end{align*}
  for $\cH'$ with respect to some appropriately defined loss $\ell'$ and leverage
  Proposition~\ref{prop:derand} to get non-learnability from Lemma~\ref{lem:nonlearn} for $\cH'$.

  The idea is a simpler version of the ``departization with randomness'' technique of
  Proposition~\ref{prop:kpart2}. First, we identify $[2^m]$ with $\{0,1\}^m$ naturally so that the
  source of randomness of $\cA'$ can be interpreted as a bitstring $b\in\{0,1\}^m$ of length $m$.

  For $x\in(V')^m$, $y\in\{0,1\}^m$ and $b\in\{0,1\}^m$, let $x^b\in\cE_m(\Omega)$ and
  $y^{b,x}\in\{0,1\}^{\binom{[m]}{2}}$ be given by
  \begin{align*}
    x^b_t
    & \df
    \begin{dcases*}
      x_t, & if $b_t=1$,\\
      z_*, & if $b_t=0$,
    \end{dcases*}
    \qquad (t\in[m]),
    \\
    y^{b,x}_{\{\alpha_1,\alpha_2\}}
    & \df
    \begin{dcases*}
      0, & if $b_{\alpha_1} = b_{\alpha_2} = 0$,\\
      y_{\alpha_2}, & if $b_{\alpha_1}=0$ and $b_{\alpha_2}=1$,\\
      y_{\alpha_1}, & if $b_{\alpha_1}=1$ and $b_{\alpha_2}=0$,\\
      y_{\alpha_{g(x_{\alpha_1},x_{\alpha_2})}} & if $b_{\alpha_1}=b_{\alpha_2}=1$ and
      $g(x_{\alpha_1},x_{\alpha_2})\neq\bot$,\\
      0, & otherwise,
    \end{dcases*}
    \qquad \left(\{\alpha_1,\alpha_2\}\in\binom{[m]}{2}\right).
  \end{align*}
  (Since we are ignoring higher-order variables, $x^b$ is seen as an element of $\Omega_1^m$.)

  We then define
  \begin{align*}
    \cA'(x,y,b)
    & \df
    \cA(x^b, y^{b,x})\down
    \qquad (x\in (V')^m, y\in\{0,1\}^m, b\in\{0,1\}^m).
  \end{align*}

  Finally, we define the loss function $\ell'\colon V'\times\{0,1\}\times\{0,1\}\to\RR_{\geq 0}$ by
  \begin{align*}
    \ell'(x,y,y')
    & \df
    \frac{\ell((z_*,x), y, y') + \ell((x,z_*), y, y')}{4}
    \qquad (x\in V', y,y'\in\{0,1\}),
  \end{align*}
  where on the right-hand side, $(z_*,x)$ is interpreted as the point $w\in\cE_2(\Omega)$ given by
  $w_{\{1\}}=z_*$ and $w_{\{2\}}=x$ and similarly for $(x,z_*)$; and $y$ and $y'$ are interpreted as
  the functions in $\{0,1\}^{S_2}$ that are constant equal to $y$ and $y'$, respectively. Note that
  the fact that $\ell$ is separated implies that $\ell'$ is also separated and we clearly have
  $\lVert\ell'\rVert_\infty\leq\lVert\ell\rVert_\infty/2$.
  
  Similarly to the proof of Proposition~\ref{prop:kpart2}, our objective now is to show that when
  $\cA'$ is attempting to learn some $F\in\cH'$, that is, some $F\colon V'\to\{0,1\}$ under some
  $\mu'\in\Pr(V')$, it is essentially simulating how $\cA$ would learn some $G_{B_F}\in\cH$ under some
  $\widehat{\mu}$ for some suitably defined $B_F\subseteq I$ and $\widehat{\mu}$. This means that we
  will want these objects to have the following properties:
  \begin{enumerate}[label={\Roman*.}]
  \item The total losses bound each other: for every $F,H\in\cH'$, we have
    \begin{align}\label{eq:pfloss}
      L_{\widehat{\mu},G_{B_F},\ell}(G_{B_H}) & \geq L_{\mu',F,\ell'}(H).
    \end{align}
  \item $G_{B_F}$ is always realizable: for every $F\in\cH'$, we have
    \begin{align}\label{eq:pfrealizable}
      L_{\widehat{\mu},G_{B_F},\ell}(G_{B_F}) & = 0.
    \end{align}
  \item The distribution of samples is correct: for $F\in\cH'$, if $\rn{x}\sim(\mu')^m$ and $\rn{b}$ is
    picked uniformly at random in $\{0,1\}^m$, independently from $\rn{x}$ and
    $\rn{\widehat{x}}\sim\widehat{\mu}^m$, then
    \begin{align}\label{eq:pfdist}
      (\rn{x}^{\rn{b}}, F^*_m(\rn{x})^{\rn{b},\rn{x}})
      & \sim
      (\rn{\widehat{x}}, (G_{B_F})^*_m(\rn{\widehat{x}})),
    \end{align}
    where $F^*_m(x)\df (F(\rn{x}_t))_{t=1}^m\in\{0,1\}^m$.
  \end{enumerate}

  Let $\widehat{\mu}\in\Pr(\Omega)$ be the measure on supported on the finite set $V'\cup\{z_*\}$
  that puts mass $1/2$ on $z_*$ and mass $\mu'(\{x\})/2$ on each of point $x$ of $V'$. It is
  straightforward to check that if $\rn{x}\sim(\mu')^m$, $\rn{b}$ is picked uniformly at random in
  $\{0,1\}^m$, independently from $\rn{x}$ and $\rn{\widehat{x}}\sim\widehat{\mu}^m$, then
  $\rn{x}^{\rn{b}}\sim\rn{\widehat{x}}$.

  To upgrade this to~\eqref{eq:pfdist}, it suffices to show that for every $x\in (V')^m$ and every
  $b\in\{0,1\}^m$, we have
  \begin{align}\label{eq:pfdistFH}
    F^*_m(x)^{b,x} & = (G_{B_F})^*_m(x^b).
  \end{align}

  To show this, note that for $\{\alpha_1,\alpha_2\}\in\binom{[m]}{2}$, we have
  \begin{align}
    F^*_m(x)^{b,x}_{\{\alpha_1,\alpha_2\}}
    & =
    \begin{dcases*}
      0, & if $b_{\alpha_1} = b_{\alpha_2} = 0$,\\
      F(x_{\alpha_2}), & if $b_{\alpha_1}=0$ and $b_{\alpha_2}=1$,\\
      F(x_{\alpha_1}), & if $b_{\alpha_1}=1$ and $b_{\alpha_2}=0$,\\
      F(x_{\alpha_{g(x_{\alpha_1},x_{\alpha_2})}}),
      & if $b_{\alpha_1}=b_{\alpha_2}=1$ and $g(x_{\alpha_1},x_{\alpha_2})\neq\bot$,\\
      0, & otherwise.
    \end{dcases*}
    \\
    (G_{B_F})^*_m(x^b)_{\{\alpha_1,\alpha_2\}}
    & =
    \One[\{x^b_{\alpha_1},x^b_{\alpha_2}\}\in B_F]
    =
    \begin{dcases*}
      \One[\chi_2(z_*,z_*)\in B_F], & if $b_{\alpha_1} = b_{\alpha_2}$,\\
      \One[\chi_2(z_*,x_{\alpha_2})\in B_F], & if $b_{\alpha_1} = 0$ and $b_{\alpha_2} = 1$,\\
      \One[\chi_2(x_{\alpha_1},z_*)\in B_F], & if $b_{\alpha_1} = 1$ and $b_{\alpha_2} = 0$,\\
      \One[\chi_2(x_{\alpha_1},x_{\alpha_2})\in B_F], & if $b_{\alpha_1}=b_{\alpha_2}=1$.
    \end{dcases*}
    \label{eq:GBF}
  \end{align}
  Now, note that all points $x_{\alpha_t}$ above are in $V'$ and for $w_1,w_2\in V'$, we have
  \begin{align*}
    \One[\chi_2(z_*,w_1)\in B_F] & = F(w_1),
    \\
    \One[\chi_2(w_1,w_2)\in B_F]
    & =
    \begin{dcases*}
      \One[\chi_1(w_{g(w_1,w_2)})\in B_F], & if $g(w_1,w_2)\neq\bot$,\\
      0, & otherwise,
    \end{dcases*}
    \\
    & =
    \begin{dcases*}
      F(w_{g(w_1,w_2)}), &  if $g(w_1,w_2)\neq\bot$,\\
      0, & otherwise,
    \end{dcases*}
  \end{align*}
  so~\eqref{eq:GBF} becomes
  \begin{align*}
    (G_{B_F})^*_m(x^b)_\alpha
    & =
    \begin{dcases*}
      0, & if $b_{\alpha_1} = b_{\alpha_2}$,\\
      F(x_{\alpha_2}), & if $b_{\alpha_1} = 0$ and $b_{\alpha_2} = 1$,\\
      F(x_{\alpha_1}), & if $b_{\alpha_1} = 1$ and $b_{\alpha_2} = 0$,\\
      F(x_{\alpha_{g(x_{\alpha_1},x_{\alpha_2})}}),
      & if $b_{\alpha_1}=b_{\alpha_2}=1$ and $g(x_{\alpha_1},x_{\alpha_2})\neq\bot$,\\
      0, & otherwise,
    \end{dcases*}
  \end{align*}
  from which~\eqref{eq:pfdistFH} follows and hence~\eqref{eq:pfdist} holds.

  \medskip

  Realizability of~\eqref{eq:pfrealizable} follows since $\ell$ is separated.

  \medskip

  To see that~\eqref{eq:pfloss} holds, note that
  \begin{align*}
    L_{\widehat{\mu},G_{B_F},\ell}(G_{B_H})
    & \geq
    \begin{multlined}[t]
      \frac{1}{4}\cdot\sum_{x\in V'}
      \mu'(\{x\})\cdot
      \bigl(\ell((z_*,x), (G_{B_H})^*_2(z_*,x), (G_{B_F})^*_2(z_*,x))
      \\
      + \ell((x,z_*), (G_{B_H})^*_2(x,z_*), (G_{B_F})^*_2(x,z_*))\bigr)
    \end{multlined}
    \\
    & =
    \begin{multlined}[t]
      \frac{1}{4}\cdot\EE_{\rn{x}\sim\mu'}[
        \ell((z_*,\rn{x}), (G_{B_H})^*_2(z_*,\rn{x}), (G_{B_F})^*_2(z_*,\rn{x}))
        \\
        + \ell((\rn{x},z_*), (G_{B_H})^*_2(\rn{x},z_*), (G_{B_F})^*_2(\rn{x},z_*))
      ]
    \end{multlined}
    \\
    & =
    L_{\mu',F,\ell'}(H).
  \end{align*}

  \medskip

  We can finally show that $\cA'$ is a randomized $k$-PAC learner for $\cH'$ with
  \begin{align*}
    m^{\PACr}_{\cH',\ell',\cA'}(\epsilon,\delta)
    & =
    m^{\PAC}_{\cH,\ell,\cA}\left(\frac{\epsilon}{2}, \widetilde{\delta}_\ell(\epsilon,\delta)\right),
  \end{align*}
  where
  \begin{align*}
    \widetilde{\delta}_\ell(\epsilon,\delta)
    & \df
    \min\left\{\frac{\epsilon\delta}{2\lVert\ell'\rVert_\infty}, \frac{1}{2}\right\}.
  \end{align*}

  Let $F\colon V'\to\{0,1\}$ be realizable with respect to $\ell'$, let $\mu'\in\Pr(V')$, let
  $\epsilon,\delta\in(0,1)$ and let $m\geq m^{\PACr}_{\cH',\ell',\cA'}(\epsilon,\delta)$ be an
  integer.

  We also let $\rn{x}\sim(\mu')^m$ and $\rn{b}$ be picked uniformly at random in $\{0,1\}^m$,
  independently from $\rn{x}$ and $\rn{\widehat{x}}\sim\widehat{\mu}^m$.

  Note now that
  \begin{align*}
    \PP_{\rn{x},\rn{b}}\Bigl[
      L_{\mu',F,\ell'}\bigl(
      \cA'(\rn{x}, F^*_m(\rn{x}), \rn{b})
      \bigr)
      \leq \frac{\epsilon}{2}
      \Bigr]
    & \geq
    \PP_{\rn{x},\rn{b}}\Bigl[
      L_{\widehat{\mu},G_{B_F},\ell}\bigl(
      \cA(\rn{x}^{\rn{b}}, F^*_m(\rn{x})^{\rn{b},\rn{x}})
      \bigr)
      \leq \frac{\epsilon}{2}
      \Bigr]
    \\
    & =
    \PP_{\rn{\widehat{x}}}\Bigl[
      L_{\widehat{\mu},G_{B_F},\ell}\bigl(
      \cA(\rn{\widehat{x}}, (G_{B_F})^*_m(\rn{\widehat{x}}))
      \bigr)
      \leq \frac{\epsilon}{2}
      \Bigr]
    \\
    & \geq
    1 - \widetilde{\delta}_\ell(\epsilon,\delta),
  \end{align*}
  where the first inequality follows from the definition of $\cA'$ and~\eqref{eq:pfloss}, the equality
  follows from~\eqref{eq:pfdist} and the second inequality follows since $\cA$ is a $2$-PAC learner
  for $\cH$ (and~\eqref{eq:pfrealizable}).

  The above can be rewritten as
  \begin{align*}
    \widetilde{\delta}_\ell(\epsilon,\delta)
    & \geq
    \PP_{\rn{x},\rn{b}}\Bigl[
      L_{\mu',F,\ell'}\bigl(
      \cA'(\rn{x}, F^*_m(\rn{x})^{\rn{b},\rn{x}}, \rn{b})
      \bigr)
      > \frac{\epsilon}{2}
      \Bigr]
    \\
    & =
    \EE_{\rn{x}}\biggl[
      \PP_{\rn{b}}\Bigl[
        L_{\mu',F,\ell'}\bigl(
        \cA'(\rn{x}, F^*_m(\rn{x})^{\rn{b},\rn{x}}, \rn{b})
        \bigr)
        > \frac{\epsilon}{2}
        \Bigr]
      \biggr],
  \end{align*}
  which by Markov's Inequality implies
  \begin{align}\label{eq:pfMarkov}
    \PP_{\rn{x}}\biggl[
      \PP_{\rn{b}}\Bigl[
        L_{\mu',F,\ell'}\bigl(
        \cA'(\rn{x}, F^*_m(\rn{x})^{\rn{b},\rn{x}}, \rn{b})
        \bigr)
        > \frac{\epsilon}{2}
        \Bigr]
      > \frac{\epsilon}{2\lVert\ell'\rVert_\infty}
      \biggr]
    & \leq
    \frac{2\cdot\lVert\ell'\rVert_\infty\cdot\widetilde{\delta}_\ell(\epsilon,\delta)}{\epsilon}
    \leq
    \delta.
  \end{align}

  Since the total loss is bounded by $\lVert\ell'\rVert_\infty$, we have the following implication:
  \begin{align*}
    & \!\!\!\!\!\!
    \PP_{\rn{b}}\Bigl[
      L_{\mu',F,\ell'}\bigl(
      \cA'(\rn{x}, F^*_m(\rn{x})^{\rn{b},\rn{x}}, \rn{b})
      \bigr)
      > \frac{\epsilon}{2}
      \Bigr]
    \leq\frac{\epsilon}{2\lVert\ell'\rVert_\infty}
    \\
    & \implies
    \EE_{\rn{b}}\Bigl[
      L_{\mu',F,\ell'}\bigl(
      \cA'(\rn{x}, F^*_m(\rn{x})^{\rn{b},\rn{x}}, \rn{b})
      \bigr)
      \Bigr]
    \leq
    \frac{\epsilon}{2} + \frac{\epsilon}{2\lVert\ell'\rVert_\infty}\cdot\lVert\ell'\rVert_\infty
    =
    \epsilon.
  \end{align*}

  By the contra-positive of the above along with~\eqref{eq:pfMarkov}, we get
  \begin{align*}
    \PP_{\rn{x}}\biggl[
      \EE_{\rn{b}}\Bigl[
        L_{\mu',F,\ell'}\bigl(
        \cA'(\rn{x},F^*_m(\rn{x})^{\rn{b},\rn{x}}, \rn{b})
        \bigr)
        \bigr]
      > \epsilon
      \biggr]
    & \leq
    \delta.
  \end{align*}
  Therefore, $\cA'$ is a randomized PAC learner for $\cH'$.

  \medskip

  By Proposition~\ref{prop:derand}, there exists a non-randomized PAC learner $\cA''$ for $\cH'$ with
  \begin{align*}
    m^{\PAC}_{\cH',\ell',\cA''}(\epsilon,\delta)
    & =
    M(\epsilon,\delta)
    +
    \ceil{
      \frac{2\cdot\lVert\ell'\rVert_\infty^2}{\xi(\epsilon,\delta)^2}\cdot
      \ln\left(
      \frac{2\cdot R_{\cA'}(M(\epsilon,\delta))}{\xi(\epsilon,\delta)}
      \right)
    },
  \end{align*}
  where $M(\epsilon,\delta)\df m^{\PACr}_{\cH',\ell',\cA'}(\xi(\epsilon,\delta),\xi(\epsilon,\delta))$
  and
  \begin{align*}
    \xi(\epsilon,\delta)
    & \df
    \min\left\{\ceil{\frac{2}{\epsilon}}^{-1},\ceil{\frac{2}{\delta}}^{-1}\right\}.
  \end{align*}
  Note that since
  \begin{align*}
    m^{\PACr}_{\cH',\ell',\cA'}(\epsilon,\delta)
    & =
    m^{\PAC}_{\cH,\ell,\cA}\left(\frac{\epsilon}{2}, \widetilde{\delta}_\ell(\epsilon,\delta)\right),
  \end{align*}
  where
  \begin{align*}
    \widetilde{\delta}_\ell(\epsilon,\delta)
    & \df
    \min\left\{\frac{\epsilon\delta}{2\lVert\ell\rVert_\infty}, \frac{1}{2}\right\},
  \end{align*}
  it follows that the value of $m^{\PAC}_{\cH',\ell',\cA''}(\epsilon,\delta)$ does not depend on $d$,
  i.e., even though both the hypothesis class $\cH'$ and the algorithm depend on $d$, the PAC
  learning sample size $m^{\PAC}_{\cH',\ell',\cA''}(\epsilon,\delta)$ of $\cA''$ does not depend on $d$.

  On the other hand, since $s(\ell')\geq s(\ell)/2$ and
  $s(\ell)/2\leq\lVert\ell'\rVert_\infty\leq\lVert\ell\rVert_\infty/2$ and recalling that $\lvert
  V'\rvert=d$ and $\cH'$ shatters $V'$, Lemma~\ref{lem:nonlearn} implies that if $\epsilon <
  s(\ell)/2$, then
  \begin{align*}
    \frac{1}{\lVert\ell\rVert_\infty/2 - \epsilon}
    \cdot\left(
    \frac{s(\ell)}{4}\left(1 - \frac{\ceil{m^{\PAC}_{\cH,\ell',\cA''}(\epsilon,\delta)}}{d}\right)
    - \epsilon\right)
    & \leq
    \delta,
  \end{align*}
  from which we conclude that
  \begin{align*}
    \ceil{m^{\PAC}_{\cH',\ell',\cA''}(\epsilon,\delta)}
    & \geq
    d
    -
    \frac{4d}{s(\ell)}
    \left(\delta\cdot\left(\frac{\lVert\ell\rVert_\infty}{2} - \epsilon\right) + \epsilon\right).
  \end{align*}
  By picking $\delta,\epsilon\in(0,1)$ small enough ($\delta <
  s(\ell)/(8\cdot\lVert\ell\rVert_\infty)$ and $\epsilon < s(\ell)/16$ suffices), the right-hand
  side of the above is at least $d/2$. Letting $d$ be large enough then yields a contradiction, so
  $\cH$ is not $2$-PAC learnable with respect to $\ell$.
\end{proof}

\section{Final remarks and open problems}
\label{sec:final}

In this paper we introduced a theory of high-arity PAC learning that is specially motivated by
learning graphs, hypergraphs and structures in finite relational languages, but extends to learning
the underlying combinatorics responsible for generating a local exchangeable distribution. In this
framework, our samples are not assumed to be i.i.d., but rather only local and exchangeable, which
allows us to take advantage of structured correlation to improve learning.

In high-arity, two different frameworks naturally arise: the non-partite case corresponding to usual
structures and jointly exchangeable distributions and the partite case corresponding to partite
structures and separately exchangeable distributions. Recall that we use the term $k$-PAC rather
than high-arity PAC when we want to emphasize the arity $k$. We showed that for partite structures
(or, more generally, $k$-partite hypothesis classes of rank at most $1$ with finite codomain), all
versions of $k$-PAC learnability, agnostic or not, with randomness or not, are equivalent to finite
$\VCN_k$-dimension and to the uniform convergence property. For the non-partite case, finite
$\VCN_k$-dimension is equivalent to agnostic $k$-PAC learnability (with randomness or not) and
implies $k$-PAC learnability. Indeed, much more is true and the interplay between the partite and
non-partite cases plays an illuminating role that has no analogue in the classical PAC setting.

Confirming our intuition that we can take advantage of structured correlation to improve learning,
there are hypothesis classes (see Examples~\ref{ex:matching} and~\ref{ex:boundeddegree}) that have
finite $\VCN_k$-dimension, but infinite Natarajan dimension; in other words, for these classes,
applying classic PAC algorithms on an i.i.d.\ subsample within a local exchangeable sample does not
PAC learn the class, while there is a $k$-ary algorithm that does $k$-PAC learn the class. This
shows that high-arity PAC learning greatly increases the scope of statistical learnability.

\notoc\subsection{Agnostic versus non-agnostic in the non-partite}

Two natural questions that we left open concern the characterization of (non-agnostic) high-arity
PAC learnability in the non-partite case and the characterization of all forms of $k$-PAC
learnability when the hypothesis classes (partite or not) are not assumed to have rank at most
$1$. For the latter question, the results of Section~\ref{sec:highorder} show that $\VCN_k$-dimension
does not characterize (partite or not) $k$-PAC learnability without the rank assumption, so it is
natural to ask for a suitable generalization of $\VCN_k$-dimension. For the former question, we
provided some evidence in Section~\ref{sec:nonlearnnonpart} that the $\VCN_k$-dimension could still
characterize non-partite $k$-PAC learnability in rank at most $1$; however, even though the results
of the aforementioned section most likely generalize to partition families of $k$-hypergraphs, they
still seem far from covering arbitrary families of graphs.

Another assumption that was crucial for the characterization of $k$-PAC learnability (even though
not necessary for all implications) was the finiteness of the codomain $\Lambda$. It is clear that
$\VCN_k$-dimension does not characterize $k$-PAC learnability when $\Lambda$ is infinite since when
$k=1$, we get the classic PAC learning theory and $\VCN_1$-dimension amounts to the Natarajan
dimension, which is known not to characterize PAC learnability when $\Lambda$ is
infinite~\cite{BCDMY22}. Instead, classic PAC learning when $\Lambda$ is potentially infinite is
characterized by the Daniely--Shalev-Shwartz ($\DS$) dimension~\cite{DS14,BCDMY22}, so it is natural to
ask if there is an appropriate high-arity version of it, say $\DS_k$-dimension, that characterizes
$k$-PAC learning when $\Lambda$ is potentially infinite.

\notoc\subsection{Asymptotics, computability and complexity}

In this paper, we focused on learnable versus non-learnable problem of high-arity PAC learning,
without attempting to optimize sampling efficiency. However, it is natural to ask how small can the
sample size be in terms of the $\VCN_k$-dimension, i.e., can we improve the dependency of
$m^{\PAC}_{\cH,\ell,\cA}$ in terms of $\VCN_k(\cH)$?

Another aspect that we only briefly mentioned is whether there is an actual algorithm that does the
learning (as opposed to a general function). Naturally, this leads to a different notion of learning
that can be called computable high-arity PAC learning (see~\cite{AABLU20} for a version of
computable PAC learning in the arity $1$ setting) and it is natural to ask how much of the
computable PAC theory extends to high-arity.

Finally, in terms of applications, it is also desirable to understand the complexity of the
high-arity learning algorithms for specific learning problems as it is done in classic PAC theory.

\notoc\subsection{Local to global}
\label{subsec:localtoglobal}

The definition of the $\VCN_k$-dimension (along with the theorems of Section~\ref{sec:mainthms})
shows that a family of bipartite graphs is $2$-PAC learnable if and only if the neighborhoods all
vertices $x$ are classically PAC learnable, uniformly in $x$. More generally, $k$-PAC learnability
of a $k$-partite hypothesis class $\cH\subseteq\cF_k(\Omega,\Lambda)$ reduces to asking whether the
family of functions $\cH(x)$ obtained by fixing all but one vertex is classically PAC learnable,
uniformly in $x$. This is very reminiscent of the ``local to global'' phenomenon of high dimensional
expanders (see~\cite{Lub18} for a survey) and it is natural to ask what happens to the intermediate
versions of $\VCN_k$-dimension. Namely, for a $k$-partite hypothesis class
$\cH\subseteq\cF_k(\Omega,\Lambda)$, if we instead compute the Natarajan (or Vapnik--Chervonenkis in
the hypergraph case) of the family of functions obtained by fixing $\ell$ vertices instead of $k-1$
(or more generally, the projection onto $\cE_\ell(\Omega)$), then what learnability notion is
captured by finiteness of this combinatorial dimension? A potential candidate is the partial
information framework of Section~\ref{subsec:partial} below.

\notoc\subsection{Graph property testing}

In Section~\ref{subsubsec:proptest}, we saw that the framework of graph property testing can be
interpreted as $2$-PAC learning, except that we only want to learn the graph up to approximate
isomorphism. From this perspective, graph property testing theory says that all hypothesis classes
are learnable up to approximate isomorphism. This contrast raises one natural question: is it
possible to have a framework in-between property testing and PAC learning in which part of the
information is completely learned and another part is learned up to (approximate) isomorphism?

\notoc\subsection{Multiplicative bounds for asymptotic total loss}
\label{subsec:nearestneighbor}

In classic learning theory, in the agnostic setting, the function that achieves the least loss is
the Bayes predictor and under mild assumptions, the nearest neighbor algorithm achieves an
asymptotic total loss that is at most twice the Bayes loss~\cite{CH67}. All these can be naturally
extended to the high-arity setting by simply ignoring most of the input and passing to a subset that
is guaranteed to be i.i.d., but it is natural to ask if one can take advantage of structured
correlation of high-arity and provide a better asymptotic multiplicative bound. See
Appendix~\ref{sec:Bayes} for more details.

\notoc\subsection{Partial information}
\label{subsec:partial}

Returning to the comparison of $k$-PAC learning with classic PAC learning, we saw that there are
$k$-ary hypothesis classes that are $k$-PAC learnable but not PAC learnable, i.e., we can learn them
if we have access to all $k$-tuples of the sample but we cannot learn them if we only have access to
disjoint $k$-tuples. Thinking of classic PAC learning as $1$-PAC learning, it is natural to ask if
there is an intermediate framework. Abstractly, can we make sense of what $\ell$-PAC learning is for
$k$-ary hypothesis classes for $2\leq\ell\leq k-1$? Or, more application oriented, what if we do not
have access to all $k$-tuples, only some of them (with some non-trivial intersection so as to not
reduce to classic PAC learning)?

Inspired by the theory of hypergraph quasirandomness~\cite{CG90,LM15,Tow17,ACHPS18}, where squashed
octahedra and different Gowers' norms control the behavior of densities of hypergraphs with some
restrictions on the intersection of their edges, we define the following partial information
$k$-partite PAC learning framework (our definition is more natural in the partite setting):

Let $\cC$ be an antichain of subsets of $[k]$ such that $\cC\neq\varnothing$ and $\cC\neq\{[k]\}$
and let $m\in\NN$. The \emph{$\cC$-octahedron of order $m$} is the $k$-partite $k$-hypergraph
$O^\cC_m$ with partition $(V^{\cC,m}_1,\ldots,V^{\cC,m}_k)$ defined as follows: let
\begin{align*}
  V^{\cC,m}_i & \df \{(i,\sigma) \mid \sigma\in [m]^{\cC-i}\}
  \qquad (i\in[k]),
\end{align*}
where
\begin{align*}
  \cC-i & \df \{A\in\cC \mid i\notin A\}
  \qquad (i\in[k]),
\end{align*}
and let
\begin{align*}
  E(O^\cC_m) & \df \{e_\sigma \mid \sigma\in[m]^\cC\}
\end{align*}
be the edge set of $O^\cC_m$, where
\begin{align*}
  e_\sigma & \df \{(i,\sigma\rest_{\cC-i}) \mid i\in[k]\}
  \qquad (\sigma\in[m]^\cC).
\end{align*}
We also let $E^{O^\cC_m}$ be the set of all $\alpha\in\prod_{i=1}^k V^{\cC,m}_i$ such that
$\im(\alpha)\in E(O^{\cC}_m)$.

In the $\cC$-PAC learning setting for a $k$-partite hypothesis class $\cH$, instead of receiving an
$([m],\ldots,[m])$-sample $(\rn{x},F^*_m(\rn{x}))$ for some $F\in\cF_k(\Omega,\Lambda)$ realizable, we
instead receive a \emph{$\cC$-sample of order $m$}: a point of the form $(\rn{x},\rn{y})$, where
$\rn{x}\sim\mu^{V^{\cC,m}_1,\ldots,V^{\cC,m}_k}$ and $\rn{y}$ is a random element of
$\Lambda^{E^{O^\cC_m}}$ given by
\begin{align*}
  \rn{y}_\alpha & \df F^*_{V^{\cC,m}_1,\ldots,V^{\cC,m}_k}(\rn{x})_\alpha
  \qquad (\alpha\in E^{O^{\cC_m}}).
\end{align*}
Note that our $\cC$-sample of order $m$ only contains information about the $k$-tuples that are
edges of $O^\cC_m$. The agnostic setup is defined analogously.

The particular case when the antichain $\cC$ is $S_k$-symmetric deserves special attention:

First, note that for $\cC_0=\{\varnothing\}$, the $\cC_0$-octahedron $O^{\cC_0}_m$ consists of a
$k$-partite matching of size $m$, so $\cC_0$-PAC learning coincides with classic PAC learning
thinking of $\cH$ as a family of functions $\cE_1(\Omega)\to\Lambda$.

On the other hand, for $\cC_{k-1}=\binom{[k]}{k-1}$, the $\cC_{k-1}$-octahedron $O^{\cC_0}_m$ is
simply a complete $k$-partite hypergraph with $m$ vertices on each part, so we recover the usual
$k$-PAC framework.

With the intuition of $\cC_0$ and $\cC_{k-1}$ above, if $\cC_\ell\df\binom{[k]}{\ell}$, then one can
expect that $\cC_\ell$-PAC learning is what it means to $(\ell+1)$-PAC learn a $k$-partite class.

Of course, $\cC$ need not be $S_k$-symmetric and these play a special role in the next subsections.

Let us also mention that one should typically assume that $\bigcap\cC=\varnothing$ as otherwise for
every $i\in\bigcap\cC$, the set $V^{\cC,m}_i$ will have only one element as
$\cC-i=\varnothing$. Without the assumption that $\bigcap\cC=\varnothing$, one should take the same
precautions of non-local PAC learning explained in Section~\ref{subsec:nonlocal} below.

Finally, given the framework above, it is natural to ask if there is a combinatorial dimension whose
finiteness characterizes $\cC$-PAC learnability. A potential candidate for the families $\cC_\ell$
is the intermediate version of $\VCN_k$-dimension described in Section~\ref{subsec:localtoglobal}
above.

\notoc\subsection{Flattening}
\label{subsec:flattening}

As mentioned before, without the assumption of rank at most $1$, our characterization of $k$-PAC
learnability no longer works. It is then natural to ask if it is possible to ``flatten'' a $k$-ary
hypothesis class $\cH$ into a $k'$-ary hypothesis class $\cH^{\kflat}$ so that $k$-PAC learnability
of $\cH$ is equivalent to some form of $k'$-PAC learnability of $\cH^{\kflat}$ but
$\rk(\cH^{\kflat})\leq 1$. It turns out that the framework of partial information of
Section~\ref{subsec:partial} gives us a natural flattening operation in the partite case.

Namely, for a Borel $k$-partite template $\Omega$, we define the \emph{flattening} of $\Omega$ as
the Borel $(2^k-1)$-partite template $\Omega^{\kflat}$ given by
\begin{align*}
  \Omega^{\kflat}_{\{t\}} & \df \Omega_{C_t}\qquad (t\in[2^k-1]),
\end{align*}
where $C_1,\ldots,C_{2^k-1}$ is a fixed enumeration of $r(k)$ and letting $\Omega^{\kflat}_B$ be a
singleton whenever $\lvert B\rvert\geq 2$. For $\mu\in\Pr(\Omega)$, the \emph{flattening} of $\mu$
is $\mu^{\kflat}\in\Pr(\Omega^{\kflat})$ given by
\begin{align*}
  \mu^{\kflat}_{\{t\}} & \df \mu_{C_t}
  \qquad (t\in [2^k-1]).
\end{align*}

Finally, we let $\cC^{\kflat}$ be the antichain of subsets of $[2^k-1]$ given by
$\cC^{\kflat}\df\{A_i \mid i\in[k]\}$, where
\begin{align*}
  A_i
  & \df
  \{t\in[2^k-1] \mid i\notin C_t\}.
\end{align*}
Note that for every $t\in[2^k-1]$, we have
\begin{align*}
  \cC^{\kflat}-t
  & =
  \{A_i \mid i\in C_t\}.
\end{align*}

We then note that for every $m\in\NN$, there is a bijection $b_m\colon V(O^{\cC^{\kflat}}_m)\to
r_k(m)$, where for every $t\in [2^k-1]$ and every $\sigma\in\cC^{\kflat}-t$, $b_m(t,\sigma)\in
r_k(m)$ is the function $C_t\to[m]$ given by
\begin{align*}
  b_m(t,\sigma)(i)
  & \df
  \sigma_{A_i}
  \qquad (i\in C_t).
\end{align*}

In turn, the bijections $b_m$ contra-variantly define Borel-isomorphisms
\begin{align*}
  b_m^*\colon\cE_m(\Omega)\to\cE_{V^{\cC^{\kflat},m}_1,\ldots,V^{\cC^{\kflat},m}_k}(\Omega^{\kflat})
\end{align*}
given by
\begin{align*}
  b_m^*(x)_B & \df x_{b_m(B)}\qquad (B\in E(O^{\cC^{\kflat}}_m)).
\end{align*}

We also have bijections $B_m\colon E^{O^{\cC^{\kflat}}_m}\to [m]^k$ given by
\begin{align*}
  B_m(\alpha)_i & \df \sigma^\alpha_{A_i}\qquad (\alpha\in E^{O^{\cC^{\kflat}}_m}, i\in[k]),
\end{align*}
where $\sigma^\alpha\in[m]^{\cC^{\kflat}}$ is the unique element such that
$\im(\alpha)=e_{\sigma^\alpha}$. In turn, these also contra-variantly define Borel-isomorphisms
$B_m^*\colon\Lambda^{[m]^k}\to\Lambda^{E^{O^{\cC^{\kflat}}_m}}$ by
\begin{align*}
  B_m^*(y)_\alpha & \df y_{B_m(\alpha)} \qquad (\alpha\in E^{O^{\cC^{\kflat}}_m}).
\end{align*}

Note now that when $m=1$, each $V^{\cC^{\kflat},m}_i$ has a single point and
$E^{O^{\cC^{\kflat}}_1}$ also has a single point, so for a $k$-ary hypothesis class
$\cH\subseteq\cF_k(\Omega,\Lambda)$, we can define
\begin{align*}
  \cH^{\kflat} & \df \{H^{\kflat} \mid H\in\cH\},
\end{align*}
where
\begin{align*}
  H^{\kflat} & \df H\comp (b_1^*)^{-1} \qquad (H\in\cH)
\end{align*}
using the natural identification. Similarly, for a loss function $\ell$, we can define
\begin{align*}
  \ell^{\kflat}(x,y,y')
  & \df
  \ell((b_1^*)^{-1}(x), (B_1^*)^{-1}(y), (B_1^*)^{-1}(y'))
  \qquad (x\in\cE_1(\Omega^{\kflat}), y,y'\in\Lambda)
\end{align*}
and for an agnostic loss function $\ell$, we can define
\begin{align*}
  \ell^{\kflat}(H,x,y)
  & \df
  \ell^{\kflat}(H\comp b_1^*, (b_1^*)^{-1}(x), (B_1^*)^{-1}(y)).
\end{align*}

It is straightforward to check that $b_m^*$ is a measure-isomorphism between the measures $\mu^m$
and $(\mu^{\kflat})^{V^{\cC^{\kflat},M}_1,\ldots,V^{\cC^{\kflat},m}_k}$, which in particular implies
that (agnostic, respectively) $k$-PAC learnability of $\cH$ with respect to $\ell$ is equivalent to
(agnostic, respectively) $\cC^{\kflat}$-PAC learnability of $\cH^{\kflat}$ with respect to
$\ell^{\kflat}$ (in both directions the argument is similar to that of Proposition~\ref{prop:kpart}).

This means that a full characterization of $\cC$-PAC learning on rank at most $1$ must in particular
cover a full characterization $k$-PAC learning without the rank assumption.

\notoc\subsection{Non-locality}
\label{subsec:nonlocal}

Throughout this article, the assumption that our exchangeable distributions are local (i.e.,
marginals of our samples in disjoint sets are independent) was crucial. However, the theory of
exchangeability also covers non-local distributions. The difference is the presence of a global
variable $x_\varnothing$ indexed by the empty set (in the partite setting, $\varnothing$ is
interpreted as the empty function), see how non-local distributions are handled in
Appendix~\ref{sec:agexch}.

It is natural to ask for a $k$-PAC learning framework in the non-local case, but one has to be
careful: if we simply give our algorithm an $[m]$-sample, the algorithm only gets information about
the behavior of the hypothesis for a single value of the global variable $x_\varnothing$. Instead,
the algorithm should receive $m'$ i.i.d.\ $[m]$-samples so as to guarantee that multiple values of
the global variable $x_\varnothing$ can be seen.

Similarly to Section~\ref{subsec:flattening}, there is a natural equivalence between partite
non-local $k$-PAC learning and local $\cC_{k-1}^{\knonlocal}$-PAC learning, where
$\cC_{k-1}^{\knonlocal}$ is the antichain of subsets of $[k+1]$ given by
\begin{align*}
  \cC_{k-1}^{\knonlocal} & \df \{[k+1]\setminus\{i\} \mid i\in[k]\}.
\end{align*}
However, the same care of having $m'$ i.i.d.\ $\cC_{k-1}^{\knonlocal}$-samples of order $m$ must be taken
as $\bigcap\cC_{k-1}^{\knonlocal}=\{k+1\}$.

More generally, there is a natural equivalence between non-local $\cC$-PAC learning and local
$\cC^{\knonlocal}$-PAC learning, where $\cC^{\knonlocal}$ is the antichain of subsets of $[k+1]$
given by
\begin{align*}
  \cC^{\knonlocal} & \df \{A\cup\{k+1\} \mid A\in\cC\}
\end{align*}
and $\bigcap\cC^{\knonlocal} = \{k+1\}\cup\bigcap\cC$.

\appendix

\section{Exchangeability in the agnostic setting}
\label{sec:agexch}

In this section, we show that the setup for agnostic $k$-PAC learning is correct, namely, we show
that all distributions with the required locality and exchangeability properties can be generated in
the way that they are done in Definitions~\ref{def:agkPAC} and~\ref{def:partagkPAC}. These results
will all be reductions to lemmas from~\cite{Kal05}, but since the notation in this book is different
from ours, we restate all these results in our language.

\notoc\subsection{Exchangeability}
\label{subsec:exch}

First, since we will have to handle non-local distributions, we need a slightly more general
notation. A \emph{non-local Borel template} is a sequence $\Omega=(\Omega_i)_{i\in\NN}$ of non-empty
Borel spaces (now we also have a space $\Omega_0$). The set of \emph{non-local probability
  templates} on $\Omega$ is the set $\Pr(\Omega)$ of sequences $(\mu_i)_{i\in\NN}$ such that
$\mu_i\in\Pr(\Omega_i)$ ($i\in\NN$).

For a (finite or) countable set $V$, we let
\begin{align*}
  \cE_V(\Omega) & \df \prod_{\substack{A\subseteq V\\ A\text{ finite}}} \Omega_{\lvert A\rvert}, &
  \mu^V & \df \bigotimes_{\substack{A\subseteq V\\ A\text{ finite}}} \mu_{\lvert A\rvert},
\end{align*}
and we use the shorthands $\cE_m(\Omega)\df\cE_{[m]}(\Omega)$ and $\mu^m\df\mu^{[m]}$ when
$m\in\NN$.

Given further a non-empty Borel space $\Lambda$, we let $\cF_k(\Omega,\Lambda)$ be the set of
measurable functions $\cE_k(\Omega)\to\Lambda$. Finally, given $F\in\cF_k(\Omega,\Lambda)$, we let
$F^*_V\colon\cE_V(\Omega)\to\Lambda^{(V)_k}$ be given by
\begin{align*}
  F^*_V(x)_\alpha
  & \df
  F(\alpha^*(x))
  \qquad (x\in\cE_V(\Omega),\alpha\in (V)_k),
\end{align*}
where
\begin{align*}
  \alpha^*(x)_B & \df x_{\alpha(B)},
  \qquad (x\in\cE_V(\Omega), B\subseteq V\text{ finite}).
\end{align*}

We start by stating a lemma that is a stepping stone towards the Aldous--Hoover Theorem from
exchangeability theory in our language.

\begin{lemma}[\protect{\cite[Lemma~7.26]{Kal05}}]\label{lem:Kallenberg}
  Let $\Omega$ be a non-local Borel template, let $\Upsilon$ be the interval $[0,1]$ equipped with
  the Borel $\sigma$-algebra, let $\Omega'$ be the non-local Borel template given by
  $\Omega'_i\df\Upsilon$ ($i\in\NN$), let $\Lambda$ be a non-empty Borel space and let $\rn{x}$
  and $\rn{y}$ be random elements of $\cE_{\NN_+}(\Omega)$ and $\Lambda^{(\NN_+)_k}$, respectively.

  Suppose that $(\rn{x},\rn{y})$ is exchangeable in the sense that for every $\sigma\in S_{\NN_+}$
  the distribution of $(\sigma^*(\rn{x}),\sigma^*(\rn{y}))$ is the same as $(\rn{x},\rn{y})$ and
  suppose that $\rn{x}\sim\mu^{\NN_+}$ for some $\mu\in\Pr(\Omega)$.

  Then there exist $F\in\cF_k(\Omega\otimes\Omega',\Lambda)$ and a random element $\rn{x'}$ of
  $\cE_{\NN_+}(\Omega')$ such that:
  \begin{enumerate}
  \item if $\mu'\in\Pr(\Omega')$ is given by $\mu'_i\df\lambda$ ($i\in\NN$), where $\lambda$ is
    the Lebesgue measure, then $(\rn{x},\rn{x'})\sim(\mu\otimes\mu')^{\NN_+}$;
  \item\label{lem:Kallenberg:as} we have
    \begin{align*}
      \rn{y} & = F^*_{\NN_+}(\rn{x},\rn{x'})
    \end{align*}
    with probability $1$. In particular, we get
    $(\rn{x},\rn{y})=(\rn{x},F^*_{\NN_+}(\rn{x},\rn{x'}))$ with probability $1$.
  \end{enumerate}
\end{lemma}

The lemma above is almost what we need for the setup of non-partite agnostic $k$-PAC learning,
except for the use of the extra variables indexed by the empty set since there is no locality
assumption. The proposition below uses a simple ad hoc argument to remove these extra variables
under the hypothesis of locality.

\begin{proposition}\label{prop:agexch}
  Let $\Omega$ be a Borel template, let $\Upsilon$ be the interval $[0,1]$ equipped with the Borel
  $\sigma$-algebra, let $\Omega'$ be the Borel template given by $\Omega'_i\df\Upsilon$
  ($i\in\NN_+$), let $\Lambda$ be a non-empty Borel
  space and let $\rn{x}$ and $\rn{y}$ be random elements of $\cE_{\NN_+}(\Omega)$ and
  $\Lambda^{(\NN_+)_k}$, respectively.

  Suppose that $(\rn{x},\rn{y})$ is exchangeable in the sense that for every $\sigma\in S_{\NN_+}$
  the distribution of $(\sigma^*(\rn{x}),\sigma^*(\rn{y}))$ is the same as $(\rn{x},\rn{y})$ and
  local in the sense that if $U,V\subseteq\NN_+$ are disjoint, then $((\rn{x}_A)_{A\in r(U)},
  (\rn{y}_\alpha)_{\alpha\in(U)_k})$ is independent from $((\rn{x}_A)_{A\in r(V)},
  (\rn{y}_\alpha)_{\alpha\in(V)_k})$ and suppose that $\rn{x}\sim\mu^{\NN_+}$ for some
  $\mu\in\Pr(\Omega)$.

  Then there exist $F\in\cF_k(\Omega\otimes\Omega',\Lambda)$ and a random element $\rn{x'}$ of
  $\cE_{\NN_+}(\Omega')$ such that:
  \begin{enumerate}
  \item if $\mu'\in\Pr(\Omega')$ is given by $\mu'_i\df\lambda$ ($i\in\NN_+$), where $\lambda$ is
    the Lebesgue measure, then $(\rn{x},\rn{x'})\sim(\mu\otimes\mu')^{\NN_+}$;
  \item\label{prop:agexch:dist} we have
    \begin{align*}
      (\rn{x},\rn{y}) & \sim (\rn{x}, F^*_{\NN_+}(\rn{x},\rn{x'})).
    \end{align*}
  \end{enumerate}
\end{proposition}

\begin{proof}
  Let $\widetilde{\Omega}$ and $\widetilde{\Omega'}$ be the non-local Borel templates obtained from
  $\Omega$ and $\Omega'$, respectively, by adding the space $\Upsilon$ at the coordinate indexed by
  $0$.

  We also let $\widetilde{\mu}\in\Pr(\widetilde{\Omega})$ be obtained from $\mu$ by adding the
  Lebesgue measure $\lambda$ at the coordinate indexed by $0$ and let $\rn{\widetilde{x}}$
  be the random element of $\cE_{\NN_+}(\widetilde{\Omega})$ obtained from $\rn{x}$ via
  $\widetilde{x}_A\df\rn{x}_A$ for every $A\in r(\NN_+)$ and sampling
  $\rn{\widetilde{x}}_\varnothing$ from $\lambda$, independently from $(\rn{x},\rn{y})$.

  Note that exchangeability of $(\rn{x},\rn{y})$ implies exchangeability of
  $(\rn{\widetilde{x}},\rn{y})$ (but not locality).
  
  By applying Lemma~\ref{lem:Kallenberg}, we get a random element $\rn{\widetilde{x'}}$ of
  $\cE_{\NN_+}(\widetilde{\Omega'})$ and
  $\widetilde{F}\in\cF_k(\widetilde{\Omega}\otimes\widetilde{\Omega'},\Lambda)$ such that
  \begin{align}\label{eq:yFtilde}
    \rn{y} & = \widetilde{F}^*_{\NN_+}(\rn{\widetilde{x}},\rn{\widetilde{x'}})
  \end{align}
  with probability $1$.

  Let $\rn{x'}$ be the random element in $\cE_{\NN_+}(\Omega')$ given by
  $\rn{x'}_A\df\rn{\widetilde{x}'}_A$ ($A\in r(\NN_+)$), that is, we simply drop the coordinate of
  $\rn{\widetilde{x'}}$ indexed by $\varnothing$.

  We claim that it suffices to show
  \begin{align}\label{eq:yFtildesim}
    \rn{y} & \sim \widetilde{F}^*_{\NN_+}(z,\rn{x},z',\rn{x'})
  \end{align}
  for $\lambda$-almost every $(z,z')\in[0,1]^2$. Indeed, if this is the case, then setting
  $F(x,x')\df\widetilde{F}(z,x,z',x')$ for any $(z,z')\in[0,1]^2$ such that~\eqref{eq:yFtildesim}
  holds yields the desired result.

  In turn, to show~\eqref{eq:yFtildesim}, it suffices to show that $\rn{y}$ is independent from
  $(\rn{\widetilde{x}}_\varnothing,\rn{\widetilde{x'}}_\varnothing)$ (as random elements). Suppose
  not. Then there exist a finite set $U\subseteq\NN_+$ and $y\in\Lambda^{(U)_k}$ such that letting
  \begin{align*}
    \rn{p} & \df \PP[\forall\alpha\in (U)_k, \rn{y}_\alpha = y_\alpha
      \given \rn{\widetilde{x}}_\varnothing,\rn{\widetilde{x'}}_\varnothing]
  \end{align*}
  we have $\Var[\rn{p}] > 0$. By~\eqref{eq:yFtilde}, if we let
  \begin{align*}
    \rn{p'}
    & \df
    \PP[\forall\alpha\in (U)_k,
      \widetilde{F}(\alpha^*(\rn{\widetilde{x}},\rn{\widetilde{x'}})) = y_\alpha
      \given \rn{\widetilde{x}}_\varnothing,\rn{\widetilde{x'}}_\varnothing],
  \end{align*}
  then $\rn{p'}=\rn{p}$ with probability $1$, so $\Var[\rn{p'}] > 0$. Let $E_U$ be the event in the
  above.

  Let now $V\subseteq\NN_+$ be a set disjoint from $U$ with $\lvert U\rvert = \lvert V\rvert$,
  let $\sigma\in S_{\NN_+}$ be such that $\sigma(U)=V$ and note that for every $\alpha\in(\NN_+)_k$,
  we have
  \begin{align*}
    \sigma^*(\rn{y})_\alpha
    & =
    \widetilde{F}^*_{\NN_+}(\rn{\widetilde{x}},\rn{\widetilde{x'}})_{\sigma\comp\alpha}
    =
    \widetilde{F}(\alpha^*(\sigma^*(\rn{\widetilde{x}},\rn{\widetilde{x'}})))
    =
    \widetilde{F}^*_{\NN_+}(\sigma^*(\rn{\widetilde{x}},\rn{\widetilde{x'}}))_\alpha,
  \end{align*}
  with probability $1$, so we get $\sigma^*(\rn{y}) =
  \widetilde{F}^*_{\NN_+}(\sigma^*(\rn{\widetilde{x}},\rn{\widetilde{x'}}))$ also with probability
  $1$. Since we have
  $\sigma^*(\rn{\widetilde{x}},\rn{\widetilde{x'}})\sim(\rn{\widetilde{x}},\rn{\widetilde{x'}})$ and
  in both sides the coordinates indexed by $\varnothing$ are equal, we get
  \begin{align*}
    (\rn{\widetilde{x}}_\varnothing,\rn{\widetilde{x'}}_\varnothing,\rn{y}) & \sim
    (\rn{\widetilde{x}}_\varnothing,\rn{\widetilde{x'}}_\varnothing,\sigma^*(\rn{y})).
  \end{align*}
  Thus, we have the following alternative formula for $\rn{p}$ that holds with probability $1$:
  \begin{align*}
    \rn{p}
    & =
    \PP[\forall\alpha\in (U)_k, \sigma^*(\rn{y})_\alpha = y_\alpha
      \given \rn{\widetilde{x}}_\varnothing,\rn{\widetilde{x'}}_\varnothing]
    \\
    & =
    \PP[\forall\gamma\in (V)_k, \rn{y}_\gamma = y_{\sigma^{-1}\comp\gamma}
      \given \rn{\widetilde{x}}_\varnothing,\rn{\widetilde{x'}}_\varnothing],
  \end{align*}
  where the second equality follows from the change of variables
  $\gamma\df\sigma\comp\alpha$. By~\eqref{eq:yFtilde}, this gives the following alternative formula
  for $\rn{p'}$ that holds with probability $1$:
  \begin{align*}
    \rn{p'}
    & =
    \PP[\forall\gamma\in (V)_k,
      \widetilde{F}(\gamma^*(\rn{\widetilde{x}},\rn{\widetilde{x'}})) = y_{\sigma^{-1}\comp\gamma}
      \given \rn{\widetilde{x}}_\varnothing,\rn{\widetilde{x'}}_\varnothing].
  \end{align*}
  Let $E_V$ be the event in the above and note that the two formulas for $\rn{p'}$ read as
  \begin{align*}
    \rn{p'}
    & =
    \PP[E_U\given\rn{\widetilde{x}}_\varnothing,\rn{\widetilde{x'}}_\varnothing]
    =
    \PP[E_V\given\rn{\widetilde{x}}_\varnothing,\rn{\widetilde{x'}}_\varnothing],
  \end{align*}
  so we get
  \begin{equation}\label{eq:varp'}
    \begin{aligned}
      \Var[\rn{p'}]
      & =
      \EE\bigl[
        \PP[E_U\given\rn{\widetilde{x}}_\varnothing,\rn{\widetilde{x'}}_\varnothing]
        \cdot
        \PP[E_V\given\rn{\widetilde{x}}_\varnothing,\rn{\widetilde{x'}}_\varnothing]
        \bigr]
      -
      \EE\bigl[\PP[E_U\given{\rn{\widetilde{x}}_\varnothing,\rn{\widetilde{x'}}_\varnothing}]\bigr]
      \cdot
      \EE\bigl[\PP[E_V\given\rn{\widetilde{x}}_\varnothing,\rn{\widetilde{x'}}_\varnothing]\bigr]
      \\
      & =
      \EE\bigl[
        \PP[E_U\cap E_V\given\rn{\widetilde{x}}_\varnothing,\rn{\widetilde{x'}}_\varnothing]
        \bigr]
      - \PP[E_U]\cdot\PP[E_V]
      \\
      & = \PP[E_U\cap E_V] - \PP[E_U]\cdot\PP[E_V],
    \end{aligned}
  \end{equation}
  where the second equality follows since the definition of the events above clearly implies that
  they are conditionally independent given
  $\rn{\widetilde{x}}_\varnothing,\rn{\widetilde{x'}}_\varnothing$ (as when given
  $\rn{\widetilde{x}}_\varnothing,\rn{\widetilde{x'}}_\varnothing$, the events are completely
  determined by sets of coordinates of $\rn{\widetilde{x}}$ and $\rn{\widetilde{x'}}$ that are
  disjoint and the coordinates of $(\rn{\widetilde{x}},\rn{\widetilde{x'}})$ are independent).

  We now claim that $E_U$ and $E_V$ are independent events. Indeed, by~\eqref{eq:yFtilde}, except
  for a probability zero set, $E_U$ is equal to
  \begin{align*}
    \forall\alpha\in(U)_k, \rn{y}_\alpha = y_\alpha
  \end{align*}
  and $E_V$ is equal to
  \begin{align*}
    \forall\alpha\in(V)_k, \rn{y}_\alpha = y_{\sigma^{-1}\comp\alpha}
  \end{align*}
  so locality implies these two events are independent. But putting this together
  with~\eqref{eq:varp'} contradicts $\Var[\rn{p'}] > 0$.
\end{proof}

\begin{remark}\label{rmk:Kallenberg}
  In fact, using a slightly stronger but standard probability argument, we could have proven the
  strengthened version of item~\ref{prop:agexch:dist} in Proposition~\ref{prop:agexch} that is the
  analogue of item~\ref{lem:Kallenberg:as} in Lemma~\ref{lem:Kallenberg}:
  \begin{align*}
    \rn{y} & = F^*_{\NN_+}(\rn{x},\rn{x'})
  \end{align*}
  with probability $1$. However, we opted for the much simpler proof above as it suffices for our
  needs.
\end{remark}

\notoc\subsection{Separate exchangeability}
\label{subsec:sepexch}

Similarly to Section~\ref{subsec:exch}, since we will have to handle non-local distributions, we
need a slightly more general notation. For $k\in\NN$ (we do allow zero), a \emph{non-local Borel
  $k$-partite template} is a sequence $\Omega=(\Omega_A)_{A\subseteq [k]}$ of non-empty Borel spaces
(now we also have a space $\Omega_\varnothing$). The set of \emph{non-local probability $k$-partite
  templates} on $\Omega$ is the set $\Pr(\Omega)$ of sequences $(\mu_A)_{A\subseteq [k]}$ such that
$\mu_A\in\Pr(\Omega_A)$ ($A\subseteq[k]$).

For (finite or) countable sets $V_1,\ldots,V_k$, we let
\begin{align*}
  \cE_{V_1,\ldots,V_k}(\Omega)
  & \df
  \prod_{f\in\widetilde{r}_k(V_1,\ldots,V_k)} \Omega_{\dom(f)},
  &
  \mu^{V_1,\ldots,V_k}
  & \df
  \bigotimes_{f\in\widetilde{r}_k(V_1,\ldots,V_k)} \mu_{\dom(f)},
\end{align*}
where
\begin{align*}
  \widetilde{r}_k(V_1,\ldots,V_k)
  & \df
  \left\{f\colon A\to\bigcup_{i=1}^k V_i
  \;\middle\vert\;
  A\subseteq[k]\land\forall i\in\dom(f), f(i)\in V_i
  \right\}.
\end{align*}
We also use the shorthands $\widetilde{r}_k(m)\df\widetilde{r}_k([m],\ldots,[m])$,
$\cE_m(\Omega)\df\cE_{[m],\ldots,[m]}(\Omega)$ and $\mu^m\df\mu^{[m],\ldots,[m]}$ when $m\in\NN$.

Given further a non-empty Borel space $\Lambda$, we let $\cF_k(\Omega,\Lambda)$ be the set of
measurable functions $\cE_1(\Omega)\to\Lambda$. Finally, given $F\in\cF_k(\Omega,\Lambda)$, we let
$F^*_{V_1,\ldots,V_k}\colon\cE_{V_1,\ldots,V_k}(\Omega)\to\Lambda$ be given by
\begin{align*}
  F^*_{V_1,\ldots,V_k}(x)_\alpha
  & \df
  F(\alpha^*(x))
  \qquad \left(x\in\cE_{V_1,\ldots,V_k}(\Omega), \alpha\in\prod_{i=1}^k V_i\right),
\end{align*}
where
\begin{align*}
  \alpha^*(x)_f & \df x_{\alpha\rest_{\dom(f)}}
  \qquad (x\in\cE_{V_1,\ldots,V_k}(\Omega), f\in \widetilde{r}_k(V_1,\ldots,V_k)).
\end{align*}

Unfortunately, in~\cite{Kal05}, the separate exchangeability analogue of Lemma~\ref{lem:Kallenberg}
is nowhere to be found, so we need to prove it from the more basic separate exchangeability lemma
below on inversion of non-local probability $k$-partite templates.

\begin{lemma}[\protect{\cite[Lemma~7.14]{Kal05}}]\label{lem:inversion}
  Let $\Omega$ and $\widetilde{\Omega}$ be non-local Borel $k$-partite templates, let
  $\mu\in\Pr(\Omega)$ and $\widetilde{\mu}\in\Pr(\widetilde{\Omega})$ be non-local probability
  $k$-partite templates. Let also $\Upsilon$ be the interval $[0,1]$ equipped with the Borel
  $\sigma$-algebra and let $\lambda\in\Pr(\Upsilon)$ be the Lebesgue measure.

  Let further $\Sigma$ be the non-local Borel $k$-partite template given by $\Sigma_A=\Upsilon$
  ($A\subseteq[k]$) and let $\nu\in\Pr(\Sigma)$ be given by $\nu_A\df\lambda$ ($A\subseteq[k]$).

  For each $A\subseteq[k]$, we also let $\Omega^A$ be the non-local Borel $\lvert A\rvert$-partite
  template given by
  \begin{align*}
    \Omega^A_B & \df \Omega_{\iota_{A,k}(B)}, \qquad (B\subseteq [\lvert A\rvert]),
  \end{align*}
  where $\iota_{A,k}$ is the unique increasing function $[\lvert A\rvert]\to[k]$ with
  $\im(\iota_{A,k})=A$.

  For a point $x\in\cE_{\NN_+,\ldots,\NN_+}(\Omega)$ and $A\subseteq[k]$, we let
  $x^A\in\cE_{\NN_+,\ldots,\NN_+}(\Omega^A)$ be given by
  \begin{align*}
    x^A_f & \df x_{f\comp\iota_{A,k}^{-1}\rest_{\iota_{A,k}(\dom(f))}}
    \qquad (f\in\widetilde{r}_{\lvert A\rvert}(\NN_+,\ldots,\NN_+)).
  \end{align*}
  We define the analogous notions over $\widetilde{\Omega}$ and $\Sigma$.

  For each $A\subseteq[k]$, let $F_A\in\cF_{\lvert A\rvert}(\Omega^A,\widetilde{\Omega}^A)$ and
  suppose that $\rn{x}\sim\mu^{\NN_+,\ldots,\NN_+}$, and
  $\rn{\widetilde{x}}\sim\widetilde{\mu}^{\NN_+,\ldots,\NN_+}$ are such that
  \begin{align*}
    \rn{\widetilde{x}}_f
    & =
    (F_{\dom(f)})^*_{\NN_+,\ldots,\NN_+}(\rn{x}^{\dom(f)})_{f\comp\iota_{\dom(f),k}}
  \end{align*}
  with probability $1$ for every $f\in\widetilde{r}_k(\NN_+,\ldots,\NN_+)$.

  Then there exist $G_A\in\cF_{\lvert A\rvert}(\widetilde{\Omega}^A\otimes\Sigma^A,\Omega^A)$
  ($A\subseteq[k]$) and a random element $\rn{z}$ of $\cE_{\NN_+,\ldots,\NN_+}(\Sigma)$ such that
  such that $(\rn{\widetilde{x}},\rn{z})\sim(\widetilde{\mu}\otimes\nu)^{\NN_+,\ldots,\NN_+}$ and
  \begin{align*}
    \rn{x}_f
    & =
    (G_{\dom(f)})^*_{\NN_+,\ldots,\NN_+}(\rn{\widetilde{x}}^{\dom(f)},\rn{z}^{\dom(f)})_{f\comp\iota_{\dom(f),k}}
  \end{align*}
  with probability $1$ for every $f\in\widetilde{r}_k(\NN_+,\ldots,\NN_+)$.
\end{lemma}

From this lemma, we can derive the following separate exchangeability analogue of
Lemma~\ref{lem:Kallenberg}:
\begin{lemma}\label{lem:sepKallenberg}
  Let $\Omega$ be a non-local Borel $k$-partite template, let $\Upsilon$ be the interval $[0,1]$
  equipped with the Borel $\sigma$-algebra, let $\Omega'$ be the non-local Borel $k$-template given
  by $\Omega'_A\df\Upsilon$ ($A\subseteq[k]$), let $\Lambda$ be a non-empty Borel space and let $\rn{x}$
  and $\rn{y}$ be random elements of $\cE_{\NN_+,\ldots,\NN_+}(\Omega)$ and $\Lambda^{\NN_+^k}$,
  respectively.

  Suppose that $(\rn{x},\rn{y})$ is separately exchangeable in the sense that for every
  $\beta_1,\ldots,\beta_k\in S_{\NN_+}$, we have
  \begin{align*}
    (\beta^{\#}(\rn{x}),\beta^{\#}(\rn{y}))
    & \sim
    (\rn{x},\rn{y}),
  \end{align*}
  where
  \begin{gather}
    \beta^{\#}(x)_f \df x_{\beta_{\#}(f)}
    \qquad (x\in\cE_{V_1,\ldots,V_k}(\Omega), f\in\widetilde{r}_k(U_1,\ldots,U_k)),
    \label{eq:betaxnonloc}
    \\
    \beta_{\#}(f)\colon\dom(f)\ni i \mapsto\beta_i(f(i))\in\bigcup_{i=1}^k V_i
    \qquad (f\in\widetilde{r}_k(U_1,\ldots,U_k)),
    \label{eq:betafnonloc}
  \end{gather}
  and
  \begin{gather*}
    \beta^{\#}(y)_\alpha \df y_{\beta_1(\alpha_1),\ldots,\beta_k(\alpha_k)}
    \left(y\in\Lambda^{\prod_{i=1}^k U_i}, \alpha\in\prod_{i=1}^k U_i\right).
  \end{gather*}
  is as in~\eqref{eq:betay}. Suppose further that $\rn{x}\sim\mu^{\NN_+,\ldots,\NN_+}$ for some
  $\mu\in\Pr(\Omega)$.

  Then there exist $F\in\cF_k(\Omega\otimes\Omega',\Lambda)$ and a random element $\rn{x'}$ of
  $\cE_{\NN_+,\ldots,\NN_+}(\Omega')$ such that:
  \begin{enumerate}
  \item\label{lem:sepKallenberg:indep} if $\mu'\in\Pr(\Omega')$ is given by $\mu'_A\df\lambda$
    ($A\subseteq[k]$), where $\lambda$ is the Lebesgue measure, then
    $(\rn{x},\rn{x'})\sim(\mu\otimes\mu')^{\NN_+,\ldots,\NN_+}$.
  \item\label{lem:sepKallenberg:as} we have
    \begin{align*}
      \rn{y} & = F^*_{\NN_+,\ldots,\NN_+}(\rn{x},\rn{x'})
    \end{align*}
    with probability $1$. In particular, we get
    $(\rn{x},\rn{y}) = (\rn{x}, F^*_{\NN_+,\ldots,\NN_+}(\rn{x},\rn{x'}))$ with probability $1$.
  \end{enumerate}
\end{lemma}

\begin{proof}
  Using the same notation as Lemma~\ref{lem:inversion} for the non-local Borel templates $\Omega^A$
  and the points $x^A$, by the separate exchangeability version of the Aldous--Hoover Theorem due to
  Hoover~\cite{Hoo79} (see also~\cite[Corollary~7.23]{Kal05}), there exist
  $\widetilde{F}_A\in\cF_{\lvert A\rvert}((\Omega')^A,\Omega^A)$ ($A\subseteq[k]$),
  $H\in\cF_k(\Omega',\Lambda)$ and a random element $\rn{\xi}$ in
  $\cE_{\NN_+,\ldots,\NN_+}(\Omega')$ such that $\rn{\xi}\sim(\mu')^{\NN_+,\ldots,\NN_+}$ and
  \begin{align}
    \rn{x}_f & = (\widetilde{F}_{\dom(f)})^*_{\NN_+,\ldots,\NN_+}(\rn{\xi}^{\dom(f)})_{f\comp\iota_{\dom(f),k}},
    \label{eq:sepx}
    \\
    \rn{y}_\alpha & = H^*_{\NN_+,\ldots,\NN_+}(\rn{\xi})_\alpha
    \label{eq:sepy}
  \end{align}
  with probability $1$ for every $f\in\widetilde{r}_k(\NN_+,\ldots,\NN_+)$ and every
  $\alpha\in\NN_+^k$.

  Using~\eqref{eq:sepx} and Lemma~\ref{lem:inversion}, there exist $G_A\in\cF_{\lvert
    A\rvert}((\Omega')^A\otimes(\Omega')^A,\Omega^A)$ and a random element $\rn{x'}$ of
  $\cE_{\NN_+,\ldots,\NN_+}(\Omega')$ such that
  $(\rn{x},\rn{x'})\sim(\mu\otimes\mu')^{\NN_+,\ldots,\NN_+}$ and
  \begin{align*}
    \rn{\xi}_f
    & =
    (G_{\dom(f)})^*_{\NN_+,\ldots,\NN_+}(\rn{x}^{\dom(f)},(\rn{x'})^{\dom(f)})_{f\comp\iota_{\dom(f),k}}
  \end{align*}
  with probability $1$ for every $f\in\widetilde{r}_k(\NN_+,\ldots,\NN_+)$. By construction,
  item~\ref{lem:sepKallenberg:indep} holds.

  We now define $F\in\cF_k(\Omega\otimes\Omega',\Lambda)$ by
  \begin{align*}
    F(x,x')
    & \df
    H(\widehat{G}(x,x')),
  \end{align*}
  where
  \begin{align*}
    \widehat{G}(x,x')_{1^A} & \df G_A(x^A,(x')^A) \qquad (A\subseteq[k])
  \end{align*}
  and $1^A\in\widetilde{r}_k(1)$ is the unique function $A\to[1]$.

  Let us prove that item~\ref{lem:sepKallenberg:as} holds. Note that for $\alpha\in\NN_+^k$, with
  probability $1$, we have
  \begin{align}\label{eq:yalpha}
    \rn{y}_\alpha
    & =
    H^*_{\NN_+,\ldots,\NN_+}(\rn{\xi})_\alpha
    =
    H(\alpha^*(\rn{\xi})).
  \end{align}
  But for $A\subseteq[k]$, with probability $1$, we have
  \begin{align}\label{eq:alphaxi}
    \alpha^*(\rn{\xi})_{1^A}
    & =
    \rn{\xi}_{\alpha\rest_A}
    =
    (G_A)^*_{\NN_+,\ldots,\NN_+}(\rn{x}^A,(\rn{x'})^A)_{\alpha\comp\iota_{A,k}}
    =
    G_A((\alpha\comp\iota_{A,k})^*(\rn{x}^A,(\rn{x'})^A)).
  \end{align}
  In turn, for $B\subseteq[\lvert A\rvert]$, with probability $1$, we have
  \begin{align*}
    (\alpha\comp\iota_{A,k})^*(\rn{x}^A,(\rn{x'})^A)_{1^B}
    & =
    (\rn{x}^A,(\rn{x'})^A)_{\alpha\comp\iota_{A,k}\rest_B}
    =
    (\rn{x},\rn{x'})_{\alpha\rest_{\iota_{A,k}(B)}}
    \\
    & =
    \alpha^*(\rn{x},\rn{x'})_{1^{\iota_{A,k}(B)}}
    =
    (\alpha^*(\rn{x})^A,\alpha^*(\rn{x'})^A)_{1^B},
  \end{align*}
  so we have $(\alpha\comp\iota_{A,k}^*)(\rn{x}^A,(\rn{x'})^A) =
  (\alpha^*(\rn{x})^A,\alpha^*(\rn{x'})^A)$ with probability $1$.

  Plugging this in~\eqref{eq:alphaxi}, we get
  \begin{align*}
    \alpha^*(\rn{\xi})_{1^A}
    & =
    G_A(\alpha^*(\rn{x})^A,\alpha^*(\rn{x'})^A)
    =
    \widehat{G}(\alpha^*(\rn{x},\rn{x'}))_{1^A},
  \end{align*}
  with probability $1$, so we have $\alpha^*(\rn{\xi}) = \widehat{G}(\alpha^*(\rn{x},\rn{x'}))$ with
  probability $1$.

  Plugging this in~\eqref{eq:yalpha}, we get
  \begin{align*}
    \rn{y}_\alpha
    & =
    H(\widehat{G}(\alpha^*(\rn{x},\rn{x'})))
    =
    F(\alpha^*(\rn{x},\rn{x'}))
    =
    F^*_{\NN_+,\ldots,\NN_+}(\rn{x},\rn{x'})_\alpha
  \end{align*}
  with probability $1$, concluding the proof.
\end{proof}

Finally, the same trick to derive Proposition~\ref{prop:agexch} from Lemma~\ref{lem:Kallenberg}
applies here to remove the extra variables from Lemma~\ref{lem:sepKallenberg} under the assumption
of locality.

\begin{proposition}\label{prop:agsepexch}
  Let $\Omega$ be a Borel $k$-partite template, let $\Upsilon$ be the interval $[0,1]$, equipped
  with the Borel $\sigma$-algebra, let $\Omega'$ be the Borel $k$-partite template given by
  $\Omega'_A\df\Upsilon$ ($A\in r(k)$), let $\Lambda$ be a non-empty Borel space and let $\rn{x}$
  and $\rn{y}$ be random elements of $\cE_{\NN_+,\ldots,\NN_+}(\Omega)$ and $\Lambda^{\NN_+^k}$,
  respectively.

  Suppose that $(\rn{x},\rn{y})$ is:
  \begin{description}
  \item[Local:] If $(U_1,\ldots,U_k)$ and $(V_1,\ldots,V_k)$ are $k$-tuples of subsets of $\NN_+$
    such that $U_i\cap V_i = \varnothing$ for every $i\in[k]$, then
    \begin{align*}
      ((\rn{x}_f)_{f\in r_k(U_1,\ldots,U_k)}, &
      (F^*_{\NN_+,\ldots,\NN_+}(\rn{x})_\alpha)_{\alpha\in\prod_{i=1}^k U_i})
      \intertext{is independent from}
      ((\rn{x}_f)_{f\in r_k(V_1,\ldots,V_k)}, &
      (F^*_{\NN_+,\ldots,\NN_+}(\rn{x})_\alpha)_{\alpha\in\prod_{i=1}^k V_i}).
    \end{align*}
  \item[Separately exchangeable:] For every $(\beta_1,\ldots,\beta_k)\in S_{\NN_+}^k$, we have
    \begin{align*}
      (\rn{x},F^*_{\NN_+,\ldots,\NN_+}(\rn{x}))
      & \sim
      (\beta^{\#}(\rn{x}), \beta^{\#}(F^*_{\NN_+,\ldots,\NN_+}(\rn{x}))),
    \end{align*}
    where the two $\beta^{\#}$ are given by~\eqref{eq:betax} and~\eqref{eq:betay}.
  \end{description}
  Suppose further that $\rn{x}\sim\mu^{\NN_+,\ldots,\NN_+}$ for some $\mu\in\Pr(\Omega)$.

  Then there exist $F\in\cF_k(\Omega\otimes\Omega',\Lambda)$ and a random element $\rn{x'}$ of
  $\cE_{\NN_+,\ldots,\NN_+}(\Omega')$ such that:
  \begin{enumerate}
  \item if $\mu'\in\Pr(\Omega')$ is given by $\mu'_A\df\lambda$ ($A\in r(k)$), where $\lambda$ is
    the Lebesgue measure, then $(\rn{x},\rn{x'})\sim(\mu\otimes\mu')^{\NN_+,\ldots,\NN_+}$;
  \item\label{prop:agsepexch:dist} we have
    \begin{align*}
      (\rn{x},\rn{y}) & \sim (\rn{x}, F^*_{\NN_+,\ldots,\NN_+}(\rn{x},\rn{x'})).
    \end{align*}
  \end{enumerate}
\end{proposition}

\begin{proof}
  Let $\widetilde{\Omega}$ and $\widetilde{\Omega'}$ be the non-local Borel $k$-templates obtained
  from $\Omega$ and $\Omega'$, respectively, by adding the space $\Upsilon$ at the coordinate
  indexed by $\varnothing$, let $\widetilde{\mu}\in\Pr(\widetilde{\Omega})$ be obtained from $\mu$
  by adding $\lambda$ to the coordinate indexed by the $\varnothing$ and let $\rn{\widetilde{x}}$ be
  obtained from $\rn{x}$ by sampling the extra coordinate $\rn{\widetilde{x}}_\varnothing$ from
  $\lambda$, independently from $(\rn{x},\rn{y})$. Separate exchangeability of $(\rn{x},\rn{y})$
  implies separate exchangeability of $(\rn{\widetilde{x}},\rn{y})$ (but not locality).

  By applying Lemma~\ref{lem:sepKallenberg}, we get a random element $\rn{\widetilde{x'}}$ of
  $\cE_{\NN_+,\ldots,\NN_+}(\widetilde{\Omega'})$ and
  $\widetilde{F}\in\cF_k(\widetilde{\Omega}\otimes\widetilde{\Omega'},\Lambda)$ such that
  \begin{align}\label{eq:sepyFtilde}
    \rn{y} & = \widetilde{F}^*_{\NN_+,\ldots,\NN_+}(\rn{\widetilde{x}},\rn{\widetilde{x'}})
  \end{align}
  with probability $1$.

  Let $\rn{x'}$ be the random element in $\cE_{\NN_+,\ldots,\NN_+}(\Omega')$ obtained from
  $\rn{\widetilde{x'}}$ by dropping the extra coordinate $\rn{\widetilde{x'}}_\varnothing$.

  We claim that it suffices to show
  \begin{align}\label{eq:sepyFtildesim}
    \rn{y} & \sim \widetilde{F}^*_{\NN_+,\ldots,\NN_+}(z,\rn{x},z',\rn{x'})
  \end{align}
  for $\lambda$-almost every $(z,z')\in[0,1]^2$. Indeed, the above implies that by setting
  $F(x,x')\df\widetilde{F}(z,x,z',x')$ for any $(z,z')\in[0,1]^2$ such that~\eqref{eq:sepyFtildesim}
  holds yields the desired result.

  In turn, to show~\eqref{eq:sepyFtildesim}, it suffices to show that $\rn{y}$ is independent from
  $(\rn{\widetilde{x}}_\varnothing,\rn{\widetilde{x'}}_\varnothing)$ (as random elements). Suppose
  not. Then there exists a finite set $U\subseteq\NN_+$ and $y\in\Lambda^{U^k}$ such that letting
  \begin{align*}
    \rn{p} & \df \PP[\forall\alpha\in U^k, \rn{y}_\alpha = y_\alpha
      \given \rn{\widetilde{x}}_\varnothing,\rn{\widetilde{x'}}_\varnothing]
  \end{align*}
  we have $\Var[\rn{p}] > 0$. By~\eqref{eq:sepyFtilde}, if we let
  \begin{align*}
    \rn{p'}
    & \df
    \PP[\forall\alpha\in U^k,
      \widetilde{F}(\alpha^*(\rn{\widetilde{x}},\rn{\widetilde{x'}})) = y_\alpha
      \given \rn{\widetilde{x}}_\varnothing,\rn{\widetilde{x'}}_\varnothing],
  \end{align*}
  then $\rn{p'}=\rn{p}$ with probability $1$, so $\Var[\rn{p'}] > 0$. Let $E_U$ be the event in the
  above.

  Let now $V\subseteq\NN_+$ be a set disjoint from $U$ with $\lvert U\rvert = \lvert V\rvert$, let
  $\sigma\in S_{\NN_+}$ be such that $\sigma(U)=V$ and let $\beta\in S_{\NN_+}^k$ be $\sigma$
  repeated $k$ times. Note that for every $\alpha\in\NN_+^k$, we have
  \begin{align*}
    \beta^{\#}(\rn{y})_\alpha
    & =
    \widetilde{F}^*_{\NN_+,\ldots,\NN_+}(\rn{\widetilde{x}},\rn{\widetilde{x'}})_{\sigma\comp\alpha}
    =
    \widetilde{F}(\alpha^*(\beta^{\#}(\rn{\widetilde{x}},\rn{\widetilde{x'}})))
    =
    \widetilde{F}^*_{\NN_+,\ldots,\NN_+}(\beta^{\#}(\rn{\widetilde{x}},\rn{\widetilde{x'}}))_\alpha,
  \end{align*}
  with probability $1$, so we get
  $\beta^{\#}(\rn{y})=\widetilde{F}^*_{\NN_+,\ldots,\NN_+}(\beta^{\#}(\rn{\widetilde{x}},\rn{\widetilde{x'}}))$
  also with probability $1$. Since
  $\beta^{\#}(\rn{\widetilde{x}},\rn{\widetilde{x'}})\sim(\rn{\widetilde{x}},\rn{\widetilde{x'}})$
  and in both sides the coordinates indexed by $\varnothing $ are equal, we get
  \begin{align*}
    (\rn{\widetilde{x}}_\varnothing,\rn{\widetilde{x'}}_\varnothing,\rn{y}) & \sim
    (\rn{\widetilde{x}}_\varnothing,\rn{\widetilde{x'}}_\varnothing,\beta^{\#}(\rn{y})).
  \end{align*}
  Thus, we have the following alternative formula for $\rn{p}$ that holds with probability $1$:
  \begin{align*}
    \rn{p}
    & =
    \PP[\forall\alpha\in U^k, \beta^{\#}(\rn{y})_\alpha = y_\alpha
      \given \rn{\widetilde{x}}_\varnothing,\rn{\widetilde{x'}}_\varnothing]
    \\
    & =
    \PP[\forall\gamma\in V^k, \rn{y}_\gamma = y_{\sigma^{-1}\comp\gamma}
      \given \rn{\widetilde{x}}_\varnothing,\rn{\widetilde{x'}}_\varnothing],
  \end{align*}
  where the second equality follows from the change of variables
  $\gamma\df\sigma\comp\alpha$. By~\eqref{eq:sepyFtilde}, this gives the following alternative
  formula for $\rn{p'}$ that holds with probability $1$:
  \begin{align*}
    \rn{p'}
    & =
    \PP[\forall\gamma\in V^k,
      \widetilde{F}(\gamma^*(\rn{\widetilde{x}},\rn{\widetilde{x'}})) = y_{\sigma^{-1}\comp\gamma}
      \given \rn{\widetilde{x}}_\varnothing,\rn{\widetilde{x'}}_\varnothing].
  \end{align*}
  Let $E_V$ be the event in the above and note that the two formulas for $\rn{p'}$ read as
  \begin{align*}
    \rn{p'}
    & =
    \PP[E_U\given\rn{\widetilde{x}}_\varnothing,\rn{\widetilde{x'}}_\varnothing]
    =
    \PP[E_V\given\rn{\widetilde{x}}_\varnothing,\rn{\widetilde{x'}}_\varnothing],
  \end{align*}
  so we get
  \begin{equation}\label{eq:sepvarp'}
    \begin{aligned}
      \Var[\rn{p'}]
      & =
      \EE\bigl[
        \PP[E_U\given\rn{\widetilde{x}}_\varnothing,\rn{\widetilde{x'}}_\varnothing]
        \cdot
        \PP[E_V\given\rn{\widetilde{x}}_\varnothing,\rn{\widetilde{x'}}_\varnothing]
        \bigr]
      -
      \EE\bigl[\PP[E_U\given{\rn{\widetilde{x}}_\varnothing,\rn{\widetilde{x'}}_\varnothing}]\bigr]
      \cdot
      \EE\bigl[\PP[E_V\given\rn{\widetilde{x}}_\varnothing,\rn{\widetilde{x'}}_\varnothing]\bigr]
      \\
      & =
      \EE\bigl[
        \PP[E_U\cap E_V\given\rn{\widetilde{x}}_\varnothing,\rn{\widetilde{x'}}_\varnothing]
        \bigr]
      - \PP[E_U]\cdot\PP[E_V]
      \\
      & = \PP[E_U\cap E_V] - \PP[E_U]\cdot\PP[E_V],
    \end{aligned}
  \end{equation}
  where the second equality follows since the definition of the events above clearly implies that
  they are conditionally independent given
  $\rn{\widetilde{x}}_\varnothing,\rn{\widetilde{x'}}_\varnothing$ (as when given
  $\rn{\widetilde{x}}_\varnothing,\rn{\widetilde{x'}}_\varnothing$, the events are completely
  determined by sets of coordinates of $\rn{\widetilde{x}}$ and $\rn{\widetilde{x'}}$ that are
  disjoint and the coordinates of $(\rn{\widetilde{x}},\rn{\widetilde{x'}})$ are independent).

  We now claim that $E_U$ and $E_V$ are independent events. Indeed, by~\eqref{eq:sepyFtilde}, except
  for a probability zero set, $E_U$ is equal to
  \begin{align*}
    \forall\alpha\in U^k, \rn{y}_\alpha = y_\alpha
  \end{align*}
  and $E_V$ is equal to
  \begin{align*}
    \forall\alpha\in V^k, \rn{y}_\alpha = y_{\sigma^{-1}\comp\alpha}
  \end{align*}
  so locality implies these two events are independent. But putting this together
  with~\eqref{eq:sepvarp'} contradicts $\Var[\rn{p'}] > 0$.
\end{proof}

\begin{remark}
  Again, similarly to Remark~\ref{rmk:Kallenberg}, a slightly stronger but standard probability
  argument gives the strengthened version of item~\ref{prop:agsepexch:dist} in
  Proposition~\ref{prop:agsepexch} that is the analogue of item~\ref{lem:sepKallenberg:as} in
  Lemma~\ref{lem:sepKallenberg}:
  \begin{align*}
    \rn{y} & = F^*_{\NN_+,\ldots,\NN_+}(\rn{x},\rn{x'})
  \end{align*}
  with probability $1$.
\end{remark}

\section{Bayes predictors}
\label{sec:Bayes}

Since in the agnostic setting the adversary is not picking a hypothesis, a very natural
question is what is the hypothesis that gives the smallest loss. Similarly to classic agnostic PAC
learning, this is given by the Bayes predictors defined below.

\begin{definition}[Bayes predictor]
  Let $\Omega$, $\Omega'$ be ($k$-partite, respectively) Borel templates, let $\Lambda$ be a finite
  non-empty Borel space, let $\mu\in\Pr(\Omega)$ and $\mu'\in\Pr(\Omega')$, let
  $\ell\colon\cE_k(\Omega)\times\Lambda^{S_k}\times\Lambda^{S_k}\to\RR_{\geq 0}$ be a $k$-ary loss
  ($\ell\colon\cE_1(\Omega)\times\Lambda\times\Lambda\to\RR_{\geq 0}$ be a $k$-partite loss,
  respectively) and let $F\in\cF_k(\Omega\otimes\Omega',\Lambda)$.

  We say that $B\in\cF_k(\Omega,\Lambda)$ is a \emph{Bayes predictor} for $(\mu,\mu',F,\ell)$ if
  \begin{align*}
    \PP_{\rn{x}\sim\mu^k}[
      \forall y\in\Lambda^{S_k},
      \EE_{\rn{x'}\sim(\mu')^k}[\ell(\rn{x}, B^*_k(\rn{x}), F^*_k(\rn{x},\rn{x'}))]
      \leq
      \EE_{\rn{x'}\sim(\mu')^k}[\ell(\rn{x}, y, F^*_k(\rn{x},\rn{x'}))]
    ] & = 1
  \end{align*}
  in the non-partite case and
  \begin{align*}
    \PP_{\rn{x}\sim\mu^k}[
      \forall y\in\Lambda,
      \EE_{\rn{x'}\sim(\mu')^1}[\ell(\rn{x}, B(\rn{x}), F(\rn{x},\rn{x'}))]
      \leq
      \EE_{\rn{x'}\sim(\mu')^1}[\ell(\rn{x}, y, F(\rn{x},\rn{x'}))]
    ] & = 1
  \end{align*}
  in the partite case.

  Intuitively, a Bayes predictor on input $x$ outputs the label $y$ that minimizes the conditional
  loss given $\rn{x}=x$.

  For the $0/1$-loss, the equations above are equivalent to
  \begin{align*}
    \PP_{\rn{x}\sim\mu^k}[
      \forall y\in\Lambda^{S_k},
      \PP_{\rn{x'}\sim(\mu')^k}[B^*_k(\rn{x}) = F^*_k(\rn{x},\rn{x'})]
      \geq
      \PP_{\rn{x'}\sim(\mu')^k}[y = F^*_k(\rn{x},\rn{x'})]
    ] & = 1
  \end{align*}
  in the non-partite case and
  \begin{align*}
    \PP_{\rn{x}\sim\mu^k}[
      \forall y\in\Lambda,
      \PP_{\rn{x'}\sim(\mu')^1}[B(\rn{x}) = F(\rn{x},\rn{x'})]
      \geq
      \PP_{\rn{x'}\sim(\mu')^1}[y = F(\rn{x},\rn{x'})]
    ] & = 1
  \end{align*}
  in the partite case.
\end{definition}

Recall from Theorems~\ref{thm:kPAC} and~\ref{thm:kPACkpart} that the natural agnostic version of a
non-agnostic loss $\ell$ is given by
\begin{align*}
  \ell^{\ag}(H,x,y) & \df \ell(x,H^*_k(x),y)
  \qquad (H\in\cH, x\in\cE_k(\Omega), y\in\Lambda^{S_k})
\end{align*}
in the non-partite case and by
\begin{align*}
  \ell^{\ag}(H,x,y) & \df \ell(x, H(x), y)
  \qquad (H\in\cH, x\in\cE_1(\Omega), y\in\Lambda)
\end{align*}
in the partite case.

It is straightforward to check that if $B\in\cF_k(\Omega,\Lambda)$ is a Bayes predictor for
$(\mu,\mu',F,\ell)$, then for every $H\in\cF_k(\Omega,\Lambda)$, we have
\begin{align*}
  L_{\mu,\mu',F,\ell^{\ag}}(B) & \leq L_{\mu,\mu',F,\ell^{\ag}}(H).
\end{align*}
One of the earliest results of learning theory due to Cover and Hart~\cite{CH67} is the fact that
under mild assumptions the nearest neighbor algorithm yields a hypothesis that is not far from the
optimal value (which is given by the Bayes predictor). We state this result below\footnote{For simplicity,
we add the assumption of completeness of the underlying metric space, but we point out that the
original result only requires that some version of a Radon--Nikodym derivative exists with the
correct continuity properties.} recalling that in the unary agnostic distributions are more easily
described as a distribution over $\Omega\times\Lambda$ rather than using an extra Borel template.

\begin{theorem}[Cover--Hart~\cite{CH67}]\label{thm:Bayes}
  Let $\Omega=(X,d)$ be a separable complete metric space equipped with the Borel $\sigma$-algebra,
  let $\Lambda$ be a finite Borel space with $\lvert\Lambda\rvert\geq 2$ and let $\nu$ be a
  probability measure on $\Omega\times\Lambda$.

  Let $\mu\in\Pr(\Omega)$ be the marginal of $\nu$ on $\Omega$ and for each $i\in\Lambda$, let
  $\theta_i\in\Pr(\Omega)$ be given by $\theta_i(A)\df\nu(A\times\{i\})$, i.e., $\theta_i$ is the marginal of
  the conditional distribution of $\nu$ given that the second coordinate is $i$.

  Suppose further that there are versions of the Radon--Nikodym derivatives
  $\frac{d\theta_i}{d\mu}$ such that the set
  \begin{align*}
    \left\{x\in\Omega
    \;\middle\vert\;
    \mu(\{x\}) > 0
    \lor
    \forall i\in\Lambda,
    \frac{d\theta_i}{d\mu}
    \text{ is continuous at } x
    \right\}
  \end{align*}
  has $\mu$-measure $1$, that is, $\mu$-almost every point of $\Omega$ is either a continuity point
  of all Radon--Nikodym derivatives or an atom of $\mu$.

  Consider the nearest neighbor classifier defined as the function
  \begin{align*}
    \cA\colon \bigcup_{m\in\NN_+} (X^m\times Y^m)
    & \to
    Y^X
  \end{align*}
  given by
  \begin{align*}
    A(x,y)(z) & \df y_{i(x)} & \qquad (x\in X^m, y\in Y^m),
  \end{align*}
  where
  \begin{align*}
    i(x) & \df \argmin\{i\in[m] \mid d(x_i,z)\}
  \end{align*}
  (breaking ties arbitrarily).

  Suppose further that $B$ is a Bayes predictor for $\nu$ and let
  \begin{align*}
    R^* & \df L_{\nu,\ell_{0/1}}(B), &
    R & \df \lim_{m\to\infty} \EE_{(\rn{x},\rn{y})\sim\nu^m}[L_{\nu,\ell_{0/1}}(A(\rn{x},\rn{y}))].
  \end{align*}

  Then we have
  \begin{align*}
    R^*
    & \leq
    R
    \leq
    R^*\cdot\left(2 - \frac{\lvert\Lambda\rvert}{\lvert\Lambda\rvert-1}\cdot R^*\right).
  \end{align*}
\end{theorem}

In plain English, the nearest neighbor classifier labels each point by the label of the nearest
point in the sample and the theorem above says that as long as the corresponding Radon--Nikodym
derivatives are reasonably continuous, then the nearest neighbor algorithm is guaranteed to
asymptotically achieve a total loss that is at most twice larger than that of a Bayes predictor (in
fact, even a bit better than twice).

The theorem above can be applied directly to the high-arity setting: since our Borel ($k$-partite,
respectively) templates are built up from Borel spaces, each $\Omega_i$ ($\Omega_A$ in the partite
case) can be endowed with a separable complete metric $d_i$ ($d_A$, respectively) generating the
underlying $\sigma$-algebra, which in turn yields a natural metric on the space $\cE_k(\Omega)$:
\begin{align*}
  d(x,x') & \df \sum_{A\in r(k)} d_{\lvert A\rvert}(x_A, x'_A)
\end{align*}
and a natural metric on the space $\cE_1(\Omega)$ in the partite case:
\begin{align*}
  d(x,x') & \df \sum_{f\in r_k(1)} d_{\dom(f)}(x_f,x'_f)
\end{align*}

Suppose now that we are given an extra Borel ($k$-partite, respectively) template $\Omega'$ and
$F\in\cF_k(\Omega\otimes\Omega',\Lambda)$, where $\Lambda$ is a finite Borel space with
$\lvert\Lambda\rvert\geq 2$, and probability measures $\mu\in\Pr(\Omega)$ and $\mu'\in\Pr(\Omega')$,
then we can let $\nu\in\Pr(\cE_k(\Omega)\times\Lambda^{S_k})$ be the distribution of
$(\rn{x},F^*_k(\rn{x},\rn{x'}))$ where $(\rn{x},\rn{x'})\sim(\mu\otimes\mu')^k$
($\nu\in\Pr(\cE_1(\Omega)\times\Lambda)$, respectively, as the distribution of
$(\rn{x},F(\rn{x},\rn{x'}))$ where $(\rn{x},\rn{x'})\sim(\mu\otimes\mu')^1$ in the partite
case). For each $i\in\Lambda^{S_k}$, we can define the measure $\theta_i$ on $\cE_k(\Omega)$ (in the
partite case, $i$ is taken in $\Lambda$ and the measure $\theta_i$ is on $\cE_1(\Omega)$) by
$\theta_i(A)\df\nu(A\times\{i\})$.

If we assume that there exist versions of the Radon--Nikodym derivatives $\frac{d\theta_i}{d\mu}$
such that the set
\begin{align*}
  \left\{x\in\cE_k(\Omega)
  \;\middle\vert\;
  \mu^k(\{x\}) > 0
  \lor
  \forall i\in\Lambda^{S_k},
  \frac{d\theta_i}{d\mu}
  \text{ is continuous at } x
  \right\}
\end{align*}
has $\mu^k$-measure $1$ (or the analogous set in the partite case), then we can apply
Theorem~\ref{thm:Bayes} above directly to conclude that a nearest neighbor algorithm also
asymptotically attains at most twice the total loss (with respect to $\ell_{0/1}$) as any Bayes
predictor.

However, let us point out that the high-arity nearest neighbor algorithm that naturally comes out of
Theorem~\ref{thm:Bayes} needs i.i.d.\ points from $\nu$, so it should ignore most of the sample
information to avoid correlation in a similar manner of a direct application of classic PAC to
high-arity PAC. As mentioned in Section~\ref{subsec:nearestneighbor}, it is then natural to ask if
one can take advantage of structured correlation of high-arity to achieve a better asymptotic
multiplicative bound against Bayes predictors.

\printbibliography

\end{document}

%% file: roadmap.tex
\begingroup
\def\bigbend{45}
\def\smallbend{25}
\def\bendshift{10}
\begin{landscape}
  \vfill
  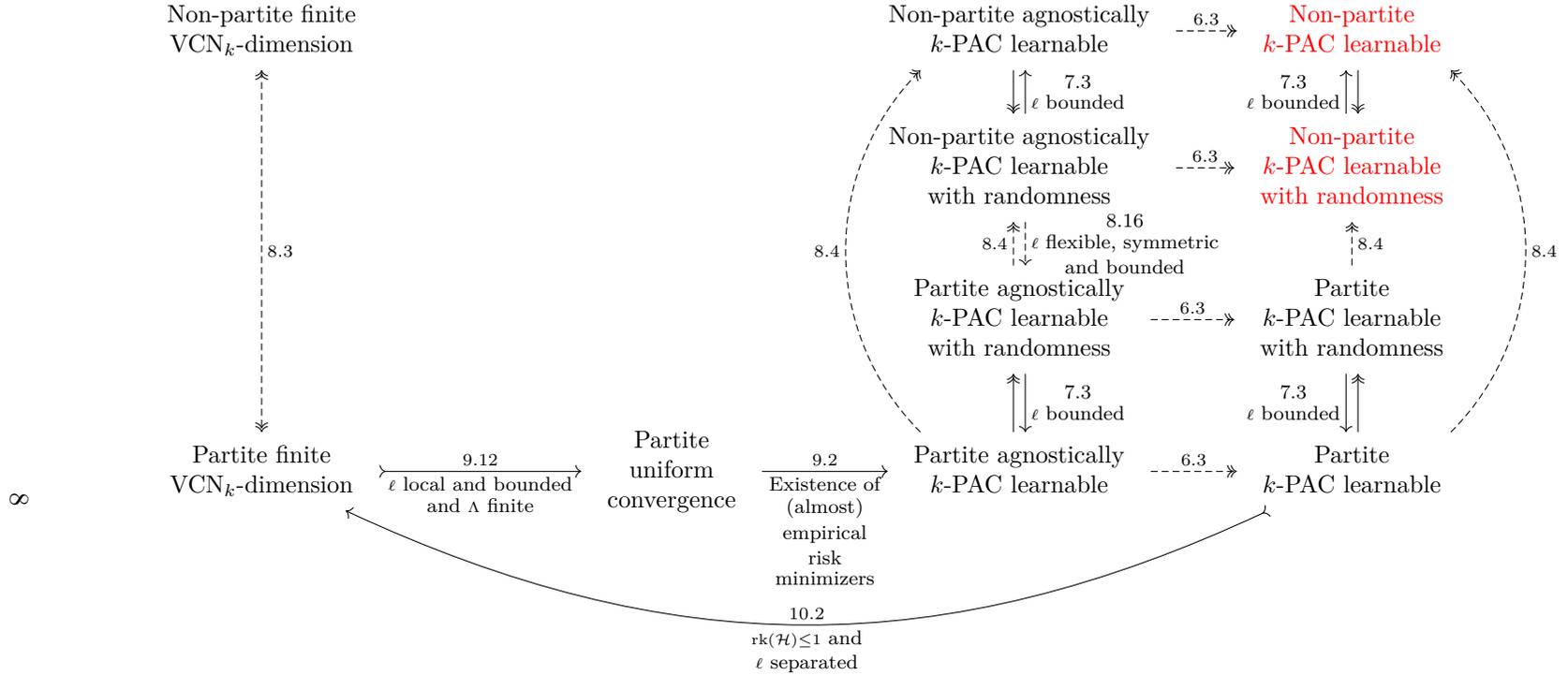
\begin{figure}[p]
    \centering
    \begin{small}
      \begin{tikzcd}
        \begin{tabular}{c}
          Non-partite finite\\
          $\VCN_k$-dimension
        \end{tabular}
        \arrow[ddd, twoheadleftarrow, two heads, dashed, "\ref{prop:kpartVCN}"]
        &[2.0cm] &[0.6cm] 
        \begin{tabular}{c}
          Non-partite agnostically\\
          $k$-PAC learnable
        \end{tabular}
        \arrow[r, two heads, dashed, "\ref{prop:agPAC->PAC}"]
        \arrow[d, two heads, shift right]
        &
              {
                \color{red}
                \begin{tabular}{c}
                  Non-partite\\
                  $k$-PAC learnable
                \end{tabular}
              }
              \arrow[d, two heads, shift left]
              \\
              & &
              \begin{tabular}{c}
                Non-partite agnostically\\
                $k$-PAC learnable\\
                with randomness
              \end{tabular}
              \arrow[r, two heads, dashed, "\ref{prop:agPAC->PAC}"]
              \arrow[u, "{\scriptsize\shortstack{\ref{prop:derand}\\$\ell$ bounded}}"', shift right]
              \arrow[d, dashed,
                "{\scriptsize\shortstack{\ref{prop:kpart3}\\$\ell$ flexible, symmetric\\ and bounded}}", shift left]
              &
                    {\color{red}
                      \begin{tabular}{c}
                        Non-partite\\
                        $k$-PAC learnable\\
                        with randomness
                      \end{tabular}
                    }
                    \arrow[u, "{\scriptsize\shortstack{\ref{prop:derand}\\$\ell$ bounded}}", shift left]
                    \\
                    & &
                    \begin{tabular}{c}
                      Partite agnostically\\
                      $k$-PAC learnable\\
                      with randomness
                    \end{tabular}
                    \arrow[r, two heads, dashed, "\ref{prop:agPAC->PAC}"]
                    \arrow[d, "{\scriptsize\shortstack{\ref{prop:derand}\\$\ell$ bounded}}", shift left]
                    \arrow[u, two heads, dashed, "\ref{prop:kpart}", shift left]
                    &
                    \begin{tabular}{c}
                      Partite\\
                      $k$-PAC learnable\\
                      with randomness
                    \end{tabular}
                    \arrow[d, "{\scriptsize\shortstack{\ref{prop:derand}\\$\ell$ bounded}}"', shift right]
                    \arrow[u, two heads, dashed, "\ref{prop:kpart}"']
                    \\
                    \begin{tabular}{c}
                      Partite finite\\
                      $\VCN_k$-dimension
                    \end{tabular}
                    \arrow[r, tail, "\ref{prop:VCNdim->UC}",
                      "{\scriptsize\shortstack{$\ell$ local and bounded\\and $\Lambda$ finite}}"']
                    &
                    \begin{tabular}{c}
                      Partite\\
                      uniform\\
                      convergence
                    \end{tabular}
                    \arrow[r, "\ref{prop:partUC->partagPAC}",
                      "{\scriptsize\shortstack{Existence of\\(almost)\\empirical\\risk\\minimizers}}"']
                    & 
                    \begin{tabular}{c}
                      Partite agnostically\\
                      $k$-PAC learnable
                    \end{tabular}
                    \arrow[r, two heads, dashed, "\ref{prop:agPAC->PAC}"]
                    \arrow[u, two heads, shift left]
                    \arrow[uuu, two heads, dashed, "\ref{prop:kpart}",
                      bend left={\bigbend}, shift left={\bendshift}]
                    &
                    \begin{tabular}{c}
                      Partite\\
                      $k$-PAC learnable
                    \end{tabular}
                    \arrow[u, two heads, shift right]
                    \arrow[uuu, two heads, dashed, "\ref{prop:kpart}"',
                      bend right={\bigbend}, shift right={\bendshift}]
                    \arrow[lll, "\ref{prop:partkPAC->VCN}"', tail,
                      "{\scriptsize \shortstack{$\rk(\cH)\leq 1$ and\\$\ell$ separated}}",
                      bend left={\smallbend}]
      \end{tikzcd}
      \captionof{figure}{Diagram of results proved in this document. Labels on arrows contain the
        number of the proposition that contains the proof of the implication and extra hypotheses
        needed. Unlabeled arrows are obvious implications. Arrows with two heads
        ($\twoheadrightarrow$) are tight in some sense with an obvious proof of tightness. Dashed
        arrows involve a construction (meaning that either the hypothesis class changes and/or the
        loss function changes) due to being in different settings; this also means that objects in one
        of the sides of the implication might not be completely general (as they are required to be in
        the image of the construction). Arrows with tails ($\rightarrowtail$) mean that exactly one of
        the sides involves a loss function (so when composing a solid arrow with a tailed arrow, the
        result might involve a construction that changes the loss function and thus be a dashed
        arrow). Under appropriate hypotheses, all items are proved equivalent except for the two top
        items on the rightmost column (in {\color{red}red}): non-partite $k$-learnability and
        non-partite $k$-learnability with randomness (see Section~\ref{sec:equivthms}).}
      \label{fig:roadmap}
    \end{small}
  \end{figure}
  \vfill
\end{landscape}
\endgroup

%% Local Variables:
%% mode: latex
%% End: